\documentclass[sigconf]{acmart}

\AtBeginDocument{%
  \providecommand\BibTeX{{%
    \normalfont B\kern-0.5em{\scshape i\kern-0.25em b}\kern-0.8em\TeX}}}

\copyrightyear{2021}
\acmYear{2021}
\setcopyright{acmcopyright}\acmConference[CCS '21]{Proceedings of the 2021 ACM
SIGSAC Conference on Computer and Communications Security}{November 15--19,
2021}{Virtual Event, Republic of Korea}
\acmBooktitle{Proceedings of the 2021 ACM SIGSAC Conference on Computer and
Communications Security (CCS '21), November 15--19, 2021, Virtual Event, Republic
of Korea}
\acmPrice{15.00}
\acmDOI{10.1145/3460120.3485258}
\acmISBN{978-1-4503-8454-4/21/11}

\usepackage{support-caption}
\usepackage{subcaption}

\usepackage{amsthm}
\usepackage{microtype}
\usepackage{graphicx}
\usepackage{booktabs}
\usepackage{siunitx}

\PassOptionsToPackage{hyphens}{url}
\usepackage{hyperref}
\usepackage{cleveref}

\usepackage{dsfont}
\usepackage{nicefrac}
\usepackage{multirow}
\usepackage{microtype}
\usepackage{graphicx}
\usepackage{booktabs} 
\usepackage{mathtools}
\usepackage{bbm}
\usepackage{fancyhdr}
\usepackage{xcolor}
\usepackage{xr}
\usepackage{float}
\usepackage{threeparttable}
\usepackage{algorithm,algorithmic}

\newcommand{\gD}{\mathcal{D}}
\newcommand{\gF}{\mathcal{F}}

\newcommand{\gO}{\mathcal{O}}
\newcommand{\gP}{\mathcal{P}}

\newcommand{\gX}{\mathcal{X}}
\newcommand{\gY}{\mathcal{Y}}
\newcommand{\gZ}{\mathcal{Z}}
\newcommand{\gL}{\mathcal{L}}

\newcommand{\gN}{\mathcal{N}}
\newcommand{\gS}{\mathcal{S}}
\newcommand{\gU}{\mathcal{U}}

\newcommand{\Id}{\mathds{1}}

\def\R{\mathbb{R}}
\def\N{\mathbb{N}}
\def\E{\mathbb{E}}
\def\Prob{\mathbb{P}}

\def\Z{\mathbb{Z}}

\DeclarePairedDelimiter\floor{\lfloor}{\rfloor}

\newtheorem{thm}{Theorem}
\newtheorem{cor}{Corollary}
\newtheorem{lem}{Lemma}
\Crefname{cor}{Corollary}{Corollaries}

\newtheorem{defn}{Definition}

\newtheorem{rem}{Remark}

\theoremstyle{plain}
\newtheorem{innercustomthm}{Theorem}
\newenvironment{customthm}[1]
{\renewcommand\theinnercustomthm{#1}\innercustomthm}
{\endinnercustomthm}

\theoremstyle{plain}
\newtheorem{innercustomlem}{Lemma}
\newenvironment{customlem}[1]
{\renewcommand\theinnercustomlem{#1}\innercustomlem}
{\endinnercustomlem}

\theoremstyle{plain}

\theoremstyle{plain}
\newtheorem{innercustomdef}{Definition}
\newenvironment{customdef}[1]
{\renewcommand\theinnercustomdef{#1}\innercustomdef}
{\endinnercustomdef}

\theoremstyle{plain}
\newtheorem{innercustomcor}{Corollary}
\newenvironment{customcor}[1]
{\renewcommand\theinnercustomcor{#1}\innercustomcor}
{\endinnercustomcor}

\theoremstyle{plain}

\numberwithin{claim}{section}

\numberwithin{fact}{claim}

\def\ddefloop#1{\ifx\ddefloop#1\else\ddef{#1}\expandafter\ddefloop\fi}
\def\ddef#1{\expandafter\def\csname bb#1\endcsname{\ensuremath{\mathbb{#1}}}}
\ddefloop ABCDEFGHIJKLMNOPQRSTUVWXYZ\ddefloop
\def\ddef#1{\expandafter\def\csname c#1\endcsname{\ensuremath{\mathcal{#1}}}}
\ddefloop ABCDEFGHIJKLMNOPQRSTUVWXYZ\ddefloop

\DeclareMathOperator*{\argmax}{arg\,max}

\def\1{\mathds{1}}

\def\diag{\textup{diag}}

\def\Exp{\textup{Exp}}

\def\atan2{\mathrm{arctan2}}

\makeatletter
\newcommand\footnoteref[1]{\protected@xdef\@thefnmark{\ref{#1}}\@footnotemark}
\makeatother
\usepackage{wrapfig}
\usepackage{xspace}

\usepackage{float}
\usepackage{tikz}
\usepackage{tikz-qtree}
\usepackage{soul}
\usepackage{version}

\usepackage[T1]{fontenc}
\usepackage{aecompl}

\usepackage[framemethod=TikZ,skipabove=2pt]{mdframed}

\newcount\Comments  
\newcount\ArXiv 
\usepackage{color}
\definecolor{darkgreen}{rgb}{0,0.5,0}
\definecolor{darkblue}{rgb}{0,0,0.5}
\definecolor{purple}{rgb}{1,0,1}
\newcommand{\kibitz}[2]{\ifnum\Comments=0\textcolor{#1}{#2}\fi}

\renewcommand{\epsilon}{\varepsilon}

\definecolor{gray}{rgb}{0.5,0.5,0.5}

\ifnum\Comments=0\else\excludeversion{old}\fi
\ifnum\Comments=0\definecolor{rev}{rgb}{0,0,1.0}\else\definecolor{rev}{rgb}{0,0,0}\fi

\newcolumntype{g}{>{\columncolor{tabgray}}c}
\usepackage{colortbl}
\definecolor{tabgray}{gray}{0.90}
\definecolor{rgray}{gray}{0.75}
\newcommand{\gray}{\cellcolor{tabgray}}

\newcolumntype{v}{>{\color{rev}}c}

\newcommand{\framework}{\textbf{TSS}\xspace}
\newcommand{\frameworkfull}{\textbf{T}ransformation-\textbf{S}pecific \textbf{S}moothing-based robustness certification\xspace}

\newsavebox{\boundingimage}

\usepackage{enumitem}

\ifnum\ArXiv=0
\else
\fi

\begin{document}

\ifnum\ArXiv=1
\fancyhead{}
\fi

\title{
   TSS: Transformation-Specific Smoothing\texorpdfstring{\\}{} for Robustness Certification
}

\author{Linyi Li*}
\thanks{* Equal contribution.}
\affiliation{%
  \institution{University of Illinois}
  \city{Urbana}
  \state{Illinois}
  \country{USA}
}
\email{linyi2@illinois.edu}

\author{Maurice Weber*}
\affiliation{%
  \institution{ETH Z\"{u}rich}
  \city{Z\"{u}rich}
  \country{Switzerland}
}
\email{maurice.weber@inf.ethz.ch}

\author{Xiaojun Xu}
\affiliation{%
  \institution{University of Illinois}
  \city{Urbana}
  \state{Illinois}
  \country{USA}
}
\email{xiaojun3@illinois.edu}

\author{Luka Rimanic}
\affiliation{%
  \institution{ETH Z\"{u}rich}
  \city{Z\"{u}rich}
  \country{Switzerland}
}
\email{luka.rimanic@inf.ethz.ch}

\author{Bhavya Kailkhura}
\affiliation{%
  \institution{Lawrence Livermore National Laboratory}
  \city{Livermore}
  \state{California}
  \country{USA}
}
\email{kailkhura1@llnl.gov}

\author{Tao Xie}
\affiliation{%
  \institution{Key Laboratory of High Confidence
    Software Technologies, MoE (Peking University)}
  \city{Beijing}
  \country{China}
}
\email{taoxie@pku.edu.cn}

\author{Ce Zhang}
\affiliation{%
  \institution{ETH Z\"{u}rich}
  \city{Z\"{u}rich}
  \country{Switzerland}
}
\email{ce.zhang@inf.ethz.ch}

\author{Bo Li}
\affiliation{%
  \institution{University of Illinois}
  \city{Urbana}
  \state{Illinois}
  \country{USA}
}
\email{lbo@illinois.edu}

\begin{abstract}

As machine learning (ML) systems become pervasive, safeguarding their security is critical.
However, recently it has been demonstrated that motivated adversaries are able to mislead ML systems by perturbing test data using semantic transformations.
While there exists a rich body of research providing provable robustness guarantees for ML models against $\ell_p$ norm bounded adversarial perturbations, guarantees against semantic perturbations remain largely underexplored.
In this paper, we provide
\framework ---{\em a unified framework for
certifying ML robustness against general adversarial semantic transformations}.
First, depending on the properties of each transformation, we divide
common transformations
into two categories, namely \textit{resolvable} (e.g., Gaussian blur) and \textit{differentially resolvable} (e.g., rotation) transformations.
For the former, we propose transformation-specific randomized smoothing strategies and obtain strong robustness certification.
The latter category covers transformations that involve interpolation errors, and we propose a novel approach based on stratified sampling to certify the robustness.
Our framework \framework leverages these certification strategies and combines with consistency-enhanced training to provide rigorous certification of robustness.
We conduct extensive experiments on {over} ten types of challenging semantic transformations and show that \framework significantly outperforms the state of the art.
Moreover, to the best of our knowledge, \framework is the first approach that achieves nontrivial certified robustness on the large-scale ImageNet dataset.
For instance, our framework achieves
$30.4\%$
certified robust accuracy against
rotation attack~(within $\pm 30^\circ$)
on ImageNet.
Moreover, to consider a broader range of transformations, we show \framework is also robust against adaptive attacks and unforeseen image corruptions such as CIFAR-10-C and ImageNet-C.

\end{abstract}



\maketitle


\section{Introduction}

\label{sec:intro}

Recent advances in machine learning~(ML) have enabled a plethora of applications in tasks such as image recognition~\cite{he2015delving} and game playing~\cite{silver2017go, moravvcik2017deepstack}.
Despite all of these advances, ML systems are also found exceedingly vulnerable to adversarial attacks:
image recognition systems can be adversarially misled~\cite{szegedy2013intriguing,goodfellow2014explaining,chaowei2018generating}, and malware detection systems can be evaded easily~\cite{liang2019improving,xu2016automatically}.

\begin{figure}[!t]
    \centering
    \includegraphics[width=0.4\textwidth]{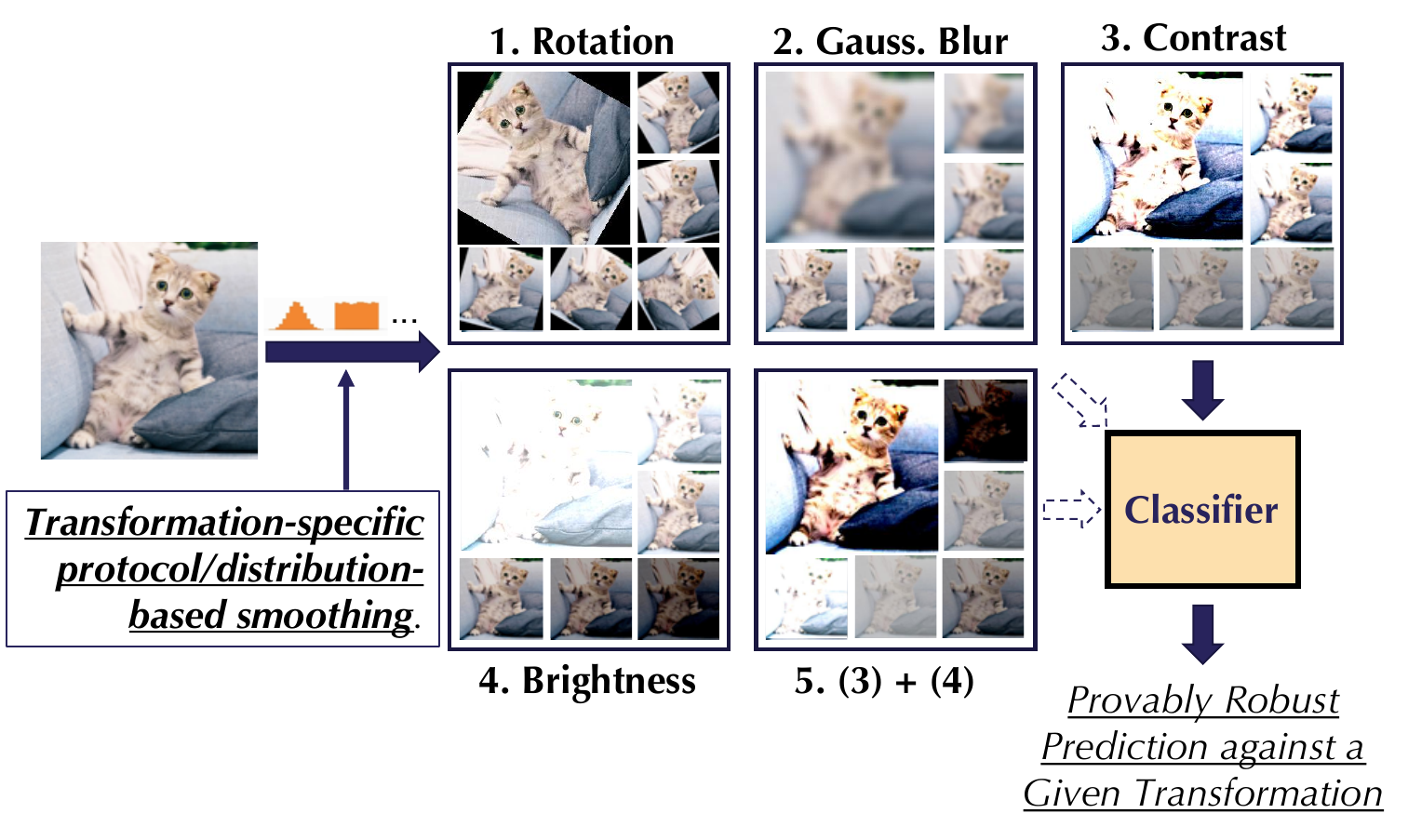}
    \vspace{-2mm}
    \caption{ \small
    {\textit{\frameworkfull}, a general robustness certification framework against various semantic transformations.} We develop a range of different transformation-specific smoothing protocols and various techniques to provide
    substantially better certified robustness bounds than state-of-the-art approaches on large-scale datasets.}
    \label{fig:transformation}
\end{figure}

The existing practice of security in ML has fallen into the cycle where new empirical defense techniques are proposed~\cite{xingjun2018characterizing,tramer2017ensemble}, followed by new adaptive attacks breaking these defenses~\cite{chaowei2018spatially,goodfellow2014explaining,ivan2018robust,athalye2018obfuscated}.
In response, recent research has attempted to provide \textit{provable robustness guarantees} for an ML model.
Such certification usually follows the form that the ML model is provably robust against arbitrary adversarial attacks, as long as the  perturbation magnitude is below a certain threshold. Different certifiable defenses and robustness verification approaches have provided nontrivial robustness guarantees against $\ell_p$ perturbations where the perturbation is bounded by small $\ell_p$ norm~\cite{kolter2017provable,tjeng2018evaluating,linyi2019robustra,cohen2019certified, xu2020automatic}.

However, certifying robustness only against $\ell_p$ perturbations is not sufficient for attacks based on semantic transformation.
For instance, image rotation, scaling, and other semantic transformations are able to mislead ML models effectively~\cite{engstrom2019exploring,ghiasi2020breaking,chaowei2018spatially, gokhale2020attribute}.
These transformations are common and practical~\cite{pei2017deepxplore,hendrycks2019benchmarking, bulusu2020anomalous}.
For example, it has been shown~\cite{hosseini2018semantic} that brightness/contrast attacks can achieve 91.6\% attack success on CIFAR-10, and 71\%-100\% attack success rate on ImageNet~\cite{hendrycks2019benchmarking}.
In practice, brightness- and contrast-based attacks have been demonstrated to be successful in autonomous driving~\cite{pei2017deepxplore,tian2018deeptest}.
These attacks incur large $\ell_p$-norm differences and are thus beyond the reach of existing certifiable defenses~\cite{salman2019convex,kumar2020curse,hayes2020extensions,blum2020random}.
{\em Can we provide provable robustness guarantees
against these semantic transformations?}

In this paper, we propose theoretical and empirical analyses to certify the ML robustness against a wide range of semantic transformations beyond $\ell_p$ bounded perturbations.
The theoretical analysis is nontrivial given different properties of the transformations,
and our empirical results set the new state-of-the-art robustness certification for a range of semantic transformations, exceeding existing work by a large margin.
In particular, we propose {\frameworkfull} ---
a \textit{general framework} based on function smoothing providing certified robustness for ML models against a range of adversarial transformations (Figure~\ref{fig:transformation}).
To this end, we first categorize semantic transformations as either \textit{resolvable} or \textit{differentially resolvable}.
We then provide certified robustness against \textit{resolvable} transformations, which include
brightness, contrast, translation, Gaussian blur, and their composition.
Second, we develop novel
certification techniques for \textit{differentially resolvable} transformations (e.g., rotation and scaling),
based on the \textit{building block} that we have developed for resolvable
transformations.

    For resolvable transformations, we leverage the framework to \emph{jointly} reason about (1)~function smoothing under different smoothing distributions and (2)~the properties inherent to each specific transformation.
    To our best knowledge, this is the first time that the interplay between smoothing distribution and semantic transformation has been analyzed as existing work~\cite{li2019certified,cohen2019certified,yang2020randomized} that studies different smoothing distributions considers only $\ell_p$ perturbations.
    Based on this analysis, we find that against certain transformations such as Gaussian blur, exponential distribution is better than Gaussian smoothing, which is commonly used in the $\ell_p$-case.

    For differentially resolvable transformations, such as rotation, scaling, and their composition with other transformations, the common \textbf{challenge} is that they naturally induce \textit{interpolation error}.
    Existing work~\cite{balunovic2019certifying,fischer2019statistical} can provide
    robustness guarantees but it
    cannot rigorously certify robustness for
    ImageNet-scale data.
    We develop a collection of novel techniques, including stratified sampling and Lipschitz bound computation to provide a tighter and sound upper bound for the interpolation error.
    We integrate these novel techniques into our \framework framework and further propose a progressive-sampling-based strategy to accelerate the robustness certification.
    We show that these techniques comprise a scalable and general framework for certifying robustness against differentially resolvable transformations.

    We conduct extensive experiments to evaluate the proposed certification framework and show that our framework significantly outperforms the state-of-the-art on different datasets including the large-scale ImageNet against a series of practical semantic transformations.
    In summary, this paper makes the following \underline{contributions}:
\begin{enumerate}[label=(\emph{\arabic*}),leftmargin=*]
    \vspace{-0.3em}
    \item We propose a \textit{general} function smoothing framework, \framework, to certify ML robustness against general semantic transformations.
    \item We categorize common adversarial semantic transformations
    in the literature into \textit{resolvable} and \textit{differentially resolvable} transformations and show that our framework is general enough to
    certify both types of transformations.
    \item We theoretically explore different smoothing strategies by sampling from different distributions including non-isotropic Gaussian, uniform, and Laplace distributions. We show that for specific transformations, such as Gaussian blur, smoothing with exponential distribution is better.
    \item We propose a pipeline, \textbf{\framework-DR}, including a stratified sampling approach, an effective Lipschitz-based bounding technique, and a progressive sampling strategy to provide rigorous, tight, and scalable robustness certification against differentially resolvable transformations such as rotation and scaling.
    \item We conduct extensive experiments and show that our framework \framework can provide significantly higher certified robustness compared with the state-of-the-art approaches, against a range of semantic transformations and their composition on MNIST, CIFAR-10, and ImageNet.
    \item
    We show that \framework also provides much higher empirical robustness against adaptive attacks and unforeseen corruptions such as CIFAR-10-C and ImageNet-C.
    \vspace{-0.3em}
\end{enumerate}
    The code implementation and all trained models are publicly available at \texttt{\small \url{https://github.com/AI-secure/semantic-randomized-smoothing}}.

\section{Background}
\label{sec:background}
We next provide an overview of different semantic transformations and explain the intuition behind the randomized smoothing~\cite{cohen2019certified} that has been proposed to certify the robustness against $\ell_p$ perturbations.

\noindent\textbf{Semantic Transformation Based Attacks.}
Beyond adversarial $\ell_p$ perturbations, a realistic threat model is given by image transformations that preserve the underlying semantics.
Examples for these types of transformations include changes to contrast or brightness levels, or rotation of the entire image.
These attacks share three common characteristics:
(1)~The perturbation stemming from a successful semantic attack typically has higher $\ell_p$ norm compared to $\ell_p$-bounded attacks.
However, these attacks still preserve the underlying semantics~(a car image rotated by $10^\circ$ still contains a car).
(2)~These attacks are governed by a low-dimensional parameter space.
For example, the rotation attack  chooses a single-dimensional rotation angle.
(3) Some of such adversarial transformations would lead to high interpolation error (e.g., rotation), which makes it challenging to certify.
Nevertheless, these types of attacks can also cause significant damage~\cite{hosseini2018semantic,hendrycks2019benchmarking} and pose realistic threats for practical ML applications such as autonomous driving~\cite{pei2017deepxplore}.
We remark that our proposed framework can be extended to certify robustness against other attacks sharing these  characteristics even beyond the image domain, such as GAN-based attacks against ML based malware detection~\cite{hu2017generating,wang2019evading}, where a limited dimension of features of the malware can be manipulated in order to preserve the malicious functionalities and such perturbation usually incurs large $\ell_p$ differences for the generated instances.

\noindent\textbf{Randomized Smoothing.}
    On a high level, randomized smoothing~\cite{lecuyer2019certified,li2019certified,cohen2019certified} provides a way to certify robustness based on randomly perturbing inputs to an ML model.
    The intuition behind such  randomized classifier is that noise smoothens the decision boundaries and suppresses regions with high curvature.
    Since adversarial examples aim to exploit precisely these high curvature regions, the vulnerability to this type of attack is reduced.
    Formally, a base classifier $h$ is smoothed by adding noise $\varepsilon$ to a test instance. The prediction of the smoothed classifier is then given by the most likely prediction under this smoothing distribution. Subsequently, a tight robustness guarantee can be obtained, based on the noise variance and the class probabilities of the smoothed classifier. It is guaranteed that, as long as the $\ell_2$ norm of the perturbation is bounded by a certain amount, the prediction on an adversarial vs. benign input will stay the same.
    This technique provides a powerful framework to study the robustness of classification models against adversarial attacks for which the primary figure of merit is a low $\ell_p$ norm with a simultaneously high success rate of fooling the classifier~\cite{Dvijotham2020A,yang2020randomized}.
    However, semantic transformations incur large $\ell_p$ perturbations,
    which renders classical randomized smoothing infeasible~\cite{kumar2020curse,hayes2020extensions,blum2020random},
    making it of great importance to generalize randomized smoothing to this kind of threat model.

\section{Threat Model \& \framework Overview}

  In this section, we first introduce the notations used throughout this paper. We then define our \textbf{threat model}, the \textbf{defense goal} and outline the \textbf{challenges} for certifying the robustness against semantic transformations.
  Finally, we will provide a brief \textbf{overview} of our \framework certification framework.

  \begin{figure}
      \centering
      \includegraphics[width=0.85\linewidth]{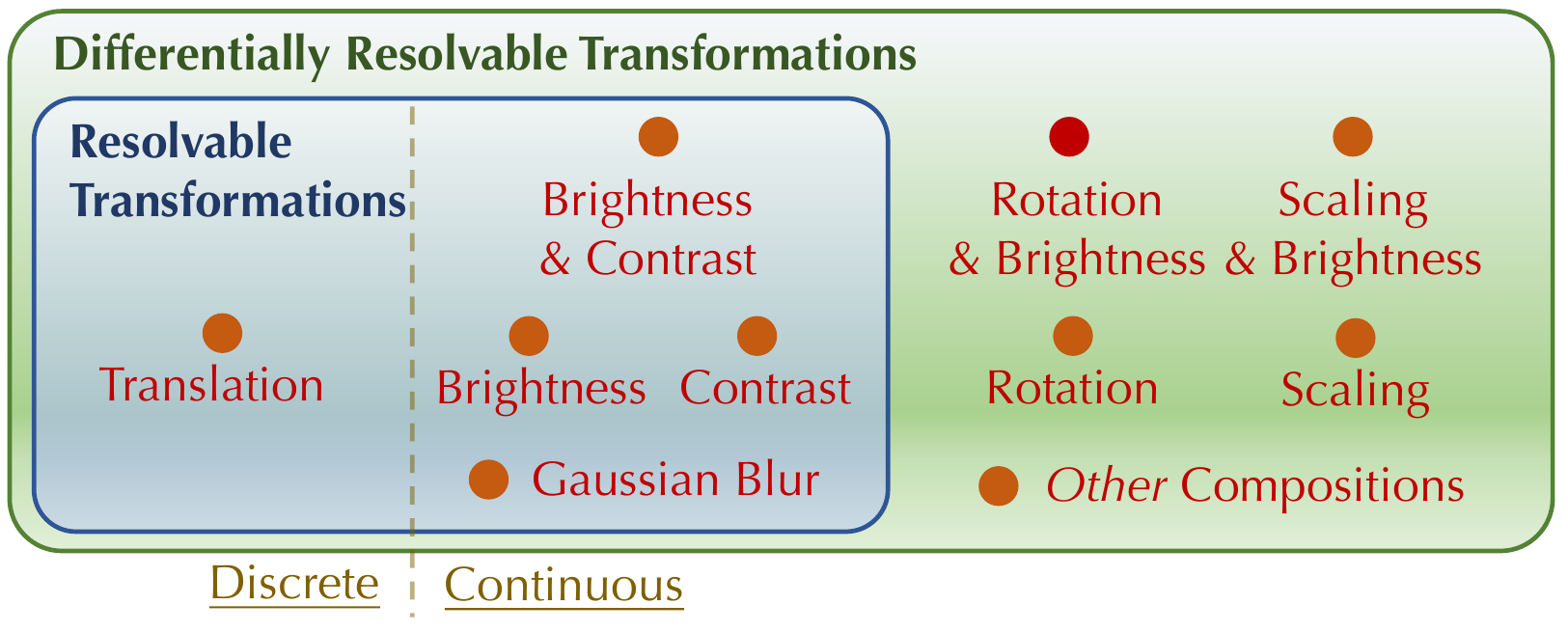}
      \caption{\small We provide strong robustness certification for both resolvable transformations and differentially revolvable transformations. These two categories cover common adversarial semantic transformations.}
      \label{fig:venn-diagram}
      \vspace{-4mm}
  \end{figure}

    We denote the space of inputs as $\gX\subseteq\R^d$ and the set of labels as $\gY=\{1,\,\ldots,\,C\}$ (where $C\geq2$ is the number of classes).
    The set of transformation parameters is given by $\gZ\subseteq\R^m$ (e.g., rotation angles).
    We use the notation $\Prob_X$ to denote the probability measure induced by the random variable $X$ and write $f_X$ for its probability density function.
    For a set $S$, we denote its probability by $\Prob_X(S)$.
    A classifier is defined to be a deterministic function $h$ mapping inputs $x\in\gX$ to classes $y\in\gY$.
    Formally, a classifier learns a conditional probability distribution $p(y\lvert\,x)$ over labels and outputs the class that maximizes $p$, i.e., $h(x) = \arg\max_{y\in\gY}p(y\lvert\,x)$.


\subsection{Threat Model and Certification Goal} \label{subsec:threat-model-cert-goal}
    \paragraph{Semantic Transformations}
    We model semantic transformations as deterministic functions $\phi\colon\gX\times\gZ\to\gX$, transforming an image $x\in\gX$ with a $\gZ$-valued parameter $\alpha$.
    For example, we use $\phi_R(x,\alpha)$ to model a rotation of the image $x$ by $\alpha$ degrees counter-clockwise with bilinear interpolation.
    We further partition semantic transformations into two different categories, namely resolvable and differentially resolvable transformations.
    We will show that these two categories could cover commonly known semantic attacks.
    This categorization depends on whether or not it is possible to write the composition of the transformation $\phi$ with itself as applying the same transformation just once, but with a different parameter, i.e., whether for any $\alpha,\,\beta\in\cZ$ there exists $\gamma$ such that $\phi(\phi(x,\,\alpha),\,\beta) = \phi(x,\,\gamma)$. Precise definitions are given in Sections~\ref{sec:resolvable} and~\ref{sec:differentially}. Figure~\ref{fig:venn-diagram} presents an overview of the transformations considered in this work.

    \paragraph{Threat model}
    We consider an adversary that launches a semantic attack, a type of data evasion attack, against a given classification model $h$ by applying a semantic transformation $\phi$ with parameter $\alpha$ to an input image $x \to \phi(x,\,\alpha)$.
    We allow the attacker to choose an \emph{arbitrary} parameter $\alpha$ within a predefined (attack) parameter space $\gS$.
    For instance, a na\`ive adversary who randomly changes brightness from within $\pm 40\%$ is able to reduce the accuracy of a state-of-the-art ImageNet classifier from $74.4\%$ to $21.8\%$ (Table~\ref{tab:main-attack}).
    While this attack is an  example random adversarial attack, our threat model also covers other types of semantic attacks and we provide the first taxonomy for semantic attacks (i.e., resolvable and differentially resolvable) in detail in Sections~\ref{sec:resolvable} and~\ref{sec:differentially}.

    \paragraph{Certification Goal}
    Since the only degree of freedom that a semantic adversary has is the parameter,
    \textbf{our goal} is to characterize a set of parameters for which the model under attack is guaranteed to be robust. Formally, we wish to find a set $\gS_{\mathrm{adv}}\subseteq\gZ$ such that, for a classifier $h$ and adversarial transformation  $\phi$, we have
    \begin{equation}
        h(x) = h(\phi(x,\,\alpha))\hspace{1em}\forall\,\alpha\in\gS_{\mathrm{adv}}.
    \end{equation}

    \paragraph{Challenges for Certifying Semantic Transformations}
    Certifying ML robustness against semantic transformations is nontrivial and requires careful analysis. We identify the following two main challenges that we aim to address in this paper:
    \begin{enumerate}
        \item[\bf(C1)] The absolute difference between semantically transformed images in terms of $\ell_p$-norms is typically high. This factor causes existing certifiable defenses against $\ell_p$ bounded perturbations to be  inapplicable~\cite{salman2019convex,kumar2020curse,hayes2020extensions,blum2020random}.
        \item[\bf(C2)] Certain semantic transformations incur additional \textit{interpolation errors}.
        To derive a robustness certificate, it is required to bound these errors, an endeavour that has been proven to be hard both analytically and computationally. This challenge applies to transformations that involve interpolation, such as rotation and scaling.
    \end{enumerate}
    We remark that it is in general not feasible to use brute-force approaches such as grid search to enumerate all possible transformation parameters~(e.g., rotation angles) since the parameter spaces are typically continuous. Given that different transformations have their own unique properties, it is crucial to provide a \textit{unified} framework that takes into account transformation-specific properties in a general way.

    To address these challenges, we generalize randomized smoothing via our proposed \textit{function smoothing framework} to certify arbitrary input transformations via different smoothing distributions, paving the way to robustness certifications that go beyond $\ell_p$ perturbations. This result addresses challenge \textbf{(C1)} in a unified way.
    Based on this generalization and depending on specific transformation properties, we address challenge \textbf{(C2)} and propose a series of smoothing strategies and computing techniques that provide robustness certifications for a diverse range of transformations.

    We next introduce our generalized function smoothing framework and show how it can be leveraged to certify semantic transformations.
    We then categorize transformations as either \textit{resolvable} transformations~(\Cref{sec:resolvable}) such as Gaussian blur, or \textit{differentially resolvable} transformations~(\Cref{sec:differentially}) such as rotations.

  \subsection{Framework Overview}
    \label{subsec:overview}
    \begin{figure}
      \centering
      \includegraphics[width=\linewidth]{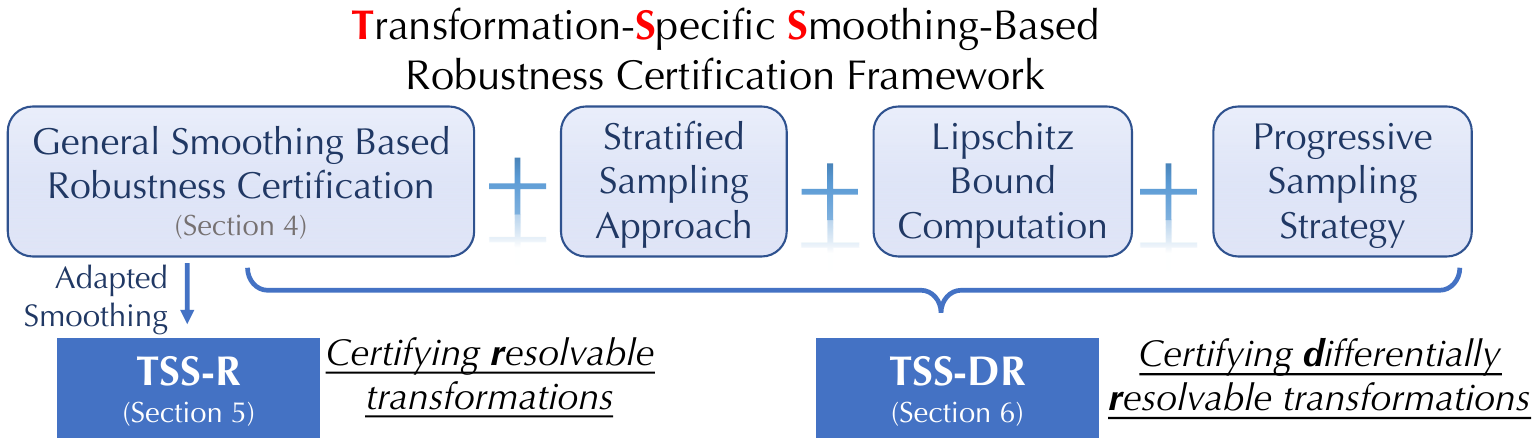}
      \vspace{-1em}
      \caption{\small An overview of \framework.}
      \label{fig:method-overview}
      \vspace{-1em}
    \end{figure}

    An overview of our proposed framework \framework is given in~\Cref{fig:method-overview}.
    We propose the function smoothing framework, a generalization of randomized smoothing, to provide robustness certifications under general smoothing distributions~(\Cref{sec:framework}). This generalization enables us to smooth the model on specific transformation dimensions.
    We then consider two different types of transformation attacks.
    For \textit{resolvable} transformations, using function smoothing framework, we adapt different smoothing strategies for specific transformations and propose \textbf{\framework-R} (Section~\ref{sec:resolvable}). We show that some smoothing distributions are more suitable for certain transformations.
    For \textit{differentially resolvable} transformations, to address the interpolation error, we combine function smoothing with the proposed stratified sampling approach and a novel technique for Lipschitz bound computation to compute a rigorous upper bound of the error.
    We then develop a progressive sampling strategy to accelerate the certification. This pipeline is termed \textbf{\framework-DR}, and we provide details and the theoretical groundwork in \Cref{sec:differentially}.


\section{TSS: Transformation Specific Smoothing based Certification}
    \label{sec:framework}
    In this section, we extend randomized smoothing and propose a function smoothing framework \framework~(\frameworkfull) for certifying robustness against semantic transformations. This framework constitutes the main building block for \textbf{\framework-R} and \textbf{\framework-DR} against specific types of adversarial transformations.

    Given an arbitrary base classifier $h$, we construct a smoothed classifier $g$ by randomly transforming inputs with parameters sampled from a smoothing distribution. Specifically, given an input $x$, the smoothed classifier $g$ predicts the class that $h$ is most likely to return when the input is perturbed by some random transformation. We formalize this intuition in the following definition.
    \begin{defn}[$\varepsilon$-Smoothed Classifier]
        \label{def:smoothed_classifier}
    	Let $\phi\colon\gX\times\gZ\to\gX$ be a transformation, $\varepsilon\sim\Prob_\varepsilon$ a random variable taking values in $\gZ$ and let $h\colon\gX\to\cY$ be a base classifier. We define the $\varepsilon$-smoothed classifier $g\colon\gX\to\cY$ as $g(x;\varepsilon) = \arg\max_{y\in\gY} q(y\lvert\,x;\,\varepsilon)$
    	where $q$ is given by the expectation with respect to the smoothing distribution $\varepsilon$, i.e.,
    	\begin{equation}
    	    \label{eq:smoothed_classifier}
    		q(y\lvert\,x;\,\varepsilon):=\E(p(y\lvert\,\phi(x,\,\varepsilon))).
    	\end{equation}
    \end{defn}
    A key to certifying robustness against a specific transformation is the choice of transformation $\phi$ in the definition of the smoothed classifier~\eqref{eq:smoothed_classifier}.
    For example, if the goal is to certify the Gaussian blur transformation, a reasonable choice is to use the same transformation in the smoothed classifier. However, for other types of transformations this choice does not lead to the desired robustness certificate, and a different approach is required. In Sections~\ref{sec:resolvable} and~\ref{sec:differentially}, we derive approaches to overcome this challenge and certify robustness against a broader family of semantic transformations.

    \paragraph{General Robustness Certification}

        Given an input $x\in\gX$ and a random variable $\varepsilon$ taking values in $\gZ$,
        suppose that the base classifier $h$ predicts $\phi(x,\varepsilon)$ to be of class $y_A$ with probability at least $p_A$ and the second most likely class with probability at most $p_B$~(i.e., \eqref{eq:4}).
        Our goal is to derive a robustness certificate for the $\varepsilon$-smoothed classifier $g$, i.e., we aim to
        find a set of perturbation parameters $\gS_{\mathrm{adv}}$
        depending on $p_A,\,p_B$, and smoothing parameter $\varepsilon$ such that, for all possible perturbation  $\alpha\in\gS_{\mathrm{adv}}$, it is guaranteed that
        \begin{equation}
            \label{eq:robustness_guarantee}
            g(\phi(x,\,\alpha);\,\varepsilon) = g(x;\,\varepsilon)
        \end{equation}
        In other words, the prediction of the smoothed classifier can never be changed by applying the transformation $\phi$ with parameters $\alpha$ that are in the robust set $\gS_{\mathrm{adv}}$.
        The following theorem provides a generic robustness condition that we will subsequently leverage to obtain conditions on transformation parameters.
        In particular, this result addresses the first challenge \textbf{(C1)} for certifying semantic transformations since this result allows to certify robustness beyond additive perturbations.
        \begin{thm}
            \label{thm:main}
            Let $\varepsilon_0\sim\Prob_0$ and $\varepsilon_1\sim\Prob_1$ be $\gZ$-valued random variables with probability density functions $f_0$ and $f_1$ with respect to a measure $\mu$ on $\gZ$ and let $\phi\colon\gX\times\gZ\to\gX$ be a semantic transformation.
            Suppose that $y_A = g(x;\,\varepsilon_0)$ and let $p_A,\,p_B\in[0,\,1]$ be bounds to the class probabilities, i.e.,
            \begin{equation}
                q(y_A\lvert\,x,\,\epsilon_0) \geq p_A > p_B \geq \max_{y\neq y_A} q(y\lvert\,x,\,\epsilon_0).
                \label{eq:4}
            \end{equation}
            For $t\geq0$, let $\underline{S}_t,\,\overline{S}_t \subseteq \gZ$ be the sets defined as $\underline{S}_t := \{f_1/f_0 < t\}$ and $\overline{S}_t := \{f_1/f_0 \leq t\}$ and define the function $\xi\colon[0,\,1] \to [0,\,1]$ by
            \begin{equation}
                \begin{gathered}
                    \xi(p) := \sup\{\Prob_1(S)\colon\,\underline{S}_{\tau_p}\subseteq S \subseteq\overline{S}_{\tau_p}\}\\
                    \mathrm{where}\hspace{1em}
                    \tau_p := \inf\{t\geq 0\colon\,\Prob_0(\overline{S}_{t}) \geq p\}.
                \end{gathered}
            \end{equation}
            Then, if the condition
            \begin{equation}
                \label{eq:robustness_condition}
                \xi(p_A) + \xi(1-p_B) > 1
            \end{equation}
            is satisfied, then it is guaranteed that $g(x;\,\varepsilon_1) = g(x;\,\varepsilon_0)$.
        \end{thm}
        A detailed proof for this statement is provided in Appendix~\ref{appendix:proof_thrms}.
        At a high level, the condition~\eqref{eq:4} defines a family of classifiers based on class probabilities obtained from smoothing the input $x$ with the distribution $\varepsilon_0$. Based on the Neyman Pearson Lemma from statistical hypothesis testing, shifting $\varepsilon_0 \to \varepsilon_1$ results in bounds to the class probabilities associated with smoothing $x$ with $\varepsilon_1$. For  class $y_A$, the lower bound is given by $\xi(p_A)$, while for any other class the upper bound is given by $1 - \xi(1 - p_B)$, leading to the the robustness condition $\xi(p_A) > 1 - \xi(1 - p_B)$.
        It is
        a more general version of what
        is proved by Cohen et al.~\cite{cohen2019certified}, and its
        generality allows us to analyze
        a larger family of threat models.
        Notice that it is not immediately clear how one can obtain the robustness guarantee~(\ref{eq:robustness_guarantee}) and deriving such a guarantee from Theorem~\ref{thm:main} is nontrivial.
        We will therefore explain in detail how this result can be instantiated to certify semantic transformations in Sections~\ref{sec:resolvable} and~\ref{sec:differentially}.

\section{TSS-R: R{\small esolveable} T{\small ransformations}}
\label{sec:resolvable}

    In this section, we define resolvable transformations and then show how Theorem~\ref{thm:main} is used to certify this class of semantic transformations. We then proceed to providing a robustness verification strategy for each specific transformation.
    In addition, we show how the generality of our framework allows us to reason about the best smoothing strategy for a given transformation, which is beyond the reach of related randomized smoothing based approaches~\cite{yang2020randomized,fischer2019statistical}.

        Intuitively, we call a semantic transformation resolvable if we can separate transformation parameters from inputs with a function that acts on parameters and satisfies certain regularity conditions.
        \begin{defn}[Resolvable transform]
            \label{def:resolvable_transform}
            A transformation $\phi\colon\gX\times\gZ\to\gX$ is called resolvable if for any $\alpha\in\gZ$ there exists a resolving function $\gamma_\alpha\colon\gZ\to\gZ$ that is injective, continuously differentiable, has non-vanishing Jacobian and for which
            \begin{equation}
                \phi(\phi(x,\,\alpha),\,\beta) = \phi(x,\,\gamma_\alpha(\beta))\hspace{1em}x\in\gX,\,\beta\in\gZ.
            \end{equation}
             Furthermore, we say that $\phi$ is additive, if $\gamma_\alpha(\beta) = \alpha + \beta$.
        \end{defn}

        The following result provides a more intuitive view on Theorem~\ref{thm:main}, expressing the condition on probability distributions as a condition on the transformation parameters.

        \begin{cor}
            \label{cor:main_thm_resolvable}
            Suppose that the transformation $\phi$ in Theorem~\ref{thm:main} is resolvable with resolving function $\gamma_\alpha$. Let $\alpha\in\gZ$ and set $\varepsilon_1:=\gamma_\alpha(\varepsilon_0)$ in the definition of the function $\xi$. Then, if $\alpha$ satisfies condition~(\ref{eq:robustness_condition}), it is guaranteed that $g(\phi(x,\alpha);\,\varepsilon_0)=g(x;\,\varepsilon_0)$.
        \end{cor}

        This corollary implies that for resolvable transformations, after we choose the smoothing distribution for the random variable $\varepsilon_0$, we can infer the distribution of $\varepsilon_1 = \gamma_\alpha(\varepsilon_0)$.
        Then, by plugging in $\varepsilon_0$ and $\varepsilon_1$ into \Cref{thm:main}, we can derive an explicit robustness condition from \eqref{eq:robustness_condition} such that for any $\alpha$ satisfying this condition, we can certify the robustness.
        In particular, for additive transformations we have $\epsilon_1 = \gamma_\alpha(\epsilon_0) = \alpha + \epsilon_0$.
        For common smoothing distributions $\epsilon_0$ along with additive transformation, we derive robustness conditions in \Cref{adx-sec:distribution_proofs}.

        In the next subsection, we focus on specific resolvable transformations. For certain transformations, this result can be applied directly. However, for some transformations, e.g.,  the composition of brightness and contrast, more careful analysis is required.
        We remark that this corollary also serves a stepping stone to certifying more complex transformations that are in general not resolvable, such as rotations as we will present in \Cref{sec:differentially}.


        \subsection{Certifying Specific Transformations}
        Here we build on our theoretical results from the previous section and provide approaches to certifying a range of different semantic transformations that are resolvable.
        We state all results here and provide proofs in appendices.

        \subsubsection{Gaussian Blur} \label{sec:main_gaussian_blur}
        This transformation is widely used in image processing to reduce noise and image detail. Mathematically, applying Gaussian blur amounts to convolving an image with a Gaussian function
        \begin{equation}
            G_\alpha(k) = \frac{1}{\sqrt{2\pi\alpha}}\exp\left(-{k^2}/{(2\alpha)}\right)
        \end{equation}
        where $\alpha>0$ is the \emph{squared} kernel radius. For $x\in\gX$, we define Gaussian blur as the transformation $\phi_B\colon\gX\times\R_{\geq0}\to\gX$ where
        \begin{equation}
            \begin{aligned}
                \phi_B(x,\,\alpha) = x\,*\,G_\alpha
            \end{aligned}
        \end{equation}
        and $*$ denotes the convolution operator. The following lemma shows that Gaussian blur is an \emph{additive transform}. Thus, existing robustness conditions for additive transformations shown in \Cref{adx-sec:distribution_proofs} are directly applicable.
        \begin{lem}
            \label{lem:gaussian_blur_additive}
            The Gaussian blur transformation is additive, i.e., for any $\alpha,\,\beta \geq 0$, we have $\phi_B(\phi_B(x,\,\alpha),\,\beta) = \phi_B(x,\,\alpha + \beta)$.
        \end{lem}
        We notice that the Gaussian blur transformation uses only  positive parameters. We therefore consider uniform noise on $[0,\,a]$ for $a>0$, folded Gaussians and exponential distribution for smoothing.

        \subsubsection{Brightness and Contrast} \label{sec:main_brightness_and_contrast}
        This transformation first changes the brightness of an image by adding a constant value $b\in\R$ to every pixel, and then alters the contrast by multiplying each pixel with a positive factor $e^k$, for some $k\in\R$. We define the brightness and contrast transformation $\phi_{BC}\colon\gX\times\R^2\to\gX$ as
        \begin{align}
            (x,\,\alpha) \mapsto \phi_{BC} (x,\,{\alpha}) := e^{k}(x + b),\hspace{1em}{\alpha}=(k,\,b)^T
        \end{align}
        where $k,\,b\in\R$ are contrast and brightness parameters, respectively.
        We remark that $\phi_{BC}$ is resolvable; however, it is not additive and applying Corollary~\ref{cor:main_thm_resolvable} directly using the resolving function $\gamma_\alpha$ leads to analytically intractable expressions.
        On the other hand, if the parameters $k$ and $b$ follow independent Gaussian distributions, we can circumvent this difficulty as follows.
        Given $\varepsilon_0 \sim \gN(0,\,\mathrm{diag}(\sigma^2,\,\tau^2))$, we compute the bounds $p_A$ and $p_B$ to the class probabilities associated with the classifier $g(x;\,\varepsilon_0)$, i.e., smoothed with $\varepsilon_0$.
        In the next step, we identify a distribution $\varepsilon_1$ with the property that
        we can map any lower bound $p$ of $q(y\lvert\,x;\,\varepsilon_0)$ to a lower bound on $q(y\lvert\,x;\,\varepsilon_1)$.
        Using $\epsilon_1$ as a bridge,
        we then derive a robustness condition, which is based on Theorem~\ref{thm:main}, and obtain the guarantee that $g(\phi_{BC}(x,\,\alpha);\,\varepsilon_0) = g(x;\,\varepsilon_0)$ whenever the transformation parameters satisfy this condition.
        The next lemma shows that the distribution $\varepsilon_1$ with the desired property (lower bound to the classifier smoothed with $\varepsilon_1$) is given by a Gaussian with transformed covariance matrix.
        \begin{lem}
            \label{lem:bc_lower_bound}
            Let $x\in\gX$, $k\in\R$, $\varepsilon_0\sim\gN(0,\,\mathrm{diag}(\sigma^2,\,\tau^2))$ and $\varepsilon_1\sim\gN(0,\,\mathrm{diag}(\sigma^2,\,e^{-2k}\tau^2))$.
            Suppose that $q(y\lvert\,x;\,\varepsilon_0)\geq p$ for some $p\in[0,\,1]$ and $y\in\gY$.
            Let $\Phi$ be the cumulative density function of the standard Gaussian.
            Then
            \begin{align}
                \label{eq:bc_lower_bound}
                q(y\lvert\,x;\,\varepsilon_1) \geq
                    \begin{cases}
                         2\Phi\left(e^k\Phi^{-1}\left(\frac{1+p}{2}\right)\right) - 1 & k \leq 0\\
                        2\left(1 - \Phi\left(e^k\Phi^{-1}(1 - \frac{p}{2})\right)\right) &k>0.
                    \end{cases}
            \end{align}
        \end{lem}
        Now suppose that $g(\cdot;\,\varepsilon_0)$ makes the prediction $y_A$ at $x$ with probability at least $p_A$. Then, the preceding lemma tells us that the prediction confidence of $g(\cdot;\,\varepsilon_1)$ satisfies the lower bound~\eqref{eq:bc_lower_bound} for the same class. Based on these confidence levels, we instantiate Theorem~\ref{thm:main} with the random variables $\varepsilon_1$ and ${\alpha} + \varepsilon_1$ to get a robustness condition.
        \begin{lem}
            \label{lem:bc_certification}
            Let $\varepsilon_0$ and $\varepsilon_1$ be as in Lemma~\ref{lem:bc_lower_bound} and suppose that
            \begin{equation}
                q(y_A\lvert\,x;\,\varepsilon_1) \geq \tilde{p}_A > \tilde{p}_B \geq \max_{y\neq y_A}q(y\lvert\,x;\,\varepsilon_1).
            \end{equation}
            Then it is guaranteed that $y_A = g(\phi_{BC}(x,\,{\alpha});\,\varepsilon_0)$ as long as ${\alpha}=(k,\,b)^T$ satisfies
            \begin{equation}
                \label{eq:bc_robustness_bound}
                \sqrt{\left({k}/{\sigma}\right)^2 + \left({b}/({e^{-k}\tau})\right)^2} < \frac{1}{2}\left(\Phi^{-1}\left(\Tilde{p}_A\right)
                -\Phi^{-1}\left(\tilde{p}_B\right)\right).
            \end{equation}
        \end{lem}
        In practice, we apply this lemma by replacing $\tilde{p}_A$ and $\tilde{p}_B$ in~\eqref{eq:bc_robustness_bound} with the bound computed from~\eqref{eq:bc_lower_bound} based on the class probability bounds $p_A$ and $p_B$ associated with the classifier $g(x;\,\varepsilon_0)$.

        \subsubsection{Translation} \label{sec:main_translation}
        Let $\Bar{\phi}_T\colon\gX\times\Z^2\to\gX$ be the transformation moving an image $k_1$ pixels to the right and $k_2$ pixels to the bottom with reflection padding. In order to handle continuous noise distributions, we define the translation transformation $\phi_T\colon\gX\times\R^2\to\gX$ as $\phi_T(x,\,\alpha) =\Bar{\phi}_T(x,\,[\alpha])$ where $[\cdot]$ denotes rounding to the nearest integer, applied element-wise. We note that $\phi_T$ is an \emph{additive transform}, allowing us to directly apply Corollary~\ref{cor:main_thm_resolvable} and derive robustness conditions. We note that if we use black padding instead of reflection padding, the transformation is not additive. However, since the number of possible translations is finite, another possibility is to use a simple brute-force approach that can handle black padding, which has already been studied extensively~\cite{mohapatra2019towards,pei2017towards}.

        \subsubsection{Composition of Gaussian Blur, Brightness, Contrast, and Translation} \label{sec:main_gaussian_blur_brightness_contrast_translation}

        Interestingly, the composition of all these four transformations is still resolvable.
        Thus, we are able to derive the explicit robustness condition for this composition based on \Cref{cor:main_thm_resolvable}, as shown in details in \Cref{appendix:transformation_compositions}.
        Based on this robustness condition, we compute practically meaningful robustness certificates as we will present in experiments in \Cref{sec:exp}.

        \subsubsection{Robustness Certification Strategies}

        \label{sec:computing_certification_summary_resolvable}
        \begin{table}[t]
            \centering
            \caption{\small Summary of the Robustness Certification Strategies for Resolvable Transformations.}
            \resizebox{\linewidth}{!}{
            \begin{tabular}{c|l|l}
                \toprule
                Transformation & Step 1 & Step 2 \\
                \midrule
                Gaussian Blur & \multirow{7}{*}{\shortstack[l]{Compute $p_A$\\ and $p_B$ \\ with\\ Monte-Carlo\\ Sampling}} & Check via \Cref{cor_apx:exponential}~(in \Cref{adx-sec:distribution_proofs})  \\
                \cline{1-1}\cline{3-3}
                Brightness & & Check via \Cref{cor:gaussian_noise}~(in \Cref{adx-sec:distribution_proofs})  \\
                \cline{1-1}\cline{3-3}
                Translation & & Check via \Cref{cor:gaussian_noise}~(in \Cref{adx-sec:distribution_proofs}) \\
                \cline{1-1}\cline{3-3}
                Brightness & & \multirow{2}{*}{\shortstack[l]{Compute $\Tilde{p}_A$ via \Cref{lem:bc_lower_bound}, then check\\ via \Cref{lem:bc_certification}~(detail in \Cref{appendix:brightness_contrast_detail})}} \\
                and Contrast & & \\
                \cline{1-1}\cline{3-3}
                Gaussian Blur, Brightness, & & \multirow{2}{*}{\shortstack[l]{Compute $\Tilde{p}_A$ via \Cref{cor:BTBC_lb}, then check\\ via \Cref{lem:BTBC_bound}~(detail in \Cref{appendix:BTBC_detail})}} \\
                Contrast and Translation & \\
                \bottomrule
            \end{tabular}
            }
            \label{tab:computing-summary-resolvable}
        \end{table}

        With these robustness conditions, for a given clean input $x$, a transformation $\phi$, and a set of parameters $\gS_{\mathrm{adv}}$, we certify the robustness of the smoothed classifier $g$ with two steps: 1)~estimate $p_A$ and $p_B$~(Equation~\eqref{eq:4}) with Monte-Carlo sampling and high-confidence bound following Cohen et al.~\cite{cohen2019certified};
        and 2)~leverage the robustness conditions to obtain the certificate.
        A summary for each transformation including the used robustness conditions are shown in \Cref{tab:computing-summary-resolvable}.

        \begin{figure}[t]
            \centering
            \includegraphics[width=.9\linewidth]{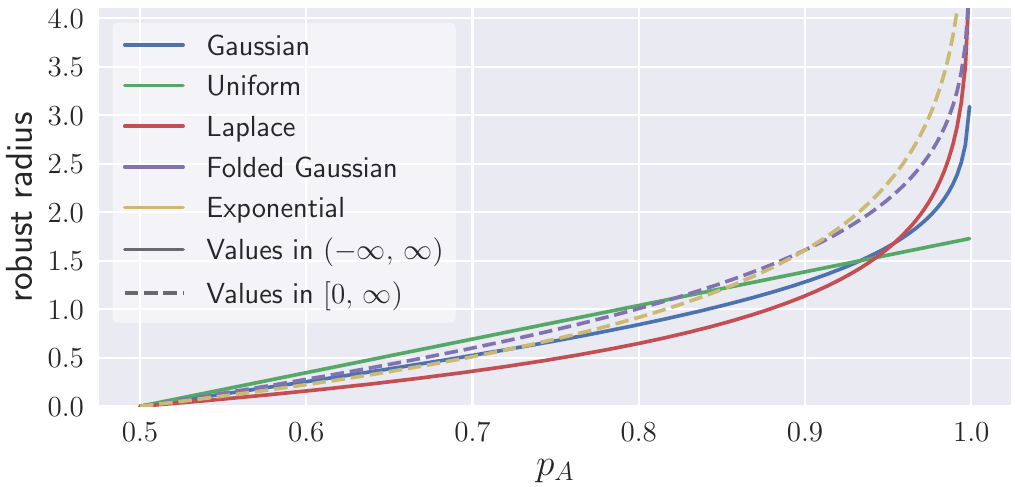}
            \vspace{-4mm}
            \caption{\small Robust radius comparison for different noise distributions, each with unit variance and dimension.}
            \label{fig:bound_compare}
        \end{figure}

        \subsection{Properties of Smoothing Distributions}
        \label{sec:noise_distributions}

        The robustness condition in Theorem~\ref{thm:main} is generic and leaves a degree of freedom with regards to which smoothing distribution should be used.
        Previous work mainly provides results for cases in which this distribution is Gaussian~\cite{cohen2019certified,Zhai2020MACER}, while it is nontrivial to extend it to
        other distributions.
        Here, we aim to answer this question and provide results for a range of distributions, and discuss their differences.
        As we will see, {\em for
        different scenarios, different distributions behave
        differently and can certify different radii.}
        We instantiate Theorem~\ref{thm:main} with an arbitrary transformation $\phi$ and with $\varepsilon_1 := \alpha + \varepsilon_0$ where $\varepsilon_0$ is the smoothing distribution and $\alpha$ is the transformation parameter.
        The robust radius is then derived by solving condition~\eqref{eq:robustness_condition} for $\alpha$.

        \Cref{fig:bound_compare} illustrates robustness radii associated with different smoothing distributions, each scaled to have unit variance. The bounds are derived in~\Cref{adx-sec:distribution_proofs} and summarized in~\Cref{adx-table:smoothing-distributions}.
        We emphasize that the contribution of this work is not merely these results on different smoothing distributions but, more importantly, the \textit{joint study between different smoothing mechanisms and different semantic transformations.}
        To compare the different radii for a fixed base classifier, we assume that
        \textit{the smoothed classifier
        $g(\cdot;\,\varepsilon)$ always has the same confidence $p_A$ for noise distributions with equal variance.}
        Finally, we provide the following conclusions and we will verify them empirically in \Cref{subsec:different-smoothing-dist}.
        \begin{enumerate}[leftmargin=*]
            \item \textit{Exponential noise can provide larger robust radius.} We notice that smoothing with exponential noise generally allows for larger adversarial perturbations than other distributions. We also observe that, while all distributions behave similarly for low confidence levels, it is only non-uniform noise distributions that converge toward $+\infty$ when $p_A\to 1$ and exponential noise converges quickest.
            \item \textit{Additional knowledge can lead to larger robust radius.} When we have additional information on the transformation, e.g., all perturbations in Gaussian blur are positive, we can take advantage of this additional information and certify larger radii. For example, under this assumption, we can use folded Gaussian noise for smoothing instead of a standard Gaussian, resulting in a larger radius.
        \end{enumerate}


\section{TSS-DR: D{\small ifferentially} R{\small esolvable} T{\small ransformations}}
\label{sec:differentially}

    As we have seen, our proposed function smoothing framework can directly deal with resolvable transformations.
    However, due to their use of interpolation, some important transformations do not fall into this category, including rotation, scaling, and their composition with resolvable transformations.
    In this section, we show that they belong to the more general class termed \emph{differentially resolvable transformations} and to address challenge \textbf{(C2)}, we propose a novel pipeline \textbf{\framework-DR} to provide rigorous robustness certification using our function smoothing framework as a central building block.

        Common semantic transformations such as rotations and scaling do not fall into the category of resolvable transformations due to their use of interpolation.
        To see this issue, consider for example the rotation transformation denoted by $\phi_R$.
        As shown in~\Cref{fig:aliasing_issue}, despite very similar, the image rotated by $30^\circ$ is different from the image rotated separately by $15^\circ$ and then again by $15^\circ$. The reason
        is the bilinear interpolation occurring during
        the
        rotation.
        Therefore, if the attacker inputs $\phi_R(x,15)$, the smoothed classifier in \Cref{sec:resolvable} outputs
        \begin{equation}
            g(\phi_R(x,15);\,\varepsilon) = \arg\max_{y\in\gY}\E\left(p(y\lvert\,\phi_R(\phi_R(x,15),\varepsilon)) \right),
        \end{equation}
        which is a weighted average over the predictions of the base classifier on the randomly perturbed set $\{\phi_R(\phi_R(x,15),\alpha)\colon \alpha\in\gZ\}$.
        However, in order to use Corollary~\ref{cor:main_thm_resolvable} and to reason about whether this prediction agrees with the prediction on the clean input (i.e., the average prediction on $\{\phi_R(x,\alpha)\colon \alpha\in\gZ\}$), we need $\phi_R$ to be resolvable.
        As it turns out, this is not the case for transformations that involve interpolation such as rotation and scaling.

        To address these issues, we define a transformation $\phi$ to be \emph{differentially resolvable}, if it can be written in terms of a resolvable transformation $\psi$ and a parameter mapping $\delta$.
        \begin{defn}[Differentially resolvable transform]
            \label{def:differentially_resolvable}
            Let $\phi\colon\gX\times\gZ_\phi\to\gX$ be a transformation with noise space $\gZ_\phi$ and let $\psi\colon\gX\times\gZ_\psi\to\gX$ be a resolvable transformation with noise space $\gZ_\psi$. We say that $\phi$ can be resolved by $\psi$ if for any $x\in\gX$ there exists function $\delta_x\colon\gZ_\phi\times \gZ_\phi \to \gZ_\psi$ such that for any $\alpha \in \gZ_\phi$ and any $\beta\in\gZ_\phi$,
            \begin{equation}
                \phi(x,\,\alpha) = \psi(\phi(x,\,\beta),\,\delta_x(\alpha,\,\beta)).
            \end{equation}
        \end{defn}
        This definition leaves open a certain degree of freedom with regard to the choice of resolvable transformation $\psi$.
        For example, we can choose the resolvable transformation corresponding to additive noise $(x,\,\delta) \mapsto \psi(x,\,\delta) := x + \delta$,
        which lets us write any transformation $\phi$ as $\phi(x,\,\alpha) = \phi(x,\,\beta) + (\phi(x,\,\alpha) - \phi(x,\,\beta)) = \psi(\phi(x,\,\beta),\,\delta)$ with $\delta = (\phi(x,\,\alpha) - \phi(x,\,\beta))$.
        In other words, $\phi(x,\alpha)$ can be viewed as first being transformed to $\phi(x,\beta)$ and then to $\phi(x,\beta) + \delta$.

        \begin{figure}[tbp]
          \centering
          \savebox{\boundingimage}{\includegraphics[width=.45\linewidth]{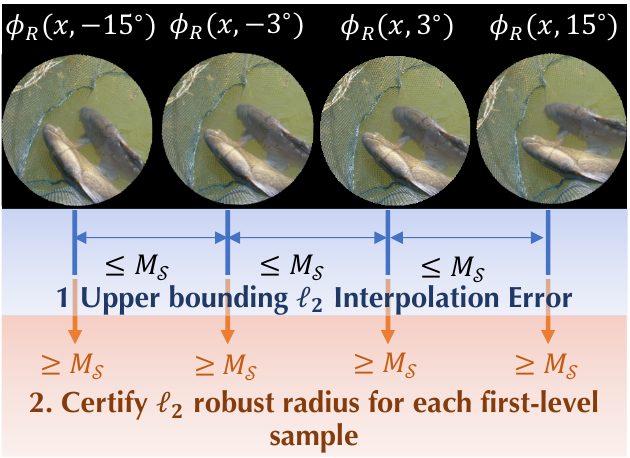}}%
          \begin{subfigure}[b]{0.45\linewidth}
            \centering
            \usebox{\boundingimage}
            \caption{\label{fig:bounding-illustration}}
          \end{subfigure}
          \hfill
          \begin{subfigure}[b]{0.52\linewidth}
            \centering
            \raisebox{\dimexpr.5\ht\boundingimage-.5\height}{%
              \includegraphics[width=\linewidth]{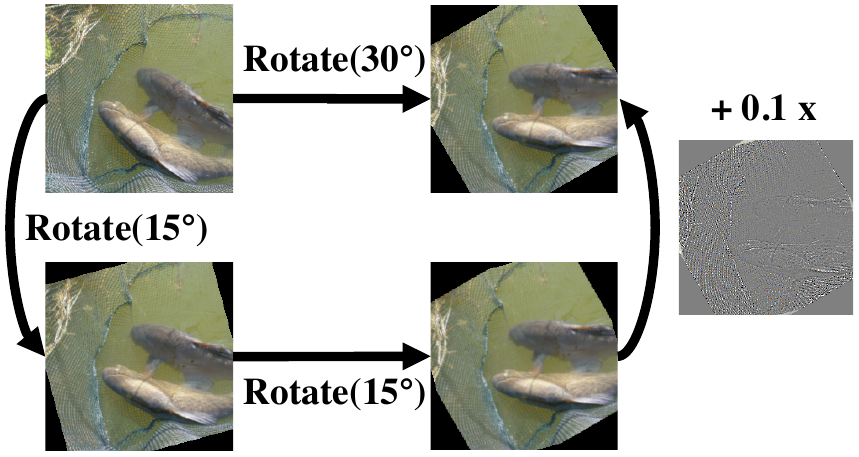}}
            \caption{\label{fig:aliasing_issue}}
          \end{subfigure}
          \vspace{-2mm}
          \caption{\small(\subref{fig:bounding-illustration}) High-level illustration of our robustness certification pipeline \textbf{\framework-DR} for differentially resolvable transformations;~(\subref{fig:aliasing_issue}) interpolation error.}
        \end{figure}

    \subsection{Overview of TSS-DR}
        \label{subsec:dif-resolvable-robust-certification}

        Here, we derive a general robustness certification strategy for differentially resolvable transformations.
        Suppose that our goal is to certify the robustness against a transformation $\phi$ that can be resolved by $\psi$ and for transformation parameters from the set $\gS \subseteq \gZ_{\phi}$. To that end, we first sample a set of parameters $\{\alpha_i\}_{i=1}^N\subseteq\gS$, and transform the input (with those sampled parameters) that yields $\{\phi(x,\alpha_i)\}_{i=1}^N$.
        In the second step, we compute the class probabilities for each transformed input $\phi(x,\,\alpha_i)$ with the classifier smoothed with the resolvable transformation $\psi$.
        Finally, the intuition is that, if every $\alpha\in\gS$ is close enough to one of the sampled parameters, then the classifier is guaranteed to be robust against parameters from the set $\gS$.
        In the next theorem, we show the existence of such a ``proximity set'' for general $\delta_x$.
        \begin{thm}
            \label{thm:main2}
            Let $\phi\colon\gX\times\gZ_\phi\to\gX$ be a transformation that is resolved by $\psi\colon\gX\times\gZ_\psi\to\gX$. Let $\varepsilon\sim\Prob_\varepsilon$ be a $\gZ_\psi$-valued random variable and suppose that the smoothed classifier $g: \gX \to \gY$ given by $q(y\lvert\,x;\varepsilon) = \E(p(y\lvert\,\psi(x,\,\varepsilon)))$ predicts $g(x;\,\varepsilon) = y_A = \argmax_y q(y\lvert\,x;\varepsilon)$.
            Let $\gS\subseteq\gZ_\phi$ and $\{\alpha_i\}_{i=1}^N\subseteq\gS$ be a set of transformation parameters such that for any $i$, the class probabilities satisfy
            \begin{equation}
                q(y_A\lvert\,\phi(x,\,\alpha_i);\,\varepsilon) \geq p_A^{(i)}\geq p_B^{(i)} \geq \max_{y\neq y_A} q(y\lvert\,\phi(x,\,\alpha_i);\,\varepsilon).
            \end{equation}
            Then there exists a set $\Delta^*\subseteq\gZ_\psi$ with the property that, if for any $\alpha\in\gS,\,\exists\alpha_i$ with $\delta_x(\alpha,\,\alpha_i)\in\Delta^*$, then it is guaranteed that
            \begin{equation}
                q(y_A\lvert\,\phi(x,\,\alpha);\varepsilon) > \max_{y\neq y_A}q(y\lvert\,\phi(x,\,\alpha);\varepsilon).
            \end{equation}
        \end{thm}
        In the theorem, the smoothed classifier $g(\cdot;\,\varepsilon)$ is based on the resolvable transformation $\psi$ that serves as a starting point to certify the target transformation $\phi$.
        To certify $\phi$ over its parameter space $\gS$, we input $N$ transformed samples $\phi(x,\alpha_i)$ to the smoothed classifier $g(\cdot;\,\varepsilon)$.
        Then, we get $\Delta^*$, the certified robust parameter set for resolvable transformation $\psi$.
        This $\Delta^*$ means that for any $\phi(x,\alpha_i)$, if we apply the transformation $\psi$ with any parameter $\delta \in \Delta^*$, the resulting instance $\psi(\phi(x,\alpha_i),\delta)$ is robust for $g(\cdot;\varepsilon)$.
        Since $\phi$ is resolvable by $\psi$, i.e., for any $\alpha\in \gS$, there exists an $\alpha_i$ and $\delta \in \Delta^*$ such that $\phi(x,\alpha) = \psi(\phi(x,\alpha_i),\delta)$, we can assert that for any $\alpha\in \gS$, the output of $g(\cdot;\varepsilon)$ on $\phi(x,\alpha)$ is robust.

        The key of using this theorem for a specific transformation is to choose the resolvable transformation $\psi$ that can enable a tight calculation of $\Delta^*$ under a specific way of sampling $\{\alpha_i\}_{i=1}^N$.
        First, we observe that a large family of transformations including rotation and scaling can be resolved by the additive transformation $\psi\colon\gX\times\gX\to\gX$ defined by $(x,\,\delta)\mapsto x + \delta$.
        Indeed, any transformation whose pixel value changes are continuous~(or with finite discontinuities) with respect to the parameter changes are differentially resolvable---they all can be resolved by the additive transformation.
        Choosing isotropic Gaussian noise $\varepsilon\sim\gN(0,\,\sigma
        ^2\Id_d)$ as smoothing noise then leads to the condition that the maximum $\ell_2$-interpolation error between the interval $\gS=[a,\,b]$ (which is to be certified) and the sampled parameters $\alpha_i$ must be bounded by a radius $R$. This result is shown in the next corollary, which is derived from Theorem~\ref{thm:main2}.
        \begin{cor}
            \label{cor:rotations_scaling_certificate}
            Let $\psi(x,\,\delta) = x + \delta$ and let $\varepsilon\sim\gN(0,\,\sigma^2 \Id_d)$. Furthermore, let $\phi$ be a transformation with parameters in $\gZ_\phi\subseteq\R^m$ and let $\gS\subseteq\gZ_\phi$ and $\{\alpha_i\}_{i=1}^N\subseteq\gS$. Let $y_A\in\cY$ and suppose that for any $i$, the $\varepsilon$-smoothed classifier defined by $q(y\lvert\,x;\varepsilon):=\E(p(y\lvert\,x+\varepsilon))$ has class probabilities that satisfy
            \begin{equation}
                 q(y_A\lvert\,\phi(x,\,\alpha_i);\,\varepsilon) \geq p_A^{(i)} \geq p_B^{(i)} \geq \max_{y\neq y_A}q(y\lvert\,\phi(x,\,\alpha_i);\,\varepsilon).
             \end{equation}
            Then it is guaranteed that $\forall\alpha\in\gS\colon\,y_A = \argmax_y q(y\lvert\,\phi(x,\,\alpha);\varepsilon)$ if the maximum interpolation error
            \begin{equation}
                \label{eq:max_interpolation_error}
                M_{\gS}:=\max_{\alpha\in\gS}\min_{1\leq i \leq N}\left\|\phi(x,\,\alpha) - \phi(x,\,\alpha_i)\right\|_2
            \end{equation}
            \begin{equation}
                \label{eq:general_Ms}
                \begin{gathered}
                    \text{satisfies}\quad
                    M_{\gS} < R:=\frac{\sigma}{2}\min_{1\leq i \leq N}\left(\Phi^{-1}\left(p_A^{(i)}\right)-\Phi^{-1}\left(p_B^{(i)}\right)\right).
                \end{gathered}
            \end{equation}
        \end{cor}
        In a nutshell, this corollary shows that if the smoothed classifier classifies
        all
        samples of transformed inputs $\{\phi(x,\alpha_i)\}_{i=1}^N$ consistent with the original input and the smallest gap between
        confidence levels $p_A^{(i)}$ and $p_B^{(i)}$ is large enough, then it is guaranteed to make the same prediction on transformed inputs $\phi(x,\alpha)$ for \emph{any} $\alpha\in\gS$.

        The main challenge now lies in computing a tight and scalable upper bound $M \geq M_{\gS}$. Given this bound, a set of transformation parameters $\gS$ can then be certified by computing $R$ in~\eqref{eq:general_Ms} and checking that $R > M_{\gS}$.
        With this methodology, we address challenge \textbf{(C2)} and provide means to certify transformations that incur interpolation errors.
        \Cref{fig:bounding-illustration} illustrates this methodology on a high level for the rotation transformation as an example.
        In the following, we present the general methodology that provides an upper bound of the interpolation error $M_\gS$ and provide closed form expressions for rotation and scaling.
        In \Cref{appendix:transformation_compositions}, we further extend this methodology to certify transformation compositions such as rotation + brightness change + $\ell_2$ perturbations.

        We remark that dealing with the interpolation error has already been tried before~\cite{balunovic2019certifying,fischer2019statistical}.
        However, these approaches either leverage explicit linear or interval bound propagation -- techniques that are either not scalable or not tight enough.
            Therefore, on large datasets such as ImageNet, they can provide only  limited certification~(e.g., against certain random attack instead of any attack).

        \begin{figure}[!t]
            \centering
            \includegraphics[width=0.85\linewidth]{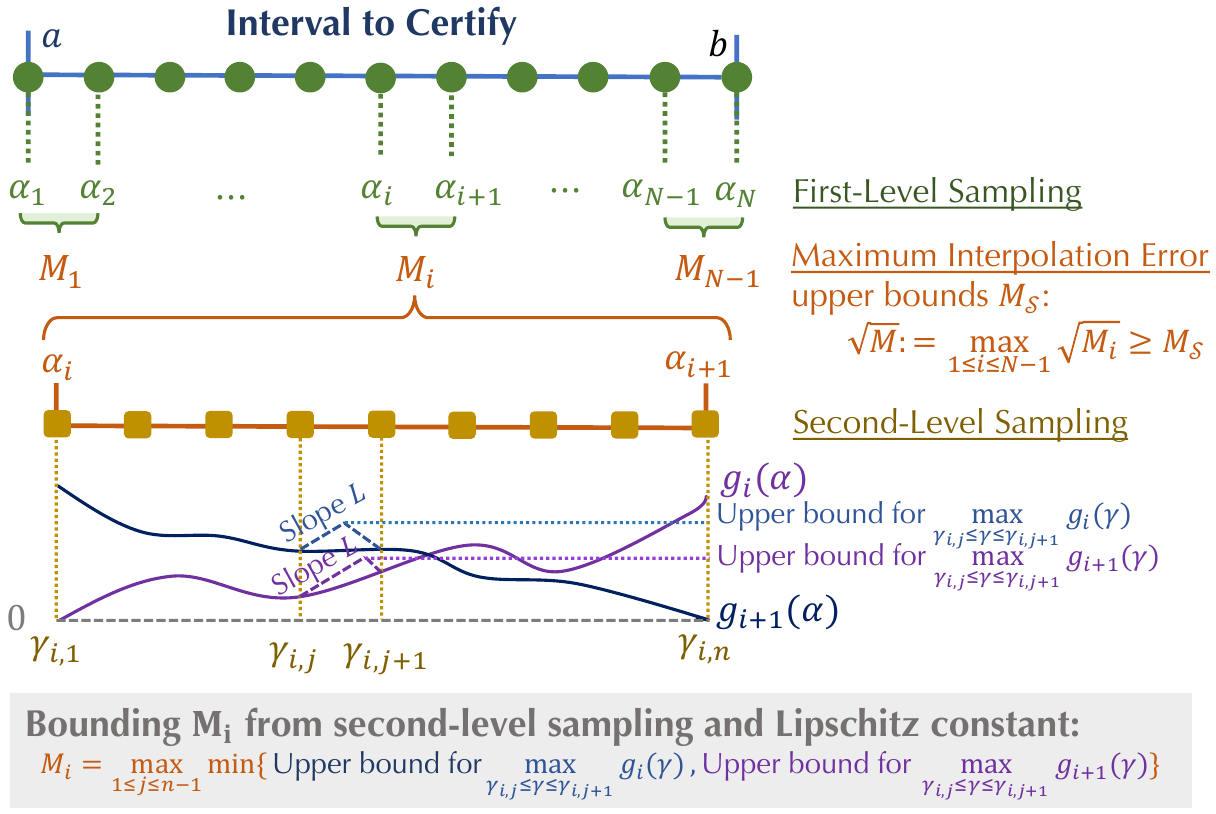}
            \caption{\small An overview of our interpolation error bounding technique based on stratified sampling and Lipschitz computation.}
            \label{fig:stratified-sampling}
        \end{figure}

    \subsection{Upper Bounding the Interpolation Error}
    \label{subsec:interpolation-bound-computation}
    Here, we present the general methodology to compute a rigorous upper bound of the interpolation error introduced in Corollary~\ref{cor:rotations_scaling_certificate}.
    The methodology presented here is based on stratified sampling and is of a general nature; an explicit computation is shown for the case of rotation and scaling toward the end of this subsection.

    Let $\gS = [a,\,b]$ be an interval of transformation parameters that we wish to certify and let $\{\alpha_i\}_{i=1}^N$ be parameters dividing $\gS$ uniformly, i.e.,
    \begin{equation}
        \alpha_i = a + (b-a)\cdot\frac{i-1}{N-1},\hspace{1em}i=1,\,\ldots,\,N. \label{eq:unif-in-a-b}
    \end{equation}
    The set of these parameters corresponds to the first-level samples in stratified sampling. With respect to these first-level samples, we define the functions $g_i\colon[a,\,b]\to\R_{\geq 0}$ as
    \begin{equation}
        \label{eq:g_i_functions_def}
        \alpha\mapsto g_i(\alpha) := \|\phi(x,\alpha) - \phi(x,\alpha_i)\|_2^2
    \end{equation}
    corresponding to the squared $\ell_2$ interpolation error between the image $x$ transformed with $\alpha$ and $\alpha_i$, respectively. For each first-level interval $[\alpha_i,\,\alpha_{i+1}]$ we look for an upper bound $M_i$ such that
    \begin{equation}
        M_i \geq \max_{\alpha_i \leq \alpha \leq \alpha_{i+1}}\min\{g_i(\alpha),\,g_{i+1}(\alpha)\}.
    \end{equation}
    It is easy to see that $\max_{1\leq i \leq N - 1} M_i\geq M_{\gS}^2$ and hence setting
    \begin{equation}
        \sqrt{M}:=\max_{1\leq i \leq N - 1} \sqrt{M_i}
        \label{eq:sqrt-M}
    \end{equation}
    is a valid upper bound to $M_{\gS}$. The problem has thus reduced to computing the upper bounds $M_i$ associated with each first-level interval $[\alpha_i,\,\alpha_{i+1}]$. To that end, we now continue with a second-level sampling within the interval $[\alpha_i,\,\alpha_{i+1}]$ for each $i$.
    Namely, let $\{\gamma_{i,j}\}_{j=1}^n$ be parameters dividing $[\alpha_i,\,\alpha_{i+1}]$ uniformly, i.e.,
    \begin{equation}
        \gamma_{i,j} = \alpha_i + (\alpha_{i+1}-\alpha_i)\cdot\frac{j-1}{n-1},\hspace{1em}j=1,\,\ldots,\,n.
        \label{eq:gamma-i-j}
    \end{equation}
    Now, suppose that $L$ is a global Lipschitz constant for all functions $\{g_i\}_{i=1}^N$. By definition, for any $1\leq i \leq N-1$, $L$ satisfies
    \begin{equation}
        \label{eq:lipschitz_constant}
        \resizebox{\linewidth}{!}{
        $
            L \ge \max\left\{\underset{c,d \in [\alpha_i, \alpha_{i+1}]}{\max} \left|\frac{g_i(c) - g_i(d)}{c-d}\right|,\underset{c,d \in [\alpha_i, \alpha_{i+1}]}{\max} \left| \frac{g_{i+1}(c) - g_{i+1}(d)}{c-d} \right| \right\}.
        $
        }
    \end{equation}
    In the following, we will derive explicit expressions for $L$ for rotation and scaling.
    Given the Lipschitz constant $L$, one can show the following closed-form expression for $M_i$:
    \begin{equation}
        \label{eq:M_i_expression}
        \begin{split}
            M_i &= \frac{1}{2}\max_{1\leq j \leq n-1}\bigg(\,\min\left\{g_i(\gamma_{i,j}) + g_i(\gamma_{i,j+1}),\right.\\&\left.\hspace{2em}g_{i+1}(\gamma_{i,j}) + g_{i+1}(\gamma_{i,j+1})\right\}\,\bigg) + L\cdot\frac{b-a}{(N-1)(n-1)}.
        \end{split}
    \end{equation}
    An illustration of this bounding technique using stratified sampling is shown in~\Cref{fig:stratified-sampling}.
    We notice that, as the number $N$ of first-level samples is increased, the interpolation error $M_i$ becomes smaller by shrinking the sampling interval $[\alpha_i, \alpha_{i+1}]$; similarly, increasing the number of second-level samples $n$ makes the upper bound of the interpolation error $M_i$ tighter since the term $L (b-a) / \left((N-1)(n-1)\right)$ decreases. Furthermore, it is easy to see that as $N\to\infty$ or $n\to\infty$ we have $M \to M_{\gS}^2$, i.e., our interpolation error estimation is \emph{asymptotically tight}.
    Finally, this tendency also highlights an important advantage of our two-level sampling approach: without stratified sampling, it is required to sample $N\times n$ $\alpha_i$'s in order to achieve the same level of approximation accuracy. As a consequence, these $N \times n$ $\alpha_i$'s in turn require to evaluate the smoothed classifier in~\Cref{cor:rotations_scaling_certificate} $N\times n$ times, compared to just $N$ times in our case.

    It thus remains to find a way to efficiently compute the Lipschitz constant $L$ for different transformations. In the following, we derive closed form expressions for rotation and scaling transformations.

        \begin{figure}[!t]
            \centering
            \includegraphics[width=0.85\linewidth]{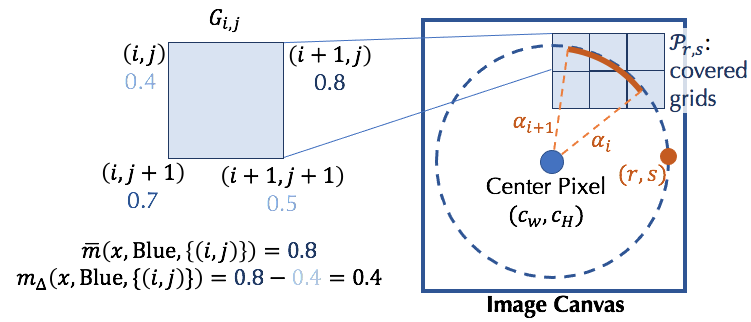}
            \caption{\small An illustration of Grid Pixel Generator $G_{i,j}$, Color Extractors $\bar m$ and $m_{\Delta}$~(take blue channel as example), and the set $\gP_{r,s}$.}
            \label{fig:grid-extractors}
            \vspace{-4mm}
        \end{figure}

    \subsection{Computing the Lipschitz Constant}
        \label{subsec:computing-lipschitz-constant}
        Here, we derive a global Lipschitz constant $L$ for the functions $\{g_i\}_{i=1}^N$ defined in~\eqref{eq:g_i_functions_def}, for rotation and scaling transformations.
        In the following, we define $K$-channel images of width $W$ and height $H$ to be tensors $x\in\R^{K\times W \times H}$ and define the region of valid pixel indices as $\Omega:=[0,\,W-1]\times[0,\,H-1] \cap \N^2$. Furthermore, for $(r,\,s)\in \Omega$, we define $d_{r,s}$ to be the $\ell_2$-distance to the center of an image, i.e.,
        \begin{equation}
            d_{r,s}=\sqrt{\left(r-(W-1)/2\right)^2 + \left(s - (H-1)/2\right)^2}.
        \end{equation}
        For ease of notation we make the following definitions that are illustrated in~\Cref{fig:grid-extractors}.
        \begin{defn}[Grid Pixel Generator]
            \label{def:grid_pixel_generator}
            For pixels $(i,\,j)\in\Omega$, we define the grid pixel generator $G_{ij}$ as
            \begin{equation}
               G_{ij}:=\{(i,\,j),\,(i+1,\,j),\,(i,\,j+1),\,(i+1,\,j+1)\}.
            \end{equation}
        \end{defn}
        \begin{defn}[Max-Color Extractor]
            \label{def:max-color-extractor}
            We define the operator that extracts the channel-wise maximum pixel wise on a grid $S\subseteq\Omega$ as the map $\overline{m}\colon\R^{K\times W \times H}\times \{0,\ldots,K-1\} \times 2^{\Omega}\to \R$ with
            \begin{equation}
                \begin{aligned}
                    \overline{m}(x,\,k,\,S):= \max_{(i,j)\in S}(\max_{(r,s)\in G_{ij}} x_{k,r,s}).
                \end{aligned}
            \end{equation}
        \end{defn}
        \begin{defn}[Max-Color Difference Extractor]
            \label{def:max-color-difference-extractor}
            We define the operator that extracts the channel-wise maximum change in color on a grid $S\subseteq\Omega$ as the map $m_\Delta\colon\R^{K\times W \times H}\times \{0,\ldots,K-1\} \times 2^{\Omega}\to \R$ with
            \begin{equation}
                m_\Delta(x,\,k,\,S):= \max_{(i,j)\in S}(\max_{(r,s)\in G_{ij}} x_{k,r,s}-\min_{(r,s)\in G_{ij}} x_{k,r,s}).
            \end{equation}
        \end{defn}

        \paragraph{Rotation} \label{subsec:rotation}
        The rotation transformation is defined as rotating an image by an angle $\alpha$ counter-clock wise, followed by bilinear interpolation $I$.
        Clearly, when rotating an image, some pixels may be padded that results in a sudden change of pixel colors.
        To mitigate this issue, we apply black padding to all pixels that are outside the largest centered circle in a given image (see~\Cref{fig:bounding-illustration} for an illustration).
        We define the rotation transformation $\phi_R$ as the (raw) rotation $\tilde{\phi}_R$ followed by interpolation and the aforementioned preprocessing step $P$ so that $\phi_R = P \circ I \circ \tilde{\phi}_R$ and refer the reader to Appendix~\ref{sec:apx_transformation_details} for details.
        We remark that our certification is independent of different rotation padding mechanisms, since these padded pixels are all refilled by black padding during preprocessing.
        The following lemma provides a closed form expression for $L$ in~\eqref{eq:M_i_expression} for rotation.
        A detailed proof is given in Appendix~\ref{sec:apx_proofs_interpolation_bound}.
        \begin{lem}
            \label{lem:rotation_lipschitz}
            Let $x\in\R^{K\times W \times H}$ be a $K$-channel image and
            let $\phi_R = P \circ I \circ \tilde{\phi}_R$ be the rotation transformation.
            Then, a global Lipschitz constant $L$ for the functions $\{g_i\}_{i=1}^N$ is given by
            \begin{equation}
                \label{eq:lem_rotation_lipschitz}
                L_{r} = \max_{1\leq i \leq N-1}\sum_{k=0}^{K-1}\sum_{r,s\in V} 2 d_{r,s}\cdot m_\Delta(x,k,\gP_{r,s}^{(i)})\cdot \overline{m}(x,\,k,\,\gP_{r,s}^{(i)})
            \end{equation}
            where $V = \left\{(r,s)\in\N^2\lvert\,d_{r,s} < \frac{1}{2}(\min\left\{W,H\right\}-1)\right\}$. The set $\gP_{r,s}^{(i)}$ is given by all integer grid pixels that are covered by the trajectory of source pixels of $(r,s)$ when rotating from angle $\alpha_i$ to $\alpha_{i+1}$.
        \end{lem}

        \vspace{-2mm}
        \paragraph{Scaling} \label{subsec:scaling}
        In \Cref{appendix:lipschitz-bound-scaling} we introduce how to compute the Lipschitz bound for the scaling transformation and provide the certification.
        The process is similar to that for rotation.
        \vspace{-2mm}

        \paragraph{Computational complexity}
        We provide pseudo-code for computing bound $M$ in Appendix~\ref{apx:algorithms}.
        The algorithm is composed of two main parts, namely the computation of the Lipschitz constant $L$, and the computation of the interpolation error bound $M$ based on $L$.
        The former is of computational complexity $\gO(N\cdot KWH)$, and the latter is of $\gO(NR\cdot KWH)$, for both scaling and rotation.
        We note that $\cP_{r,s}$ contains only a constant number of pixels since each interval $[\alpha_i,\alpha_{i+1}]$ is small.
        Thus, the bulk of costs come from the transformation operation.
        We improve the speed by implementing a fast and fully-parallelized \texttt{C} kernel for rotation and scaling of images.
        As a result, on CIFAR-10, the algorithm takes less than $\SI{2}{s}$ on average with $10$ processes for rotation with $N = 556$ and $n = 200$ and the time for scaling is faster.
        We refer readers to \Cref{sec:exp} for detailed experimental evaluation.
        Also, we remark that the algorithm is model-independent.
        Thus, we can precompute $M$ for test set and reuse for any models that need a certification.

        \subsection{Discussion}
            \label{subsec:differentially-resolvable-discussion}
            Here, we briefly summarize the computation procedure of robustness certification, introduce an acceleration strategy---progressive sampling---and discuss the extensions beyond rotation and scaling.

            \vspace{-0.5em}
            \subsubsection{Computation of Robustness Certification}
                \label{subsec:computation-robust-certification}
                With the methodology mentioned above,
                for differentially resolvable transformations such as rotation and scaling, computing robustness certification follows two steps: (1)~computing the interpolation error bound $M$; (2)~generate transformed samples $\{\phi(x,\alpha_i)\}_{i=1}^N$, compute $p_A^{(i)}$ and $p_B^{(i)}$ for each sample, and check whether $M_{\gS} < R$ holds for each sample according to \Cref{cor:rotations_scaling_certificate}.

            \vspace{-0.5em}
            \subsubsection{Acceleration: Progressive Sampling}
                \label{subsec:progressive-sampling-with-interpolation-bound}
                In step (2) above, we need to estimate $p_A^{(i)}$ and $p_B^{(i)}$ for each sample $\phi(x,\alpha_i)$ to check whether $M_{\gS} < R$.
                In the brute-force approach,
                to obtain a high-confidence bound on $p_A^{(i)}$ and $p_B^{(i)}$, we typically sample $n_s = 10,000$ or more~\cite{cohen2019certified} then apply the binomial statistical test.
                In total, we thus need to sample the classifier's prediction $N \times n_s$ times, which is costly.

                To accelerate the computation, we design a \emph{progressive sampling strategy} from the following two insights: (1)~we only need to check whether $R > M_{\gS}$, but are not required to compute $R$ precisely; (2)~for any sample $\phi(x,\alpha_i)$ if the check fails, the model is not certifiably robust and there is no need to proceed.
                Based on (1), for the current $\phi(x,\,\alpha_i)$, we sample $n_s$ samples in batches and maintain  high-confidence lower bound of $R$ based on  existing estimation.
                Once the lower bound exceeds $M_{\gS}$ we proceed to the next $\phi(x,\,\alpha_{i+1})$.
                Based on (2), we terminate early if the check $R > M_{\gS}$ for the current $\phi(x,\,\alpha_i)$ fails.
                More details are provided
                in Appendix~\ref{apx:algorithms}.

            \vspace{-0.5em}
            \subsubsection{Extension to More Transformations}
                \label{subsec:interpolation-bound-compositions}
                For other transformations that involve interpolation, we can similarly compute the interpolation error bound using intermediate results in our above lemmas.
                For transformation compositions, we extend our certification pipeline for the composition of (1)~rotation/scaling with brightness, and (2)~rotation/scaling with brightness and $\ell_p$-bounded additive perturbations.
                These compositions simulate an attacker who does not precisely perform the specified transformation.
                We present these extensions in \Cref{appendix:scaling_rotation_brightness_detail} and \Cref{appendix:scaling_rotation_brightness_l2_detail} in detail, and in \Cref{appendix:discussion_on_more_compositions} we discuss how to analyze possible new transformations and then extend \framework to provide certification.

            \vspace{-0.5em}
\section{Experiments}
    \label{sec:exp}

    We validate our framework \framework by certifying robustness over semantic transformations experimentally. We compare with state of the art for each transformation, highlight our main results, and present some interesting findings and ablation studies.

            \vspace{-0.5em}
    \subsection{Experimental Setup}

        \subsubsection{Dataset}
            Our experiments are conducted on three classical image classification datasets: MNIST, CIFAR-10, and ImageNet.
            For all images, the pixel color is normalized to $[0,1]$.
            We follow common practice to resize and center cropping the ImageNet images to $224 \times 224$ size~\cite{torchvisionzoo,cohen2019certified,yang2020randomized,jeong2020consistency}.
            To our best knowledge, we are the \emph{first} to provide rigorous certifiable robustness against semantic transformations on the large-scale \emph{standard} ImageNet dataset.

        \vspace{-0.5em}
        \subsubsection{Model}
            The undefended model is very vulnerable even under simple random semantic attacks.
            Therefore, we apply existing data augmentation training~\cite{cohen2019certified} combined with consistency regularization~\cite{jeong2020consistency} to train the base classifiers.
            We then use the introduced smoothing strategies, to obtain the models for robustness certification.
            On MNIST and CIFAR-10, the models are trained from scratch while on ImageNet, we either finetune undefended models in \texttt{torchvision} library or finetune from state-of-the-art certifiably robust models against $\ell_2$ perturbations~\cite{salman2019provably}.
            Details are available in \Cref{adxsec:model-preparation}.
            We remark that our framework focuses on robustness certification and
            did not fully explore the training methods for improving the certified robustness or tune the hyperparameters.

        \vspace{-0.5em}
        \subsubsection{Implementation and Hardware}
            We implement our framework \framework based on \texttt{PyTorch}.
            We improve the running efficiency by tensor parallelism and embedding \texttt{C} modules.
            Details are available in \Cref{adxsec:implementation_details}.
            All experiments were run on $24$-core Intel Xeon Platinum 8259CL CPU and one Tesla T4 GPU with $\SI{15}{GB}$ RAM.

        \vspace{-0.5em}
        \subsubsection{Evaluation Metric}
            \label{subsec:evaluation-metric}
            \label{subsec:random-attack}
            On each dataset, we uniformly pick $500$ samples from the test set and evaluate all results on this \emph{test subset} following Cohen et al~\cite{cohen2019certified}.
            In line with related work~\cite{cohen2019certified,yang2020randomized,salman2019provably,jeong2020consistency}, we report the \textbf{certified
            robust accuracy} that is defined as the fraction of samples (within the test subset) that are both \emph{certified robust} and \emph{classified correctly}, and set the certification confidence level to $p=0.1\%$.
            We use $n_s=10^5$ samples to obtain a confidence lower bound $\underline{p_A}$ for resolvable transformations,
            and $n_s=10^4$ samples to obtain each $\underline{p_A}^{(i)}$ for differentially resolvable transformations.
            Due to Progressive Sampling~(\Cref{algo:progressive-sampling-cert}), the actual samples used for differentially resolvable transformations are usually far fewer than $n_s$.
            In addition, we report the \textbf{benign accuracy} in \Cref{adxsec:benign-acc} defined as the fraction of \emph{correctly classified} samples when no attack is present,
            and the \textbf{empirical robust accuracy}, defined as the fraction of samples in the test subset that are classified correctly under either a simple random attack~(following \cite{balunovic2019certifying,fischer2019statistical}) or two adaptive attacks~(namely Random+ Attack and PGD Attack).
            We introduce all these attacks in \Cref{adxsec:attack_details} and provide a detailed comparison in \Cref{appendix:exp-compare-of-adaptive-attacks}.
            Note that the empirical robust accuracy under any attacks is lower bounded by the certified accuracy.

        \vspace{-0.5em}
        \subsubsection{Notations for Robust Radii}
            In the tables, we use these notations:
            $\alpha$ for squared kernel radius for Gaussian blur;
            $\sqrt{\Delta x^2 + \Delta y^2}$ for translation distance;
            $b$ and $c$ for brightness shift and contrast change respectively as in $x \mapsto (1+c)x + b$;
            $r$ for rotation angle;
            $s$ for size scaling ratio;
            and $\|\delta\|_2$ for $\ell_2$ norm of additional perturbations.

        \vspace{-0.5em}
        \subsubsection{Vanilla Models and Baselines}
            \label{subsec:baselines}
            We compare with vanilla (undefended) models and baselines from related work.
            The vanilla models are trained to achieve high accuracy only on clean data.
            For fairness, on all datasets we use the same model architectures as in our approach.
            On the test subset, the \emph{benign accuracy} of vanilla models is $98.6\%$/$88.6\%$/$74.4\%$ on MNIST/CIFAR-10/ImageNet.
            We also report their empirical robust accuracy under attacks in \Cref{tab:main-attack}.
            Since vanilla models are not smoothed, we cannot have certified robust accuracy for them.
            In terms of baselines, we consider the approaches that provide certification against semantic transformations: DeepG~\cite{balunovic2019certifying}, Interval~\cite{singh2019abstract}, VeriVis~\cite{pei2017towards}, Semantify-NN~\cite{mohapatra2019towards}, and DistSPT~\cite{fischer2019statistical}.
            In~\Cref{adxsec:baselines}, we provide more detailed discussion and comparison with these baseline  approaches, and list how we run these approaches for fair comparison.

\begin{table*}[!t]
	\centering
	\caption{\small Comparison of certified robust accuracy achieved by our framework \framework and other known baselines and empirical robust accuracy achieved by \framework and vanilla models under random and adaptive attacks.
	``-'' denotes the settings where the baselines cannot support.
	The parentheses show the weaker baseline settings.
	For certified robust accuracy, the existing state of the art is \textbf{bolded}.
	For empirical robust accuracy, the higher accuracy under each setting are \textbf{bolded}.
	}
		\centering
		\resizebox{\textwidth}{!}{
		\begin{tabular}{c c c c | c c c c c c | c c c c}
			\toprule
			\multirow{3}{*}{Transformation} & \multirow{3}{*}{Type} & \multirow{3}{*}{Dataset} & \multirow{3}{*}{Attack Radius} & \multicolumn{6}{c|}{\textbf{Certified Robust Accuracy}} & \multicolumn{4}{c}{Empirical Robust Accuracy} \\
			& & & & \multirow{2}{*}{\framework} & \multirow{2}{*}{DeepG~\cite{balunovic2019certifying}} & \multirow{2}{*}{Interval~\cite{singh2019abstract}} & \multirow{2}{*}{VeriVis~\cite{pei2017towards}} & \multirow{2}{*}{Semantify-NN~\cite{mohapatra2019towards}} & \multirow{2}{*}{DistSPT~\cite{fischer2019statistical}} &
			\multicolumn{2}{c}{Random Attack} & \multicolumn{2}{c}{Adaptive Attacks} \\
			& & & & & & & & & & \framework & Vanilla & \framework & Vanilla \\
			\toprule

			\multirow{3}{*}{Gaussian Blur} & \multirow{3}{*}{\shortstack[c]{Resolvable}}

			& MNIST    & Squared Radius $\alpha \le 36$ & $\mathbf{90.6\%}$  & - & - & - & - & -
			& $\textbf{91.4\%}$ & $ 12.2\%$ & $\textbf{91.2\%}$ & $ 12.2\%$ \\
			& & CIFAR-10  & Squared Radius $\alpha \le 16$ & $\mathbf{63.6\%}$ & - & - & - & - & -
			& $\textbf{65.8\%}$ & $ 3.4\%$ & $\textbf{65.8\%}$ & $ 3.4\%$ \\
			& & ImageNet & Squared Radius $\alpha \le 36$ & $\mathbf{51.6\%}$  & - & - & - & - & -
			& $\textbf{52.8\%}$ & $ 8.4\%$ & $\textbf{52.6\%}$ & $ 8.2\%$ \\

			\midrule
			\multirow{3}{*}{\shortstack[c]{Translation\\ (Reflection Pad.)}} &
			\multirow{3}{*}{\shortstack[c]{Resolvable,\\ Discrete}}

			& MNIST    & $\sqrt{\Delta x^2 + \Delta y^2} \le 8$ & $\mathbf{99.6}\%$  & - & - & $98.8\%$ & $98.8\%$ & -
			& $ \textbf{99.6\%}$ & $ 0.0\%$ & $ \textbf{99.6\%}$ & $ 0.0\%$ \\
			& & CIFAR-10  & $\sqrt{\Delta x^2 + \Delta y^2} \le 20$ & $\mathbf{80.8}\%$ & - & - & $65.0\%$ & $65.0\%$ & -
			& $ \textbf{86.2\%}$ & $ 4.4\%$ & $ \textbf{86.0\%}$ & $ 4.2\%$ \\
			& & ImageNet & $\sqrt{\Delta x^2 + \Delta y^2} \le 100$ & $\mathbf{50.0}\%$ & - & - & $43.2\%$ & $43.2\%$ & -
			& $ \textbf{69.2\%}$ & $ 46.6\%$ & $ \textbf{69.2\%}$ & $ 46.2\%$ \\

			\midrule
			\multirow{3}{*}{\shortstack[c]{Brightness}} &
			\multirow{3}{*}{\shortstack[c]{Resolvable}}

			& MNIST    & $b\pm 50\%$ & $\mathbf{98.2\%}$  & - & - & - & - & -
			& $ \textbf{98.2\%}$ & $ 96.6\%$ & $ \textbf{98.2\%}$ & $ 96.6\%$ \\
			& & CIFAR-10 & $b\pm 40\%$ & $\mathbf{87.0\%}$ & - & - & - & - & -
			& $ \textbf{87.2\%}$ & $ 44.4\%$ & $ \textbf{87.4\%}$ & $ 42.6\%$ \\
			& & ImageNet & $b\pm 40\%$ & $\mathbf{70.0\%}$ & - & - & - & - & -
			& $ \textbf{70.4\%}$ & $ 19.6\%$ & $ \textbf{70.4\%}$ & $ 18.4\%$ \\

			\midrule
			\multirow{5}{*}{\shortstack[c]{Contrast\\ and\\ Brightness}} &
			\multirow{5}{*}{\shortstack[c]{Resolvable,\\ Composition}}

			& \multirow{2}{*}{MNIST}    & \multirow{2}{*}{$c\pm 50\%, b\pm 50\%$} & \multirow{2}{*}{$\mathbf{97.6\%}$}  & $\le 0.4\%$ & $0.0\%$ & \multirow{2}{*}{-} & $\le 74\%$ & \multirow{2}{*}{-}
			& \multirow{2}{*}{$\textbf{98.0\%}$} & \multirow{2}{*}{$94.6\%$} & \multirow{2}{*}{$\textbf{98.0\%}$} & \multirow{2}{*}{$93.2\%$} \\
			& & & &  &($c,b\pm 30\%$) & ($c,b\pm 30\%$) & & ($c\pm 5\%,b\pm 50\%$) & \\
			& & \multirow{2}{*}{CIFAR-10} & \multirow{2}{*}{$c\pm 40\%, b\pm 40\%$} & \multirow{2}{*}{$\mathbf{82.4\%}$} & $0.0\%$ & $0.0\%$ & \multirow{2}{*}{-} & \multirow{2}{*}{-} & \multirow{2}{*}{-}
			& \multirow{2}{*}{$\textbf{86.0\%}$} & \multirow{2}{*}{$21.0\%$} & \multirow{2}{*}{$\textbf{85.8\%}$} & \multirow{2}{*}{$9.6\%$} \\
			& & & &  &($c,b\pm 30\%$) & ($c,b\pm 30\%$) & & &  \\
			& & ImageNet & $c\pm 40\%, b\pm 40\%$ & $\mathbf{61.4\%}$ & - & - & - & - & -
			& $\textbf{68.4\%}$ & $ 1.2\%$ & $\textbf{68.4\%}$ & $ 0.0\%$ \\

			\midrule
			\multirow{3}{*}{\shortstack[c]{Gaussian Blur,\\ Translation, Bright-\\ ness, and Contrast}} & \multirow{3}{*}{\shortstack[c]{Resolvable,\\ Composition}}
			&  MNIST    &  $\alpha \le 1, \sqrt{\Delta x^2 + \Delta y^2} \le 5, c,b \pm 10\%$ &  $\textbf{90.2\%}$ & - & - & - & - & -
			& $ \textbf{97.2\%}$ & $ 0.4\%$ & $ \textbf{97.0\%}$ & $ 0.4\%$ \\
			& &  CIFAR-10  &  $\alpha \le 1, \sqrt{\Delta x^2 + \Delta y^2} \le 5, c,b \pm 10\%$ &  $\textbf{58.2\%}$ & - & - & - & - & -
			& $ \textbf{67.6\%}$ & $ 9.6\%$ & $ \textbf{67.8\%}$ & $ 5.6\%$ \\
			& &  ImageNet &  $\alpha \le 10, \sqrt{\Delta x^2 + \Delta y^2} \le 10, c,b\pm 20\%$ &  $\textbf{32.8\%}$ & - & - & - & - & -
			& $ \textbf{48.8\%}$ & $ 9.4\%$ & $ \textbf{47.4\%}$ & $ 4.0\%$ \\

			\midrule
			\multirow{5}{*}{Rotation} &
			\multirow{5}{*}{\shortstack[c]{Differentially\\ Resolvable}}

			& \multirow{2}{*}{MNIST}    & \multirow{2}{*}{$r\pm 50^\circ$} & \multirow{2}{*}{$\mathbf{97.4\%}$} & $\le 85.8\%$ & $\le 6.0\%$ & \multirow{2}{*}{-} & \multirow{2}{*}{$\le 92.48\%$} & \multirow{2}{*}{$82\%$}
			& \multirow{2}{*}{$\textbf{98.4\%}$} & \multirow{2}{*}{$12.2\%$} & \multirow{2}{*}{$\textbf{98.2\%}$} & \multirow{2}{*}{$11.0\%$} \\
			& & & & & ($r\pm 30^\circ$) & ($r\pm 30^\circ$) &  & & \\
			& & \multirow{2}{*}{CIFAR-10} & $r\pm 10^\circ$ & $\mathbf{70.6\%}$ & $62.5\%$ & $20.2\%$ & - & - & $37\%$
			& $ \textbf{76.6\%}$ & $ 65.6\%$ & $ \textbf{76.4\%}$ & $ 65.4\%$ \\
			& & & $r\pm 30^\circ$ & $\mathbf{63.6\%}$ & $10.6\%$ & $0.0\%$ & - & $\le 49.37\%$ & $22\%$
			& $ \textbf{69.2\%}$ & $ 21.6\%$ & $ \textbf{69.4\%}$ & $ 21.4\%$ \\
			& & ImageNet & $r\pm 30^\circ$ & $\mathbf{30.4\%}$ & - & - & - & - & $16 \%$  (rand. attack)
			& $ 37.8\%$ & $ \textbf{40.0\%}$ & $ \textbf{37.8\%}$ & $ 37.0\%$ \\

			\midrule
			\multirow{3}{*}{\shortstack[c]{Scaling}} &
			\multirow{3}{*}{\shortstack[c]{Differentially\\ Resolvable}}
			& MNIST    & $s \pm 30\%$ & $\mathbf{97.2\%}$ & $85.0\%$ & $16.4\%$ & - & - & -
			& $ \textbf{99.2\%}$ & $ 90.2\%$ & $ \textbf{99.2\%}$ & $ 89.2\%$ \\
			& & CIFAR-10 & $s \pm 30\%$ & $\mathbf{58.8\%}$ & $0.0\%$ & $0.0\%$ & - & - & -
			& $ \textbf{67.2\%}$ & $ 51.6\%$ & $ \textbf{67.0\%}$ & $ 51.2\%$ \\
			& & ImageNet & $s \pm 30\%$ & $\mathbf{26.4\%}$ & - & - & - & - & -
			& $ 37.4\%$ & $ \textbf{50.0\%}$ & $ 36.4\%$ & $ \textbf{49.8\%}$ \\

			\midrule
			\multirow{4}{*}{\shortstack[c]{Rotation\\ and\\ Brightness}} &
			\multirow{4}{*}{\shortstack[c]{Differentially\\ Resolvable,\\ Composition}}

			& MNIST    & $r\pm 50^\circ, b\pm 20\%$ & $\mathbf{97.0}\%$ & - & - & - & - & -
			& $ \textbf{98.2\%}$ & $ 11.0\%$ & $ \textbf{98.0\%}$ & $ 10.4\%$ \\
			& & \multirow{2}{*}{CIFAR-10} & $r\pm 10^\circ, b\pm 10\%$ & $\mathbf{70.2}\%$ & - & - & - & - & -
			& $ \textbf{76.6\%}$ & $ 59.4\%$ & $ \textbf{76.0\%}$ & $ 56.8\%$ \\
			& & & $r\pm 30^\circ, b\pm 20\%$ & $\mathbf{61.4}\%$ & - & - & - & - & -
			& $ \textbf{68.4\%}$ & $ 13.0\%$ & $ \textbf{68.2\%}$ & $ 9.0\%$ \\
			& & ImageNet & $r\pm 30^\circ, b\pm 20\%$ & $\mathbf{26.8}\%$ & - & - & - & - & -
			& $ \textbf{37.4\%}$ & $ 22.4\%$ & $ \textbf{36.8\%}$ & $ 21.2\%$ \\

			\midrule
			\multirow{3}{*}{\shortstack[c]{Scaling\\ and\\ Brightness}} &
			\multirow{3}{*}{\shortstack[c]{Differentially\\ Resolvable,\\ Composition}}
			& MNIST    & $s \pm 50\%, b\pm 50\%$ & $\mathbf{96.6}\%$ & - & - & - & - & -
			& $ \textbf{97.8\%}$ & $ 24.8\%$ & $ \textbf{97.8\%}$ & $ 15.6\%$ \\
			& & CIFAR-10 & $s \pm 30\%, b\pm 30\%$ & $\mathbf{54.2}\%$  & - & - & - & - & -
			& $ \textbf{67.2\%}$ & $ 17.4\%$ & $ \textbf{66.8\%}$ & $ 11.6\%$ \\
			& & ImageNet & $s \pm 30\%, b\pm 30\%$ & $\mathbf{23.4}\%$  & - & - & - & - & -
			& $ \textbf{36.4\%}$ & $ 16.0\%$ & $ \textbf{36.0\%}$ & $ 8.8\%$ \\

			\midrule
			\multirow{4}{*}{\shortstack[c]{Rotation,\\ Brightness,\\ and $\ell_2$}} &
			\multirow{4}{*}{\shortstack[c]{Differentially\\ Resolvable,\\ Composition}}

			& MNIST    & $r\pm 50^\circ, b\pm 20\%, \|\delta\|_2\le .05$ & $\mathbf{96.6\%}$ & - & - & - & - & -
			& $ \textbf{97.6\%}$ & $ 10.8\%$ & $ \textbf{97.4\%}$ & $ 9.0\%$ \\
			& & \multirow{2}{*}{CIFAR-10} & $r\pm 10^\circ, b\pm 10\%, \|\delta\|_2\le .05$ & $\mathbf{64.2\%}$ & - & - & - & - & -
			& $ \textbf{71.6\%}$ & $ 31.8\%$ & $ \textbf{71.2\%}$ & $ 29.6\%$ \\
			& & & $r\pm 30^\circ, b\pm 20\%, \|\delta\|_2\le .05$ & $\mathbf{55.2\%}$ & - & - & - & - & -
			& $ \textbf{65.2\%}$ & $ 0.8\%$ & $ \textbf{64.0\%}$ & $ 0.4\%$ \\
			& & ImageNet & $r\pm 30^\circ, b\pm 20\%, \|\delta\|_2\le .05$ & $\mathbf{26.6\%}$ & - & - & - & - & -
			& $ \textbf{37.0\%}$ & $ 17.6\%$ & $ \textbf{36.4\%}$ & $ 14.0\%$ \\

			\midrule
			\multirow{3}{*}{\shortstack[c]{Scaling,\\ Brightness, \\ and $\ell_2$}} &
			\multirow{3}{*}{\shortstack[c]{Differentially\\ Resolvable,\\ Composition}}
			& MNIST    & $s \pm 50\%, b\pm 50\%, \|\delta\|_2\le .05$ & $\mathbf{96.4\%}$ & - & - & - & - & -
			& $ \textbf{97.6\%}$ & $ 22.2\%$ & $ \textbf{97.6\%}$ & $ 12.2\%$ \\
			& & CIFAR-10 & $s \pm 30\%, b\pm 30\%, \|\delta\|_2\le .05$ & $\mathbf{51.2\%}$  & - & - & - & - & -
			& $ \textbf{65.0\%}$ & $ 4.4\%$ & $ \textbf{61.8\%}$ & $ 2.6\%$ \\
			& & ImageNet & $s \pm 30\%, b\pm 30\%, \|\delta\|_2\le .05$ & $\mathbf{22.6\%}$ & - & - & - & - & -
			& $\textbf{36.0\%}$ & $ 7.4\%$ & $\textbf{35.6\%}$ & $ 4.8\%$ \\
			\bottomrule
		\end{tabular}
		}
	\label{tab:main}
    \label{tab:main-attack}
	\vspace{-2mm}
\end{table*}

    \subsection{Main Results}
        \label{sec:main_results}
        Here, we present our main results from five aspects:
            (1)~certified robustness compared to baselines;
            (2)~empirical robustness comparison;
            (3)~certification time statistics;
            (4)~empirical robustness under unforeseen physical attacks;
            (5)~certified robustness under attacks exceeding the certified radii.

        \vspace{-0.5em}
        \subsubsection{Certified Robustness Compared to Baselines}

            Our results are summarized in \Cref{tab:main}.
            For each transformation, we ensure that our setting is either the same as or strictly stronger than all other baselines.\footnote{The only exception is Semantify-NN~\cite{mohapatra2019towards} on brightness and contrast changes, where Semantify-NN considers these changes composed with clipping to $[0,1]$ while we consider pure brightness and contrast changes to align with other baselines. We refer the reader to \Cref{adxsec:baselines} for a detailed discussion.}
            When our setting is strictly stronger, the baseline setting is shown in corresponding parentheses, and our certified robust accuracy implies a higher or equal certified robust accuracy in the corresponding baseline setting.
            To our best knowledge, we are the first to provide certified robustness for Gaussian blur, brightness, composition of rotation and brightness, etc.
            Moreover, on the  large-scale standard ImageNet dataset, we are the first to provide nontrivial certified robustness against certain semantic attacks.
            Note that DistSPT~\cite{fischer2019statistical} is theoretically feasible to provide robustness certification for the ImageNet dataset.
            However, its certification is not tight enough to handle ImageNet  and it provides  robustness certification for only a certain random attack instead of arbitrarily worst-case attacks~\cite[Section 7.4]{fischer2019statistical}.
            We observe that,
            across transformations, our framework \emph{significantly} outperforms the state of the art, if present, in terms of robust accuracy.
            For example, on the composition of contrast and brightness, we improve the certified robust accuracy from $74\%$ to $97.6\%$ on MNIST, from $0.0\%$~(failing to certify) to $82.4\%$ on CIFAR-10, and from $0\%$~(absence of baseline) to $61.4\%$ on ImageNet.
                On the rotation transformation, we improve the certified robust accuracy from $92.48\%$ to $97.4\%$ on MNIST, from $49.37\%$ to $63.6\%$ on CIFAR-10~(rotation angle within $30^\circ$), and from $16\%$ against a certain random attack to $30.4\%$ against arbitrary attacks on ImageNet.
                Some baselines are able to provide certification under other certification goals and
               the readers can refer  to~\Cref{adxsec:baselines} for a detailed discussion.

        \vspace{-0.5em}
        \subsubsection{Comparison of Empirical Robust Accuracy}

            In \Cref{tab:main-attack}, we report the empirical robust accuracy for both (undefended) vanilla models and trained \framework models.
            The empirical robust accuracy is either evaluated under random attack or two adaptive attacks--Random+ and PGD attack.
            When it is under adaptive attacks, we report the lower accuracy to evaluate against stronger attackers.
            \begin{enumerate}[leftmargin=*]
                \item For almost all settings, \framework models have significantly higher \textit{empirical robust accuracy}, which means that
                \framework models are also practical in terms of defending against existing attacks.
                The only exception is rotation and scaling on ImageNet.
                The reason is that a single rotation/scaling transformation is too weak to attack even an undefended model. At the same time, our robustness certification comes at the cost of benign accuracy, which also affects the empirical robust accuracy.
                This exception is eliminated when rotation and scaling are composed with other transformations.

                \item Similar observations arise when comparing the \textit{empirical robust accuracy of the vanilla model with the certified robust accuracy of ours}.
                Hence, even compared to \emph{empirical} metrics, our \emph{certified} robust accuracy is nontrivial and guarantees high accuracy.

                \item
                Our \textit{certified} robust accuracy is always lower or equal compared to the \textit{empirical} one, verifying the validity of our robustness certification.
                The gaps range from $\sim2\%$ on MNIST to $\sim10\%$ - $20\%$ on ImageNet.
                Since empirical robust accuracy is an upper bound of the certified accuracy, this implies that our certified bounds are usually tight, particularly on small datasets.

                    \item
                    The adaptive attack decreases the empirical accuracy of \framework models \emph{slightly}, while it decreases that of vanilla models {significantly}.
                    Taking contrast and brightness on CIFAR-10 as example, \framework accuracy decreases from $86\%$ to $85.8\%$ while the vanilla model accuracy decreases from $21.0\%$ to $9.6\%$.
                    Thus, \framework is still robust against adaptive attacks.
                    Indeed, \framework has robustness guarantee against any attack within the certified radius.
            \end{enumerate}

        \vspace{-0.5em}
        \subsubsection{Certification Time Statistics}
            \label{subsec:certification-time-statistics}
            Our robustness certification time is usually less than $\SI{100}{s}$ on MNIST and $\SI{200}{s}$ on CIFAR-10; on ImageNet it is around $\SI{200}{s}$ - $\SI{2000}{s}$.
            Compared to other baselines, ours is slightly faster and achieves much higher certified robustness.
            For fairness, we give $\SI{1000}{s}$ time limit per instance when running baselines on MNIST and CIFAR-10.
            Note that other baselines cannot scale up to ImageNet.
            Our approach is scalable due to the blackbox nature of smoothing-based certification, the tight interpolation error upper bound, and the efficient progressive sampling strategy.
            Details on hyperparameters including smoothing variance and average certification time are given in~\Cref{appendix:exp-dist-running-time}.

        \vspace{-0.5em}
        \subsubsection{Generalization to Unforeseen Common Corruptions}

            \label{subsec:generalization-to-unforeseen-physical-attacks}
            \emph{Are \framework models still more robust when it comes to potential unforeseen physical attacks?}
            To answer this question, we evaluate the robustness of \framework models on the realistic CIFAR-10-C and ImageNet-C datasets~\cite{hendrycks2019benchmarking}.
            These two datasets are comprised of corrupted images from CIFAR-10 and ImageNet.
            They apply around 20 types of common corruptions to model \textit{physical attacks}, such as fog, snow, and frost.
            We evaluate the \emph{empirical robust accuracy} against the highest corruption level~(level 5) to model the strongest physical attacker.
            We apply \framework models trained against a transformation composition attack, Gaussian blur + brightness + contrast + translation, to defend against these  corruptions.
            We select two baselines: vanilla models and AugMix~\cite{hendrycks2019augmix}.
            AugMix is the state of the art model on CIFAR-10-C and ImageNet-C~\cite{croce2020robustbench}.

            The results are shown in \Cref{tab:cifar10-imagenet-c}.
            The answer is \emph{yes}---\framework models are more robust than undefended vanilla models.
            It even exceeds the state of the art, AugMix, on CIFAR-10-C.
            On ImageNet-C, \framework model's empirical accuracy is between vanilla and AugMix.
            We emphasize that in contrast to \framework, both vanilla and AugMix fail to provide  robustness certification.
            Details on evaluation protocols and additional findings are in \Cref{appendix:exp-corruption}.

            \begin{table}[t]
                \centering
                \caption{\small
                Comparison of \textbf{empirical accuracy} of different models under physical corruptions (CIFAR-10-C and ImageNet-C) and \textbf{certified accuracy} against composition of transformations.
                \framework  achieves higher or comparable empirical accuracy against unforeseen corruptions and significantly higher certified accuracy
                (under attack radii in \Cref{tab:main}).
                }
                \label{tab:cifar10-imagenet-c}
                \resizebox{\linewidth}{!}{
                \begin{tabular}{c|ccg|ccg}
                    \toprule
                     & \multicolumn{3}{c|}{CIFAR-10} & \multicolumn{3}{c}{ImageNet} \\
                    \hline
                    & Vanilla & AugMix~\cite{hendrycks2019augmix} & \framework
                    & Vanilla & AugMix~\cite{hendrycks2019augmix} & \framework \\
                    \hline
                    \multirow{2}{*}{\shortstack[c]{\textbf{Empirical Accuracy}\\ on CIFAR-10-C and ImageNet-C}}  & & & & & & \\
                    &
                    \multirow{-2}{*}{$53.9\%$} & \multirow{-2}{*}{$65.6\%$} & \multirow{-2}{*}{$\textbf{67.4\%}$} & \multirow{-2}{*}{$18.3\%$} & \multirow{-2}{*}{$\textbf{25.7\%}$} & \multirow{-2}{*}{$21.9\%$} \\
                    \hline
                    \multirow{3}{*}{\shortstack[c]{\textbf{Certified Accuracy} against\\ Composition of Gaussian Blur, \\ Translation, Brightness, and Contrast}} &
                    & & & & & \\
                    & & & & & & \\
                    & \multirow{-3}{*}{$0.0\%$} & \multirow{-3}{*}{$0.4\%$} & \multirow{-3}{*}{$\textbf{58.2\%}$} & \multirow{-3}{*}{$0.0\%$} & \multirow{-3}{*}{$0.0\%$} & \multirow{-3}{*}{$\textbf{32.8\%}$} \\
                    \bottomrule
                \end{tabular}
                }
            \end{table}

        \vspace{-0.5em}
        \subsubsection{Evaluation on Attacks Beyond Certified Radii}
            \label{subsec:generalization-to-attack-beyond-radii}

            The semantic attacker in the physical world may not constrain itself to be within the specified attack radii.
            In \Cref{appendix:exp-beyond-radii} we present a thorough evaluation of \framework's robustness when the attack radii go beyond the certified ones.
            We show, for example, for \framework model defending against $\pm 40\%$ brightness change on ImageNet, when the radius increases to $50\%$, the certified accuracy only slightly drops from $70.4\%$ to $70.0\%$.
            In a nutshell, there is no significant or immediate degradation on both certified robust accuracy and empirical robust accuracy when the attack radii go beyond the certified ones.

\begin{figure}[!tbp]
    \centering
    \begin{subfigure}{.3848\linewidth}
        \centering
        \includegraphics[width=\linewidth]{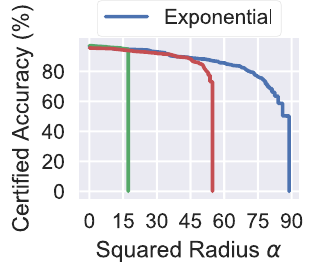}
        \caption{MNIST}
        \label{fig:exp-dist-comparison-mnist}
    \end{subfigure}
    \begin{subfigure}{.2964\linewidth}
        \centering
        \includegraphics[width=\linewidth]{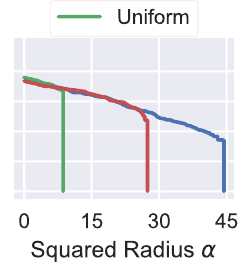}
        \caption{CIFAR-10}
        \label{fig:exp-dist-comparison-cifar}
    \end{subfigure}
    \begin{subfigure}{.286\linewidth}
        \centering
        \includegraphics[width=\linewidth]{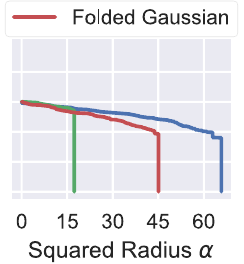}
        \caption{ImageNet}
        \label{fig:exp-dist-comparison-imagenet}
    \end{subfigure}
    \vspace{-2mm}
    \caption{\small
    Certified accuracy for different smoothing distributions for Gaussian blur. On MNIST/CIFAR-10/ImageNet the noise std. is $10/5/10$.
    }
    \label{fig:exp-dist-comparison}
    \vspace{-4mm}
\end{figure}

\begin{table}[t]
    \centering
    \caption{\small Study of the impact of different smoothing variance levels on certified robust accuracy and benign accuracy on \textbf{ImageNet} for \framework. The attack radii are consistent with \Cref{tab:main}.  ``Dist.'' refers to both training and smoothing distribution. The variance used in \Cref{tab:main} is labeled in gray.}

    \resizebox{0.96\linewidth}{!}{

        \begin{tabular}{cccccc}
            \toprule
            \multirow{2}{*}{Transformation} & \multirow{2}{*}{\shortstack[c]{Attack\\ Radii}} & \multicolumn{4}{c}{Certified Accuracy and Benign Accuracy} \\
            & & \multicolumn{4}{c}{under Different Variance Levels} \\
            \midrule

            \multirow{3}{*}{\shortstack[c]{Gaussian Blur}} & \multirow{3}{*}{$\alpha \le 36$} & Dist. of $\alpha$ & $\Exp(1/5)$ & $\Exp(1/10)$ \gray & $\Exp(1/20)$ \\
            \cline{3-6}
            & & Cert. Rob. Acc. & $0.0\%$ & $\textbf{51.6\%}$ & $48.4\%$ \\
            & & Benign Acc. & $\textbf{63.4\%}$ & $59.2\%$ & $53.2\%$ \\
            \hline

            \multirow{3}{*}{\shortstack[c]{Translation\\ (Reflection Pad.)}} & \multirow{3}{*}{\shortstack[c]{$\sqrt{\Delta x^2 + \Delta y^2}$\\  $\le 100$}} & Dist. of $(\Delta x, \Delta y)$ & $\gN(0,20^2 I)$ & $\gN(0, 30^2 I)$ \gray & $\gN(0, 40^2 I)$ \\
            \cline{3-6}
            & & Cert. Rob. Acc. & $0.0\%$ & $50.0\%$ & $\textbf{55.4\%}$ \\
            & & Benign Acc. & $70.0\%$ & $\textbf{72.6\%}$ & $70.0\%$ \\
            \hline

            \multirow{3}{*}{\shortstack[c]{Brightness}} & \multirow{3}{*}{$b \pm 40\%$} & Dist. of $(c,b)$ & $\gN(0, 0.3^2 I)$ & $\gN(0, 0.4^2 I)$ \gray & $\gN(0, 0.5^2 I)$ \\
            \cline{3-6}
            & & Cert. Rob. Acc. & $\textbf{70.2\%}$ & $70.0\%$ & $67.6\%$ \\
            & & Benign Acc. & $\textbf{73.2\%}$ & $72.2\%$ & $69.4\%$ \\
            \hline

            \multirow{3}{*}{\shortstack[c]{Contrast}} & \multirow{3}{*}{$c \pm 40\%$} &  Dist. of $(c,b)$ & $\gN(0, 0.3^2 I)$ & $\gN(0, 0.4^2 I)$ \gray & $\gN(0, 0.5^2 I)$ \\
            \cline{3-6}
            & & Cert. Rob. Acc. & $58.4\%$ & $63.6\%$ & $\textbf{65.0\%}$ \\
            & & Benign Acc. & $\textbf{72.8\%}$ & $71.4\%$ & $68.6\%$ \\
            \hline

            \multirow{3}{*}{Rotation} & \multirow{3}{*}{$r \pm 30^\circ$} & Dist. of $\epsilon$ & $\gN(0, 0.25^2 I)$ & $\gN(0, 0.50^2 I)$ \gray & $\gN(0, 1.00^2 I)$ \\
            \cline{3-6}
            & & Cert. Rob. Acc. & $9.8\%$ & $\textbf{30.4\%}$ & $20.0\%$ \\
            & & Benign Acc. & $\textbf{55.6\%}$ & $46.2\%$ & $32.2\%$ \\
            \hline

            \multirow{3}{*}{Scaling} & \multirow{3}{*}{$s \pm 30\%$} & Dist. of $\epsilon$ & $\gN(0, 0.25^2 I)$ & $\gN(0, 0.50^2 I)$ \gray & $\gN(0, 1.00^2 I)$ \\
            \cline{3-6}
            & & Cert. Rob. Acc. & $7.2\%$ & $\textbf{26.4\%}$ & $17.4\%$ \\
            & & Benign Acc. & $\textbf{58.8\%}$ & $50.8\%$ & $33.8\%$ \\
            \bottomrule
        \end{tabular}
    }
    \label{tab:var-level-imagenet}
\end{table}

    \subsection{Ablation Studies}
    \label{subsec:ablation-studies}
    Here, we provide two ablation studies:
    (1)~Comparison of different smoothing distributions;
    (2)~Comparison of different smoothing variances.
    In \Cref{appendix:exp-tightness-efficiency-trade-off}, we present another ablation study on different numbers of samples for differentially resolvable transformations, which reveals a tightness-efficiency trade-off.

    \vspace{-0.5em}
    \subsubsection{Comparison of Smoothing Distributions}
        \label{subsec:different-smoothing-dist}
        To study the effects of different smoothing distributions, we compare the certified robust accuracy for Gaussian blur when the model is smoothed by different smoothing distributions.
        We consider three smoothing distributions, namely exponential~({blue} line), uniform~({green} line), and folded Gaussian~({red} line).
        On each dataset, we adjust the distribution parameters such that each distribution has the same variance. All other hyperparameters are kept the same throughout training and certification.
        As shown~\Cref{fig:exp-dist-comparison}, we notice that on all three datasets, the exponential distribution has the highest average certified radius. This observation is in line with our theoretical reasoning in~\Cref{sec:noise_distributions}.

    \vspace{-0.5em}
    \subsubsection{Comparison of Different Smoothing Variances}
        \label{subsec:different-smoothing-variance}

        The variance of the smoothing distribution is a hyperparameter that controls the accuracy-robustness trade-off.
        In \Cref{tab:var-level-imagenet}, we evaluate different smoothing variances for several transformations on ImageNet and report both the certified accuracy and benign accuracy.
        The results on MNIST and CIFAR-10 and more discussions are in \Cref{appendix:exp-different-variance-levels}.
        From these results, we observe that usually, when the smoothing variance increases, the benign accuracy drops and the certified robust accuracy first rises and then drops.
        This tendency is also observed in classical randomized smoothing~\cite{cohen2019certified,yang2020randomized}.
        However, the range of acceptable variance is usually wide.
        Thus, without careful tuning the smoothing variances, we are able to achieve high certified  and  benign accuracy as reported in \Cref{tab:main} and \Cref{tab:benign-acc}.

\section{Related Work}
    \label{sec:related-work}

    \paragraph{Certified Robustness against $\ell_p$ perturbations.}
    Since the studies of adversarial vulnerability of neural networks~\cite{szegedy2013intriguing,goodfellow2014explaining}, there has emerged a rich body of research on evasion attacks~(e.g., \cite{carlini2017towards,chaowei2018generating,athalye2018obfuscated,tramer2020adaptive}) and  empirical defenses~(e.g., \cite{madry2017towards,samangouei2018defense,shafahi2019adversarial}).
    To provide robustness certification, different robustness training and verification approaches have been proposed.
    In particular, interval bound propagation~\cite{gowal2018effectiveness,zhang2019towards}, linear relaxations~\cite{weng2018towards,kolter2017provable,wong2018scaling,mirman2018differentiable, xu2020automatic}, and semidefinite programming~\cite{raghunathan2018semidefinite,dathathri2020enabling}  have been applied to certify NN robustness.
    Recently, robustness certification based on randomized smoothing has shown to be scalable and with tight guarantees~\cite{lecuyer2019certified,li2019certified,cohen2019certified}.
    With improvements on optimizing the smoothing distribution~\cite{yang2020randomized,teng2020ell,Dvijotham2020A} and better training mechanisms~\cite{carmon2019unlabeled,salman2019provably,Zhai2020MACER,jeong2020consistency}, the verified robustness of randomized smoothing is further improved.
    A recent survey summarizes certified robustness approaches~\cite{li2020sok}.

    \vspace{-0.5em}
    \paragraph{Semantic Attacks for Neural Networks.}
    Recent work has shown that semantic transformations are able to mislead ML models~\cite{chaowei2018spatially,hosseini2018semantic,ghiasi2020breaking}.
    For instance, image rotations and translations can attack ML models with $40\%$ - $99\%$ degradation on MNIST, CIFAR-10, and ImageNet on both vanilla models and models that are robust against $\ell_p$-bounded perturbations~\cite{engstrom2019exploring}.
    Brightness/contrast attacks can achieve $91.6$\% attack success on CIFAR-10, and $71\%$-$100\%$ attack success rate on ImageNet~\cite{hendrycks2019benchmarking}.
    Our evaluation on empirical robust accuracy~(\Cref{tab:main-attack}) for vanilla models also confirms these observations.
    Moreover, brightness attacks have been shown to be of practical concern in autonomous driving~\cite{pei2017deepxplore}.
    Empirical defenses against semantic transformations have been investigated~\cite{engstrom2019exploring,hendrycks2019benchmarking}.

    \vspace{-0.5em}
    \paragraph{Certified Robustness against Semantic Transformations.}
    While heuristic defenses against semantic attacks have been proposed,
    \emph{provable} robustness requires further investigation.
    Existing certified robustness against transformations is based on heuristic enumeration, interval bound propagation, linear relaxation, or smoothing.
    Efficient enumeration in VeriVis~\cite{pei2017towards} can handle only discrete transformations.
    Interval bound propagation
    has been used to certify common semantic transformations~\cite{singh2019abstract,balunovic2019certifying,fischer2019statistical}.
    To tighten the interval bounds, linear relaxations are introduced.
    DeepG~\cite{balunovic2019certifying} optimizes linear relaxations for given semantic transformations, and Semantify-NN~\cite{mohapatra2019towards} encodes semantic transformations by neural networks and applies linear relaxations for NNs~\cite{weng2018towards,zhang2019towards}.
    However, linear relaxations are loose and computationally intensive compared to our \framework.
    Recently, Fischer et al~\cite{fischer2019statistical} have applied a smoothing scheme to provide provable robustness against transformations but on the large ImageNet dataset, it can provide certification only against random attacks that draw transformation parameters from a pre-determined distribution.
    More details are available in \Cref{adxsec:baselines}.

\section{Conclusion}
\label{sec:conclusion}
In this paper, we have presented a unified framework, \framework, for certifying ML robustness against general semantic adversarial transformations.
Extensive experiments have shown that
\framework
significantly outperforms the state of the art or, if no previous work exists,
set new baselines.
In future work, we plan to further improve the efficiency and tightness of our robustness certification and explore more transformation-specific smoothing strategies.

\begin{acks}
    This work was performed under the auspices of the U.S. Department of Energy by the Lawrence Livermore National Laboratory under Contract No. DE-AC52-07NA27344, Lawrence Livermore National Security, LLC.
    The views and opinions of the authors expressed herein do not necessarily state or reflect those of the United States Government or Lawrence Livermore National Security, LLC, and shall not be used for advertising or product endorsement purposes. This work was supported by LLNL Laboratory Directed Research and Development project 20-ER-014 and released with LLNL tracking number LLNL-CONF-822465.
    This work is supported in part by NSF under grant no. CNS-2046726, CCF-1910100, CCF-1816615, and 2020 Amazon Research Award.

    CZ and the DS3Lab gratefully acknowledge the support from the Swiss National Science Foundation (Project Number 200021\_184628, and 197485), Innosuisse/SNF BRIDGE Discovery (Project Number 40B2-0\_187132), European Union Horizon 2020 Research and Innovation Programme (DAPHNE, 957407), Botnar Research Centre for Child Health, Swiss Data Science Center, Alibaba, Cisco, eBay, Google Focused Research Awards, Kuaishou Inc., Oracle Labs, Zurich Insurance, and the Department of Computer Science at ETH Zurich.

    The authors thank the anonymous reviewers for valuable feedback, Adel Bibi (University of Oxford) for pointing out a bug in the initial implementation, and the authors of \cite{fischer2019statistical}, especially Marc Fischer and Martin Vechev, for insightful discussions and support on experimental evaluation.
\end{acks}


\bibliographystyle{ACM-Reference-Format}
\balance
\bibliography{paper}

\appendix
\newpage

\allowdisplaybreaks

\textbf{\Cref{appendix:lipschitz-bound-scaling}} introduces our method for Lipschitz bound computation for scaling transformation.
\textbf{\Cref{appendix:transformation_compositions}} introduces the certification procedure of \framework for common transformation compositions and discusses the extension to more compositions.
\textbf{\Cref{appendix:proof_thrms}} contains the proofs for \framework general framework, which is introduced in \Cref{sec:framework}.
\textbf{\Cref{adx-sec:distribution_proofs}} contains the theorem statements and proofs for the robustness conditions derived for common smoothing distributions. These statements are instantiations of \Cref{thm:main}, and serve for both certifying resolvable transformations and differentially resolvable transformations.
\textbf{\Cref{adx-sec:smoothing-distributions}} compares the closed-form expressions of the robustness radii derived for different smoothing distributions, which corresponds to \Cref{fig:bound_compare}.
\textbf{\Cref{appendix:proofs_resolvable}} contains the proofs of robustness conditions for various resolvable transformations.
\textbf{\Cref{apx:proofs_differentially_resolvable}} contains the proofs of general theorems for certifying differentially resolvable transformations.
\textbf{\Cref{sec:apx_transformation_details}} formally defines the rotation and scaling transformations, two typical differentially resolvable transformations.
\textbf{\Cref{sec:apx_proofs_interpolation_bound}} proves our supporting theorems for computing the interpolation bound, which is used when certifying differentially resolvable transformations.
\textbf{\Cref{apx:algorithms}} contains concrete algorithm descriptions for certifying differentially resolvable transformations.
Finally, \textbf{\Cref{adxsec:experiment-details}} presents the omitted details in experiments, including experiment settings, detailed discussion of baseline approaches, implementation details, and additional results.

    \section{Computing the Lipschitz Bound for Scaling Transformation}
        \label{appendix:lipschitz-bound-scaling}

        The scaling transformation $\phi_S$ first stretches height and width of the input image by a factor $\alpha\in\R_+$ where values $\alpha < 1$~($> 1$) correspond to shrinking~(enlarging) an image.
        Then, bilinear interpolation is applied, followed by black padding to determine pixel values.
        We refer the reader to Appendix~\ref{sec:apx_transformation_details} for a formal definition.
        Due to black padding, the functions $g_i$ may contain discontinuities.
        To circumvent this issue, we enumerate all these discontinuities as $\gD$. It can be shown that $\gD$ contains at most $H+W$ elements.
        Hence, for large enough $N$, the interval $[\alpha_i,\,\alpha_{i+1}]$ contains at most one discontinuity.
        We thus modify the upper bounds $M_i$ in~\eqref{eq:M_i_expression} as
        \begin{align}
            \small
            M_i :=
                \begin{cases}
                    \max\limits_{\alpha_{i}\leq\alpha\leq\alpha_{i+1}}\min\{g_i(\alpha),\,g_{i+1}(\alpha)\} \,&[\alpha_{i},\,\alpha_{i+1}]\cap \gD = \varnothing\\
                    \max\left\{
                        \max\limits_{\alpha_{i}\leq\alpha\leq t_i} g_{i+1}(\alpha),\,
                        \max\limits_{t_i\leq\alpha \leq \alpha_{i+1}} g_{i}(\alpha)\right
                    \}
                    &[\alpha_{i},\,\alpha_{i+1}]\cap \gD = \{t_i\}
                \end{cases}
        \end{align}
        In either case, the quantity $M_i$ can again be bounded by a Lipschitz constant. With this definition, the following lemma provides a closed form expression for the Lipschitz constant $L$ in~\eqref{eq:M_i_expression} for scaling. A detailed proof is given in Appendix~\ref{sec:apx_proofs_interpolation_bound}.
        \begin{lem}
            \label{lem:scaling_lipschitz}
            Let $x\in\R^{K\times W \times H}$ be a $K$-channel image and
            let $\phi_S$ be the scaling transformation.
            Then, a global Lipschitz constant $L$ for the functions $\{g_i\}_{i=1}^N$ is given by
            \begin{equation}
                \small
                L_{\mathrm{s}} = \max_{1\leq i \leq N-1}\sum_{k=0}^{K-1}\sum_{r,s\in \Omega\cap\N^2} \frac{\sqrt{2} d_{r,s}}{a^2}\cdot m_\Delta(x,k,\gP_{r,s}^{(i)})\cdot \overline{m}(x,\,k,\,\gP_{r,s}^{(i)})
            \end{equation}
            where $\Omega=[0,\,W-1]\times[0,\,H-1]$ and $a$ is the lower boundary value in $\gS=[a,\,b]$. The set $\gP_{r,s}^{(i)}$ is given by all integer grid pixels that are covered by the trajectory of source pixels of $(r,s)$ when scaling with factors from $\alpha_{i+1}$ to $\alpha_{i}$.
        \end{lem}

\section{Certification of Transformation Compositions}
    \label{appendix:transformation_compositions}

    Here we state how \framework certifies typical transformation compositions in detail and discuss how \framework can be directly extended for providing robustness certificates of other transformations or their compositions.

    \subsection{Brightness and Contrast}
        \label{appendix:brightness_contrast_detail}
        As noted in \Cref{sec:main_brightness_and_contrast}, we certify the composition of brightness and contrast based on \Cref{lem:bc_lower_bound} and \Cref{lem:bc_certification}.
        To this end, we first obtain $p_A$, a lower bound of $q(y_A|\,x,\,\varepsilon_0)$ by Monte-Carlo sampling, where $\epsilon_0\sim\gN(0,\,\diag(\sigma^2,\,\tau^2))$ is the smoothing distribution.
        For the given $k, b \in \R$ that we would like to certify $g(\phi_{BC}(x,\,(k,\,b)^T);\,\varepsilon_0) = y_A$,
        we then trigger \Cref{lem:bc_lower_bound} to get $\Tilde{p}_A$, a lower bound of $q(y_A|\,x,\,\varepsilon_1)$, and set $\Tilde{p}_B = 1 - \Tilde{p}_A$.
        Finally, we use the explicit condition in \Cref{lem:bc_certification} to obtain the certification.

        In the actual computation, instead of certifying a single pair $(k,\,b)$, we usually certify the robustness against a set of transformation parameters
        \begin{equation}
            \gS_{\mathrm{adv}} = \{ (k,\,b) |\, k\in [-k_0,\,k_0], b\in [-b_0,\,b_0]\},
            \label{eq:bc_robustness_S}
        \end{equation}
        which stands for any contrast change within $e^{k_0}$ and brightness change within $b_0$.
        It is infeasible to check every $(k,\,b) \in \gS_{\mathrm{adv}}$.
        To mitigate this, we relax the robustness condition in \Cref{lem:bc_certification} from
        \begin{equation}
            \label{eq:bc_robustness_bound_repeat}
            \sqrt{\left({k}/{\sigma}\right)^2 + \left({b}/({e^{-k}\tau})\right)^2} < \frac{1}{2}\left(\Phi^{-1}\left(\Tilde{p}_A\right)
            -\Phi^{-1}\left(\tilde{p}_B\right)\right)
        \end{equation}
        to
        \begin{equation}
            \label{eq:bc_robustness_bound_relax}
            \sqrt{\left({k}/{\sigma}\right)^2 + \left({b}/({\min\{e^{-k},\,1\}\tau})\right)^2} < \frac{1}{2}\left(\Phi^{-1}\left(\Tilde{p}_A\right)
            -\Phi^{-1}\left(\tilde{p}_B\right)\right).
        \end{equation}
        Thus, we only need to check the condition \eqref{eq:bc_robustness_bound_relax} for $(k_0,\,b_0)$ and $(-k_0,\,b_0)$ to certify the robustness for any $(k,\,b)$ in \eqref{eq:bc_robustness_S}.
        This is because the LHS of \eqref{eq:bc_robustness_bound_relax} is monotonically increasing w.r.t. $|k|$ and $|b|$, and the RHS of \eqref{eq:bc_robustness_bound_relax} is equal to $\Phi^{-1}(\Tilde{p}_A)$ that is monotonically decreasing w.r.t. $|k|$.
        Throughout the experiments, we use this strategy for certification of brightness and contrast.

    \subsection{Gaussian Blur, Brightness, Contrast, and Translation}

        \label{appendix:BTBC_detail}
        The certification generally follows the same procedure as in certifying brightness and contrast.
        In the following, we first provide a formal definition of this transformation composition.
        Specifically, the transformation $\phi_{BTBC}$ is defined as:
        \begin{equation}
            \phi_{BTBC}(x,\alpha) := \phi_B( \phi_T( \phi_{BC}(x, \alpha_k, \alpha_b), \alpha_{Tx}, \alpha_{Ty}), \alpha_B ),
        \end{equation}
        where $\phi_B$, $\phi_T$ and $\phi_{BC}$ are Gaussian blur, translation, and brightness and contrast transformations respectively as defined before;
        $\alpha := (\alpha_k, \alpha_b, \alpha_{Tx}, \alpha_{Ty}, \alpha_B)^T \in \R^4 \times \R_{\ge 0}$ is the transformation parameter.

        Our certification relies on the following corollary~(extended from \Cref{lem:bc_lower_bound}) and lemma, which are proved in \Cref{appendix:proofs_resolvable}.
        \begin{cor}
            Let $x \in \gX$, $k\in\R$
            and let $\epsilon_0 := (\epsilon_0^a, \epsilon_0^b)^T$ be a random variable defined as
            \begin{equation}
                \epsilon_0^a \sim \gN(0, \diag(\sigma_k^2,\,\sigma_b^2,\,\sigma_{T}^2,\,\sigma_{T}^2))
                \,\text{and}\,
                \epsilon_0^b \sim \Exp(\lambda_B).
                \label{eq:epsilon_0_BTBC}
            \end{equation}
            Similarly, let $\epsilon_1 := (\epsilon_1^a, \epsilon_1^b)$ be a random variable with
            \begin{equation}
                \epsilon_1^a \sim \gN(0, \diag(\sigma_k^2,\,e^{-2k}\sigma_b^2,\,\sigma_{T}^2,\,\sigma_{T}^2))
                \,\text{and}\,
                \epsilon_1^b \sim \Exp(\lambda_B).
            \end{equation}
            For either random variable~(denoted as $\epsilon$), recall that $q(y|x;\,\epsilon) := \E(p(y|\phi_{BTBC}(x,\,\epsilon)))$.
            Suppose that $q(y|x;\,\epsilon_0) \ge p$ for some $p\in [0,\,1]$ and $y\in \cY$. Then $q(y|x;\,\epsilon_1)$ satisfies Eq.~(11).
            \label{cor:BTBC_lb}
        \end{cor}

        \begin{lemma}
            Let $\epsilon_0$ and $\epsilon_1$ be as in \Cref{cor:BTBC_lb} and suppose that
            \begin{equation}
                q(y_A|x;\,\varepsilon_1) \ge \tilde p_A > \tilde p_B \ge \max_{y\neq y_A} q(y|x;\,\varepsilon_1).
            \end{equation}
            Then it is guaranteed that $y_A = g(\phi_{BTBC}(x,\,\alpha);\varepsilon_0)$
            as long as $p_A' > p_B'$, where
            \begin{equation}
                p_A' = \left\{
                \begin{aligned}
                    & 0, & \hspace{-4.5em} \text{if }\tilde p_A \le 1 - \exp(-\lambda_B \alpha_B), \\
                    &\begin{aligned}
                        &\Phi\left(\Phi^{-1}\left( 1 - (1 - \tilde p_A)\exp(\lambda_B \alpha_B) \right) \right.\\
                        &\hspace{1em}\left. - \sqrt{\nicefrac{\alpha_k^2}{\sigma_k^2} + \nicefrac{\alpha_b^2}{(e^{-2\alpha_k}\sigma_b^2)} + \nicefrac{(\alpha_{Tx}^2+\alpha_{Ty}^2)}{\sigma_{T}^2}}\right)
                    \end{aligned}&\text{otherwise}
                \end{aligned}
                \right.
                \label{eq:BTBC_bound_p_A'}
            \end{equation}
            and
            \begin{equation}
                p_B' = \left\{
                \begin{aligned}
                    & 1, & \hspace{-4.5em} \text{if } \tilde p_B \ge \exp(-\lambda_B\alpha_B), \\
                    & \begin{aligned}
                        &1 - \Phi\left( \Phi^{-1}\left( 1 - \tilde p_B \exp(\lambda_B\alpha_B)\right) \right.\\
                        &\hspace{1em}\left. - \sqrt{\nicefrac{\alpha_k^2}{\sigma_k^2} + \nicefrac{\alpha_b^2}{(e^{-2\alpha_k}\sigma_b^2)} + \nicefrac{(\alpha_{Tx}^2+\alpha_{Ty}^2)}{\sigma_{T}^2}}\right).
                    \end{aligned}&\text{otherwise}
                \end{aligned}
                \right.
                \label{eq:BTBC_bound_p_B'}
            \end{equation}
            \label{lem:BTBC_bound}
        \end{lemma}

        The $\epsilon_0$ specified by \eqref{eq:epsilon_0_BTBC} is the smoothing distribution.
        Similar as in brightness and contrast certification, we first obtain $p_A$, a lower bound of $q(y_A|\,x,\,\varepsilon_0)$ by Monte-Carlo sampling.
        For a given transformation parameter $\alpha := (\alpha_k,\,\alpha_b,\,\alpha_{Tx},\,\alpha_{Ty},\,\alpha_B)^T$, we then trigger \Cref{cor:BTBC_lb} to get $\Tilde{p}_A$, a lower bound of $q(y_A|\,x,\,\varepsilon_1)$ and set $\Tilde{p}_B = 1 - \tilde{p}_A$.
        Finally, we use the explicit condition in \Cref{lem:BTBC_bound} to obtain the certification.
        Indeed, with $\Tilde{p}_B = 1 - \Tilde{p}_A$, \Cref{lem:BTBC_bound} can be simplified to the following corollary.

        \begin{cor}
            Let $\epsilon_0$ and $\epsilon_1$ be as in \Cref{cor:BTBC_lb} and suppose that
            \begin{equation}
                q(y_A|x;\,\varepsilon_1) \ge \tilde p_A.
            \end{equation}
            Then it is guaranteed that $y_A = g(\phi_{BTBC}(x,\,\alpha);\varepsilon_0)$ as long as
            \begin{equation}
                \begin{aligned}
                    \tilde p_A > 1 - &\exp(-\lambda_B \alpha_B) \Bigg( 1 - \\
                    &\hspace{3em}\Phi\left(\sqrt{\frac{\alpha_k^2}{\sigma_k^2} + \frac{\alpha_b^2}{e^{-2\alpha_k}\sigma_b^2} + \frac{\alpha_{Tx}^2+\alpha_{Ty}^2}{\sigma_{T}^2}}\right) \Bigg).
                \end{aligned}
                \label{eq:cor-BTBC_bound}
            \end{equation}
            \label{cor:BTBC_bound}
        \end{cor}

        To certify against a set of transformation parameters
        \begin{equation}
            \begin{aligned}
                \gS_{\mathrm{adv}} = & \{ (\alpha_k,\,\alpha_b,\,\alpha_{Tx},\,\alpha_{Ty},\,\alpha_B)^T | \\
                & \hspace{2em} \alpha_k \in [-k_0,\,k_0], \alpha_b \in [-b_0,\,b_0], \\
                & \hspace{2em} \|(\alpha_{Tx},\alpha_{Ty})\|_2 \le T, \alpha_B \le B_0 \},
            \end{aligned}
        \end{equation}
        we relax the robust condition in \eqref{eq:cor-BTBC_bound} to
        \begin{equation}
            \begin{aligned}
                \tilde p_A > 1 - &\exp(-\lambda_B \alpha_B) \Bigg( 1 - \\
                &\hspace{2em}\Phi\left(\sqrt{\frac{\alpha_k^2}{\sigma_k^2} + \frac{\alpha_b^2}{\min\{e^{-2\alpha_k},\,1\}\sigma_b^2} + \frac{\alpha_{Tx}^2+\alpha_{Ty}^2}{\sigma_{T}^2}}\right) \Bigg).
            \end{aligned}
            \label{eq:BTBC_bound_relax}
        \end{equation}
        The LHS of \Cref{eq:BTBC_bound_relax} is monotonically decreasing w.r.t. $|\alpha_k|$ and the RHS is monotonically increasing w.r.t. $|\alpha_k|, |\alpha_b|, \|(\alpha_{Tx},\alpha_{Ty}\|_2$, and $|\alpha_B|$, and the RHS is symmetric w.r.t. $\alpha_b$ and $\|(\alpha_{Tx},\alpha_{Ty})\|_2$.
        As a result, we only need to check the condition for $(-k_0,\,b_0,\,T,\,0,\,B_0)$ and $(k_0,\,b_0,\,T,\,0,\,B_0)$ to certify the entire set $\gS_{\mathrm{adv}}$.
        Throughout the experiments, we use this strategy for certification.

    \subsection{Scaling/Rotation and Brightness}

        \label{appendix:scaling_rotation_brightness_detail}
        To certify the composition of scaling and brightness or rotation and brightness, we follow the same methodology as certifying scaling or rotation alone and reuse the computed interpolatation error $M_\gS$.
        We only make the following two changes:
        (1)~alter the smoothing distribution from additive Gaussian noise $\psi(x,\delta)=x+\delta$ where $\delta\sim\gN(0,\sigma^2\Id_d)$ to additive Gaussian noise and Gaussian brightness change $\psi(x,\delta,\delta_b) = x+\delta + b \cdot \Id_d$ where $\delta\sim\gN(0,\sigma^2 \Id_d), b \sim\gN(0,\sigma_b^2)$;
        (2)~change the robustness condition from $R > M_{\gS}$ in \Cref{cor:rotations_scaling_certificate} to $R > \sqrt{M_{\gS}^2 + (\nicefrac{\sigma^2}{\sigma_b^2})b_0^2}$.
        We formalize this robustness condition in the following corollary, and the proof is entailed in \Cref{apx:proofs_differentially_resolvable}.

        \begin{cor}
            \label{cor:rotations_scaling_brightness_certificate}
            Let $\psi_B(x,\,\delta,\,b) = x + \delta + b\cdot \Id_d$ and let $\varepsilon\sim\gN(0,\,\sigma^2 \Id_d)$, $\varepsilon_b\sim\gN(0,\,\sigma_b^2)$. Furthermore, let $\phi$ be a transformation with parameters in $\gZ_\phi\subseteq\R^m$ and let $\gS\subseteq\gZ_\phi$ and $\{\alpha_i\}_{i=1}^N\subseteq\gS$. Let $y_A\in\cY$ and suppose that for any $i$, the $(\varepsilon,\varepsilon_b)$-smoothed classifier $q(y\lvert\,x;\,\varepsilon,\varepsilon_b):=\E(p(y\lvert\,\psi_B(x,\,\varepsilon,\,\varepsilon_b))$ satisfies
            \begin{equation}
                q(y_A\lvert\,x;\,\varepsilon,\varepsilon_b) \geq p_A^{(i)} > p_B^{(i)} \geq \max_{y\neq y_A}q(y\lvert\,x;\,\varepsilon,\varepsilon_b).
            \end{equation}
            for each $i$. Let
            \begin{equation}
                R := \frac{\sigma}{2}\min_{1\leq i \leq N}\left(\Phi^{-1}\left(p_A^{(i)}\right)-\Phi^{-1}\left(p_B^{(i)}\right)\right)
            \end{equation}
            Then, $\forall \alpha \in \gS$ and $\forall b \in [-b_0, b_0]$ it is guaranteed that $y_A = \argmax_y q(y\lvert\,\phi(x, \alpha) + b\cdot\Id_d;\, \varepsilon, \varepsilon_b)$ as long as
            \begin{equation}
                R > \sqrt{M_{\gS}^2 + \dfrac{\sigma^2}{\sigma_b^2}b_0^2},
                \label{eq:general_Ms_brightness_stmt}
            \end{equation}
            where $M_\gS$ is defined as in \Cref{cor:rotations_scaling_certificate}.
        \end{cor}

    \subsection{Scaling/Rotation, Brightness, and \texorpdfstring{$\ell_2$}{L2} Perturbations}

        \label{appendix:scaling_rotation_brightness_l2_detail}
        We use the same smoothing distribution as above, and the following corollary directly allows us to certify the robustness against the composition of scaling/rotation, brightness, and an additional $\ell_2$-bounded perturbations---we only need to change the robustness condition from $R > \sqrt{M_{\gS}^2 + (\nicefrac{\sigma^2}{\sigma_b^2})b_0^2}$ to $R > \sqrt{(M_{\gS}+r)^2 + (\nicefrac{\sigma^2}{\sigma_b^2})b_0^2}$.
        The proof is given in \Cref{apx:proofs_differentially_resolvable}.

        \begin{cor}
            \label{cor:rotations_scaling_brightness_l2_certificate}
            Under the same setting as in~\Cref{cor:rotations_scaling_brightness_certificate},
            for $\forall \alpha \in \gS$, $\forall b \in [-b_0, b_0]$ and $\forall \delta\in\R^d$ such that $\|\delta\|_2 \le r$, it is guaranteed that $y_A = \argmax_k q(y\lvert\,\phi(x,\alpha) + b\cdot\Id_d + \delta;\,\varepsilon,\varepsilon_d)$
            as long as
            \begin{equation}
                \label{eq:general_Ms_brightness_l2_stmt}
                R > \sqrt{(M_{\gS}+r)^2 + \dfrac{\sigma^2}{\sigma_b^2}b_0^2},
            \end{equation}
            where $M_\gS$ is defined as in \Cref{cor:rotations_scaling_certificate}.
        \end{cor}

    \subsection{Discussion on More Transformations and Compositions}

        \label{appendix:discussion_on_more_compositions}
        Our \framework is not limited to specific transformations.
        Here we briefly discuss how to extend \framework for any new transformations or new compositions.

        For a new transformation, we first identify the parameter space $\gZ$, where the identified parameter should completely and deterministically decide the output after transformation for any given input.
        Then, we use ~\Cref{def:resolvable_transform} to check whether the transformation is resolvable.
        If so, we can write down the function $\gamma_\alpha$. 
        Next, we choose a smoothing distribution, i.e., the distribution of the random variable $\varepsilon_0$, and identify the distribution of $\varepsilon_1 = \gamma_\alpha(\varepsilon_0)$.
        Finally, we use \Cref{thm:main} to derive the robustness certificates and follow the two-step template~(\Cref{sec:computing_certification_summary_resolvable}) to compute the robustness certificate.

        If the transformation is not resolvable, we identify a dimension in $\gZ$ for which the transformation is \emph{resolvable}.
        For example, the composition of rotation and brightness has a rotation and a brightness axis, where the brightness axis is itself resolvable.
        As a result, we can write the parameter space as Cartesian product of non-resolvable subspace and resolvable subspace: $\gZ := \gZ_{\text{no-resolve}} \times \gZ_{\text{resolve}}$.
        We perform smoothing on the resolvable subspace and sample enough points in the non-resolvable subspace.
        Next, we bound the interpolation error between sampled points and arbitrary points in the non-resolvable subspace, using either $\ell_p$ difference as we did for rotation and scaling or other regimes.
        Specifically, our \Cref{lem:lipschitz_constant} shown in \Cref{sec:apx_proofs_interpolation_bound} is a useful tool to bound $\ell_p$ difference caused by interpolation error.
        Finally, we instantiate \Cref{thm:main2} to compute the robustness certificate.

        Theoretically, we can certify against the composition of all the discussed transformations: Gaussian blur, brightness, contrast, translation, rotation, and scaling.
        However, as justified in \cite[Figure 3]{hendrycks2019augmix}, the composition of more than two transformations leads to unrealistic images that are even hard to distinguish by humans.
        Moreover,
        if the composition contains too many transformations, the parameter space would be no longer low dimensional. Therefore, there would be much more axes that are differentially resolvable (instead of resolvable).
        As a consequence, much more samples are required to obtain a small bound on interpolation error~(which is necessary for a nontrivial robustness certification).
        Therefore, we mainly evaluate on single transformation or composition of two transformations to simulate a practical attack.

\section{Proofs for the general certification framework}
\label{appendix:proof_thrms}

    Here we provide the proof for Theorem~\ref{thm:main}. For that purpose, recall the following definition from the main part of this paper:
    \begin{customdef}{\ref{def:smoothed_classifier}}[restated]
        Let $\phi\colon\gX\times\gZ\to\gX$ be a transformation, $\varepsilon\sim\Prob_\varepsilon$ a random variable taking values in $\gZ$ and let $h\colon\gX\to\cY$ be a base classifier. We define the $\varepsilon$-smoothed classifier $g\colon\gX\to\cY$ as $g(x;\varepsilon) = \arg\max_{y\in\gY} q(y\lvert\,x;\,\varepsilon)$
    	where $q$ is given by the expectation with respect to the smoothing distribution $\varepsilon$, i.e.,
    	\begin{equation}
    		q(y\lvert\,x;\,\varepsilon):=\E(p(y\lvert\,\phi(x,\,\varepsilon))).
    	\end{equation}
    \end{customdef}
    Here, we additionally define the notion of level sets separately. These sets originate from statistical hypothesis testing correspond to rejection regions of likelihood ratio tests.
    \begin{defn}[Lower level sets]
        \label{def:lower_level_sets}
        Let $\varepsilon_0\sim\Prob_0$, $\varepsilon_1\sim\Prob_1$ be $\gZ$-valued random variables with probability density functions $f_0$ and $f_1$ with respect to a measure $\mu$. For $t\geq0$ we define lower and strict lower level sets as
        \begin{equation}
            \begin{gathered}
                \underline{S_t} := \left\{z\in\gZ\colon\,\Lambda(z) < t \right\},
                \hspace{1em}\overline{S_t} := \left\{z\in\gZ\colon\,\Lambda(z) \leq t \right\},\\
                \text{where}\hspace{1em}\Lambda(z):=\frac{f_1(z)}{f_0(z)}.
            \end{gathered}
        \end{equation}
    \end{defn}
    We also make the following definition in order to reduce clutter and simplify the notation. This definition will be used throughout the proofs presented here.
    \begin{defn}[$(p_A,p_B)$-Confident Classifier]
        \label{def:pa_pb_confidence}
    	Let $x\in\gX$, $y_A\in\gY$ and $p_A,\,p_B\in[0,\,1]$ with $p_A > p_B$. We say that the $\varepsilon$-smoothed classifier $q$ is $(p_A,\,p_B)$-confident at $x$ if
    	\begin{equation}
    		\label{eq:pa_pb_effective}
    		q(y_A\lvert\,x;\,\varepsilon) \geq p_A \geq p_B \geq \max_{y\neq y_A}q(y\lvert\,x;\,\varepsilon).
    	\end{equation}
    \end{defn}

    \subsection{Auxiliary Lemmas}
    \begin{lem}
        \label{lem:sandwich}
        Let $\varepsilon_0$ and $\varepsilon_1$ be random variables taking values in $\gZ$ and with probability density functions $f_0$ and $f_1$ with respect to a measure $\mu$. Denote by $\Lambda$ the likelihood ratio $\Lambda(z) = f_1(z)/f_0(z)$. For $p\in[0,\,1]$ let $\tau_p:=\inf\{t\geq0\colon\,\Prob_0(\overline{S_t}) \geq p\}$. Then, it holds that
        \begin{equation}
            \label{eq:sandwich}
            \Prob_0\left(\underline{S}_{\tau_p}\right) \leq p \leq \Prob_0\left(\overline{S}_{\tau_p}\right).
        \end{equation}
    \end{lem}
    \begin{proof}
        We first show the RHS of inequality~(\ref{eq:sandwich}). This follows directly from the definition of $\tau_p$ if we show that the function $t\mapsto \Prob_0\left(\overline{S}_{t}\right)$ is right-continuous. For that purpose, let $t\geq 0$ and let $\{t_n\}_n$ be a sequence in $\R_{\geq0}$ such that $t_n\downarrow t$. Define the sets $A_n:=\{z\colon\,\Lambda(z) \leq t_n\}$ and note that $A_{n+1}\subseteq A_n$. Clearly, if $z\in \overline{S}_{t}$, then $\forall\,n\colon\,\Lambda(z) \leq t \leq t_n$, thus $z\in\cap_n A_n$ and hence $\overline{S}_{t} \subseteq \cap_n A_n$. If on the other hand $z\in \cap_n A_n$, then $\forall n\colon \Lambda(z) \leq t_n \to t$ as $n\to \infty$ and thus $z\in \overline{S}_{t}$, yielding $\overline{S}_{t} = \cap_n A_n$. Hence for any $t\geq0$ we have that
        \begin{equation}
            \lim_{n\to\infty}\Prob_0\left(A_n\right) = \Prob_0\left(\bigcap_n\,A_n\right) = \Prob_0\left(\underline{S}_{t}\right).
        \end{equation}
        Thus, the function $t\mapsto \Prob_0\left(\overline{S}_{t}\right)$ is right continuous and in particular it follows that $\Prob_0\left(\overline{S}_{\tau_p}\right) \geq p$. We now show the LHS of inequality~(\ref{eq:sandwich}). Consider the sets $B_n:=\{z\colon\,\Lambda(z) < \tau_p - \nicefrac{1}{n}\}$ and note that $B_n \subseteq B_{n+1}$. Clearly, if $z\in\cup_n\,B_n$, then $\exists\,n$ such that $\Lambda(z) < \tau_p - \nicefrac{1}{n} < \tau_p$ and thus $z\in \underline{S}_{\tau_p}$. If on the other hand $z\in \underline{S}_{\tau_p}$, then we can choose $n$ large enough such that $\Lambda(z) < \tau_p - \nicefrac{1}{n}$ and thus $z\in\cup_n\,B_n$ yielding $\underline{S}_{\tau_p} = \cup_n\,B_n$. Furthermore, by the definition of $\tau_p$ and since for any $n\in \N$ we have that $\Prob_0(B_n) = \Prob_0\left(\underline{S}(\tau_p-\nicefrac{1}{n})\right) < p$ it follows that
        \begin{equation}
            \Prob_0\left(\underline{S}_{\tau_p}\right) = \Prob_0\left(\bigcup_n\,B_n\right) = \lim_{n\to\infty}\Prob_0\left(B_n\right) \leq p
        \end{equation}
        concluding the proof.
    \end{proof}
    \begin{lem}
        \label{lem:np_extended}
        Let $\varepsilon_0$ and $\varepsilon_1$ be random variables taking values in $\gZ$ and with probability density functions $f_0$ and $f_1$ with respect to a measure $\mu$. Let $h\colon\gZ\to[0,\,1]$ be a determinstic function. Then, for any $t\geq 0$ the following implications hold:
        \begin{enumerate}[leftmargin=*]
            \item[(i)] For any $S\subseteq\gZ$ with $\underline{S_t} \subseteq S \subseteq \overline{S_t}$ the following implication holds:
            \begin{equation}
                \E[h(\varepsilon_0)] \geq \Prob_0(S)\Rightarrow\E[h(\varepsilon_1)] \geq \Prob_1(S).
            \end{equation}
    	   \item[(i)] For any $S\subseteq\gZ$ with $\overline{S_t}^c \subseteq S \subseteq \underline{S_t}^c$ the following implication holds:
    	   \begin{equation}
    	       \E[h(\varepsilon_0)] \leq \Prob_0(S)\Rightarrow\E[h(\varepsilon_1)] \leq \Prob_1(S).
    	   \end{equation}
        \end{enumerate}
    \end{lem}
    \begin{proof}
    	We first prove (i). For that purpose, consider
    	\begin{align}
    	    \E[f(\varepsilon_1)] - \Prob_1(S) &= \int hf_1\,d\mu - \int_{S}f_1\,d\mu\\
    	    &= \int_{S^c}hf_1\,d\mu - \left(\int_{S}(1 - h)f_1\,d\mu\right)\\
    	    &=\int_{S^c}h\Lambda f_0\,d\mu - \left(\int_{S}(1 - h)\Lambda f_0\,d\mu\right)\\
    	    &\geq t\cdot\int_{S^c}h f_0\,d\mu - t\cdot \left(\int_{S}(1 - h) f_0\,d\mu\right)\label{eq:proof_np_inequality_1}\\
    	    &= t\cdot\left(\int h f_0\,d\mu - \int_{S}f_0\,d\mu\right)\\
    	    &= t\cdot\left(\E[f(\varepsilon_0)] - \Prob_0(S)\right) \geq 0 \label{proof:proof_np_inequality_2}.
    	\end{align}
    	The inequality in~(\ref{eq:proof_np_inequality_1}) follows from the fact that whenever $z\in S^c$, then $f_1(z) \geq t\cdot f_0(z)$ and if $z\in S$, then $f_1(z) \leq t\cdot f_0(z)$ since $S$ is a lower level set. Finally, the inequality in~(\ref{proof:proof_np_inequality_2}) follows from the assumption. The proof of $(ii)$ is analogous and omitted here.
    \end{proof}

    \subsection{Proof of Theorem~\ref{thm:main}}
    \begin{customthm}{\ref{thm:main}}[restated]
        Let $\varepsilon_0\sim\Prob_0$ and $\varepsilon_1\sim\Prob_1$ be $\gZ$-valued random variables with probability density functions $f_0$ and $f_1$ with respect to a measure $\mu$ on $\gZ$ and let $\phi\colon\gX\times\gZ\to\gX$ be a semantic transformation.
        Suppose that $y_A = g(x;\,\varepsilon_0)$ and let $p_A,\,p_B\in[0,\,1]$ be bounds to the class probabilities, i.e.,
        \begin{equation}
            q(y_A\lvert\,x,\,\epsilon_0) \geq p_A > p_B \geq \max_{y\neq y_A} q(y\lvert\,x,\,\epsilon_0).
        \end{equation}
        For $t\geq0$, let $\underline{S}_{t},\,\overline{S}_{t} \subseteq \gZ$ be the sets defined as $\underline{S}_{t} := \{f_1/f_0 < t\}$ and $\overline{S}_{t} := \{f_1/f_0 \leq t\}$ and define the function $\xi\colon[0,\,1] \to [0,\,1]$ by
        \begin{equation}
            \begin{gathered}
                \xi(p) := \sup\{\Prob_1(S)\colon\,\underline{S}_{\tau_p}\subseteq S \subseteq\overline{S}_{\tau_p}\}\\
                \mathrm{where}\hspace{1em}
                \tau_p := \inf\{t\geq 0\colon\,\Prob_0(\overline{S}_{t}) \geq p\}.
            \end{gathered}
        \end{equation}
        If the condition
        \begin{equation}
            \xi(p_A) + \xi(1-p_B) > 1
        \end{equation}
        is satisfied, then it is guaranteed that $g(x;\,\varepsilon_1) = g(x;\,\varepsilon_0)$.

    \end{customthm}
    \begin{proof}
        For ease of notation, let $\zeta$ be the function defined by
        \begin{equation}
            t\mapsto \zeta(t) := \Prob_0(\overline{S}_t)
        \end{equation}
        and notice that $\tau_p = \zeta^{-1}(p)$ where $\zeta^{-1}$ denotes the generalized inverse of $\zeta$.
        Furthermore, let $\tau_A:=\tau_{p_A}$, $\tau_B:=\tau_{1-p_B}$, $\underline{S}_A:=\underline{S}(\tau_A)$, $\underline{S}_B:=\underline{S}(\tau_B)$, $\overline{S}_A:=\overline{S}(\tau_A)$ and $\overline{S}_B:=\overline{S}(\tau_B)$.
        We first show that $q(y_A\lvert\,x,\,\varepsilon_1)$ is lower bounded by $\xi(\tau_A)$. For that purpose, note that by Lemma~\ref{lem:sandwich} we have that $\zeta(\tau_A) = \Prob_0(\overline{S}_A) \geq p_A \geq \Prob_0(\underline{S}_A)$. Thus, the collection of sets
        \begin{equation}
            \gS_A := \{S\subseteq\gZ\colon \underline{S}_A \subseteq S \subseteq \overline{S}_A,\,\Prob_0(S) \leq p_A\}
        \end{equation}
        is not empty. Pick some $A \in\gS_A$ arbitrary and note that, since by assumption $g(\cdot;\,\varepsilon_0)$ is $(p_A,\,p_B)$-confident at $x$ it holds that
        \begin{equation}
            \E(p(y_A\lvert\,\phi(x,\,\varepsilon_0))) = q(y_A\lvert\,x;\,\varepsilon_0) \geq p_A \geq \Prob_0(A).
        \end{equation}
        Since $\underline{S}_A \subseteq A \subseteq \overline{S}_A$ we can apply part $(i)$ of Lemma~\ref{lem:np_extended} and obtain the lower bound
        \begin{equation}
            q(y_A\lvert\,x,\,\varepsilon_1) = \E(p(y_A\lvert\,\phi(x,\,\varepsilon_1))) \geq \Prob_1(A).
        \end{equation}
        Since $A\in\gS_A$ was arbitrary, we take the $\sup$ over all $A\in\gS_A$ and obtain
        \begin{equation}
            \label{eq:main_proof_lower_bound}
            q(y_A\lvert\,x;\,\varepsilon_1) \geq \sup_{A\in\gS_A}\Prob_1(A) = \xi(p_A)
        \end{equation}
        We now show that for any $y\neq y_A$ the prediction $q(y\lvert\,x;\,\varepsilon_1)$ is upper bounded by $1 - \xi(1-p_B)$. For that purpose, note that by Lemma~\ref{lem:sandwich} we have that $\zeta(\tau_B) = \Prob_0(\overline{S}_A) \geq 1 - p_B \geq \Prob_0(\underline{S}_B)$. Thus, the collection of sets
        \begin{equation}
            \gS_B := \{S\subseteq\gZ\colon \underline{S}_B \subseteq S \subseteq \overline{S}_B,\,\Prob_0(S) \leq 1 - p_B\}
        \end{equation}
        is not empty. Pick some $B \in\gS_A$ arbitrary and note that, since by assumption $g(\cdot;\,\varepsilon_0)$ is $(p_A,\,p_B)$-confident at $x$ it holds that
        \begin{equation}
            \begin{aligned}
                \E(p(y\lvert\,\phi(x,\,\varepsilon_0))) &= q(y\lvert\,x;\,\varepsilon_0) \leq p_B \\
                &= 1 - (1-p_B) \leq 1 - \Prob_0(B).
            \end{aligned}
        \end{equation}
        Since $\underline{S}_B^c \subseteq B^c \subseteq \overline{S}_B^c$ we can apply part $(ii)$ of Lemma~\ref{lem:np_extended} and obtain the upper bound
        \begin{equation}
            q(y\lvert\,x;\varepsilon_1) = \E(p(y\lvert\,\phi(x,\,\varepsilon_1))) \leq 1 - \Prob_1(B).
        \end{equation}
        Since $B\in\gS_B$ was arbitrary, we take the $\inf$ over all $B\in\gS_B$ and obtain
        \begin{equation}
            \label{eq:main_proof_upper_bound}
            q(y\lvert\,x;\varepsilon_1) \leq \inf_{B\in\gS_B}(1 - \Prob_1(B)) = 1 - \xi(1-p_B).
        \end{equation}
        Combining together~(\ref{eq:main_proof_upper_bound}) and~(\ref{eq:main_proof_lower_bound}), we find that, whenever
        \begin{equation}
            \xi(p_A) + \xi(1-p_B) > 1
        \end{equation}
        it is guaranteed that
        \begin{equation}
            q(y_A\lvert\,x;\,\varepsilon_1) > \max_{y\neq y_A}q(y_\lvert\,x;\,\varepsilon_1)
        \end{equation}
        what concludes the proof.
    \end{proof}

\section{Proofs for Certification with Different Smoothing Distributions}
    \label{adx-sec:distribution_proofs}

    \ifnum\ArXiv=0
    Hereby we only list selected main results and leave the proofs to the arXiv version~\cite{arxiv} due to space limit.

    \begin{cor}[Non-Isotropic Gaussian Smoothing]
        \label{cor:gaussian_noise}
        Suppose $\gZ=\R^m$, $\Sigma:=\mathrm{diag}(\sigma_1^2,\,\ldots,\sigma_m^2)$, $\varepsilon_0\sim\gN(0,\,\Sigma)$ and $\varepsilon_1:=\alpha+\varepsilon_0$ for some $\alpha\in\R^m$.
        Suppose that the $\varepsilon_0$-smoothed classifier $g$ is $(p_A,\,p_B)$-confident at $x\in\gX$ for some $y_A\in\gY$. Then, it holds that $q(y_A\lvert\,x;\,\varepsilon_1) > \max_{y\neq y_A}q(y\lvert\,x;\,\varepsilon_1)$ if $\alpha$ satisfies
        \begin{equation}
            \label{eq:gaussian_condition}
            \sqrt{\sum_{i=1}^{m}\left(\dfrac{\alpha_i}{\sigma_i}\right)^2} < \dfrac{1}{2}\left(\Phi^{-1}(p_A) - \Phi^{-1}(p_B) \right).
        \end{equation}
    \end{cor}

    \begin{cor}[Exponential Smoothing]
        \label{cor_apx:exponential}
        Suppose $\gZ=\R^m_{\geq 0}$, fix some $\lambda>0$ and let $\varepsilon_{0,i}\overset{\mathrm{iid}}{\sim}\mathrm{Exp}(1/\lambda)$, $\varepsilon_0:=(\varepsilon_{0,1},\,\ldots,\,\varepsilon_{0,m})^T$ and $\varepsilon_1:=\alpha + \varepsilon_0$ for some $\alpha\in\R^m_{\geq0}$.
        Suppose that the $\varepsilon_0$-smoothed classifier $g$ is $(p_A,\,p_B)$-confident at $x\in\gX$ for some $y_A\in\gY$. Then, it holds that $q(y_A\lvert\,x;\,\varepsilon_1) > \max_{y\neq y_A}q(y\lvert\,x;\,\varepsilon_1)$ if $\alpha$ satisfies
        \begin{equation}
            \label{eq:exponential_condition}
            \|\alpha\|_1 < -\dfrac{\log(1 - p_A + p_B)}{\lambda}.
        \end{equation}
    \end{cor}

    \fi

    \ifnum\ArXiv=1
    Here, we instantiate Theorem~\ref{thm:main} with different smoothing distributions and solve the robustness condition~\eqref{eq:robustness_condition} for the case where the distribution of $\varepsilon_1$ results from shifting the distribution of $\varepsilon_0$, i.e., $\varepsilon_1 = \alpha + \varepsilon_0$.
    For ease of notation, let $\zeta\colon\R_{\geq0}\to[0,\,1]$ be the function defined by
    \begin{equation}
        t\mapsto\zeta(t):=\Prob_0\left(\overline{S}_{t}\right)
    \end{equation}
    where $\Prob_0$ is the distribution of $\varepsilon_0$ and $\overline{S}_{t}$ is a lower level set; recall that the definitions of lower level sets is
    \begin{equation}
        \begin{gathered}
            \underline{S}_{t} := \left\{z\in\gZ\colon\,\Lambda(z) < t \right\},
            \hspace{1em}\overline{S}_{t} := \left\{z\in\gZ\colon\,\Lambda(z) \leq t \right\},\\
            \text{where}\hspace{1em}\Lambda(z):=\frac{f_1(z)}{f_0(z)}.
        \end{gathered}
    \end{equation}
    Note that the generalized inverse of $\zeta$ corresponds to $\tau_p$, i.e.,
    \begin{equation}
        \zeta^{-1}(p) = \inf\{t\geq 0\lvert\,\zeta(t) \geq p\} = \tau_p
    \end{equation}
    and the function $\xi$ is correspondingly given by
    \begin{equation}
        \xi(p) = \sup\{\Prob_1(S)\lvert\,\underline{S}(\zeta^{-1}(p)) \subseteq S \subseteq \overline{S}(\zeta^{-1}(p))\}
    \end{equation}

    \subsection{Gaussian Smoothing}
    \begin{cor}
        \label{cor:gaussian_noise}
        Suppose $\gZ=\R^m$, $\Sigma:=\mathrm{diag}(\sigma_1^2,\,\ldots,\sigma_m^2)$, $\varepsilon_0\sim\gN(0,\,\Sigma)$ and $\varepsilon_1:=\alpha+\varepsilon_0$ for some $\alpha\in\R^m$.
        Suppose that the $\varepsilon_0$-smoothed classifier $g$ is $(p_A,\,p_B)$-confident at $x\in\gX$ for some $y_A\in\gY$. Then, it holds that $q(y_A\lvert\,x;\,\varepsilon_1) > \max_{y\neq y_A}q(y\lvert\,x;\,\varepsilon_1)$ if $\alpha$ satisfies
        \begin{equation}
            \label{eq:gaussian_condition}
            \sqrt{\sum_{i=1}^{m}\left(\dfrac{\alpha_i}{\sigma_i}\right)^2} < \dfrac{1}{2}\left(\Phi^{-1}(p_A) - \Phi^{-1}(p_B) \right).
        \end{equation}
    \end{cor}
    \begin{proof}
        By Theorem~\ref{thm:main} we know that if $\varepsilon_1$ satisfies
        \begin{equation}
            \xi(p_A) + \xi(1-p_B) > 1\label{eq:gaussian_condition_proof},
        \end{equation}
        then it is guaranteed that
        \begin{equation}
            q(y_A\lvert\,x;\,\varepsilon_1) > \max_{y\neq y_A}q(y\lvert\,x;\,\varepsilon_1).
        \end{equation}
        The proof is thus complete if we show that~\eqref{eq:gaussian_condition_proof} reduces to~\eqref{eq:gaussian_condition}.
        For that purpose, denote by $f_0$ and $f_1$ density functions of $\varepsilon_0$ and $\varepsilon_1$,  respectively. Let $A:=\Sigma^{-1}$ and note that the bilinear form $(z_1,\,z_2)\mapsto z_1^T A z_2 =: \langle z_1,\,z_2 \rangle_A$ defines an inner product on $\R^m$. Let $z\in\R^m$ and consider
        \begin{align}
        	\Lambda(z) = \frac{f_1(z)}{f_0(z)} &= \frac{\exp\left(-\frac{1}{2}\langle z - \alpha,\,z-\alpha\rangle_A\right)}{\exp\left(-\frac{1}{2}\langle z,\,z\rangle_A\right)}\\
        	&= \exp\left(\langle z,\,\alpha\rangle_A - \frac{1}{2}\langle \alpha,\,\alpha\rangle_A\right)
        \end{align}
        and thus
        \begin{align}
            \Lambda(z) \leq t \iff \langle z,\,\alpha\rangle_A \leq \log(t) + \frac{1}{2}\langle \alpha,\,\alpha\rangle.
        \end{align}
        Let $Z\sim\gN(0,\,1)$ and notice that
        \begin{equation}
            \frac{\langle \varepsilon_0,\,\alpha\rangle_A}{\sqrt{\langle \alpha,\,\alpha\rangle_A}} \overset{d}{=}Z \overset{d}{=} \frac{\langle \varepsilon_1,\,\alpha\rangle_A - \langle \alpha,\,\alpha\rangle_A}{\sqrt{\langle \alpha,\,\alpha\rangle_A}}.
        \end{equation}
        Let $\partial_t := \overline{S}_t\setminus\underline{S}_t = \{z\colon\,\Lambda(z) = t\}$ and notice that $\Prob_0\left(\partial_t\right) = \Prob_1\left(\partial_t\right) = 0$ and $\Prob_0(\underline{S}_t) = \Prob_0(\overline{S}_t)$. Similarly, it holds that $\Prob_1(\underline{S}_t) = \Prob_1(\overline{S}_t)$. The function $p\mapsto \xi(p)$ is thus given by
        \begin{align}
            \xi(p) = \Prob_1\left(\overline{S}_{\zeta^{-1}(p)}\right).
        \end{align}
        We compute $\zeta$ as
        \begin{align}
        	\zeta(t) &= \Prob\left(\Lambda(\varepsilon_0) \leq t\right) = \Prob\left(\langle \varepsilon_0,\,\alpha\rangle_A \leq \log(t) + \frac{1}{2}\langle \alpha,\,\alpha\rangle_A\right)\\
        	&= \Phi\left(\dfrac{\log(t) + \frac{1}{2}\langle \alpha,\,\alpha\rangle_A}{\sqrt{\langle \alpha,\,\alpha\rangle_A}}\right)
        \end{align}
        and for $p\in[0,\,1]$ its inverse is
        \begin{align}
        	\zeta^{-1}(p) &= \exp\left(\Phi^{-1}(p)\sqrt{\langle \alpha,\,\alpha\rangle_A} - \dfrac{1}{2}\langle \alpha,\,\alpha\rangle_A\right).
        \end{align}
        Thus
        \begin{align}
        	&\Prob\left(\Lambda(\varepsilon_1) \leq \zeta^{-1}(p)\right)\\
        	&\hspace{2em}= \Prob\left(\frac{\langle \varepsilon_1,\,\alpha\rangle_A - \langle \alpha,\,\alpha\rangle_A}{\sqrt{\langle \alpha,\,\alpha\rangle_A}} \leq \frac{\log(\zeta^{-1}(p)) - \frac{1}{2}\langle \alpha,\,\alpha\rangle_A }{\sqrt{\langle \alpha,\,\alpha\rangle_A}}\right)\\
        	&\hspace{2em}= \Phi\left(\dfrac{\left(\Phi^{-1}(p)\sqrt{\langle \alpha,\,\alpha\rangle_A} - \frac{1}{2}\langle \alpha,\,\alpha\rangle_A\right) - \frac{1}{2}\langle \alpha,\,\alpha\rangle_A}{\sqrt{\langle \alpha,\,\alpha\rangle_A}}\right)\\
        	&\hspace{2em}= \Phi\left(\Phi^{-1}(p) - \sqrt{\langle \alpha,\,\alpha\rangle_A}\right).
        \end{align}
        Finally, algebra shows that
        \begin{equation}
            \Phi\left(\Phi^{-1}(p_A) - \sqrt{\langle \alpha,\,\alpha\rangle_A}\right) + \Phi\left(\Phi^{-1}(1-p_B) - \sqrt{\langle \alpha,\,\alpha\rangle_A}\right) > 1
        \end{equation}
        is equivalent to
        \begin{gather}
            \begin{aligned}
        	\sqrt{\sum_{i=1}^{m}\left(\dfrac{\alpha_i}{\sigma_i}\right)^2} < \dfrac{1}{2}\left(\Phi^{-1}(p_A) - \Phi^{-1}(p_B)\right)
        	\end{aligned}
        \end{gather}
        what concludes the proof.
    \end{proof}

    \subsection{Exponential Smoothing}
    \begin{cor}
        \label{cor_apx:exponential}
        Suppose $\gZ=\R^m_{\geq 0}$, fix some $\lambda>0$ and let $\varepsilon_{0,i}\overset{\mathrm{iid}}{\sim}\mathrm{Exp}(1/\lambda)$, $\varepsilon_0:=(\varepsilon_{0,1},\,\ldots,\,\varepsilon_{0,m})^T$ and $\varepsilon_1:=\alpha + \varepsilon_0$ for some $\alpha\in\R^m_{\geq0}$.
        Suppose that the $\varepsilon_0$-smoothed classifier $g$ is $(p_A,\,p_B)$-confident at $x\in\gX$ for some $y_A\in\gY$. Then, it holds that $q(y_A\lvert\,x;\,\varepsilon_1) > \max_{y\neq y_A}q(y\lvert\,x;\,\varepsilon_1)$ if $\alpha$ satisfies
        \begin{equation}
            \label{eq:exponential_condition}
            \|\alpha\|_1 < -\dfrac{\log(1 - p_A + p_B)}{\lambda}.
        \end{equation}
    \end{cor}
    \begin{proof}
        By Theorem~\ref{thm:main} we know that if $\varepsilon_1$ satisfies
        \begin{equation}
            \xi(p_A) + \xi(1-p_B) > 1\label{eq:exponential_condition_proof},
        \end{equation}
        then it is guaranteed that
        \begin{equation}
            q(y_A\lvert\,x;\,\varepsilon_1) > \max_{y\neq y_A}q(y\lvert\,x;\,\varepsilon_1).
        \end{equation}
        The proof is thus complete if we show that~(\ref{eq:exponential_condition_proof}) reduces to~(\ref{eq:exponential_condition}).
        For that purpose, denote by $f_0$ and $f_1$ density functions of $\varepsilon_0$ and $\varepsilon_1$,  respectively,
        and note that
        \begin{align}
            f_1(z) &=
                \begin{cases}
                    \lambda\cdot\exp(-\lambda\|z - \alpha\|_1),\hspace{0.5em} &\min_i(z_i - \alpha_i) \geq 0,\\
                    0, &\mathrm{otherwise},
                \end{cases}\\
            f_0(z) &=
                \begin{cases}
                    \lambda\cdot\exp(-\lambda\|z\|_1),\hspace{0.5em} &\min_i(z_i) \geq 0,\\
                    0, &\mathrm{otherwise},
                \end{cases}
        \end{align}
        and $\forall i$, $z_i - \alpha_i \leq z_i$ and hence $f_0(z) = 0 \Rightarrow f_1(z) = 0$.
        Thus
        \begin{equation}
            \Lambda(z) = \frac{f_1(z)}{f_0(z)} =
                \begin{cases}
                    \exp\left(\lambda\cdot\|\alpha\|_1\right) \hspace{0.5em} &\min_i(z_i-\alpha_i) \geq 0,\\
                    0, &\mathrm{otherwise.}
                \end{cases}
        \end{equation}
        Let $S_0:=\{z\in\R^m_{\geq0}\colon\,\min_i(z_i-\alpha_i) < 0\}$ and note that due to independence
        \begin{align}
            \Prob_0\left(S_0 \right) &= \Prob\left(\bigcup_{i=1}^{m}\{\varepsilon_{0,i} < \alpha_i\}\right)\\
            &= 1 - \Prob\left(\bigcap_{i=1}^{m}\{\varepsilon_{0,i} \geq \alpha_i\}\right)
            = 1 - \prod_{i=1}^{m}\Prob\left(\varepsilon_{0,i} \geq \alpha_i\right)\\
            &= 1 - \prod_{i=1}^{m}\left(1 - \left(1 - \exp\left(-\lambda\,\alpha_i\right)\right)\right)\\
            &= 1 - \exp\left(-\lambda\|\alpha\|_1\right). \label{eq:exponential_proof_s0}
        \end{align}
        Let $t_\alpha:=\exp(\lambda\|\alpha\|_1)$ and compute $\zeta$ as
        \begin{align}
            \zeta(t) &= \Prob\left(\Lambda(\varepsilon_0) \leq t\right)\\
            &= \Prob\left(\Id\{\min_i(\varepsilon_{0,i} - \alpha_i) \geq 0\} \leq t\cdot\exp\left(-\lambda\|\alpha\|_1\right)\right)\\
            &=\begin{cases}
                    1 - \exp\left(-\lambda\|\alpha\|_1\right) \hspace{0.5em} &t < t_\alpha,\\
                    1 \hspace{0.5em} &t \geq t_\alpha.
                \end{cases}
        \end{align}
        Recall that $\zeta^{-1}(p) := \inf\{t\geq 0\colon\,\zeta(t)\geq p\}$ for $p\in[0,\,1]$ and hence
        \begin{align}
            \zeta^{-1}(p) =
                \begin{cases}
                    0 \hspace{1em} &p\leq 1 - \exp(-\lambda\|\alpha\|_1),\\
                    \exp(\lambda\|\alpha\|_1) &p > 1 - \exp(-\lambda\|\alpha\|_1).
                \end{cases}
        \end{align}
        In order to evaluate $\xi$ we compute the lower and strict lower level sets at $t=\zeta^{-1}(p)$.
        Recall that $\underline{S}_{t} = \{z\in\R^m_{\geq0}\colon\, \Lambda(z) < t\}$ and $\overline{S}_{t} = \{z\in\R^m_{\geq0}\colon\, \Lambda(z) \leq t\}$ and consider
        \begin{align}
            \begin{split}
            \underline{S}_{\zeta^{-1}(p)} &=
                \left(S_0^c \cap \left\{z\in\R^m_{\geq0}\colon\, \exp(\lambda\|\alpha\|_1) < \zeta^{-1}(p)\right\}\right) \\
                &\hspace{4em}\cup \left(S_0 \cap \left\{z\in\R^m_{\geq 0}\,|\, 0 < \zeta^{-1}(p)\right\}\right)
            \end{split}\\
            &=\begin{cases}
                \varnothing \hspace{1em} &p \leq 1 - \exp(-\lambda\|\alpha\|_1),\\
                S_0 &p > 1 - \exp(-\lambda\|\alpha\|_1)
            \end{cases}
        \end{align}
        and
        \begin{align}
            \begin{split}
                \overline{S}_{\zeta^{-1}(p)} &= \left(S_0^c \cap \left\{z\in\R^m_{\geq 0}\colon\,\exp(\lambda\|\alpha\|_1) \leq \zeta^{-1}(p)\right\}\right)\\ &\hspace{4em}\dot\cup \left(S_0 \cap \left\{z\in\R^m_{\geq 0}\colon\, 0 \leq \zeta^{-1}(p)\right\}\right)
            \end{split}\\
            &=\begin{cases}
                S_0 \hspace{1em} &p \leq 1 - \exp(-\lambda\|\alpha\|_1),\\
                \R^m_+ &p > 1 - \exp(-\lambda\|\alpha\|_1).
            \end{cases}
        \end{align}
        Suppose that $p\leq 1 - \exp(-\lambda\|\alpha\|_1)$. Then we have that $\underline{S}_{\zeta^{-1}(p)} = \varnothing$ and  $\overline{S}_{\zeta^{-1}(p)} = S_0$ and hence
        \begin{equation}
            \begin{aligned}
                p&\leq 1 - \exp(-\lambda\|\alpha\|_1)\\
                &\Rightarrow \xi(p) = \sup\{\Prob_1(S)\colon\,S\subseteq S_0 \land \Prob_0(S)\leq p\} = 0.
            \end{aligned}
        \end{equation}
        Condition~(\ref{eq:exponential_condition_proof}) can thus be satisfied only if $p_A > 1 - \exp(-\lambda\|\alpha\|_1)$ and $1 - p_B > 1 - \exp(-\lambda\|\alpha\|_1)$.
        In this case  $\underline{S}_{\zeta^{-1}(p)} = S_0$ and  $\overline{S}_{\zeta^{-1}(p)} = \R^m_{\geq_0}$.
        For $p\in[0,\,1]$ let $\gS_p=\{S\subseteq\R^m_{\geq0}\colon\,S_0\subseteq S \subseteq \R^m_{\geq0},\,\Prob_0(S)\leq p\}$.
        Then
        \begin{equation}
            p > 1 - \exp(-\lambda\|\alpha\|_1) \Rightarrow \xi(p) = \sup_{S\in\gS_p}\Prob_1(S).
        \end{equation}
        We can write any $S\in\gS_p$ as the disjoint union $S=S_0\,\dot\cup\,T$ for some $T\subseteq\R^m_{\geq0}$ such that $\Prob_0(S_0\,\dot\cup\,T) \leq p$.
        Note that $\Prob_1\left(S_0\right) = 0$ and since $S_0\cap T = \varnothing$ any $z\in T$ satisfies $0 \leq\min_i\,(z_i - \alpha_i) \leq \min_i\,z_i$ and hence $\Lambda(z) = \exp(\lambda\|\alpha\|_1)$.
        Thus
        \begin{align}
            \Prob_1\left(S\right) &= \Prob_1\left(T\right) = \int_{T}f_1(z)\,dz\\
            & = \int_{T}\exp(\lambda\|\alpha\|_1)f_0(z)\,dz
            = \exp(\lambda\|\alpha\|_1)\cdot\Prob_0\left(T\right).
        \end{align}
        Thus, The supremum of the left hand side over all $S\in\gS_p$ equals the supremum of the right hand side over all $T\in\{T'\subseteq S_0^c\colon\,\Prob_0(T') \leq 1 - \Prob_0(S_0)\}$
        \begin{align}
            \begin{split}
                &\sup_{S\in\gS_p}\,\Prob_1\left(S\right) =\\
                &\exp(\lambda\|\alpha\|_1)
                \cdot \sup\,\{\Prob_1(T')\colon\,T'\subseteq S_0^c,\,\Prob_0(T') \leq p - \Prob_0(S_0)\}
            \end{split}\\
            &= \exp(\lambda\|\alpha\|_1)\cdot(p - \Prob_0(S_0)).
        \end{align}
        Computing $\xi$ at $p_A$ thus yields
        \begin{align}
            \xi(p_A) &= \sup_{S\in\gS_{p_A}}\,\Prob_1\left(S\right) = \exp(\lambda\|\alpha\|_1)\cdot\left(p_A - \Prob_0\left(S_0\right)\right)\\
            &= \exp(\lambda\|\alpha\|_1)\cdot\left(p_A - \left(1 - \exp\left(-\lambda\|\alpha\|_1\right)\right)\right)\\
            &= \exp(\lambda\|\alpha\|_1)\cdot (p_A + \exp\left(-\lambda\|\alpha\|_1\right) - 1)
        \end{align}
        where the third equality follows from~(\ref{eq:exponential_proof_s0}).
        Similarly, computing $\xi$ at $1-p_B$ yields
        \begin{align}
            \xi(1 - p_B) &= \sup_{S\in\gS_{1-p_B}}\,\Prob_1\left(S\right)\\
            &= \exp(\lambda\|\alpha\|_1)\cdot\left(1 - p_B - \Prob_0\left(S_0\right)\right)\\
            &= \exp(\lambda\|\alpha\|_1)\cdot\left(1 - p_B - \left(1 - \exp\left(-\lambda\|\alpha\|_1\right)\right)\right)\\
            &= \exp(\lambda\|\alpha\|_1)\cdot\left(-p_B + \exp\left(-\lambda\|\alpha\|_1\right)\right).
        \end{align}
        Finally, condition~(\ref{eq:exponential_condition_proof}) is satisfied whenever $\alpha$ satisfies
        \begin{gather}
            \begin{split}
                &\exp(\lambda\|\alpha\|_1)\cdot\left(p_A + \exp\left(-\lambda\|\alpha\|_1\right) - 1\right)\\
                &\hspace{4em}+ \exp(\lambda\|\alpha\|_1)\cdot\left(-p_B + \exp\left(-\lambda\|\alpha\|_1\right)\right) > 1
            \end{split}\\
            \iff\nonumber\\
            \begin{split}
                &\exp(-\lambda\|\alpha\|_1) + p_B - \exp(-\lambda\|\alpha\|_1) \\
                &\hspace{6em}< p_A + \exp\left(-\lambda\|\alpha\|_1\right) - 1
            \end{split}\\
            \iff\nonumber\\
            1 - p_A + p_B < \exp\left(-\lambda\|\alpha\|_1\right)\\
            \iff\nonumber\\
            \|\alpha\|_1 < -\dfrac{\log(1 - p_A + p_B)}{\lambda}
        \end{gather}
        what completes the proof.
    \end{proof}

    \subsection{Uniform Smoothing}
    \begin{cor}
        \label{cor_apx:uniform}
        Suppose $\gZ=\R^m$, and $\varepsilon_0\sim\gU([a,\, b]^m)$ for some $a < b$. Set $\varepsilon_1:=\alpha + \varepsilon_0$ for $\alpha\in\R^m$.
        Suppose that the $\varepsilon_0$-smoothed classifier $g$ is $(p_A,\,p_B)$-confident at $x\in\gX$ for some $y_A\in\gY$. Then, it holds that $q(y_A\lvert\,x;\,\varepsilon_1) > \max_{y\neq y_A}q(y\lvert\,x;\,\varepsilon_1)$ if $\alpha$ satisfies
        \begin{equation}
            \label{eq:uniform_condition}
            1 - \left(\dfrac{p_A - p_B}{2}\right) < \prod_{i=1}^m\left(1 - \frac{|\alpha_i|}{b-a}\right)_+
        \end{equation}
        where $(x)_+:=\max\{x,\,0\}$.
    \end{cor}
    \begin{proof}
        By Theorem~\ref{thm:main} we know that if $\varepsilon_1$ satisfies
        \begin{equation}
            \xi(p_A) + \xi(1-p_B) > 1\label{eq:uniform_condition_proof},
        \end{equation}
        then it is guaranteed that
        \begin{equation}
            q(y_A\lvert\,x;\,\varepsilon_1) > \max_{y\neq y_A}q(y\lvert\,x;\,\varepsilon_1).
        \end{equation}
        The proof is thus complete if we show that~(\ref{eq:uniform_condition_proof}) reduces to~(\ref{eq:uniform_condition}).
        For that purpose, denote by $f_0$ and $f_1$ density functions of $\varepsilon_0$ and $\varepsilon_1$,  respectively,
        and let $I_0=[a,\,b]^m$ and $I_1:=\prod_{i=1}^m[a+\alpha_i,\,b+\alpha_i]$ bet the support of $\varepsilon_0$ and $\varepsilon_1$. Consider
        \begin{align}
            f_0(z) &=
                \begin{cases}
                (b - a)^{-m} \hspace{1em} &z\in I_0,\\
                0 &\mathrm{otherwise}
                \end{cases}\\
            f_1(z) &=
                \begin{cases}
                (b-a)^{-m} \hspace{1em} &z\in I_1,\\
                0 &\mathrm{otherwise}.
                \end{cases}
        \end{align}
        Let $S_0:=I_0\setminus I_1$. Then, for any $z\in I_0\cup I_1$
        \begin{align}
            \Lambda(z) = \frac{f_1(z)}{f_0(z)} =
                \begin{cases}
                    0 \hspace{1em} &z\in S_0,\\
                    1 & z \in I_0\cap I_1,\\
                    \infty &z \in I_1 \setminus I_0.
                \end{cases}
        \end{align}
        Note that
        \begin{align}
            \Prob_0\left(S_0\right) &= 1 - \Prob_0\left(I_1\right)\\
            & = 1 - \prod_{i=1}^m\Prob\left(a+\alpha_i \leq \varepsilon_{0,i} \leq b+\alpha_i\right)\\
            & = 1 -  \prod_{i=1}^m\left(1 - \dfrac{|\alpha_i|}{b-a}\right)_+
        \end{align}
        where $(x)_+ = \max\{x,\,0\}$. We then compute $\zeta$ for $t\geq 0$
        \begin{align}
            \zeta(t) &= \Prob\left(\Lambda(\varepsilon_0) \leq t\right) =
                \begin{cases}
                    \Prob_0\left(S_0\right) \hspace{0.5em}&t<1,\\
                    \Prob_0\left(I_0\right) & t\geq 1.
                \end{cases}\\
                &=
                \begin{cases}
                    1 -  \prod_{i=1}^m\left(1 - \frac{|\alpha_i|}{b-a}\right)_+ \hspace{0.5em}&t<1,\\
                    1 & t\geq 1.
                \end{cases}
        \end{align}
        Recall that $\zeta^{-1}(p) := \inf\{t\geq 0\colon\,\zeta(t)\geq p\}$ for $p\in[0,\,1]$ and hence
        \begin{align}
            \zeta^{1}(p) =
                \begin{cases}
                    0 \hspace{0.5em} &p \leq 1 -  \prod_{i=1}^m\left(1 - \frac{|\alpha_i|}{b-a}\right)_+,\\
                    1 \hspace{0.5em} &p > 1 -  \prod_{i=1}^m\left(1 - \frac{|\alpha_i|}{b-a}\right)_+.
                \end{cases}
        \end{align}
        In order to evaluate $\xi$, we compute the lower and strict lower level sets at $t=\zeta^{-1}(p)$.
        Recall that $\underline{S}_{t} = \{z\in\R^m_{\geq0}\colon\, \Lambda(z)< t\}$ and $\overline{S}_{t} = \{z\in\R^m_{\geq 0}\colon\, \Lambda(z) \leq t\}$ and consider
        \begin{align}
            \underline{S}_{\zeta^{-1}(p)} =
                \begin{cases}
                    \varnothing \hspace{0.5em} & p \leq 1 -  \prod_{i=1}^m\left(1 - \frac{|\alpha_i|}{b-a}\right)_+,\\
                    S_0 \hspace{0.5em} & p > 1 -  \prod_{i=1}^m\left(1 - \frac{|\alpha_i|}{b-a}\right)_+
                \end{cases}
        \end{align}
        and
        \begin{align}
            \overline{S}_{\zeta^{-1}(p)} =
                \begin{cases}
                    S_0 \hspace{0.5em} & p \leq 1 -  \prod_{i=1}^m\left(1 - \frac{|\alpha_i|}{b-a}\right)_+,\\
                    I_0 \hspace{0.5em} & p > 1 -  \prod_{i=1}^m\left(1 - \frac{|\alpha_i|}{b-a}\right)_+
                \end{cases}
        \end{align}
        Suppose $p \leq 1 -  \prod_{i=1}^m\left(1 - \frac{|\alpha_i|}{b-a}\right)_+$. Then $\underline{S}_{\zeta^{-1}(p)}=\varnothing$ and $\overline{S}_{\zeta^{-1}(p)}=S_0$ and hence
        \begin{equation}
            \begin{split}
                p &\leq 1 -  \prod_{i=1}^m\left(1 - \frac{|\alpha_i|}{b-a}\right)_+\\ &\hspace{2em}\Rightarrow \xi(p) = \sup\{\Prob_1(S)\colon\, S \subseteq S_0,\,\Prob_0(S)\leq p\} = 0.
            \end{split}
        \end{equation}
        Condition~(\ref{eq:uniform_condition_proof}) can thus be satisfied only if $p_A > 1 -  \prod_{i=1}^m\left(1 - \frac{|\alpha_i|}{b-a}\right)_+$ and $1 - p_B > 1 -  \prod_{i=1}^m\left(1 - \frac{|\alpha_i|}{b-a}\right)_+$.
        In this case  $\underline{S}_{\zeta^{-1}(p)} = S_0$ and  $\overline{S}_{\zeta^{-1}(p)} = I_0$.
        For $p\in[0,\,1]$ let $\gS_p=\{S\subseteq\R^m\colon\,S_0\subseteq S \subseteq I_0,\,\Prob_0(S)\leq p\}$.
        Then
        \begin{align}
            p > 1 -  \prod_{i=1}^m\left(1 - \frac{|\alpha_i|}{b-a}\right)_+\Rightarrow \xi(p) = \sup_{S\in\gS_p}\Prob_1(S).
        \end{align}
        We can write any $S\in\gS_p$ as the disjoint union $S=S_0\,\dot\cup\,T$ for some $T\subseteq I_0 \cap I_1$ such that $\Prob_0(S_0\,\dot\cup\,T) \leq p$.
        Note that $\Prob_1\left(S_0\right) = 0$ and for any $z\in T$, we have $f_0(z) = f_1(z)$.
        Hence
        \begin{align}
            \Prob_1\left(S\right) &= \Prob_1(T) = \Prob_0(T)\\
            &\leq p - \Prob_0(S_0) = p - \left(1 - \prod_{i=1}^m\left(1 - \frac{|\alpha_i|}{b-a}\right)_+\right).
        \end{align}
        Thus, The supremum of the left hand side over all $S\in\gS_p$ equals the supremum of the right hand side over all $T\in\{T'\subseteq I_0\cap I_1\colon\,\Prob_0(T') \leq 1 - \Prob_0(S_0)\}$
        \begin{align}
            \begin{split}
                \sup_{S\in\gS_p}\,\Prob_1\left(S\right) &= \sup\,\{\Prob_1(T')\colon\,T'\subseteq I_0\cap I_1,\\
                &\hspace{6em}\Prob_0(T') \leq p - \Prob_0(S_0)\}
            \end{split}\\
            &= p - \left(1 - \prod_{i=1}^m\left(1 - \frac{|\alpha_i|}{b-a}\right)_+\right).
        \end{align}
        Hence, computing $\xi$ at $p_A$ and $1 - p_B$ yields
        \begin{align}
            \xi(p_A) &= p_A - \left(1 - \prod_{i=1}^m\left(1 - \frac{|\alpha_i|}{b-a}\right)_+\right),\\
            \xi(1 - p_B) &= 1 - p_B - \left(1 - \prod_{i=1}^m\left(1 - \frac{|\alpha_i|}{b-a}\right)_+\right).
        \end{align}
        Finally, condition~(\ref{eq:uniform_condition_proof}) is satisfied whenever $\alpha$ satisfies
        \begin{gather}
            \begin{split}
            &1 - \left(1 - p_B - \left(1 - \prod_{i=1}^m\left(1 - \frac{|\alpha_i|}{b-a}\right)_+\right)\right)\\
            &\hspace{8em} < p_A - \left(1 - \prod_{i=1}^m\left(1 - \frac{|\alpha_i|}{b-a}\right)_+\right)
            \end{split}\\
            \iff\nonumber\\
            p_B + 1 - \prod_{i=1}^m\left(1 - \frac{|\alpha_i|}{b-a}\right)_+<p_A - 1 + \prod_{i=1}^m\left(1 - \frac{|\alpha_i|}{b-a}\right)_+\\
            \iff 2 - p_A + p_B < 2\cdot \prod_{i=1}^m\left(1 - \frac{|\alpha_i|}{b-a}\right)_+\\
            \iff 1 - \left(\dfrac{p_A - p_B}{2}\right) < \prod_{i=1}^m\left(1 - \frac{|\alpha_i|}{b-a}\right)_+
        \end{gather}
        what concludes the proof.
    \end{proof}
    \subsection{Laplacian Smoothing}
    \begin{cor}
        \label{cor_apx:laplacian}
        Suppose $\gZ=\R$ and $\varepsilon_0\sim\gL(0,\,b)$ follows a Laplace distribution with mean $0$ and scale parameter $b>0$. Let $\varepsilon_1:=\alpha + \varepsilon_0$ for $\alpha\in\R$.
        Suppose that the $\varepsilon_0$-smoothed classifier $g$ is $(p_A,\,p_B)$-confident at $x\in\gX$ for some $y_A\in\gY$. Then, it holds that $q(y_A\lvert\,x;\,\varepsilon_1) > \max_{y\neq y_A}q(y\lvert\,x;\,\varepsilon_1)$ if $\alpha$ satisfies
        \begin{align}
            \label{eq:laplace_condition}
            |\alpha| <
                \begin{cases}
                    -b\cdot\log\left(4\,p_B\,(1-p_A)\right)\,&
                    \!\begin{aligned}
                       &(p_A =\frac{1}{2} \,\land \ \, p_B < \frac{1}{2})\\
                       &\hspace{0.3em}\lor\, (p_A> \frac{1}{2}\,\land\,p_B = \frac{1}{2}),
                    \end{aligned}\\
                    -b\cdot \log\left(1 - p_A + p_B\right) & p_A > \frac{1}{2} \,\land\, p_B < \frac{1}{2}.
                \end{cases}
        \end{align}
    \end{cor}
    \begin{proof}
        By Theorem~\ref{thm:main} we know that if $\varepsilon_1$ satisfies
        \begin{equation}
            \xi(p_A) + \xi(1-p_B) > 1\label{eq:laplace_condition_proof},
        \end{equation}
        then it is guaranteed that
        \begin{equation}
            q(y_A\lvert\,x;\,\varepsilon_1) > \max_{y\neq y_A}q(y\lvert\,x;\,\varepsilon_1).
        \end{equation}
        The proof is thus complete if we show that~(\ref{eq:laplace_condition_proof}) reduces to~(\ref{eq:laplace_condition}).
        For that purpose denote by $f_0$ and $f_1$ density functions of $\varepsilon_0$ and $\varepsilon_1$,  respectively,
        and consider
        \begin{align}
            f_0(z) = \frac{1}{2b}\exp\left(-\frac{|z|}{b}\right), \hspace{1em} f_1(z) = \frac{1}{2b}\exp\left(-\frac{|z - \alpha|}{b}\right).
        \end{align}
        Due to symmetry, assume without loss of generality that $\alpha \geq 0$.
        Then for $z\in\R$
        \begin{align}
            \Lambda(z) &= \frac{f_1(z)}{f_0(z)} = \exp\left(-\frac{|z-\alpha| - |z|}{b}\right)\\&=
                \begin{cases}
                    \exp\left(-\frac{\alpha}{b}\right)\hspace{0.5em} & z < 0,\\
                    \exp\left(\frac{2z - \alpha}{b}\right)\hspace{0.5em} & 0 \leq z < \alpha,\\
                    \exp\left(\frac{\alpha}{b}\right)\hspace{0.5em} & z \geq \alpha.
                \end{cases}
        \end{align}
        Note that the CDFs for $\varepsilon_0$ and $\varepsilon_1$ are given by
        \begin{align}
            F_0(z) &=
                \begin{cases}
                    \frac{1}{2}\exp\left(\frac{z}{b}\right) \hspace{0.5em} & z\leq 0,\\
                    1 - \frac{1}{2}\exp\left(-\frac{z}{b}\right) & z>0,
                \end{cases}\\
            F_1(z) &=
                \begin{cases}
                    \frac{1}{2}\exp\left(\frac{z-\alpha}{b}\right) \hspace{0.5em} & z\leq \alpha,\\
                    1 - \frac{1}{2}\exp\left(-\frac{z-\alpha}{b}\right) & z>\alpha.
                \end{cases}
        \end{align}
        Note that for $\exp\left(-\frac{\alpha}{b}\right) \leq t < \exp\left(\frac{\alpha}{b}\right)$ we have
        \begin{align}
            \begin{split}
                &\Prob_0\left(\exp\left(\frac{2\varepsilon_0 - \alpha}{b}\right) \leq t \,\land \, 0 \leq \varepsilon_0 < \alpha\right) \\
                &\hspace{4em}= \Prob_0\left(\exp\left(-\frac{\alpha}{b}\right) \leq \exp\left(\frac{2\varepsilon_0 - \alpha}{b}\right) \leq t\right)
            \end{split}\\
            &\hspace{4em}=\Prob_0\left(0 \leq \varepsilon_0 \leq \frac{b\log(t) + \alpha}{2}\right)\\
            &\hspace{4em}= F_0\left(\frac{b\log(t) + \alpha}{2}\right) - F_0(0)\\
            &\hspace{4em}= \frac{1}{2} - \frac{1}{2}\exp\left(-\frac{1}{b}\left(\frac{b\log(t) + \alpha}{2}\right)\right)\\
            &\hspace{4em}= \frac{1}{2} - \frac{1}{2\sqrt{t}}\exp\left(-\frac{\alpha}{2b}\right).
        \end{align}
        Computing $\zeta$ yields
        \begin{align}
            \zeta(t) &= \Prob\left(\Lambda(\varepsilon_0) \leq t\right)\\
            \begin{split}
                &= \Prob\left(\exp\left(-\frac{\alpha}{b}\right) \leq t \,\land \, \varepsilon_0 < 0\right)\\
                &\hspace{2em} + \Prob\left(\exp\left(\frac{\alpha}{b}\right) \leq t \,\land \, \varepsilon_0 \geq \alpha\right)\\
                &\hspace{4em} + \Prob\left(\exp\left(\frac{2\varepsilon_0 - \alpha}{b}\right) \leq t \,\land \, 0 \leq \varepsilon_0 < \alpha\right)
            \end{split}\\
            &= \begin{cases}
                    0 \hspace{0.5em} &t < \exp\left(-\frac{\alpha}{b}\right),\\
                    1 - \frac{1}{2\sqrt{t}}\exp\left(-\frac{\alpha}{2b}\right) &\exp\left(-\frac{\alpha}{b}\right) \leq t < \exp\left(\frac{\alpha}{b}\right),\\
                    1 &t \geq \exp\left(\frac{\alpha}{b}\right).
                \end{cases}
        \end{align}
        The inverse is then given by
        \begin{align}
            \zeta^{-1}(p) =
                \begin{cases}
                    0 \hspace{0.5em} & p< \frac{1}{2},\\
                    \frac{1}{4(1-p)^2}\exp\left(-\frac{\alpha}{b}\right) & \frac{1}{2} \leq p < 1 - \frac{1}{2}\exp(-\frac{\alpha}{b}),\\
                    \exp\left(\frac{\alpha}{b}\right) & p \geq 1 - \frac{1}{2}\exp(-\frac{\alpha}{b}).
                \end{cases}
        \end{align}
        In order to evaluate $\xi$, we compute the lower and strict lower level sets at $t=\zeta^{-1}(p)$.
        Recall that $\underline{S}_{t} = \{z\in\R\colon\, \Lambda(z) < t\}$ and $\overline{S}_{t} = \{z\in\R\colon\, \Lambda(z)\leq t\}$ and consider
        \begin{align}
            \underline{S}_{\zeta^{-1}(p)} =
                \begin{cases}
                    \varnothing\hspace{0.5em}&p \leq \frac{1}{2},\\
                    \left(-\infty,\,b\cdot\log\left(\frac{1}{2(1-p)}\right)\right)&\frac{1}{2} < p< 1 - \frac{1}{2}\exp\left(-\frac{\alpha}{b}\right),\\
                    \left(-\infty,\,\alpha\right], &p \geq 1 - \frac{1}{2}\exp\left(-\frac{\alpha}{b}\right)
                \end{cases}
        \end{align}
        and
        \begin{align}
            \overline{S}_{\zeta^{-1}(p)} =
                \begin{cases}
                    \varnothing \hspace{0.5em} &p < \frac{1}{2},\\
                    \left(-\infty,\,b\cdot\log\left(\frac{1}{2(1-p)}\right)\right]&\frac{1}{2} \leq p< 1 - \frac{1}{2}\exp\left(-\frac{\alpha}{b}\right),\\
                    \R &p \geq 1 - \frac{1}{2}\exp\left(-\frac{\alpha}{b}\right).
                \end{cases}
        \end{align}
        Suppose $p < \nicefrac{1}{2}$.
        Then $\underline{S}_{\zeta^{-1}(p)} = \overline{S}_{\zeta^{-1}(p)} = \varnothing$ and hence $\xi(p)=0$ and condition~(\ref{eq:laplace_condition_proof}) cannot be satisfied.
        If $p=\nicefrac{1}{2}$, then $\underline{S}_{\zeta^{-1}(p)} = \varnothing$ and $\overline{S}_{\zeta^{-1}(p)} = (-\infty,\,0]$.
        Note that for $z\leq 0$ we have $f_1(z) = f_0(z)\exp(-\nicefrac{\alpha}{b})$ and hence for any $S\subseteq\overline{S}_{\zeta^{-1}(\nicefrac{1}{2})}$ we have $\Prob_1(S) = \exp(-\nicefrac{\alpha}{b})\cdot\Prob_0(S)$.
        We can thus compute $\xi$ at $1/2$ as
        \begin{align}
            \begin{split}
                &p=\frac{1}{2} \\
                &\Rightarrow \xi\left(1/2\right) = \sup\,\left\{\Prob_1(S)\colon\, S \subseteq\,(-\infty,\,0],\,\Prob_0(S)\leq \frac{1}{2}\right\} = \frac{1}{2}.
            \end{split}\label{eq:p_equal_1half}
        \end{align}
        Now suppose $\nicefrac{1}{2} < p < 1 - \nicefrac{1}{2}\exp(-\nicefrac{\alpha}{b})$.
        In this case, $\underline{S}_{\zeta^{-1}(p)} = (-\infty,\,b\cdot\log(\nicefrac{1}{2(1-p)}))$ and $\overline{S}_{\zeta^{-1}(p)} = \left(-\infty,\,b\cdot\log\left(\nicefrac{1}{2(1-p)}\right)\right]$.
        Since the singleton $\{b\cdot\log(\nicefrac{1}{2(1-p)})\}$ has no probability mass under both $\Prob_0$ and $\Prob_1$, the function $\xi$ is straight forward to compute: if $\frac{1}{2} < p < 1 - \frac{1}{2}\exp(-\frac{\alpha}{b})$, then
        \begin{align}
            \xi(p) &= \Prob\left(\varepsilon_1 \leq b\cdot\log\left(\frac{1}{2(1-p)}\right)\right)\\
            &= \frac{1}{2}\exp\left(\frac{b\cdot\log\left(\frac{1}{2(1-p)}\right)-\alpha}{b}\right)\\
            &=\frac{1}{4(1-p)}\exp\left(-\frac{\alpha}{b}\right). \label{eq:p_bigger_1half}
        \end{align}
        Finally, consider the case where $p \geq 1 - \nicefrac{1}{2}\exp(-\nicefrac{\alpha}{b})$.
        Then $\underline{S}_{\zeta^{-1}(p)} = (-\infty,\,\alpha]$ and $\overline{S}_{\zeta^{-1}(p)} = \R$.
        Any $(-\infty,\,\alpha] \subseteq S \subseteq \R$ can then be written as $S=(-\infty,\,\alpha]\,\dot\cup\,T$ for some $T \subseteq (\alpha,\,\infty)$.
        Hence
        \begin{align}
            \Prob_1(S) &= \Prob(\varepsilon_1 \leq \alpha) + \Prob_1(T) = \frac{1}{2} + \exp\left(\frac{\alpha}{b}\right)\Prob_0(T),\\
            \Prob_0(S) &= \Prob(\varepsilon_0 \leq \alpha) + \Prob_0(T) = 1 - \frac{1}{2}\exp(-\frac{\alpha}{b}) + \Prob_0(T).
        \end{align}
        Thus, if $p \geq 1 - \frac{1}{2}\exp(-\frac{\alpha}{b})$, then
        \begin{align}
            \xi\left(p\right) &= \sup\left\{\Prob_1(S)\colon \,(-\infty,\,\alpha] \subseteq S \subseteq \R,\,\Prob_0(S)\leq p\right\}\\
            \begin{split}
                &= \frac{1}{2} + \sup\bigg\{\Prob_1(T)\colon\,T\subseteq (\alpha,\,\infty),\\
                &\hspace{6em}\,\Prob_0(T)\leq p - 1  + \frac{1}{2}\exp\left(-\frac{\alpha}{b}\right)\bigg\}
            \end{split}\\
            &= \frac{1}{2} + \exp\left(\frac{\alpha}{b}\right)\left(p - 1 + \frac{1}{2}\exp\left(-\frac{\alpha}{b}\right)\right)\\
            &= 1 - \exp\left(\frac{\alpha}{b}\right)\left(1 - p\right). \label{eq:p_bigger_1minus_exp}
        \end{align}
        In order to evaluate condition~(\ref{eq:laplace_condition_proof}), consider
        \begin{align}
             &1 - \xi\left(1-p_B\right) =
                \begin{cases}
                    1 \hspace{0.5em} & p_B > \frac{1}{2}\\
                    \frac{1}{2} & p_B=\frac{1}{2}\\
                    1 - \frac{1}{4p_B}\exp\left(-\frac{\alpha}{b}\right) & \frac{1}{2} > p_B > \exp\left(-\frac{\alpha}{b}\right)\\
                    \exp\left(\frac{\alpha}{b}\right)p_B & \exp\left(-\frac{\alpha}{b}\right) \geq p_B,
                \end{cases}\\
            &\xi\left(p_A\right) =
                \begin{cases}
                    0 \hspace{0.5em} & p_A < \frac{1}{2}\\
                    \frac{1}{2} & p_A=\frac{1}{2}\\
                    \dfrac{1}{4(1-p_A)}\exp\left(-\dfrac{\alpha}{b}\right) &\frac{1}{2} < p_A < 1 - \frac{1}{2}\exp(-\frac{\alpha}{b})\\
                    1 - \exp\left(\dfrac{\alpha}{b}\right)\left(1 - p_A\right) &p_A \geq 1 - \frac{1}{2}\exp(-\frac{\alpha}{b}).
                \end{cases}
        \end{align}
        Note that the case $p_B>\nicefrac{1}{2}$ can be ruled out, since by assumption $p_A\geq p_B$.
        If $p_A = \nicefrac{1}{2}$, then we need $p_B < \nicefrac{1}{2}$.
        Thus, if $p_A=\nicefrac{1}{2}$, then condition~(\ref{eq:laplace_condition_proof}) is satisfied if $p_B < \nicefrac{1}{2}$ and
        \begin{gather}
            \max\left\{1 - \frac{1}{4p_B}\exp\left(-\frac{\alpha}{b}\right),\, \exp\left(\frac{\alpha}{b}\right)\cdot p_B \right\} < \frac{1}{2}\\
            \iff p_B\cdot\exp\left(\frac{\alpha}{b}\right) < \frac{1}{2}\\
            \iff \alpha < -b\cdot\log\left(2p_B\right).
        \end{gather}
        Now consider the case where $p_A > \nicefrac{1}{2}$. If $p_B=\nicefrac{1}{2}$, then condition~(\ref{eq:laplace_condition_proof}) is satisfied if
        \begin{gather}
            \frac{1}{2} < \min\left\{\frac{1}{4(1-p_A)}\exp\left(-\frac{\alpha}{b}\right),\,1 - \exp\left(\frac{\alpha}{b}\right)\left(1 - p_A\right)\right\}\\
            \iff \frac{1}{2} < 1 - \exp\left(\frac{\alpha}{b}\right)(1-p_A)\\
            \iff \alpha < -b\cdot\log\left(2(1-p_A)\right).
        \end{gather}
        If on the other hand, $p_A > \nicefrac{1}{2}$ and $p_B < \nicefrac{1}{2}$, condition~(\ref{eq:laplace_condition_proof}) is satisfied if
        \begin{gather}
            \begin{split}
                &\max\left\{1 - \frac{1}{4p_B}\exp\left(-\frac{\alpha}{b}\right),\, \exp\left(\frac{\alpha}{b}\right)\cdot p_B \right\} < \\
                &\hspace{2em}\min\left\{\frac{1}{4(1-p_A)}\exp\left(-\frac{\alpha}{b}\right),\,1 - \exp\left(\frac{\alpha}{b}\right)\left(1 - p_A\right)\right\}
            \end{split}\\
            \iff\nonumber\\
            p_B\cdot\exp\left(\frac{\alpha}{b}\right) < 1 - \exp\left(\frac{\alpha}{b}\right)(1-p_A)\\
            \iff\nonumber\\
            \alpha < -b \cdot \log\left(1 - p_A + p_B\right).
        \end{gather}
        Finally, we get that condition~(\ref{eq:laplace_condition_proof}) is satisfied, if
        \begin{align}
            |\alpha| <
                \begin{cases}
                    -b\cdot\log\left(4\,p_B\,(1-p_A)\right)\,&
                    \!\begin{aligned}
                        &(p_A =\frac{1}{2} \,\land \ \, p_B < \frac{1}{2})\\
                        &\lor\, (p_A> \frac{1}{2}\,\land\,p_B = \frac{1}{2})
                    \end{aligned}\\
                    -b\cdot \log\left(1 - p_A + p_B\right) & p_A > \frac{1}{2} \,\land\, p_B < \frac{1}{2}
                \end{cases}
        \end{align}
        what concludes the proof.
    \end{proof}
    \subsection{Folded Gaussian Smoothing}
    \begin{cor}
        \label{cor_apx:folded_gaussian}
        Suppose $\gZ=\R_{\geq0}$, $\varepsilon_0\sim|\gN(0,\,\sigma)|$ and $\varepsilon_1:=\alpha+\varepsilon_0$ for some $\alpha>0$.
        Suppose that the $\varepsilon_0$-smoothed classifier $g$ is $(p_A,\,p_B)$-confident at $x\in\gX$ for some $y_A\in\gY$. Then, it holds that $q(y_A\lvert\,x;\,\varepsilon_1) > \max_{y\neq y_A}q(y\lvert\,x;\,\varepsilon_1)$ if $\alpha$ satisfies
        \begin{equation}
            \label{eq:folded_gaussian_condition}
            \alpha < \sigma\cdot\left(\Phi^{-1}\left(\dfrac{1 + \min\{p_A,\,1-p_B\}}{2}\right)-\Phi^{-1}\left(\dfrac{3}{4}\right)\right).
        \end{equation}
    \end{cor}
    \begin{proof}
        By Theorem~\ref{thm:main} we know that if $\varepsilon_1$ satisfies
        \begin{equation}
            \xi(p_A) + \xi(1-p_B) > 1,\label{eq:folded_gaussian_condition_proof}
        \end{equation}
        then it is guaranteed that
        \begin{equation}
            q(y_A\lvert\,x;\,\varepsilon_1) > \max_{y\neq y_A}q(y\lvert\,x;\,\varepsilon_1).
        \end{equation}
        The proof is thus complete if we show that~(\ref{eq:folded_gaussian_condition_proof}) reduces to~(\ref{eq:folded_gaussian_condition}).
        For that purpose denote by $f_0$ and $f_1$ density functions of $\varepsilon_0$ and $\varepsilon_1$,  respectively,
        and consider
        \begin{align}
            f_0(z) &=
                \begin{cases}
                    \frac{2}{\sqrt{2\pi}\sigma}\exp\left(-\frac{z^2}{2\sigma^2}\right) \hspace{0.5em} &z\geq 0\\
                    0 & z < 0
                \end{cases}\\
            f_1(z) &=
                \begin{cases}
                    \frac{2}{\sqrt{2\pi}\sigma}\exp\left(-\frac{(z-\alpha)^2}{2\sigma^2}\right) \hspace{0.5em}&z\geq \alpha\\
                    0 & z<\alpha.
                \end{cases}
        \end{align}
        Then, for $z\geq 0$,
        \begin{align}
            \Lambda(z) = \frac{f_1(z)}{f_0(z)} =
                \begin{cases}
                    0 & z < \alpha,\\
                    \exp\left(\frac{z\alpha}{\sigma^2} - \frac{\alpha^2}{2\sigma^2}\right) \hspace{0.5em} &z\geq \alpha.
                \end{cases}
        \end{align}
        Let $t_\alpha := \exp\left(\frac{\alpha^2}{2\sigma^2}\right)$ and suppose $t<t_\alpha$. Then
        \begin{align}
            \zeta(t) &= \Prob\left(\Lambda(\varepsilon_0)\leq t\right]
            = \Prob\left(\varepsilon_0 < \alpha\right)\\
            &= \int_{0}^{\alpha}\frac{2}{\sqrt{2\pi}\sigma}\exp\left(-\frac{z^2}{2\sigma^2}\right)\,dz \\
            &= 2\cdot\int_{0}^{\alpha/\sigma}\frac{1}{\sqrt{2\pi}}\exp\left(-\frac{s^2}{2}\right)\,ds
            = 2\cdot\Phi\left(\frac{\alpha}{\sigma}\right) - 1.
        \end{align}
        If $t\geq t_\alpha$, then
        \begin{align}
            \zeta(t) &= \Prob\left(\Lambda(\varepsilon_0)\leq t\right)\\
            &= \Prob\left(\frac{\varepsilon_0\,\alpha}{\sigma^2} - \frac{\alpha^2}{2\sigma^2} \leq \log(t)\,\land\,\varepsilon_0 \geq \alpha\right) + \Prob\left(\varepsilon_0 < \alpha\right)\\
            &= \Prob\left(\varepsilon_0  \leq \frac{\sigma^2}{\alpha}\log(t) + \frac{1}{2}\alpha\right)\\
            &= 2\cdot\Phi\left(\frac{\sigma}{\alpha}\log(t) + \frac{\alpha}{2\sigma}\right) - 1
        \end{align}
        and hence
        \begin{align}
            \zeta(t) =
                \begin{cases}
                    2\cdot\Phi\left(\frac{\alpha}{\sigma}\right) - 1 & t < t_\alpha\\
                    2\cdot\Phi\left(\frac{\sigma}{\alpha}\log\left(t\right) + \frac{\alpha}{2\sigma}\right) - 1 & t \geq t_\alpha.
                \end{cases}
        \end{align}
        Note that $\zeta(t_\alpha) = 2\cdot\Phi\left(\frac{\alpha}{\sigma}\right) - 1$ and let $p_\alpha := \zeta(t_\alpha)$.
        Recall that $\zeta^{-1}(p) := \inf\{t\geq 0\colon\,\zeta(t) \geq p\}$, which yields
        \begin{align}
            \zeta^{-1}(p) =
                \begin{cases}
                    0 & p \leq p_\alpha\\
                    \exp\left(\frac{\alpha}{\sigma}\Phi^{-1}\left(\frac{1 + p}{2}\right) - \frac{\alpha^2}{2\sigma^2}\right) & p > p_\alpha.
                \end{cases}
        \end{align}
        In order to evaluate $\xi$ we compute the lower and strict lower level sets at $t=\zeta^{-1}(p)$. Recall that $\underline{S}_{t} = \{z\in\R_{\geq0}\colon\, \Lambda(z) < t\}$ and $\overline{S}_{t} = \{z\in\R_{\geq0}\colon\, \Lambda(z) \leq t\}$.
        Let $S_0:=[0,\,\alpha)$ and note that if $p \leq p_\alpha$, we have $\zeta^{-1}(p) = 0$ and hence
        $\underline{S}_{\zeta^{-1}(p)}=\varnothing$ and $\overline{S}_{\zeta^{-1}(p)} = S_0$.
        If, on the other hand $p>p_\alpha$, then
        \begin{align}
            \underline{S}_{\zeta^{-1}(p)} &= \left\{z\geq 0\colon\,\Lambda(z) < \zeta^{-1}(p)\right\}\\
            &\hspace{-1em}= S_0 \cup \left\{z\geq \alpha\colon\,\frac{z\,\alpha}{\sigma^2} - \frac{\alpha^2}{2\sigma^2} < \frac{\alpha}{\sigma}\Phi^{-1}\left(\frac{1 + p}{2}\right) -  \frac{\alpha^2}{2\sigma^2}\right\}\\
            &\hspace{-1em}= S_0 \cup \left\{z\geq \alpha\colon\, z < \sigma\cdot\Phi^{-1}\left(\frac{1+p}{2}\right)\right\}\\
            &\hspace{-1em}= S_0 \cup \left[\alpha,\,\sigma\cdot\Phi^{-1}\left(\frac{1+p}{2}\right)\right)
        \end{align}
        and
        \begin{align}
            \overline{S}_{\zeta^{-1}(p)} &= \left\{z\geq 0\colon\,\Lambda(z)\leq\zeta^{-1}(p)\right\}\\
            &\hspace{-1em}= S_0 \cup \left\{z\geq \alpha\colon\,\frac{z\,\alpha}{\sigma^2} - \frac{\alpha^2}{2\sigma^2} \leq \frac{\alpha}{\sigma}\Phi^{-1}\left(\frac{1 + p}{2}\right) - \frac{\alpha^2}{2\sigma^2}\right\}\\
            &\hspace{-1em}= S_0 \cup \left\{z\geq \alpha\colon\,z\leq\sigma\cdot\Phi^{-1}\left(\frac{1+p}{2}\right)\right\}\\
            &\hspace{-1em}= S_0 \cup \left[\alpha,\,\sigma\cdot\Phi^{-1}\left(\frac{1+p}{2}\right)\right]\\
            &\hspace{-1em}= \underline{S}_{\zeta^{-1}(p)} \cup \left\{\sigma\cdot\Phi^{-1}\left(\frac{1+p}{2}\right)\right\}.
        \end{align}
        In other words
        \begin{align}
            \underline{S}_{\zeta^{-1}(p)} &=
                \begin{cases}
                    \varnothing \hspace{0.5em} &p\leq p_\alpha,\\
                    S_0 \cup \left[\alpha,\,\sigma\cdot\Phi^{-1}\left(\frac{1+p}{2}\right)\right) &p >p_\alpha,
                \end{cases}\\
            \overline{S}_{\zeta^{-1}(p)} &=
                \begin{cases}
                    S_0 \hspace{0.5em} &p\leq p_\alpha,\\
                    S_0 \cup \left[\alpha,\,\sigma\cdot\Phi^{-1}\left(\frac{1+p}{2}\right)\right] &p > p_\alpha.
                \end{cases}
        \end{align}
        Let $\gS_{p}:=\{S\subseteq\R_{\geq0}\colon\,\underline{S}_{\zeta^{-1}(p)} \subseteq S \subseteq \overline{S}_{\zeta^{-1}(p)},\,\Prob_0(S) \leq p\}$ and recall that $\xi(p) = \sup_{S \in \gS_{p}}\,\Prob_1(S)$.
        Note that for $p\leq p_\alpha$, we have $\gS_{p} = \{S\subseteq\R_{\geq0}\colon\, S \subseteq S_0\,\land\,\Prob_0(S) \leq p\}$ and for $S\subseteq S_0$, it holds that $\Prob_1(S) = 0$.
        Hence
        \begin{align}
            p\leq p_\alpha \Rightarrow \xi\left(p\right) = \sup_{S\in \gS_{p}}\,\Prob_1(S) = 0.
        \end{align}
        If $p> p_\alpha$, then
        \begin{align}
            \begin{split}
                \gS_{p} &=\bigg\{S\subseteq\R_{\geq0}\colon\,S_0 \cup \left[\alpha,\,\sigma\cdot\Phi^{-1}\left(\nicefrac{1+p}{2}\right)\right) \subseteq S\\
                &\hspace{4em}\subseteq S_0 \cup \left[\alpha,\,\sigma\cdot\Phi^{-1}\left(\nicefrac{1+p}{2}\right)\right],
                \,\land\,\Prob_0(S) \leq p\bigg\}.
            \end{split}
        \end{align}
        Since the singleton $\left\{\sigma\cdot\Phi^{-1}\left(\frac{1+p}{2}\right)\right\}$ has no mass under both $\Prob_0$ and $\Prob_1$, we find that if $p> p_\alpha$, then
        \begin{align}
            \xi\left(p\right) &= \Prob\left(0 \leq \varepsilon_1 \leq\sigma\cdot\Phi^{-1}\left(\frac{1+p}{2}\right)\right)\\
            & = \Prob\left(0 \leq \varepsilon_0\leq\sigma\cdot\Phi^{-1}\left(\frac{1+p}{2}\right) -\alpha \right)\\
            & = 2\cdot\Phi\left( \Phi^{-1}\left(\dfrac{1 + p}{2}\right) - \dfrac{\alpha}{\sigma}\right) - 1.
        \end{align}
        Condition~(\ref{eq:folded_gaussian_condition_proof}) can thus be satisfied only if $p_B < p_A$ and
        \begin{gather}
            \begin{split}
                2\cdot\Phi\left(\frac{\alpha}{\sigma}\right) - 1 < \min\{p_A,\,1-p_B\}\\
                \land\ \xi\left(p_A\right) + \xi\left(1-p_b\right) > 1
            \end{split}
        \end{gather}
        that is equivalent to
        \begin{align}
            \begin{split}
                &\alpha < \sigma\cdot \Phi^{-1}\left(\frac{1 + \min\{p_A,\,1-p_B\}}{2}\right)\,\\
                &\hspace{4em}\land\,\Phi\left( \Phi^{-1}\left(\frac{1 + (1-p_B)}{2}\right) - \frac{\alpha}{\sigma}\right) \\
                &\hspace{8em}+ \Phi\left( \Phi^{-1}\left(\frac{1+p_A}{2}\right) - \frac{\alpha}{\sigma}\right) > \frac{3}{2}.\label{eq:folded_gaussian_inequalitites}
            \end{split}
        \end{align}
        Thus, the following is a sufficient condition for the two inequalities in~(\ref{eq:folded_gaussian_inequalitites}) and hence~(\ref{eq:folded_gaussian_condition_proof}) to hold
        \begin{align}
            \alpha < \sigma\cdot\left(\Phi^{-1}\left(\frac{1 + \min\{p_A,\,1-p_B\}}{2}\right)-\Phi^{-1}\left(\frac{3}{4}\right)\right)
        \end{align}
        what completes the proof.
    \end{proof}
    \fi

\section{Comparison of Different Smoothing Distributions}
    \label{adx-sec:smoothing-distributions}
    Here, we provide robustness radii derived from different smoothing distributions in \Cref{adx-table:smoothing-distributions}, each scaled to have unit variance.
    The figure comparison is in \Cref{fig:bound_compare} in main text.

    \begin{table}[!t]
        \centering
        \caption{\small Comparison of certification radii with $p_A + p_B = 1$, the variance and noise dimension are set to $1$ for each distribution.
        }
        \resizebox{\linewidth}{!}{
        \begin{tabular}{c  c  c}
            \toprule
            Distribution & Value Space & Robust Radius\\
            \midrule
            Gaussian$(0,\,1)$ & $(-\infty,\,\infty)$ & $\Phi^{-1}(p_A)$\\
            \midrule
            Laplace$(0,\,{1}/{\sqrt{2}})$ & $(-\infty,\,\infty)$ & $-\log(2-2p_A)/{\sqrt{2}}$\\
            \midrule
            Uniform$[-\sqrt{3},\,-\sqrt{3}]$ & $(-\infty,\,\infty)$ & $2\sqrt{3}\cdot(p_A - {1}/{2})$\\
            \midrule
            Exponential($1$) & $[0,\,\infty)$ & $-\log(2-2p_A)$\\
            \midrule
            FoldedGaussian$(0,\,\sqrt{\frac{\pi}{\pi - 2}})$ & $[0,\,\infty)$ & $\sqrt{\frac{\pi}{\pi - 2}}\cdot\left(\Phi^{-1}\left(\frac{1 + p_A}{2}\right) - \Phi^{-1}\left(\frac{3}{4}\right)\right)$\\
            \bottomrule
        \end{tabular}}
        \label{adx-table:smoothing-distributions}
    \end{table}

\section{Proofs for Certifying Resolvable Transformations}
    \label{appendix:proofs_resolvable}

    \ifnum\ArXiv=0
    We leave the proofs to the arXiv version~\cite{arxiv} due to space limit.
    \fi

    \ifnum\ArXiv=1
    Here, we state the proofs and technical details concerning our results for resolvable transformations.
    Recall the definition of resolvable transformations from the main part of this paper.
    \begin{customdef}{\ref{def:resolvable_transform}}[restated]
        A transformation $\phi\colon\gX\times\gZ\to\gX$ is called resolvable if for any $\alpha\in\gZ$ there exists a resolving function $\gamma_\alpha\colon\gZ\to\gZ$ that is injective, continuously differentiable, has non-vanishing Jacobian and for which
        \begin{equation}
            \phi(\phi(x,\,\alpha),\,\beta) = \phi(x,\,\gamma_\alpha(\beta))\hspace{1em}x\in\gX,\,\beta\in\gZ.
        \end{equation}
        We say that $\phi$ is additive, if $\gamma_\alpha(\beta) = \alpha + \beta$.
    \end{customdef}
    \subsection{Proof of Corollary~\ref{cor:main_thm_resolvable}}
    \begin{customcor}{\ref{cor:main_thm_resolvable}}[restated]
        Suppose that the transformation $\phi$ in Theorem~\ref{thm:main} is resolvable with resolving function $\gamma_\alpha$. Let $\alpha\in\gZ$ and set $\varepsilon_1:=\gamma_\alpha(\varepsilon_0)$ in the definition of the functions $\zeta$ and $\xi$. Then, if $\alpha$ satisfies condition~(\ref{eq:robustness_condition}), it is guaranteed that $g(\phi(x,\alpha);\,\varepsilon_0)=g(x;\,\varepsilon_0)$.
    \end{customcor}
    \begin{proof}
        Since $\phi$ is a resolvable transformation, by definition $\gamma_\alpha$ is injective, continuously differentiable and has non-vanishing Jacobian. By Jacobi's transformation formula~(see e.g., \cite{protter2000}), it follows that the density of $\varepsilon_1$ vanishes outside the image of $\gamma_\alpha$ and is elsewhere given by
        \begin{equation}
            f_1(z) = f_0(\gamma_\alpha^{-1}(z))\lvert\mathrm{det}(J_{\gamma_\alpha^{-1}(z)})\rvert \hspace{1em}\text{for any } z \in \mathrm{Im}(\gamma_\alpha)
        \end{equation}
        where $J_{\gamma_\alpha^{-1}(z)}$ is the Jacobian of $\gamma_\alpha^{-1}(z)$. Since $f_1$ is paramterized by $\alpha$, it follows from Theorem~\ref{thm:main} that if $\alpha$ satisfies~\eqref{eq:robustness_condition} it is guaranteed that $\argmax_y q(y\lvert\,x,\,\varepsilon_1)=\argmax_y q(y\lvert\,x,\,\varepsilon_0)$. The statement of the corollary then follows immediately from the observation that for any $y\in\gY$ we have
        \begin{align}
            q(y\lvert\,x;\,\varepsilon_1) &= \E(p(y\lvert\,\phi(x,\,\varepsilon_1)))\\
            &= \E(p(y\lvert\,\phi(x,\,\gamma_\alpha(\varepsilon_0))))\\
            &= \E(p(y\lvert\,\phi(\phi(x,\,\alpha),\,\varepsilon_0)))\\
            &= q(y\lvert\,\phi(x,\,\alpha);\,\varepsilon_0).
        \end{align}
    \end{proof}

    \subsection{Gaussian Blur}
    Recall that the Gaussian blur transformation is given by a convolution with a Gaussian kernel
    \begin{equation}
        G_\alpha(k) = \dfrac{1}{\sqrt{2\pi\alpha}}\exp\left(-\dfrac{k^2}{2\alpha}\right)
    \end{equation}
    where $\alpha>0$ is the squared kernel radius. Here we show that the transformation $x\mapsto \phi_B(x):= x\,*\,G$ is additive.
    \begin{customlem}{\ref{lem:gaussian_blur_additive}}[restated]
        The Gaussian blur transformation is additive, i.e., for any $\alpha,\,\beta \geq 0$, we have $\phi_B(\phi_B(x,\,\alpha),\,\beta) = \phi_B(x,\,\alpha + \beta)$.
    \end{customlem}

    \begin{proof}
    Note that associativity of the convolution operator implies that
    \begin{align}
        \phi_B(\phi_B(x,\,\alpha),\,\beta) &= (\phi_B(x,\,\alpha)\,*\,G_\beta)\\
        &= ((x\,*\,G_\alpha)\,*\,G_\beta)\\
        &= (x\,*\,(G_\alpha\,*\,G_\beta)).
    \end{align}
    The claim thus follows, if we can show that $(G_\alpha\,*\,G_\beta) = G_{\alpha + \beta}$. Let $\gF$ denote the Fourier transformation and $\gF^{-1}$ the inverse Fourier transformation and note that by the convolution theorem $(G_\alpha\,*\,G_\beta) = \gF^{-1}\{\gF(G_\alpha)\cdot\gF(G_\beta)\}$. Therefore we have to show that $\gF(G_\alpha)\cdot\gF(G_\beta) = \gF(G_{\alpha+\beta})$. For that purpose, consider
    \begin{align}
        \gF(G_\alpha)(\omega) &= \int_{-\infty}^{\infty}\,G_\alpha(y)\exp(-2\,\pi i \omega y)\,dy \\
        &\hspace{-4em}= \int_{-\infty}^{\infty}\dfrac{1}{\sqrt{2\pi\alpha}}\exp\left(-\dfrac{y^2}{2\alpha}\right)\exp\left(-2\,\pi i \omega y\right)\,dy\\
        &\hspace{-4em}=\dfrac{1}{\sqrt{2\pi\alpha}}\int_{-\infty}^{\infty}\exp\left(-\dfrac{y^2}{2\alpha}\right)\left(\cos\left(2\pi\omega y\right) + i \sin\left(2\pi \omega y\right)\right)\,dy\\
        &\hspace{-4em}\overset{(i)}{=}\dfrac{1}{\sqrt{2\pi\alpha}}\int_{-\infty}^{\infty}\exp\left(-\dfrac{y^2}{2\alpha}\right)\,\cos\left(2\pi\omega y\right)\,dy\\
        &\hspace{-4em}\overset{(ii)}{=} \exp\left(-\omega^2\pi^2 2\alpha\right),
    \end{align}
    where $(i)$ follows from the fact that the second term is an integral of an odd function over a symmetric range and $(ii)$ follows from $\int_{-\infty}^{\infty}\exp\left(-a\,y^2\right)\,\cos\left(2\pi\omega y\right)\,dy = \sqrt{\frac{\pi}{a}}\exp(\frac{-(\pi\omega)^2}{a})$ with $a=\frac{1}{2\alpha}$~(see p.~302, eq.~7.4.6 in \cite{abramowitz1972handbook}). This concludes our proof since
    \begin{align}
        (\gF(G_\alpha)\cdot\gF(G_\beta))(\omega) &= \exp\left(-\omega^2\pi^2 2\alpha\right) \cdot \exp\left(-\omega^2\pi^2 2\beta\right)\\
        &= \exp\left(-\omega^2\pi^2 2(\alpha+\beta)\right)\\
        &= \gF(G_{\alpha+\beta})(\omega)
    \end{align}
    and hence
    \begin{align}
        (G_\alpha\,*\,G_\beta) &= \gF^{-1}\{\gF(G_\alpha)\cdot\gF(G_\beta)\} \\
        &= \gF^{-1}\{\gF(G_{\alpha+\beta})\}\\
        &= G_{\alpha+\beta}.
    \end{align}
    \end{proof}
    \begin{rem}
        We notice that the preceding theorem naturally extends to higher dimensional Gaussian kernels of the form
        \begin{equation}
            G_\alpha(k) = \dfrac{1}{(2\pi\alpha)^{\frac{m}{2}}} \exp\left(-\dfrac{\|k\|^2}{2\alpha}\right),\hspace{2em}k\in\R^m.
        \end{equation}
        Consider
        \begin{align}
            \gF(G_\alpha)(\omega) &= \int_{\R^m}G_\alpha(y)\exp\left(-2\pi i \langle\omega,\,y\rangle\right)\,dy\\
            &= \dfrac{1}{(2\pi\alpha)^\frac{m}{2}}\int_{\R^m}\exp\left(-\frac{\left\|y\right\|_2^2}{2\alpha}-2\pi i \langle\omega,\,y\rangle\right)\,dy\\
            &=\prod_{j=1}^m\left(\dfrac{1}{\sqrt{2\pi\alpha}}\int_{\R}\exp\left(-\frac{y_j^2}{2\alpha} - 2\pi i \omega_j y_j\right)\,dy_j\right)\\
            &= \exp\left(-\left\|\omega\right\|_2^2\pi^22\alpha\right)
        \end{align}
        that leads to $(G_\alpha * G_\beta) = G_{\alpha + \beta}$, and hence additivity.
    \end{rem}

    \subsection{Brightness and contrast}
    Recall that the brightness and contrast transformation is defined as
    \begin{equation}
        \begin{aligned}
            \phi_{BC}\colon\gX\times\R^2\to\gX,\hspace{1em}
            (x,\,\alpha)\mapsto e^{\alpha_1}(x + \alpha_2).
        \end{aligned}
    \end{equation}
    \begin{customlem}{\ref{lem:bc_lower_bound}}[restated]
        Let $x\in\gX$, $k\in\R$, $\varepsilon_0\sim\gN(0,\,\mathrm{diag}(\sigma^2,\,\tau^2))$ and $\varepsilon_1\sim\gN(0,\,\mathrm{diag}(\sigma^2,\,e^{-2k}\tau^2))$.
        Suppose that $q(y\lvert\,x;\,\varepsilon_0)\geq p$ for some $p\in[0,\,1]$ and $y\in\gY$. Then
        \begin{align}
            q(y\lvert\,x;\,\varepsilon_1) \geq
                \begin{cases}
                     2\Phi\left(e^k\Phi^{-1}\left(\frac{1+p}{2}\right)\right) - 1 & k \leq 0\\
                    2\left(1 - \Phi\left(e^k\Phi^{-1}(1 - \frac{p}{2})\right)\right) &k>0.
                \end{cases}
        \end{align}
    \end{customlem}
    \begin{proof}
        Note that $\varepsilon_0\sim\gN(0,\,\Sigma)$ and $\varepsilon_1=A\,\varepsilon_0\sim\gN(0,\,A^2\,\Sigma)$ where
        \begin{align}
            A = \begin{pmatrix}
                1 & 0\\
                0 & e^{-k}
                \end{pmatrix},
            \hspace{1em}
            \Sigma = \begin{pmatrix}
                \sigma^2 & 0\\
                0 & \tau^2
                \end{pmatrix}
        \end{align}
        and denote by $f_0$ and $f_1$ the probability density functions of $\varepsilon_0$ and $\varepsilon_1$,  respectively, and denote by $\Prob_0$ and $\Prob_1$ the corresponding probability measures.
        Recall the Definition of Lower Level sets (Definition~\ref{def:lower_level_sets}: for $t\geq0$, (strict) lower level sets are defined as
        \begin{equation}
            \begin{gathered}
                \underline{S_t} := \left\{z\in\gZ\colon\,\Lambda(z) < t \right\},
                \hspace{1em}\overline{S_t} := \left\{z\in\gZ\colon\,\Lambda(z) \leq t \right\},\\
                \text{where}\hspace{1em}\Lambda(z):=\frac{f_1(z)}{f_0(z)}.
            \end{gathered}
        \end{equation}
        Furthermore, recall that the function $\zeta$ is given by
        \begin{equation}
            t\mapsto\zeta(t):=\Prob_0\left(\overline{S}_{t}\right)
        \end{equation}
        where $\Prob_0$ is the distribution of $\varepsilon_0$ and note that the generalized inverse of $\zeta$ corresponds to $\tau_p$, i.e.,
        \begin{equation}
            \zeta^{-1}(p) = \inf\{t\geq 0\lvert\,\zeta(t) \geq p\} = \tau_p
        \end{equation}
        and the function $\xi$ is correspondingly given by
        \begin{equation}
            \xi(p) = \sup\{\Prob_1(S)\lvert\,\underline{S}(\zeta^{-1}(p)) \subseteq S \subseteq \overline{S}(\zeta^{-1}(p))\}.
        \end{equation}

        By assumption we know that $\E(p(y\lvert\,\phi(x,\,\varepsilon_0))) = q(y\lvert\,x;\,\varepsilon_o)\geq p$.
        Note that by Lemma~\ref{lem:sandwich}, for any $p\in[0,\,1]$ we have that
        \begin{equation}
            \Prob_0(\underline{S}_{\zeta^{-1}(p)}) \leq p.
        \end{equation}
        Let $S\subseteq\gZ$ be such that $\underline{S}_{\zeta^{-1}(p)}\subseteq S \subseteq \overline{S}_{\zeta^{-1}(p)}$ and $\Prob_0(S) \leq p$. Then, from part $(i)$ of Lemma~\ref{lem:np_extended}, it follows that $\E(p(y\lvert\,\phi(x,\,\varepsilon_1))) = q(y\lvert\,x;\,\varepsilon_1) \geq \Prob_1(S)$. Note that
        \begin{align}
            \Lambda(z) &= \frac{f_1(z)}{f_0(z)} \\
            &= \frac{\left((2\pi)^2|A^2\,\Sigma|\right)^{-\frac{1}{2}}\exp(-\frac{1}{2}(z^T(A^2\Sigma)^{-1}z))}{\left((2\pi)^2|\Sigma|\right)^{-\frac{1}{2}}\exp(-\frac{1}{2}(z^T(\Sigma)^{-1}z))}\\
            &=\dfrac{1}{|A|}\exp\left(-\frac{1}{2}z^T\,((A^2\,\Sigma)^{-1}  - \Sigma^{-1})\,z\right)\\
            &= \exp\left(k-\frac{z_2^2}{2\tau^2}\left(e^{2k} - 1\right)\right).
        \end{align}
        Note that, if $k=0$, then $\varepsilon_1=\varepsilon_0$ and hence the statement holds in this case.
        Suppose that $k > 0$ and consider
        \begin{align}
            \zeta(t) &= \Prob_0\left(\overline{S}_t\right) = \Prob\left(\exp\left(k-\frac{\varepsilon_{0,2}^2}{2\tau^2}\left(e^{2k} - 1\right)\right) \leq t\right)\\
            &= 1 - \Prob\left(\left(\frac{\varepsilon_{0,2}}{\tau}\right)^2 \leq 2\cdot\frac{k - \log(t)}{e^{2k} - 1}\right)\\
            &= 1 - F_{\chi^2}\left(2\cdot\frac{k - \log(t)}{e^{2k} - 1}\right)\\
            &=\begin{cases}
                0 \hspace{0.5em} &t = 0,\\
                1 - F_{\chi^2}\left(2\cdot\frac{k - \log(t)}{e^{2k} - 1}\right) & 0 < t < e^k,\\
                1 &t \geq e^k,
            \end{cases}
        \end{align}
        where $F_{\chi^2}$ denotes the CDF of the $\chi^2$-distribution with one degree of freedom. Note that for any $t\geq 0$ we have that $\Prob_0(\overline{S}_t) = \Prob_0(\underline{S}_t)$ and thus
        the inverse $\zeta^{-1}(p)=\inf\{t\geq 0\colon\,\zeta(t)\geq p\}$ is given by
        \begin{align}
            \zeta^{-1}(p) =
                \begin{cases}
                    0 & p= 0\\
                    \exp\left(k - F_{\chi^2}^{-1}(1 - p)\cdot\frac{e^{2k} - 1}{2}\right) & 0 < p < 1\\
                    e^k &p=1.
                \end{cases}
        \end{align}
        Thus, for any $p\in[0,\,1]$, we find that
        \begin{align}
            \Prob_0(\overline{S}_{\zeta^{-1}(p)})=\Prob_0(\underline{S}_{\zeta^{-1}(p)}) = \zeta(\zeta^{-1}(p)) = p
        \end{align}
        and
        \begin{align}
            \E_0(p(y\lvert\,\phi(x,\,\varepsilon_0))) = q(y\lvert\,x;\,\varepsilon_0)\geq p = \Prob_0(\overline{S}_{\zeta^{-1}(p)}).
        \end{align}
        Part $(i)$ of  Lemma~\ref{lem:np_extended} implies that $q(y\lvert\,x;\,\varepsilon_1) \geq \Prob_1(\overline{S}_{\zeta^{-1}(p)})$.
        Computing $\Prob_1(\overline{S}_{\zeta^{-1}(p)})$ yields
        \begin{align}
            q(y\lvert\,x;\,\varepsilon_1) &\geq \Prob_1(\overline{S}_{\zeta^{-1}(p)}) \\
            &\hspace{-2em}= 1 - \Prob\left(\left(\frac{\varepsilon_{1,2}}{\tau^2}\right)^2 \leq (k - \log(\zeta^{-1}(p)))\frac{2}{e^{2k} - 1}\right)\\
            &\hspace{-2em}= 1 - \Prob\left(\left(\dfrac{\varepsilon_{0,2}}{\tau^2}\right)^2 \leq (k - \log(\zeta^{-1}(p)))\dfrac{2e^{2k}}{e^{2k} - 1}\right)\\
            &\hspace{-2em}= 1 - F_{\chi^2}\left((k - \log(\zeta^{-1}(p)))\dfrac{2e^{2k}}{e^{2k} - 1}\right)\\
            &\hspace{-2em}= 1- F_{\chi^2}\left(\left(k - \left(k - \dfrac{e^{2k} - 1}{2}F_{\chi^2}^{-1}(1 - p)\right)\right)\dfrac{2e^{2k}}{e^{2k} - 1}\right)\\
            &\hspace{-2em}=1 - F_{\chi^2}\left(e^{2k}F_{\chi^2}^{-1}(1-p)\right).
        \end{align}
        If, on the other hand, $k < 0$, then
        \begin{align}
            \zeta(t) &= \Prob_0\left(\overline{S}_t\right)\\
            &= \Prob\left(\exp\left(k+\frac{\varepsilon_{0,2}^2}{2\tau^2}\left|e^{2k} - 1\right|\right) \leq t\right)\\
            &=\Prob\left(\left(\frac{\varepsilon_{0,2}}{\tau}\right)^2 \leq 2\cdot\frac{\log(t) - k}{\left|e^{2k} - 1\right|}\right)\\
            &= F_{\chi^2}\left(2\cdot\frac{\log(t) - k}{\left|e^{2k} - 1\right|}\right)\\
            &=
            \begin{cases}
                0  &t \leq e^k,\\
                F_{\chi^2}\left(2\cdot\frac{\log(t) - k}{\left|e^{2k} - 1\right|}\right) &t > e^k.
            \end{cases}
        \end{align}
        A similar computation as in the case where $k>0$ leads to an expression for the inverse $\zeta^{-1}(p)=\inf\{t\geq0\lvert\, \zeta(t)\geq p\}$
        \begin{align}
            \zeta^{-1}(p) =
                \begin{cases}
                    0 \hspace{0.5em}&p=0,\\
                    \exp\left(k + F^{-1}_{\chi^2}(p)\cdot\frac{\left|e^{2k} - 1\right|}{2}\right) &p > 0.
                \end{cases}
        \end{align}
        Thus, for any $p\in[0,\,1]$, we find that
        \begin{equation}
            \Prob_0(\overline{S}_{\zeta^{-1}(p)})=\Prob_0(\underline{S}_{\zeta^{-1}(p)}) = \zeta(\zeta^{-1}(p)) = p
        \end{equation}
        and
        \begin{equation}
            \E(p(y\lvert\,\phi(x,\,\varepsilon_0))) = q(y\lvert\,x;\,\varepsilon_0)\geq p = \Prob_0(\overline{S}_{\zeta^{-1}(p)}).
        \end{equation}
        Part $(i)$ of  Lemma~\ref{lem:np_extended} implies that $g_c^{\varepsilon_1}(x) \geq \Prob_1(\overline{S}_{\zeta^{-1}(p)})$.
        Computing $\Prob_1(\overline{S}_{\zeta^{-1}(p)})$ yields
        \begin{align}
             q(y\lvert\,x;\,\varepsilon_1) &\geq \Prob_1(\overline{S}_{\zeta^{-1}(p)})\\
             &= \Prob\left(\left(\frac{\varepsilon_{1,2}}{\tau}\right)^2 \leq 2\cdot\frac{\log(\zeta^{-1}(p)) - k}{|e^{2k} - 1|}\right)\\
             &=\Prob\left(\left(\frac{\varepsilon_{0,2}}{\tau}\right)^2 \leq 2e^{2k}\cdot\frac{\log(\zeta^{-1}(p)) - k}{|e^{2k} - 1|}\right)\\
             &=F_{\chi^2}\left(\left(\left(k + F^{-1}_{\chi^2}(p)\frac{|e^{2k} - 1|}{2}\right) - k\right)\frac{2\,e^{2k}}{|e^{2k} - 1|}\right)\\
             &=F_{\chi^2}\left(e^{2k}F_{\chi^2}^{-1}\left(p\right)\right).
        \end{align}
        Finally, note the following relation between the $\chi^2(1)$ and the standard normal distribution.
        Let $Z\sim\gN(0,1)$ and denote by $\Phi$ the CDF of $Z$.
        Then, for any $z\geq 0$, $F_{\chi^2}(z) = \Prob(Z^2 \leq z) = \Prob(-\sqrt{z} \leq Z \leq \sqrt{z}) = \Phi(\sqrt{z})-\Phi(-\sqrt{z}) = 2\Phi(\sqrt{z}) - 1$ and the inverse is thus given by $F_{\chi^2}^{-1}(p) = (\Phi^{-1}(\frac{1+p}{2}))^2$.
        It follows that
        \begin{align}
            q(y\lvert,\,x;\,\varepsilon_1)\geq
                \begin{cases}
                    2\Phi\left(e^k\Phi^{-1}\left(\frac{1+p}{2}\right)\right) - 1 & k \leq 0,\\
                    2\left(1 - \Phi\left(e^k\Phi^{-1}(1 - \frac{p}{2})\right)\right) &k>0,
                \end{cases}
        \end{align}
        what concludes the proof.
    \end{proof}

    The following lemma establishes another useful property of the distribution of $\varepsilon_1$.
    \begin{lem}
        \label{lem:bc_distribution_shift}
        Let $\varepsilon_0\sim\gN(0,\,\mathrm{diag}(\sigma^2,\,\tau^2))$, $\alpha=(k,\,b)^T\in\R^2$ and $\varepsilon_1\sim\gN(0,\,\mathrm{diag}(\sigma^2,\,e^{-2k}\tau^2))$. Then, for all $x\in\gX$, it holds that $g(\phi_{BC}(x,\,\alpha);\,\varepsilon_0) = g(x;\,\alpha + \varepsilon_1)$.
    \end{lem}
    \begin{proof}
        Let $x\in\gX$, and write $\varepsilon_i = (\varepsilon_{i,1},\,\varepsilon_{i,2})^T$ for $i=0,\,1$. Note that
        \begin{align}
            \phi_{BC}(\phi_{BC}(x,\,\alpha),\,\varepsilon_0)
            &= e^{\varepsilon_{0,1}}\left(\phi_{BC}(x,\,\alpha) + \varepsilon_{0,2}\right)
            = e^{\varepsilon_{0,1}}\left(e^{k}\left(x + b\right) + \varepsilon_{0,2}\right)\\
            &= e^{\varepsilon_{0,1} + k}\left(x + \left(b + e^{-k}\varepsilon_{0,2}\right)\right)
            = \phi_{BC}(x, \alpha + \Tilde{\varepsilon}_0)
        \end{align}
        where $\Tilde{\varepsilon}_0=(\varepsilon_{0,1},\,e^{-k}\varepsilon_{0,2})^T$. Note that  $\Tilde{\varepsilon}_0$ follows a Gaussian distribution since
        \begin{align}
            \Tilde{\varepsilon}_0 = A\cdot \varepsilon_0,\hspace{1em} A = \begin{pmatrix}1 & 0 \\ 0 & e^{-k}\end{pmatrix}
        \end{align}
        and hence $\E\left(\Tilde{\varepsilon}_0\right) = A\cdot\E\left(\varepsilon_0\right) = 0$ and
        \begin{align}
            \mathrm{Cov}\left(\Tilde{\varepsilon}_0\right) = \E\left(\varepsilon_0\,A\,A^T\,\varepsilon_0^T\right) = A^2\cdot \begin{pmatrix}\sigma^2 & 0 \\ 0 & \tau^2\end{pmatrix} = \begin{pmatrix}\sigma^2 & 0 \\ 0 & e^{-2k}\tau^2\end{pmatrix}.
        \end{align}
        The choice $\varepsilon_1 \equiv \Tilde{\varepsilon}_0\sim\gN(0,\,\mathrm{diag}(\sigma_1^2,\,e^{-2k}\sigma_2^2))$ shows that for any $y\in\gY$
        \begin{align}
            q(y\lvert\,\phi_{BC}(x,\,\alpha);\,\varepsilon_0) &= \E\left(p(y\lvert\,\phi(\phi(x,\,\alpha),\,\varepsilon_0)\right) \\
            &= \E\left(p(y\lvert\,\phi(x,\,\alpha+\varepsilon_1)\right)\\
            &= q(y\lvert\,x;\,\alpha+\varepsilon_1)
        \end{align}
        what concludes the proof.
    \end{proof}
    These observations, together with the Gaussian robustness bound from Corollary~\ref{cor:gaussian_noise} allow us to prove Lemma~\ref{lem:bc_certification}.
    \begin{customlem}{\ref{lem:bc_certification}}[restated]
        Let $\varepsilon_0$ and $\varepsilon_1$ be as in Lemma~\ref{lem:bc_lower_bound} and suppose that
        \begin{equation}
            q(y_A\lvert\,x;\,\varepsilon_1) \geq \tilde{p}_A > \tilde{p}_B \geq \max_{y\neq y_A}q(y\lvert\,x;\,\varepsilon_1).
        \end{equation}
        Then it is guaranteed that $y_A = g(\phi_{BC}(x,\,{\alpha});\,\varepsilon_0)$ as long as ${\alpha}=(k,\,b)^T$ satisfies
        \begin{equation}
            \sqrt{\left(\frac{k}{\sigma}\right)^2 + \left(\frac{b}{e^{-k}\tau}\right)^2} < \frac{1}{2}\left(\Phi^{-1}\left(\Tilde{p}_A\right)
            -\Phi^{-1}\left(\tilde{p}_B\right)\right).
        \end{equation}
    \end{customlem}
    \begin{proof}
        Since $\varepsilon_1\sim\gN(0,\,\mathrm{diag}(\sigma^2,\,e^{-2k}\tau^2))$, it follows from Corollary~\ref{cor:gaussian_noise} that whenever $\alpha=(k,\,b)^T$ satisfies
        \begin{equation}
            \sqrt{\left(\frac{k}{\sigma}\right)^2 + \left(\frac{b}{e^{-k}\tau}\right)^2} < \frac{1}{2}\left(\Phi^{-1}\left(\Tilde{p}_A\right)
            -\Phi^{-1}\left(\tilde{p}_B\right)\right),
        \end{equation}
        then it is guaranteed that $y_A = g(x;\,\varepsilon_1)$. The statement now directly follows from Lemma~\ref{lem:bc_distribution_shift}.
    \end{proof}

    \subsection{Gaussian Blur, Brightness, Contrast, and Translation}
        \label{adxsec:proofs-BTBC}
        Recall that the composition of Gaussian Blur, with brightness, contrast and translation is defined as
        \begin{equation}
            \phi_{BTBC}(x,\alpha) := \phi_B( \phi_T( \phi_{BC}(x, \alpha_k, \alpha_b), \alpha_{Tx}, \alpha_{Ty}), \alpha_B ),
        \end{equation}
        where $\phi_B$, $\phi_T$ and $\phi_{BC}$ are Gaussian blur, translation, and brightness and contrast transformations respectively as defined before and
        $\alpha := (\alpha_k, \alpha_b, \alpha_{Tx}, \alpha_{Ty}, \alpha_B)^T \in \R^4 \times \R_{\ge 0}$ is the transformation parameter.
        It is easy to see that this transformation composition satisfies the following properties:
        \begin{itemize}[leftmargin=*]
            \item \textbf{(P1)}~For arbitrary $\alpha^{(1)}, \alpha^{(2)} \in \R^4 \times \R_{\ge 0}$,
            \begin{align}
                \phi_{BTBC}(\phi_{BTBC}(x, \alpha^{(1)}), \alpha^{(2)})=\phi_{BTBC}(x, \alpha)
            \end{align}
            where
            \begin{align}
                \alpha &= \left( \alpha_k^{(1)} + \alpha_k^{(2)}, \alpha_b^{(1)} + \alpha_b^{(2)} / e^{\alpha_k^{(1)}}, \right. \nonumber \\
                & \left. \hspace{1em} \alpha_{Tx}^{(1)} + \alpha_{Tx}^{(2)}, \alpha_{Tx}^{(1)} + \alpha_{Tx}^{(2)}, \alpha_B^{(1)} + \alpha_B^{(2)} \right) .
            \end{align}

            \item \textbf{(P2)}~For an arbitrary $\alpha \in \R^4 \times \R_{\ge 0}$, define the parameterized operators:
            \begin{equation}
                \begin{gathered}
                    \phi_B^\alpha := \phi_B(\cdot; \alpha_B), \hspace{1em}
                    \phi_T^\alpha := \phi_T(\cdot; \alpha_{Tx}, \alpha_{Ty}),\\
                    \phi_{BC}^\alpha := \phi_{BC}(\cdot; \alpha_{k}, \alpha_b)
                \end{gathered}
            \end{equation}
            and let $\phi_1^\alpha, \phi_2^\alpha, \phi_3^\alpha$ be an arbitrary permutation of the above three operators. Then, we have that
            \begin{equation}
                \phi_{BTBC}(x,\alpha) = \phi_1^\alpha \circ \phi_2^\alpha \circ \phi_3^\alpha (x).
            \end{equation}
        \end{itemize}
        \hspace{0em}\\
        The property \textbf{(P1)} states that $\phi_{BTBC}$ is almost additive where the exception happens only on the brightness dimension~($\alpha_b$).
        The brightness dimension is subject to the same contrast effect implied and proved in Lemma~8~(in main paper).
        The property \textbf{(P2)} states that all the three transformations $\phi_B$, $\phi_T$, and $\phi_{BC}$ are commutative.
        The reason is that: (1)~$\phi_{BC}$ is a per-pixel color shift and independent of $\phi_B$ and $\phi_T$;
        (2)~$\phi_B$, Gaussian blur, relies on relative position of pixels and the translation with reflection padding, $\phi_T$, does not change it.

        Based on these two properties, we prove the key results as follows.
        \begin{customcor}{\ref{cor:BTBC_lb}}[restated]
            Let $x \in \gX$, $k\in\R$ and
            let $\epsilon_0 := (\epsilon_0^a, \epsilon_0^b)^T$ be a random variable defined as
            \begin{equation}
                \epsilon_0^a \sim \gN(0, \diag(\sigma_k^2,\,\sigma_b^2,\,\sigma_{T}^2,\,\sigma_{T}^2))
                \,\text{and}\,
                \epsilon_0^b \sim \Exp(\lambda_B).
            \end{equation}
            Similarly, let $\epsilon_1 := (\epsilon_1^a, \epsilon_1^b)$ be a random variable with
            \begin{equation}
                \epsilon_1^a \sim \gN(0, \diag(\sigma_k^2,\,e^{-2k}\sigma_b^2,\,\sigma_{T}^2,\,\sigma_{T}^2))
                \,\text{and}\,
                \epsilon_1^b \sim \Exp(\lambda_B).
            \end{equation}
            For either random variable~(denoted as $\epsilon$), recall that $q(y|x;\,\epsilon) := \E(p(y|\phi_{BTBC}(x,\,\epsilon)))$.
            Suppose that $q(y|x;\,\epsilon_0) \ge p$ for some $p\in [0,\,1]$ and $y\in \cY$. Then $q(y|x;\,\epsilon_1)$ satisfies Eq.~(11).
        \end{customcor}
        \begin{proof}
            According to the commutative property \textbf{(P2)}, we can view $q(y|x;\,\epsilon)$ as
            \begin{align}
                q(y|x,\,\epsilon) & = \E_\epsilon p(y|\phi_{BTBC}(x,\,\epsilon)) \\
                & \hspace{-2.4em} = \E_{\epsilon_k, \epsilon_b} \underbrace{\E_{\epsilon_{Tx},\epsilon_{Ty},\epsilon_B} p(y|\phi_{BC}(\phi_T(\phi_B(x,\epsilon_B),\epsilon_{Tx},\epsilon_{Ty}),\epsilon_k,\epsilon_b))}_{=:q'(y|x;\,\epsilon_k,\,\epsilon_b)}.
            \end{align}
            Notice that $q'(y|x;\,\epsilon_k,\,\epsilon_b)$ is a deterministic value in $[0,\,1]$.
            Its value is dependent on the distribtuion of $\epsilon_{Tx}, \epsilon_{Ty}, \epsilon_B$ and the underlying base classifier.
            Luckily, the random variables $\epsilon_0$ and $\epsilon_1$ have the same distribution over the components $\epsilon_{Tx}, \epsilon_{Ty}$ and $\epsilon_B$.
            Thus, they share the same $q'$ and we write $q(y|x;\,\epsilon_0)$ and $q(y|x;\,\epsilon_1)$ as
            \begin{align}
                q(y|x;\,\epsilon_0) &=  \underset{(\epsilon_k,\epsilon_b)\sim\gN(0,\diag(\sigma_k^2,\sigma_b^2))}{\E} q'(y|x;\,\epsilon_k,\epsilon_b), \\
                q(y|x;\,\epsilon_1) &=  \underset{(\epsilon_k,\epsilon_b)\sim\gN(0,\diag(\sigma_k^2,e^{-2k}\sigma_b^2))}{\E}  q'(y|x;\,\epsilon_k,\epsilon_b).
            \end{align}
            Now, we directly apply \Cref{lem:bc_lower_bound} and the desired lower bound for $q(y|x;\,\epsilon_1)$ follows.
        \end{proof}

        \begin{customlem}{\ref{lem:BTBC_bound}}[restated]
            Let $\epsilon_0$ and $\epsilon_1$ be as in \Cref{cor:BTBC_lb} and suppose that
            \begin{equation}
                q(y_A|x;\,\varepsilon_1) \ge \tilde p_A > \tilde p_B \ge \max_{y\neq y_A} q(y|x;\,\varepsilon_1).
            \end{equation}
            Then it is guaranteed that $y_A = g(\phi_{BTBC}(x,\,\alpha);\varepsilon_0)$
            as long as $p_A' > p_B'$,
            \begin{equation}
                p_A' = \left\{
                \begin{aligned}
                    &0, &\hspace{-5em} \text{if }\tilde p_A \le 1 - \exp(-\lambda_B \alpha_B), \\
                    &\begin{aligned}
                        &\Phi\left(\Phi^{-1}\left( 1 - (1 - \tilde p_A)\exp(\lambda_B \alpha_B) \right) \right.\\
                        &\hspace{1em} \left. - \sqrt{\nicefrac{\alpha_k^2}{\sigma_k^2} + \nicefrac{\alpha_b^2}{(e^{-2\alpha_k}\sigma_b^2)} + \nicefrac{(\alpha_{Tx}^2+\alpha_{Ty}^2)}{\sigma_{T}^2}}\right),
                    \end{aligned}&\text{otherwise}
                \end{aligned}
                \right.
                \label{eq:BTBC_bound_p_A'_pf}
            \end{equation}
            and
            \begin{equation}
                p_B' = \left\{
                \begin{aligned}
                    & 1, & \hspace{-4.5em} \text{if } \tilde p_B \ge \exp(-\lambda_B\alpha_B), \\
                    & \begin{aligned}
                        &1 - \Phi\left( \Phi^{-1}\left( 1 - \tilde p_B \exp(\lambda_B\alpha_B)\right) \right.\\
                        & \hspace{1em} \left. - \sqrt{\nicefrac{\alpha_k^2}{\sigma_k^2} + \nicefrac{\alpha_b^2}{(e^{-2\alpha_k}\sigma_b^2)} + \nicefrac{(\alpha_{Tx}^2+\alpha_{Ty}^2)}{\sigma_{T}^2}}\right).
                    \end{aligned}&\text{otherwise}
                \end{aligned}
                \right.
                \label{eq:BTBC_bound_p_B'_pf}
            \end{equation}
        \end{customlem}

        \begin{proof}
            We notice that for any $y\in\cY$,
            \begin{align}
                q(y|&\phi_{BTBC}(x,\,\alpha);\,\epsilon_0)\nonumber \\
                &=\E_{\epsilon_0} p(y|\phi_{BTBC}(\phi_{BTBC}(x,\,\alpha),\,\epsilon_0)) \\
                \begin{split}
                    &\overset{(a.)}{=} \E_{\epsilon_0} p\Big(y\big|x;\,\alpha + \\
                    &\hspace{3em}((\epsilon_0)_k, (\epsilon_0)_b / e^{\alpha_k}, (\epsilon_0)_{Tx}, (\epsilon_0)_{Ty}, (\epsilon_0)_B)^T \Big)
                \end{split}\\
                &\overset{(b.)}{=}E_{\epsilon_1} p\left(y|x; \alpha + \epsilon_1\right).
            \end{align}
            The step $(a.)$ uses the property \textbf{(P1)} of transformation $\phi_{BTBC}$,
            and the step $(b.)$ follows the definition of $\varepsilon_1$ in \Cref{cor:BTBC_lb}~(we define $k := \alpha_k$ hereinafter for simplicity).
            Thus, $g(\phi_{BTBC}(x,\,\alpha);\epsilon_0) = g(x;\alpha + \epsilon_1)$, and the robustness condition is equivalent to $g(x;\alpha + \epsilon_1) = g(x;\epsilon_1) = y_A$.

            According to \Cref{thm:main}, to prove the robustness, we only need to show that $\xi(\tilde p_A) + \xi(1 - \tilde p_B) > 1$ given $p_A' > p_B'$.
            Note that in the definition of $\xi$, the density functions $f_0$ and $f_1$ are for distributions of $\epsilon_1 \sim \Prob_0$ and $(\alpha + \epsilon_1) \sim \Prob_1$ respectively.

            In the proof below, we will compute the closed-form solution of $\xi(p)$ for any $0\le p\le 1$, and show that $\xi(\tilde p_A) + \xi(1 - \tilde p_B) > 1$ given $p_A' > p_B'$.
            To begin with, we write down $f_0$ and $f_1$.
            \begin{align}
                & f_0(z) =  \dfrac{\lambda_B}{(2\pi)^2 \sigma_k \sigma_b \sigma_T^2} \exp\left(
                -\lambda_B z_B - \nicefrac{(z_{Tx}^2+z_{Ty}^2)}{2\sigma_T^2} \right. \nonumber \\
                & \hspace{12.5em} \left. - \nicefrac{z_k^2}{2\sigma_k^2} - \nicefrac{z_b^2}{2e^{-2k}\sigma_b^2}
                \right), \\
                & f_1(z) = \left\{
                \begin{aligned}
                    \begin{split}
                        &\dfrac{\lambda_B \exp(\lambda_B\alpha_B)}{(2\pi)^2 \sigma_k \sigma_b \sigma_T^2} \exp\left(
                        -\lambda_B z_B \right.\\
                        &\hspace{2.5em} - \nicefrac{(z_{Tx}-\alpha_{Tx})^2}{2\sigma_T^2} - \nicefrac{(z_{Ty}-\alpha_{Ty})^2}{2\sigma_T^2} \\
                        &\hspace{2.5em} \left.  - \nicefrac{(z_k-\alpha_k)^2}{2\sigma_k^2} - \nicefrac{(z_b-\alpha_b)^2}{2e^{-2k}\sigma_b^2}
                        \right),
                    \end{split}& \text{if } z_B \ge \alpha_B,\\
                    & 0, & \text{otherwise}, \\
                \end{aligned}
                \right. \label{eq:BTBC_pf_f_1}
            \end{align}
            where $z = (z_k,z_b,z_{Tx},z_{Ty},z_B)^T \in \R^4 \times R_{\ge 0}$.
            As a result, function $\Lambda = f_1/f_0$ in \Cref{thm:main} writes as
            \begin{align}
                \Lambda(z) = \left\{
                \begin{aligned}
                    \begin{split}
                        &\exp\left(\lambda_B \alpha_B - \frac{\alpha_{Tx}^2}{2\sigma_T^2} - \frac{\alpha_{Ty}^2}{2\sigma_T^2} - \frac{\alpha_k^2}{2\sigma_k^2} - \frac{\alpha_b^2}{2e^{-2k}\sigma_b^2} \right. \\
                        & \hspace{0em} \left. + \frac{\alpha_{Tx} z_{Tx}}{\sigma_T^2} + \frac{\alpha_{Ty} z_{Ty}}{\sigma_T^2} + \frac{\alpha_k z_k}{\sigma_k^2} + \frac{\alpha_b z_b}{e^{-2k}\sigma_b^2}\right),
                    \end{split}&\text{if } z_B \ge \alpha_B,\\
                    & 0 & \text{otherwise}.
                \end{aligned}
                \right.
            \end{align}
            It turns out that for any $t > 0$,
            \begin{align}
                & \underline{S}_t = \{f_1/f_0 < t\} \nonumber \\
                = & \left\{(\hat z_k\sigma_k,\hat z_be^{-k}\sigma_b,\hat z_{Tx}\sigma_T,\hat z_{Ty}\sigma_T)^T \right. \nonumber \\
                & \hspace{1em} \,|\, \hat\alpha_{Tx}\hat z_{Tx} + \hat\alpha_{Ty}\hat z_{Ty} + \hat\alpha_k\hat z_k + \hat\alpha_b\hat z_b \nonumber \\
                & \hspace{1em} < \left. \ln t + \hat\alpha_{Tx}^2/2 + \hat\alpha_{Ty}^2/2 + \hat\alpha_k^2/2 + \hat\alpha_b^2/2 - \lambda_B \alpha_B \right\} \times [\alpha_B, +\infty) \nonumber \\
                & \cup \R^4 \times [0,\alpha_B), \\
                & \overline{S}_t = \{f_1/f_0 \le t\} \nonumber \\
                = & \left\{(\hat z_k\sigma_k,\hat z_be^{-k}\sigma_b,\hat z_{Tx}\sigma_T,\hat z_{Ty}\sigma_T)^T \right. \nonumber \\
                & \hspace{1em} \,|\, \hat\alpha_{Tx}\hat z_{Tx} + \hat\alpha_{Ty}\hat z_{Ty} + \hat\alpha_k\hat z_k + \hat\alpha_b\hat z_b \nonumber \\
                & \hspace{1em} \le \left. \ln t + \hat\alpha_{Tx}^2/2 + \hat\alpha_{Ty}^2/2 + \hat\alpha_k^2/2 + \hat\alpha_b^2/2 - \lambda_B \alpha_B \right\} \times [\alpha_B, +\infty) \nonumber \\
                & \cup \R^4 \times [0,\alpha_B),
            \end{align}
            where $\hat \alpha_{Tx} = \alpha_{Tx} / \sigma_T, \hat \alpha_{Ty} = \alpha_{Ty} / \sigma_T, \hat \alpha_k = \alpha_k / \sigma_k, \hat \alpha_b = \alpha_b / (e^{-k}\sigma_b)$.
            When $t = 0$, $\underline{S}_t = \emptyset$ and $\overline{S}_t = \R^4 \times [0,\alpha_B)$.
            Then, the probability integration shows that
            \begin{align}
                & \tau_p = \inf\{t \ge 0: \Prob_0(\overline{S}_t) \ge p \} \nonumber \\
                = &
                \left\{
                \begin{aligned}
                    & 0, \hspace{14em} \text{if } p \le 1 - \exp(-\lambda_B\alpha_B), \\
                    & \exp\left( \lambda_B \alpha_B + \|\hat \alpha_{:-1}\| \Phi^{-1}\left( 1 - \exp(\lambda_B \alpha_B) (1-p) \right) - 1/2 \|\hat\alpha_{:-1}\|^2 \right), \\
                    & \hspace{15em} \text{otherwise,}
                \end{aligned}
                \right.
                \label{eq:BTBC_pf_tau_p}
            \end{align}
            where $\|\hat \alpha_{:-1}\| = \sqrt{\hat\alpha_{Tx}^2 + \hat\alpha_{Ty}^2 + \hat\alpha_k^2 + \hat\alpha_b^2}$.
            Now we are ready to compute $\xi(p) = \sup\{ \Prob_1(S):\, \underline{S}_{\tau_p} \subset S \subset \overline{S}_{\tau_p} \}$.
            When $p \le 1 - \exp(-\lambda_B \alpha_B)$, we have $S \subset \R^4 \times [0,\alpha_B)$ and $\Prob_1(S) = 0$ because $\Prob_1$ has zero mass for any $z_B < \alpha_B$~(see (\ref{eq:BTBC_pf_f_1})).
            When $p > 1 - \exp(-\lambda_B \alpha_B)$, $\tau_p > 0$.
            Again, from probability integration, we get
            \begin{equation}
                \Prob_1(\underline{S}_{\tau_p}) = \Prob_1(\overline{S}_{\tau_p}) = \Phi\left( \frac{\ln \tau_p - \lambda_B \alpha_B}{\|\hat\alpha_{:-1}\|} - \frac{1}{2} \|\hat\alpha_{:-1}\| \right).
            \end{equation}
            We inject the closed-form solution of $\tau_p$ in (\ref{eq:BTBC_pf_tau_p}) and yield
            \begin{equation}
                \xi(p) = \Prob_1(S) = \Phi\left( \Phi^{-1}\left(1 - (1-p)\exp(\lambda_B \alpha_B)\right) - \|\hat\alpha_{:-1}\| \right)
            \end{equation}
            for any $S$ satisfying $\underline{S}_{\tau_p} \subset S \subset \overline{S}_{\tau_p}$.
            We summarize the above equations and write down the closed-form solution of $\xi(p)$ as such:
            \begin{equation}
                \xi(p) = \left\{
                \begin{aligned}
                    & 0, & \hspace{-10em} \text{if } p \le 1 - \exp(-\lambda_B \alpha_B), \\
                    & \Phi\left( \Phi^{-1}\left(1 - (1-p)\exp(\lambda_B \alpha_B)\right) - \|\hat\alpha_{:-1}\| \right), & \text{otherwise}.
                \end{aligned}
                \right.
            \end{equation}
            We can easily observe that $p_A'$ in lemma statement~(\ref{eq:BTBC_bound_p_A'_pf}) is indeed $\xi(\tilde p_A)$,
            and $p_B'$~(\ref{eq:BTBC_bound_p_B'_pf}) is indeed $1 - \xi(1 - \tilde p_B)$.
            Therefore,
            \begin{equation}
                \xi(\tilde p_A) + \xi(1 - \tilde p_B) > 1
                \iff
                p'_A > p'_B
            \end{equation}
            and using \Cref{thm:main} concludes the proof.
        \end{proof}

        Finally, we simplify the statement of \Cref{lem:BTBC_bound} with slight relaxation.
        \begin{customcor}{\ref{cor:BTBC_bound}}[restated]
            Let $\epsilon_0$ and $\epsilon_1$ be as in \Cref{cor:BTBC_lb} and suppose that
            \begin{equation}
                q(y_A|x;\,\varepsilon_1) \ge \tilde p_A.
            \end{equation}
            Then it is guaranteed that $y_A = g(\phi_{BTBC}(x,\,\alpha);\varepsilon_0)$ as long as
            \begin{equation}
                \tilde p_A > 1 - \exp(-\lambda_B \alpha_B) \left( 1 - \Phi\left(\sqrt{\frac{\alpha_k^2}{\sigma_k^2} + \frac{\alpha_b^2}{e^{-2\alpha_k}\sigma_b^2} + \frac{\alpha_{Tx}^2+\alpha_{Ty}^2}{\sigma_{T}^2}}\right) \right).
                \label{eq:cor-BTBC_bound_pf}
            \end{equation}
        \end{customcor}

        \begin{proof}
            Since $q(y_A|x;\,\varepsilon_1) \ge \tilde p_A$, according to the complement rule, $\max_{y\neq y_A} q(y|x;\,\varepsilon_1) < 1 - \tilde p_A =: \tilde p_B$.
            Inject the $\tilde p_A$ and $\tilde p_B$ into \Cref{lem:BTBC_bound} we find that $p_A' + p_B' = 1$ always hold.
            Therefore, $p_A' > 0.5$ guarantees that $p_A' > p_B'$ and thus the robustness.
            Indeed, from simple algebra,
            $
                p_A' > 0.5 \iff \text{(\ref{eq:cor-BTBC_bound_pf}}).
            $
        \end{proof}
    \fi

\section{Proofs for Differentially Resolvable Transformations}
    \label{apx:proofs_differentially_resolvable}

    \ifnum\ArXiv=0
    In the corresponding appendix of the arXiv version~\cite{arxiv},
    we provide proofs and technical details for theoretical results about certifying differentially resolvable transformations.
    \fi

    \ifnum\ArXiv=1
    Here we provide proofs and technical details for theoretical results about certifying differentially resolvable transformations. First, let us recall the definition of differentially resolvable transformations.
    \begin{customdef}{\ref{def:differentially_resolvable}}[restated]
        Let $\phi\colon\gX\times\gZ_\phi\to\gX$ be a transformation with noise space $\gZ_\phi$ and let $\psi\colon\gX\times\gZ_\psi\to\gX$ be a resolvable transformation with noise space $\gZ_\psi$. We say that $\phi$ can be resolved by $\psi$ if for any $x\in\gX$ there exists function $\delta_x\colon\gZ_\phi\times \gZ_\phi \to \gZ_\psi$ such that for any $\beta\in\gZ_\phi$
        \begin{equation}
            \phi(x,\,\alpha) = \psi(\phi(x,\,\beta),\,\delta_x(\alpha,\,\beta)).
        \end{equation}
    \end{customdef}
    \begin{customthm}{\ref{thm:main2}}[restated]
            Let $\phi\colon\gX\times\gZ_\phi\to\gX$ be a transformation that is resolved by $\psi\colon\gX\times\gZ_\psi\to\gX$. Let $\varepsilon\sim\Prob_\varepsilon$ be a $\gZ_\psi$-valued random variable and suppose that the smoothed classifier $g: \gX \to \gY$ given by $q(y\lvert\,x;\varepsilon) = \E(p(y\lvert\,\psi(x,\,\varepsilon)))$ predicts $g(x;\,\varepsilon) = y_A = \argmax_y q(y\lvert\,x;\varepsilon)$.
            Let $\gS\subseteq\gZ_\phi$ and $\{\alpha_i\}_{i=1}^N\subseteq\gS$ be a set of transformation parameters such that for any $i$, the class probabilities satisfy
            \begin{equation}
                q(y_A\lvert\,\phi(x,\,\alpha_i);\,\varepsilon) \geq p_A^{(i)}\geq p_B^{(i)} \geq \max_{y\neq y_A} q(y\lvert\,\phi(x,\,\alpha_i);\,\varepsilon).
            \end{equation}
            Then there exists a set $\Delta^*\subseteq\gZ_\psi$ with the property that, if for any $\alpha\in\gS,\,\exists\alpha_i$ with $\delta_x(\alpha,\,\alpha_i)\in\Delta^*$, then it is guaranteed that
            \begin{equation}
                q(y_A\lvert\,\phi(x,\,\alpha);\varepsilon) > \max_{y\neq y_A}q(y\lvert\,\phi(x,\,\alpha);\varepsilon).
            \end{equation}
    \end{customthm}
    \begin{proof}
        We prove the theorem by explicitly constructing a region $\Delta^*$ with the desired property by applying Theorem~\ref{thm:main}. For that purpose let $\delta\in\gZ_\psi$ and denote by $\gamma_\delta\colon\gZ_\psi\to\gZ_\psi$ the resolving function of $\psi$, i.e.,
        \begin{equation}
            \psi(\psi(x,\,\delta),\,\delta') = \psi(x,\,\gamma_\delta(\delta')).
        \end{equation}
        Let $\Prob_\gamma$ be the distribution of the random variable $\gamma:=\gamma_\delta(\varepsilon)$ with density function $f_\gamma$ and let
        \begin{gather}
            \underline{S}_t = \{z\in\gZ_\psi\colon\,\Lambda(z) < t\},\hspace{1em}
            \overline{S}_t = \{z\in\gZ_\psi\colon\,\Lambda(z) \leq t\}\\
            \mathrm{where}\hspace{1em}
            \Lambda(z) = \frac{f_\gamma(z)}{f_\varepsilon(z)}.
        \end{gather}
        Furthermore, recall the definition of the function $\zeta\colon\R_{\geq0}\to[0,\,1]$ that is given by $t\mapsto\zeta(t):=\Prob_\varepsilon(\overline{S}_t)$ with generalized inverse $\zeta^{-1}(p):=\inf\{t\geq0\colon\,\zeta(t) \geq p\}$. For $t\geq 0$ and the function $\xi\colon[0,\,1]\to[0,\,1]$ is given by by

        \begin{equation}
            \xi(p):=\sup\{\Prob_\gamma(S)\colon\,\underline{S}_{\zeta^{-1}(p)}\subseteq S \subseteq \overline{S}_\zeta^{-1}(p),\,\Prob_\varepsilon(S)\leq p\}.
        \end{equation}
        By assumption, for every $i=1,\,\ldots,\,n$, the $\varepsilon$-smoothed classifier $g$ is $(p_A^{(i)},\,p_B^{(i)})$-confident at $\phi(x,\,\alpha_i)$. Identify $\Delta_i\subseteq\gZ_\psi$ with the set of perturbations that satisfy the robustness condition~\eqref{eq:robustness_condition} in Theorem~\ref{thm:main}, i.e.,
        \begin{equation}
            \Delta_i\equiv \{\delta\in\gZ_\psi\colon\,1 - \xi(1-p_B^{(i)}) < \xi(p_A^{(i)})\}.
        \end{equation}
        Thus, by Theorem~\ref{thm:main}, we have that for any $\delta \in \Delta_i$
        \begin{equation}
            q(y_A\lvert\,\psi(\phi(x,\,\alpha_i),\,\delta);\,\varepsilon) > \max_{y\neq y_A}q(y\lvert\,\psi(\phi(x,\,\alpha_i),\,\delta);\,\varepsilon).
        \end{equation}
        Finally, note that for the set
        \begin{equation}
           \Delta^* \equiv \bigcap_{i=1}^N \Delta_i
       \end{equation}
       it holds that, if for $\alpha\in\gS$ there exists $\alpha_i$ with $\delta_x(\alpha,\,\alpha_i)\in\Delta^*$, then in particular $\delta_x(\alpha,\,\alpha_i)\in\Delta_i$ and hence, by Theorem~\ref{thm:main} it is guaranteed that
       \begin{align}
           q(y_A\lvert\,\phi(x,\,\alpha);\,\varepsilon) &= q(y_A\lvert\,\psi(\phi(x,\,\alpha_i),\,\delta_x(\alpha,\,\alpha_i));\,\varepsilon)\\
           &> \max_{y\neq y_A} q(y\lvert\,\psi(\phi(x,\,\alpha_i),\,\delta_x(\alpha,\,\alpha_i));\,\varepsilon)\\
           &= \max_{y\neq y_A} q(y\lvert\,\phi(x,\,\alpha);\varepsilon)
       \end{align}
       what concludes the proof.
    \end{proof}

    \begin{customcor}{\ref{cor:rotations_scaling_certificate}}[restated]
            Let $\psi(x,\,\delta) = x + \delta$ and let $\varepsilon\sim\gN(0,\,\sigma^2 \Id_d)$. Furthermore, let $\phi$ be a transformation with parameters in $\gZ_\phi\subseteq\R^m$ and let $\gS\subseteq\gZ_\phi$ and $\{\alpha_i\}_{i=1}^N\subseteq\gS$. Let $y_A\in\cY$ and suppose that for any $i$, the $\varepsilon$-smoothed classifier defined by $q(y\lvert\,x;\varepsilon):=\E(p(y\lvert\,x+\varepsilon))$ has class probabilities that satisfy
            \begin{equation}
                 q(y_A\lvert\,\phi(x,\,\alpha_i);\,\varepsilon) \geq p_A^{(i)} \geq p_B^{(i)} \geq \max_{y\neq y_A}q(y\lvert\,\phi(x,\,\alpha_i);\,\varepsilon).
             \end{equation}
            Then it is guaranteed that $\forall\alpha\in\gS\colon\,y_A = \argmax_y q(y\lvert\,\phi(x,\,\alpha);\varepsilon)$ if the maximum interpolation error
            \begin{equation}
                M_{\gS}:=\max_{\alpha\in\gS}\min_{1\leq i \leq N}\left\|\phi(x,\,\alpha) - \phi(x,\,\alpha_i)\right\|_2
            \end{equation}
            satisfies
            \vspace{-0.7em}
            \begin{equation}
                \begin{gathered}
                    M_{\gS} < R:=\frac{\sigma}{2}\min_{1\leq i \leq N}\left(\Phi^{-1}\left(p_A^{(i)}\right)-\Phi^{-1}\left(p_B^{(i)}\right)\right).
                \end{gathered}
            \end{equation}
    \end{customcor}
    \begin{proof}
        Since the resolvable transformation $\psi$ is given by $\psi(x,\,\delta) = x + \delta$ we can write
        \begin{equation}
            \phi(x,\,\alpha) = \phi(x,\,\alpha_i) + \underbrace{(\phi(x,\,\alpha) - \phi(x,\,\alpha_i))}_{=:\delta_x(\alpha,\,\alpha_i)}.
        \end{equation}
        Furthermore, by assumption $\varepsilon\sim\gN(0,\,\sigma^2\Id_d)$ and $g(\cdot;\,\varepsilon)$ is $(p_A^{(i)},\,p_B^{(i)})$-confident at $\phi(x,\,\alpha_i)$ for $y_A$ and for all $i$. Thus, by Corollary~\ref{cor:gaussian_noise}, if $\delta$ satisfies
        \begin{equation}
            \left\|\delta\right\|_2 < R_i := \frac{\sigma}{2}\left(\Phi^{-1}\left(p_A^{(i)}\right)-\Phi^{-1}\left(p_B^{(i)}\right)\right)
        \end{equation}
        then it is guaranteed that $y_A = \argmax_y q(y\lvert\,\phi(x,\,\alpha_i) + \delta;\,\varepsilon)$.
        Let $\Delta_i:=B_{R_i}(0)$ and notice that $R \equiv \min_i R_i$ and thus
        \begin{equation}
            \bigcap_{i=1}^N B_{R_i}(0) = B_R(0) = \Delta^*.
        \end{equation}
        To see that $\Delta^*$ has the desired property, consider
        \begin{gather}
            \forall\alpha\in\gS\,\exists\alpha_i\colon\,\delta_x(\alpha,\,\alpha_i) \in\Delta^*\\
            \iff\,\forall\alpha\in\gS\,\exists\alpha_i\colon\,\left\|\phi(x,\,\alpha) - \phi(x,\,\alpha_i)\right\|_2 < R.
        \end{gather}
        Since $R \leq R_i$ it follows that for $\delta_i = \phi(x,\,\alpha) - \phi(x,\,\alpha_i)$ it is guaranteed that
        \begin{align}
            y_A &= \argmax_y q(y\lvert\,\phi(x,\,\alpha_i) + \delta_i;\,\varepsilon)\\
            &=\argmax_y q(y\lvert\,\phi(x,\,\alpha);\,\varepsilon).
        \end{align}
        Thus, the set $\Delta^*$ has the desired property. In particular, since
        \begin{gather}
            \forall\alpha\in\gS\,\exists\alpha_i\colon\,\left\|\phi(x,\,\alpha) - \phi(x,\,\alpha_i)\right\|_2 < R\\
            \iff \max_{\alpha\in\gS}\min_{1\leq i \leq N}\left\|\phi(x,\,\alpha) - \phi(x,\,\alpha_i)\right\|_2 < R
        \end{gather}
        the statement follows.
    \end{proof}

    \begin{customcor}{\ref{cor:rotations_scaling_brightness_certificate}}[restated]
        Let $\psi_B(x,\,\delta,\,b) = x + \delta + b\cdot \Id_d$ and let $\varepsilon\sim\gN(0,\,\sigma^2 \Id_d)$, $\varepsilon_b\sim\gN(0,\,\sigma_b^2)$. Furthermore, let $\phi$ be a transformation with parameters in $\gZ_\phi\subseteq\R^m$ and let $\gS\subseteq\gZ_\phi$ and $\{\alpha_i\}_{i=1}^N\subseteq\gS$. Let $y_A\in\cY$ and suppose that for any $i$, the $(\varepsilon,\varepsilon_b)$-smoothed classifier $q(y\lvert\,x;\,\varepsilon,\varepsilon_b):=\E(p(y\lvert\,\psi_B(x,\,\varepsilon,\,\varepsilon_b))$ satisfies
        \begin{equation}
            q(y_A\lvert\,x;\,\varepsilon,\varepsilon_b) \geq p_A^{(i)} > p_B^{(i)} \geq \max_{y\neq y_A}q(y\lvert\,x;\,\varepsilon,\varepsilon_b).
        \end{equation}
        for each $i$. Let
        \begin{equation}
            R := \frac{\sigma}{2}\min_{1\leq i \leq N}\left(\Phi^{-1}\left(p_A^{(i)}\right)-\Phi^{-1}\left(p_B^{(i)}\right)\right)
        \end{equation}
        Then, $\forall \alpha \in \gS$ and $\forall b \in [-b_0, b_0]$ it is guaranteed that $y_A = \argmax_y q(y\lvert\,\phi(x, \alpha) + b\cdot\Id_d;\, \varepsilon, \varepsilon_b)$ as long as
        \begin{equation}
            R > \sqrt{M_{\gS}^2 + \dfrac{\sigma^2}{\sigma_b^2}b_0^2},
            \label{eq:general_Ms_brightness}
        \end{equation}
        where $M_\gS$ is defined as in \Cref{cor:rotations_scaling_certificate}.
    \end{customcor}

    \begin{proof}
        Since the resolvable transformation $\psi_B$ is given by
        \begin{equation}
            \psi_B(x,\,\delta,\,b) = x + \delta + b \cdot \Id_d,
        \end{equation}
        we can write
        \begin{equation}
            \phi(x,\,\alpha) + b \cdot \Id_d = \phi(x,\,\alpha_i) + \underbrace{\left(\phi(x,\,\alpha) - \phi(x,\,\alpha_i)\right) + b \cdot \Id_d}_{=: \delta_x((\alpha, b),\,(\alpha_i,0))}.
        \end{equation}
        Furthermore, by assumption $\varepsilon \sim \gN(0,\sigma^2\Id_d)$, $\varepsilon_b \sim \gN(0,\sigma_b^2)$ and $g(\cdot;\,\varepsilon, \varepsilon_b)$ is $(p_A^{(i)}, p_B^{(i)})$-confident at $\phi(x,\alpha_i)$ for $y_A$ and all $i$.
        Thus, by Corollary~\ref{cor:gaussian_noise}, if $\delta$ and $b$ satisfy
        \begin{equation}
            \label{eq:cor_compisition1_vanilla_condition}
            \sqrt{\dfrac{\|\delta\|_2^2}{\sigma^2} + \dfrac{b^2}{\sigma_b^2}} < \frac{1}{2} \left( \Phi^{-1}(p_A^{(i)}) - \Phi^{-1}(p_B^{(i)}) \right),
        \end{equation}
        then it is guaranteed that
        \begin{equation}
            y_A = \argmax_y q(y| \phi(x,\alpha_i) + \delta + b \cdot \Id_d;\, \varepsilon, \varepsilon_b).
        \end{equation}
        Let
        \begin{equation}
            R_i:=\frac{\sigma}{2} \left(\Phi^{-1}(p_A^{(i)}) - \Phi^{-1}(p_B^{(i)})\right)
        \end{equation}
        and note that without loss of generality we can assume that $R_i > {\sigma}/{\sigma_b}b_0$, because otherwise the robustness condition is violated.
        Rearranging terms in~\eqref{eq:cor_compisition1_vanilla_condition} leads to the condition
        \begin{equation}
            \|\delta\|_2 < \sqrt{R_i^2 - \frac{\sigma^2}{\sigma_b^2} b^2}
        \end{equation}
        that can be turned into a sufficient robustness condition holding for any $b\in[-b_0,\,b_0]$ simultaneously
        \begin{equation}
            \|\delta\|_2 < \sqrt{R_i^2 - \frac{\sigma^2}{\sigma_b^2} b_0^2}
        \end{equation}
        Note that, without loss of generality
        For each $i$ let $\Delta_i$ be the set defined as
        \begin{equation}
            \Delta_i := \left\{\delta + b\cdot \mathbf{1}_d \in\R^d\colon\,\|\delta\|_2 < \sqrt{R_i^2 - \frac{\sigma^2}{\sigma_b^2} b_0^2},\,\left|b\right| \leq b_0\right\}
        \end{equation}
        and note that
        \begin{align}
            \Delta^*&:=\bigcap_{i=1}^N \Delta_i\\
            &= \left\{\delta + b\cdot\mathbf{1}_d\in\R^d\colon\,\|\delta\|_2 < \sqrt{R^2 - \frac{\sigma^2}{\sigma_b^2} b_0^2},\,\left|b\right| \leq b_0\right\}
        \end{align}
        with $R:=\min_i R_i$. Clearly, if $\forall \alpha\in \gS,\,\forall b \in [-b_0,\,b_0]\,\exists i$ such that
        \begin{equation}
            \delta_x((\alpha, b),\,(\alpha_i,0)) \in \Delta^*
        \end{equation}
        then it is guaranteed that
        \begin{align}
            y_A &= \argmax_y q(y\lvert\,\phi(x,\alpha_i) + \delta_x((\alpha, b),\,(\alpha_i,0)))\\
            &=\argmax_y q(y\lvert\,\phi(x,\,\alpha) + b \cdot \Id_d).
        \end{align}
        We can thus reformulate the robustness condition as
        \begin{gather}
            \forall \alpha\in \gS,\,\forall b \in [-b_0,\,b_0]\,\exists i \,\mathrm{\,s.t.\,}\,\delta_x((\alpha, b),\,(\alpha_i,0)) \in \Delta^*\\
            \iff\nonumber\\
            \begin{gathered}
                \forall \alpha\in \gS,\,\forall b \in [-b_0,\,b_0]\,\exists i \,\mathrm{\,s.t.\,}\\\,\|\phi(x,\,\alpha) - \phi(x,\,\alpha_i)\|_2 < \sqrt{R^2 - \frac{\sigma^2}{\sigma_b^2} b_0^2}
            \end{gathered}\\
            \iff\nonumber\\
            \max_{\alpha\in\gS}\min_{1\leq i\leq N}\|\phi(x,\,\alpha) - \phi(x,\,\alpha_i)\|_2 < \sqrt{R^2 - \frac{\sigma^2}{\sigma_b^2} b_0^2}
        \end{gather}
        that, written in terms of the maximum $\ell_2$ interpolation error $M_\gS$, is equivalent to
        \begin{equation}
            R > \sqrt{M_{\gS}^2 + \dfrac{\sigma^2}{\sigma_b^2}b_0^2}
        \end{equation}
        what concludes the proof.
    \end{proof}

    \begin{customcor}{\ref{cor:rotations_scaling_brightness_l2_certificate}}[restated]
        Under the same setting as in~\Cref{cor:rotations_scaling_brightness_certificate},
        for $\forall \alpha \in \gS$, $\forall b \in [-b_0, b_0]$ and $\forall \delta\in\R^d$ such that $\|\delta\|_2 \le r$, it is guaranteed that $y_A = \argmax_k q(y\lvert\,\phi(x,\alpha) + b\cdot\Id_d + \delta;\,\varepsilon,\varepsilon_d)$
        as long as
        \begin{equation}
            \label{eq:general_Ms_brightness_l2}
            R > \sqrt{(M_{\gS}+r)^2 + \dfrac{\sigma^2}{\sigma_b^2}b_0^2},
        \end{equation}
        where $M_\gS$ is defined as in \Cref{cor:rotations_scaling_certificate}.
    \end{customcor}

    \begin{proof}
        Note that we can write the transformed input as
        \begin{align}
            \begin{split}
                \phi(x,\alpha) + & b\cdot \Id_d + \delta \\
                    = & \phi(x,\alpha_i) +
                    \underbrace{
                    (\phi(x,\alpha) - \phi(x,\alpha_i) + \delta) + b\cdot \Id_d
                    }_{=: \delta_x((\alpha,b,\delta),\,(\alpha_i,0,0))}.
            \end{split}
        \end{align}

        Since we use the same smoothing protocol as in \Cref{cor:rotations_scaling_brightness_certificate},
        the general proof idea is similar to \Cref{cor:rotations_scaling_brightness_certificate} ---
        we use the same resolvable transformation $\psi_B$ and define the same set $\Delta_i$, namely
        \begin{equation}
            \begin{split}
                \Delta_i &:= \bigg\{\delta' + b\cdot \mathbf{1}_d + \delta \in\R^d\colon\\
                &\hspace{1em}\,\|\delta'+ \delta\|_2 < \sqrt{R_i^2 - \frac{\sigma^2}{\sigma_b^2} b_0^2},\,\left|b\right| \leq b_0,\,\left\|\delta\right\|_2 \leq r\bigg\}.
            \end{split}
        \end{equation}
        and set
        \begin{align}
            \Delta^*&:=\bigcap_{i=1}^N \Delta_i\\
            \begin{split}
                &= \bigg\{\delta' + b\cdot\mathbf{1}_d + \delta \in\R^d\colon\\
                &\hspace{1em}\,\|\delta' + \delta\|_2 < \sqrt{R^2 - \frac{\sigma^2}{\sigma_b^2} b_0^2},\,\left|b\right| \leq b_0,\,\left\|\delta\right\|_2 \leq r\bigg\}
            \end{split}
        \end{align}
        with $R:=\min_i R_i$. Clearly, if $\forall \alpha\in \gS,\,\forall b \in [-b_0,\,b_0],\,\left\|\delta\right\|_2 \leq r\,\exists i$ such that
        \begin{equation}
            \delta_x((\alpha, b,\,\delta),\,(\alpha_i,0)) \in \Delta^*
        \end{equation}
        then it is guaranteed that
        \begin{align}
            y_A &= \argmax_y q(y\lvert\,\phi(x,\alpha_i) + \delta_x((\alpha, b),\,(\alpha_i,0)))\\
            &=\argmax_y q(y\lvert\,\phi(x,\,\alpha) + b \cdot \Id_d).
        \end{align}
        We can thus reformulate the robustness condition as
        \begin{gather}
            \begin{gathered}
                \forall \alpha\in \gS,\,\forall b \in [-b_0,\,b_0],\,\left\|\delta\right\|_2 \leq r\,\exists i\\
                \mathrm{\,s.t.\,}\,\delta_x((\alpha, b,\,\delta),\,(\alpha_i,\,0,\,0)) \in \Delta^*
            \end{gathered}\\
            \iff\nonumber\\
            \begin{gathered}
                \forall \alpha\in \gS,\,\forall b \in [-b_0,\,b_0],\,\left\|\delta\right\|_2 \leq r\,\exists i \,\mathrm{\,s.t.\,}\\\,\|\phi(x,\,\alpha) - \phi(x,\,\alpha_i) + \delta\|_2 < \sqrt{R^2 - \frac{\sigma^2}{\sigma_b^2} b_0^2}
            \end{gathered}\\
            \iff\nonumber\\
            \max_{\alpha\in\gS}\min_{1\leq i\leq N}\|\phi(x,\,\alpha) - \phi(x,\,\alpha_i) + \delta\|_2 < \sqrt{R^2 - \frac{\sigma^2}{\sigma_b^2} b_0^2}.
        \end{gather}
        Note that by the triangle inequality have
        \begin{align}
            \max_{\alpha\in\gS}\min_{1\leq i\leq N}\|\phi(x,\,\alpha) - \phi(x,\,\alpha_i) + \delta\|_2 &\leq M_\gS + \|\delta\|_2\\
            &\leq M_\gS + r
        \end{align}
        and thus, robustness is implied by
        \begin{equation}
            R > \sqrt{(M_{\gS} + r)^2 + \dfrac{\sigma^2}{\sigma_b^2}b_0^2}
        \end{equation}
        what concludes the proof.
    \end{proof}
    \fi

\section{Transformation Details}
\label{sec:apx_transformation_details}

\ifnum\ArXiv=0
    Due to space limit, we provide detailed definitions of rotation and scaling transformation in the corresponding appendix of the arXiv version~\cite{arxiv}.
\fi

\ifnum\ArXiv=1
In this section, we provide detailed definitions of rotation and scaling transformation.
\subsection{Bilinear Interpolation}
    Let $\Omega_K:=\{0,\,\ldots,\,K-1\}$ and $\Omega:=[0,\,W-1]\times[0,\,H-1]$. We define bilinear interpolation to be the map $Q\colon\R^{K\times W\times H}\to L^2(\Omega_K\times\R^2,\,\R),\,x\mapsto Q(x)=: Q_x$ where $Q_x$ is given by
    \begin{equation}
        \label{eq:bilinear_interpolation}
        \begin{aligned}
        (k,\,i,\,j) \mapsto Q_x(k,\,i,\,j):=
            \begin{cases}
                0\hspace{0.5em}&(i,\,j)\notin \Omega\\
                x_{k,i,j} & (i,\,j)\in \Omega\cap\N^2\\
                \Tilde{x}_{k,\,i,\,j} & (i,\,j)\in \Omega\setminus\N^2.
            \end{cases}
        \end{aligned}
    \end{equation}
    and where
    \begin{equation}
        \begin{aligned}
            \Tilde{x}_{k,i,j}:=&\left(1-(i-\floor{i})\right)\cdot\left(\left(1-(j-\floor{j})\right)\cdot x_{k,\floor{i},\floor{j}}\right.\\
             & \hspace{2em} \left. + (j-\floor{j})\cdot x_{k,\floor{i},\floor{j}+1}\right)\\
             &+ (i-\floor{i})\cdot\left(\left(1-(j-\floor{j})\right)\cdot x_{k,\floor{i}+1,\floor{j}} \right. \\
             & \hspace{2em} + \left. (j-\floor{j})\cdot x_{k,\floor{i}+1,\floor{j}+1}\right).
        \end{aligned}
    \end{equation}
\subsection{Rotation}
    \label{appendix:rotation_details}
    The rotation transformation is denoted as $\phi_R\colon\R^{K\times W\times H}\times\R\to \R^{K\times W\times H}$ and acts on an image in three steps that we will highlight in greater detail. First, it rotates the image by $\alpha$ degrees counter-clockwise. After rotation, pixel values are determined using bilinear interpolation~(\ref{eq:bilinear_interpolation}). Finally, we apply black padding to all pixels $(i,\,j)$ whose $\ell_2$-distance to the center pixel is larger than half of the length of the shorter side, and denote this operation by $P$. Let $c_W$ and $c_H$ be the center pixels
    \begin{equation}
        c_W := \dfrac{W-1}{2},\hspace{2em} c_H := \dfrac{H-1}{2}.
    \end{equation}
    and
    \begin{equation}
        \begin{aligned}
        d_{i,j} &= \sqrt{\left(i - c_W\right)^2 + \left(j - c_H\right)^2},\\
        g_{i,j} &= \atan2\left(j - c_H,\,i - c_W\right).
    \end{aligned}
    \end{equation}
    We write $\Tilde{\phi}_R$ for the rotation transformation before black padding and decompose $\phi_R$ as $\phi_R=P\circ \Tilde{\phi}_R$, where $\Tilde{\phi}_R\colon\R^{K\times W\times H}\times\R\to \R^{K\times W\times H}$ is defined by
    \begin{equation}
        \begin{aligned}
            \Tilde{\phi}_R(x,\,\alpha)_{k,i,j}&:=Q_x(k,\,c_W + d_{i,j}\,\cos(g_{i,j} - \alpha),\\
            &\hspace{6em}\,c_H + d_{i,j}\,\sin(g_{i,j} - \alpha))
        \end{aligned}
    \end{equation}
    and $P\colon \R^{K\times W\times H}\to \R^{K\times W\times H}$ by
    \begin{equation}
        \label{eq:black_padding}
        f\mapsto P(f)_{k,i,j} =
            \begin{cases}
                f(k,\,i,\,j) \hspace{0.5em} & d_{i,j} < \min\left\{c_W,\,c_H\right\}\\
                0 &\mathrm{otherwise}
            \end{cases}.
    \end{equation}

        The rotation transformation in practice may use different padding mechanisms.
        For example, the rotation in the physical world may fill in boundary pixels with real elements captured by the camera.
        We remark that our \framework against the transformation $\phi_R$ implies the defense against rotation with \emph{any other} padding mechanisms, because we first apply black-padding $P$ to any such rotated input and then feed into \framework models so that \framework models always receive black-padded inputs.

\subsection{Scaling}
    \label{appendix:scaling_details}
    The scaling transformation is denoted as $\phi_S\colon\R^{K\times W\times H}\times\R\to \R^{K\times W\times H}$. Similar as for rotations, $\phi_S$ acts on an image in three steps. First, it stretches height and width by a fixed ratio $\alpha\in\R$. Second, we determine missing pixel values with bilinear interpolation. Finally, we apply black padding to regions with missing pixel values if the image is scaled by a factor smaller than 1. Let $c_W$ and $c_H$ be the center pixels
    \begin{equation}
        c_W := \dfrac{W-1}{2},\hspace{2em} c_H := \dfrac{H-1}{2}.
    \end{equation}
    We notice that black padding is naturally applied during bilinear interpolation in cases where the scaling factor is smaller than 1 (that is, when we make images smaller). We can thus write the scaling operation as $\phi_S\colon\R^{K\times W\times H}\times\R_{>0}\to \R^{K\times W\times H},\,(x,\,\alpha)\mapsto\phi(x,\,\alpha)$ where
    \begin{equation}
        \phi_S(x,\,\alpha)_{k,i,j}:=Q_x\left(k,\,c_W + \frac{i - c_W}{\alpha},\,c_H + \frac{j - c_H}{\alpha}\right).
    \end{equation}

        When the scaling transformation in practice uses different padding mechanisms, we can simply apply black padding to the outer pixels during preprocessing.
        For example, if we know the semantic attacker could choose $0.7$ as the smallest scaling ratio, we can apply black padding to all pixels that are out of canvas after $0.7$ scaling.
        Therefore, we overwrite all different padding mechanisms and ensure the generalizability.
        As a trade-off, the classifier has a narrower reception field that affects the clean accuracy.
\fi

\section{Proofs for Interpolation Bound Computation}
    \label{sec:apx_proofs_interpolation_bound}

    \ifnum\ArXiv=0
        We state the proofs for the theoretical results governing our approach to certifying rotations and scaling transformations using randomized smoothing in the corresponding appendix of the arXiv version~\cite{arxiv}.
        As a quick glance, the following auxiliary lemma is used for both rotation and scaling:
    \begin{lem}
        \label{lem:lipschitz_constant}
        Let $x\in \R^{K\times W \times H}$, $-\infty < t_1<t_2 < \infty$  and suppose $\rho\colon[t_1,\,t_2]\to [0,\,W-1]\times[0,\,H-1]$ is a curve of class $C^1$. Let
        \begin{align}
            \psi_k\colon[t_1,\,t_2] \to \R,\hspace{0.5em} \psi_k(t):=Q_x(k,\,\rho_1(t),\,\rho_2(t))
        \end{align}
        where $k\in\Omega_K$ and $Q_x$ denotes bilinear interpolation. Then $\psi_k$ is $L_k$-Lipschitz continuous with constant
        \begin{align}
            L_k = \max_{t\in[t_1,t_2]}\left(\sqrt{2}\left\|\dot\rho(t)\right\|_2\cdot m_\Delta(x,\,k,\,\floor{\rho(t)})\right)
        \end{align}
    \end{lem}
    \fi

    \ifnum\ArXiv=1
    In this section we state the proofs for the theoretical results governing our approach to certifying rotations and scaling transformations using randomized smoothing. We first define the maximum $\ell_2$ interpolation error. First, let us recall the following definitions from the main part of this paper.
    \begin{defn}[$\ell_2$ interpolation error]
        Let $x\in\gX$, $\phi\colon\gX\times\gZ\to\gX$ a transformation, $\gS=[a,\,b]$, $N\in\N$ and suppose $\{\alpha_i\}_{i=1}^N\subseteq\gS$. The maximum $\ell_2$ interpolation error is defined as
        \begin{equation}
            M_{\gS}:=\max_{a\leq \alpha \leq b}\,\min_{1\leq i \leq N}\|\phi(x,\,\alpha) - \phi(x,\,\alpha_i) \|_2.
        \end{equation}
    \end{defn}
    \begin{customdef}{\ref{def:grid_pixel_generator}}[restated]
        For pixels $(i,\,j)\in\Omega$, we define the grid pixel generator $G_{ij}$ as
        \begin{equation}
           G_{ij}:=\{(i,\,j),\,(i+1,\,j),\,(i,\,j+1),\,(i+1,\,j+1)\}.
        \end{equation}
    \end{customdef}
    \begin{customdef}{\ref{def:max-color-extractor}}[restated]
        We define the operator that extracts the channel-wise maximum pixel wise on a grid $S\subseteq\Omega$ as the map $\overline{m}\colon\R^{K\times W \times H}\times \{0,\ldots,K-1\} \times 2^{\Omega}\to \R$ with
        \begin{equation}
            \begin{aligned}
                \overline{m}(x,\,k,\,S):= \max_{(i,j)\in S}\left(\max_{(r,s)\in G_{ij}} x_{k,r,s}\right)
            \end{aligned}
        \end{equation}
    \end{customdef}
    \begin{customdef}{\ref{def:max-color-difference-extractor}}[restated]
        We define the operator that extracts the channel-wise maximum change in color on a grid $S\subseteq\Omega$ as the map $m_\Delta\colon\R^{K\times W \times H}\times \{0,\ldots,K-1\} \times 2^{\Omega}\to \R$ with
        \begin{equation}
            m_\Delta(x,\,k,\,S):= \max_{(i,j)\in S}\left(\max_{(r,s)\in G_{ij}} x_{k,r,s}-\min_{(r,s)\in G_{ij}} x_{k,r,s}\right)
        \end{equation}
    \end{customdef}
    The following auxiliary lemma is used for both rotation and scaling:
    \begin{lem}
        \label{lem:lipschitz_constant}
        Let $x\in \R^{K\times W \times H}$, $-\infty < t_1<t_2 < \infty$  and suppose $\rho\colon[t_1,\,t_2]\to [0,\,W-1]\times[0,\,H-1]$ is a curve of class $C^1$. Let
        \begin{align}
            \psi_k\colon[t_1,\,t_2] \to \R,\hspace{0.5em} \psi_k(t):=Q_x(k,\,\rho_1(t),\,\rho_2(t))
        \end{align}
        where $k\in\Omega_K$ and $Q_x$ denotes bilinear interpolation. Then $\psi_k$ is $L_k$-Lipschitz continuous with constant
        \begin{align}
            L_k = \max_{t\in[t_1,t_2]}\left(\sqrt{2}\left\|\dot\rho(t)\right\|_2\cdot m_\Delta(x,\,k,\,\floor{\rho(t)})\right)
        \end{align}
    \end{lem}
    \begin{proof}
        Note that the function $t\mapsto\floor{\rho(t)}$ is piecewise constant and let $t_1=:u_1<u_2<\ldots<u_{N_0}:=t_2$ such that $\floor{\rho(t)}$ is constant on $\left[u_i,\,u_{i+1}\right)$ for all $1\leq i\leq N_0-1$ and $\dot\cup_{i=1}^{N_0}\left[u_i,\,u_{i+1}\right) = [t_1,\,t_2)$. We notice that $\psi_k$ is a continuous real-valued function since it is the composition of the continuous $Q_x$ and $C^1$-curve $\rho$. $L_k$-Lipschitz continuity on $[t_1,\,t_2)$ thus follows if we show that $\psi_k$ is $L_k$-Lipschitz on each interval in the partition. For that purpose, let $1\leq i\leq N_0$ be arbitrary and fix some $t\in\left[u_i,\,u_{i+1}\right)$. Let $(w,\,h):=\floor{\rho(t)}$ and $\gamma(t) := \rho(t) - \floor{\rho(t)}$ and notice that $\gamma(t)\in[0,\,1)^2$. Let
        \begin{gather}
            V_1:=x_{k,w,h},\,V_2:=x_{k,w,h+1},\\ V_3:=x_{k,w+1,h},\,V_4:=x_{k,w+1,h+1},\,
        \end{gather}
        Then, for any $u\in\left[u_i,\,u_{i+1}\right)$
        \begin{align}
            \psi_k(u) &= Q_x(k,\,\rho_1(u),\,\rho_2(u))\\
            \begin{split}
                &=(1-\gamma_1(u))\cdot((1-\gamma_2(u))\cdot V_1 + \gamma_2(u)\cdot V_2)\\
                &\hspace{4em} + \gamma_1(u)\cdot((1-\gamma_2(u)\cdot V_3 + \gamma_2(u)\cdot V_4).
            \end{split}
        \end{align}
        Let $m_\Delta:=m_\Delta(x,\,k\,\floor{\rho(t)})$ and notice that by definition
        \begin{equation}
            m_\Delta = \max_i\,V_i - \min_i\,V_i
        \end{equation}
        and in particular
        \begin{equation}
            \left|V_i - V_j\right| \leq m_\Delta \hspace{1em}\forall\,i,j.
        \end{equation}
        Since $V_i$ is constant for each $i$ and $\gamma$ is differentiable, $\psi_k$ is differentiable on $\left[u_i,\,u_{i+1}\right)$ and hence
        \begin{align}
            \dot\psi_k(u) &= (\dot\gamma_1(u)\gamma_2(u) + \gamma_1(u)\dot\gamma_2(u))(V_1 - V_2 - V_3 + V_4) \\
            &\hspace{3em} + \dot\gamma_1(u) (V_3 - V_1) + \dot\gamma_2(u) (V_2 - V_1).
        \end{align}
        Note that the derivative $\dot\psi_k$ is linear in $\gamma_1$ and $\gamma_2$ and hence its extreme values are bounded when evaluated at extreme values of $\gamma$, that is $(\gamma_1,\,\gamma_2)\in\{0,\,1\}^2$. We treat each case separately:
        \begin{itemize}[leftmargin=*]
            \item $\gamma_1 = \gamma_2 = 0$. Then,
                \begin{align}
                    \left|\dot\psi_k\right| &\leq \left|\dot\gamma_1(V_3 - V_1) + \dot\gamma_2(V_2 - V_1)\right|\\
                    &\leq \left|\dot\gamma_1\right|\cdot\left|V_3 - V_1\right| + \left|\dot\gamma_2\right|\cdot \left|V_2-V_1\right| \\
                    &\leq m_\Delta(\left|\dot\gamma_1\right| + \left|\dot\gamma_2\right|)
                \end{align}
            \item $\gamma_1 = \gamma_2 = 1$. Then,
                \begin{align}
                    \left|\dot\psi_k\right| &\leq \left|\dot\gamma_1(V_4 - V_2) + \dot\gamma_2(V_4 - V_3)\right|\\
                    &\leq \left|\dot\gamma_1\right|\cdot\left|V_4 - V_2\right| + \left|\dot\gamma_2\right|\cdot \left|V_4-V_3\right|\\
                    &\leq m_\Delta(\left|\dot\gamma_1\right| + \left|\dot\gamma_2\right|)
                \end{align}
            \item $\gamma_1 = 0,\, \gamma_2 = 1$. Then,
                \begin{align}
                    \left|\dot\psi_k\right| &\leq \left|\dot\gamma_1(V_4 - V_2) + \dot\gamma_2(V_2 - V_1)\right|\\
                    &\leq \left|\dot\gamma_1\right|\cdot\left|V_4 - V_2\right| + \left|\dot\gamma_2\right|\cdot\left|V_2 - V_1\right|\\
                    &\leq m_\Delta(\left|\dot\gamma_1\right| + \left|\dot\gamma_2\right|)
                \end{align}
            \item $\gamma_1 = 1,\, \gamma_2 = 0$. Then,
                \begin{align}
                    \left|\dot\psi_k\right| &\leq \left|\dot\gamma_1(V_3 - V_1) + \dot\gamma_2(V_4 - V_3)\right|\\
                    &\leq \left|\dot\gamma_1\right|\cdot\left|V_3 - V_1\right| + \left|\dot\gamma_2\right|\cdot\left|V_4 - V_3\right|\\
                    &\leq m_\Delta(\left|\dot\gamma_1\right| + \left|\dot\gamma_2\right|)
                \end{align}
        \end{itemize}
        Hence, for any $u\in\left[u_i,\,u_{i+1}\right)$, the modulus of the derivative is bounded by $m_\Delta(\left|\dot\gamma_1\right| + \left|\dot\gamma_2\right|)$. We can further bound this by observing the following connection between $\ell_1$ and $\ell_2$ distance
        \begin{equation}
            \forall\,x\in\R^n:\hspace{1em} \left\|x\right\|_1 = \left|\langle\left|x\right|,\,\mathbf{1}\rangle\right| \leq \left\|x\right\|_2\left\|\mathbf{1}\right\|_2 = \sqrt{n}\left\|x\right\|_2
        \end{equation}
        and hence $\forall\,u\in \left[u_i,\,u_{i+1}\right)$
        \begin{align}
            \left|\psi_k(u)\right| &\leq m_\Delta\left\|\dot\gamma(u)\right\|_1\\
            &\leq m_\Delta\sqrt{2}\left\|\dot\gamma(u)\right\|_2\\
            &= m_\Delta\sqrt{2}\left\|\dot\rho(u)\right\|_2.
        \end{align}
        Since $\psi_k$ is differentiable on $\left[u_i,\,u_{i+1}\right)$, its Lipschitz constant is bounded by the maximum absolute value of its derivative. Hence
        \begin{align}
            \begin{split}
                &\max_{u\in[u_i,\,u_{i+1})}m_\Delta\sqrt{2}\left\|\dot\rho(u)\right\|_2\\
                &\hspace{4em}= \max_{u\in[u_i,\,u_{i+1})}m_\Delta(x,\,k,\,\floor{\rho(u)})\sqrt{2}\left\|\dot\rho(u)\right\|_2
            \end{split}\\
            &\hspace{4em}\leq
            \max_{u\in[t_1,\,t_2)}m_\Delta(x,\,k,\,\floor{\rho(u)})\sqrt{2}\left\|\dot\rho(u)\right\|_2 = L_k
        \end{align}
        is a Lipschitz constant for $\psi_k$ on $\left[u_i,\,u_{i+1}\right)$. Note that $L_k$ does not depend on $i$. Furthermore, $i$ was chosen arbitrarily and hence $L_k$ is a Lipschitz constant for $\psi_k$ on $\left[t_1,\,t_2\right)$ and due to continuity on $\left[t_1,\,t_2\right]$, concluding the proof.
    \end{proof}

    \subsection{Rotation}
    \begin{customlem}{\ref{lem:rotation_lipschitz}}[restated]
        Let $x\in\R^{K\times W \times H}$ be a $K$-channel image and
        let $\phi_R = P \circ I \circ \tilde{\phi}_R$ be the rotation transformation.
        Then, a global Lipschitz constant $L$ for the functions $\{g_i\}_{i=1}^N$ is given by
        \begin{equation}
            L_{r} = \max_{1\leq i \leq N-1}\sum_{k=0}^{K-1}\sum_{r,s\in V} 2 d_{r,s}\cdot m_\Delta(x,k,\gP_{r,s}^{(i)})\cdot \overline{m}(x,\,k,\,\gP_{r,s}^{(i)})
        \end{equation}
        where $V = \left\{(r,s)\in\N^2\lvert\,d_{r,s} < \frac{1}{2}(\min\left\{W,H\right\}-1)\right\}$. The set $\gP_{r,s}^{(i)}$ is given by all integer grid pixels that are covered by the trajectory of source pixels of $(r,s)$ when rotating from angle $\alpha_i$ to $\alpha_{i+1}$.
    \end{customlem}
    \begin{proof}
        Recall that $\phi_R$ acts on images $x\in\R^{K\times W\times H}$ and that $g_i$ is defined as
        \begin{equation}
            \label{eq:gi_expanded}
            \begin{aligned}
                g_i(\alpha) &= \left\|\phi_R(x,\,\alpha) - \phi_R(x,\,\alpha_i)\right\|_2^2\\
            &= \sum_{k=0}^{K-1}\sum_{r=0}^{W-1}\sum_{s=0}^{H-1}\,\left(\phi_R(x,\,\alpha)_{k,r,s}-\phi_R(x,\,\alpha_i)_{k,r,s}\right)^2
            \end{aligned}
        \end{equation}
        Let $c_W$ and $c_H$ denote the center pixels
        \begin{equation}
            c_W := \dfrac{W-1}{2},\hspace{2em} c_H := \dfrac{H-1}{2}.
        \end{equation}
        and recall the following quantities from the definition of $\phi_R$ (\Cref{appendix:rotation_details}):
        \begin{equation}
            \begin{aligned}
                d_{r,s} &= \sqrt{\left(r - c_W\right)^2 + \left(s - c_H\right)^2},\\
                g_{r,s} &= \arctan2\left(s - c_H,\,r - c_W\right)
            \end{aligned}
        \end{equation}
        Note that
        \begin{equation}
            d_{r,s} \geq \min\{c_W,\,c_H\}\hspace{0.5em}\Rightarrow\hspace{0.5em}\phi_R(x,\,\alpha)_{k,r,s}=0.
        \end{equation}
        We thus only need to consider pixels that lie inside the centered disk. We call the collection of such pixels \emph{valid} pixels, denoted by V:
        \begin{equation}
            V:=\left\{(r,\,s)\in\N^2\mid d_{r,s} < \min\{c_W,\,c_H\}\right\}.
        \end{equation}
        Let $f_1^{r,s}\colon\R\to\R$ and $f_2^{r,s}\colon\R\to\R$ be functions defined as
        \begin{equation}
            \begin{aligned}
                f_1^{r,s}(\alpha) &= c_W + d_{r,s}\,\cos(g_{r,s} - \alpha),\\
                f_2^{r,s}(\alpha) &= c_H + d_{r,s}\,\sin(g_{r,s} - \alpha).
            \end{aligned}
        \end{equation}
        Then for any valid pixel $(r,\,s)\in V$, the value of the rotated image $\phi_R(x,\,\alpha)$ is given by
        \begin{equation}
            \label{eq:rotation_short_def}
            \phi_R(x,\,\alpha)_{k,r,s} = Q_x(k,\,f_1^{r,s}(\alpha),\,f_2^{r,s}(\alpha))
        \end{equation}
        where $Q_x$ denotes bilinear interpolation. We define the shorthand
        \begin{equation}
            g_i^{k,r,s}(\alpha) := \left(\phi_R(x,\,\alpha)_{k,r,s}-\phi_R(x,\,\alpha_i)_{k,r,s}\right)^2
        \end{equation}
        and denote by $L_i^{k,r,s}$ and $L_{i+1}^{k,r,s}$ the Lipschitz constants of $g_i^{k,r,s}$ and $g_{i+1}^{k,r,s}$ on $[\alpha_i,\,\alpha_{i+1}]$. We can write~(\ref{eq:gi_expanded}) as
        \begin{equation}
            \begin{aligned}
                g_i(\alpha) &= \sum_{k=0}^{K-1}\sum_{(r,\,s)\in V} g_i^{k,r,s}(\alpha),\\
                g_{i+1}(\alpha) &= \sum_{k=0}^{K-1}\sum_{(r,\,s)\in V} g_{i+1}^{k,r,s}(\alpha)
            \end{aligned}
        \end{equation}
        and note that Lipschitz constants of $g_i$ and $g_{i+1}$ on $[\alpha_i,\,\alpha_{i+1}]$ are given by
        \begin{align}
                \max_{c,\,d\in[\alpha_i,\,\alpha_{i+1}]}\dfrac{\left|g_i(c) - g_i(d)\right|}{\left|c - d\right|} &\leq \left(\sum_{k=0}^{K-1}\sum_{(r,\,s)\in V} L_i^{k,r,s}\right) =: L_i\\
                \max_{c,\,d\in[\alpha_i,\,\alpha_{i+1}]}\dfrac{\left|g_{i+1}(c) - g_{i+1}(d)\right|}{\left|c - d\right|}&\leq \left(\sum_{k=0}^{K-1}\sum_{(r,\,s)\in V} L_{i+1}^{k,r,s}\right)=:L_{i+1}
        \end{align}
        We can hence determine $L$ according to equation~(\ref{eq:lipschitz_constant}) as
        \begin{equation}
            L = \max_i\left\{\max\left\{L_i,\,L_{i+1}\right\}\right\}.
        \end{equation}
        Without loss of generality, consider $L_i^{k,r,s}$ and note that
        \begin{align}
            &\max_{c,d\in[\alpha_i,\,\alpha_{i+1}]}\left|\dfrac{g_i^{k,r,s}(c) - g_i^{k,r,s}(d) }{c - d}\right|\\
            \begin{split}
                &\hspace{2em}=\max_{c,d\in[\alpha_i,\,\alpha_{i+1}]}\left|\dfrac{\phi_R(x,\,c)_{k,r,s} - \phi_R(x,\,d)_{k,r,s}}{c - d}\right|\\
                &\hspace{4em}\cdot\left|\phi_R(x,\,c)_{k,r,s} + \phi_R(x,\,d)_{k,r,s} - 2\phi_R(x,\,\alpha_i)_{k,r,s}\right|
            \end{split}\\
            \begin{split}
                &\hspace{2em}\leq \max_{c,d\in[\alpha_i,\,\alpha_{i+1}]}\underbrace{\left|\dfrac{\phi_R(x,\,c)_{k,r,s} - \phi_R(x,\,d)_{k,r,s}}{c - d}\right|}_{(\mathrm{I})}\\
                &\hspace{4em}\cdot2 \max_{\theta\in[\alpha_i,\,\alpha_{i+1}]}\underbrace{\left|\phi_R(x,\,\theta)_{k,r,s} - \phi_R(x,\,\alpha_i)_{k,r,s}\right|}_{(\mathrm{II})}\label{eq:upper_bound_gi}.
            \end{split}
        \end{align}
        To compute a Lipschitz constant for $g_i^{k,r,s}$ on the interval $[\alpha_i,\,\alpha_{i+1}]$ we thus only need to compute a Lipschitz constant for $\phi_R(x,\,\cdot)$ on $[\alpha_i,\,\alpha_{i+1}]$ and an upper bound on (II). For that purpose, note that $\phi_R$ takes only positive values and consider
        \begin{align}
            \label{eq:lipschitz_rot_upper_bound_2}
            \mathrm{(II)} &\leq \max_{\theta\in[\alpha_i,\,\alpha_{i+1}]}\left\{\phi_R(x,\,\theta)_{k,r,s},\,\phi_R(x,\,\alpha_i)_{k,r,s}\right\} \\
            &= \max_{\theta\in[\alpha_i,\,\alpha_{i+1}]}\,\phi_R(x,\,\theta)_{k,r,s}
        \end{align}
        Notice that now both $L_i^{k,r,s}$ and $L_{i+1}^{k,r,s}$ share the same upper bound.
        Recall~\eqref{eq:rotation_short_def}, i.e.,
        \begin{align}
            \phi_R(x,\,\theta)_{k,r,s} =  Q_x(k,\,f_1^{r,s}(\theta),\,f_2^{r,s}(\theta)).
        \end{align}
        Now, we upper bound~\eqref{eq:lipschitz_rot_upper_bound_2} by finding all integer grid pixels that are covered by the trajectory $(f_1^{r,s}(\theta),\,f_2^{r,s}(\theta))$. Specifically, let
        \begin{equation}
            \gP_{r,s}^{(i)}:=\bigcup_{\theta\in[\alpha_i,\,\alpha_{i+1}]}\left(\floor{f_1^{r,s}(\theta)},\, \floor{f_2^{r,s}(\theta)}\right).
        \end{equation}
        Since $\phi_R$ is interpolated from integer pixels, we can consider the maximum over $\gP_{r,s}^{(i)}$ in order to upper bound ~(\ref{eq:lipschitz_rot_upper_bound_2}):
        \begin{align}
            \begin{split}
                &\max_{\theta\in[\alpha_i,\,\alpha_{i+1}]}\,\phi_R(x,\,\theta)_{k,r,s} =\\ &\hspace{4em}\max_{\theta\in[\alpha_i,\,\alpha_{i+1}]}\,Q_x(k,\,f_1^{r,s}(\theta),\,f_2^{r,s}(\theta))
            \end{split}\\
            \begin{split}
                &\hspace{4em} \leq \max_{(i,j)\in\gP_{r,s}}\max\big\{x(k,\,i,\,j),\,x(k,\,i+1,\,j),\\
                &\hspace{10em} \,x(k,\,i,\,j+1),\,x(k,\,i+1,\,j+1)\big\}
            \end{split}\\
            &\hspace{4em}= \Bar{m}(x,\,k,\,\gP_{r,s}^{(i)}).
        \end{align}
        We now have to find an upper bound of (I), that is, a Lipschitz constant of $\phi_R(x,\,\cdot)_{k,r,s}$ on the interval $[\alpha_i,\,\alpha_{i+1}]$. For that purpose, consider the following. Note that the curve $\rho\colon[\alpha_i,\,\alpha_{i+1}]\to \R^2$, $\rho(t):=(f_1^{r,s}(t),\,f_2^{r,s}(t))$ is of class $C^1$ and
        \begin{align}
            \frac{df_1^{r,s}(t)}{dt} &= \frac{d}{dt} \left(c_W + d_{r,s}\,\cos(g_{r,s} - t)\right)\\
            &= d_{r,s}\,\sin(g_{r,s} - t)\\
            \frac{df_2^{r,s}(t)}{dt} &= \frac{d}{dt} \left(c_H + d_{r,s}\,\sin(g_{r,s} - t)\right)\\
            &= -d_{r,s}\,\cos(g_{r,s} - t)
        \end{align}
        and hence
        \begin{align}
            \left\|\dot\rho(t)\right\|_2 = \sqrt{\left(\dfrac{df_1^{r,s}(t)}{dt}\right)^2 + \left(\dfrac{df_2^{r,s}(t)}{dt}\right)^2} = \sqrt{2}\,d_{r,s}.
        \end{align}
        By Lemma~\ref{lem:lipschitz_constant} a Lipschitz constant for the function $\phi_R(x,\,\cdot)_{k,r,s}$ is thus given by
        \begin{equation}
            \begin{aligned}
                &\max_{c,d\in[\alpha_i,\,\alpha_{i+1}]}\left|\dfrac{\phi_R(x,\,c)_{k,r,s} - \phi_R(x,\,d)_{k,r,s}}{c - d}\right| \\
                &\hspace{6em}\leq 2\,d_{r,s}\cdot m_\Delta(x,\,k,\,\gP_{r,s}^{(i)}).
            \end{aligned}
        \end{equation}
        We can thus upper bound (I) and (II) in~(\ref{eq:upper_bound_gi}) yielding a Lipschitz constant for $g_i^{k,r,s}$ and $g_{i+1}^{k,r,s}$ on $[\alpha_i,\,\alpha_{i+1}]$
        \begin{align}
            \begin{split}
                &\max_{c,d\in[\alpha_i,\,\alpha_{i+1}]}\left|\dfrac{g_i^{k,r,s}(c) - g_i^{k,r,s}(d) }{c - d}\right|\\
                &\hspace{6em}\leq 2\,d_{r,s}\cdot m_\Delta(x,\,k,\,\gP_{r,s}^{(i)}) \cdot \Bar{m}(x,\,k,\,\gP_{r,s}^{(i)})
            \end{split}\\
            &\hspace{6em}= L_i^{k,r,s} (=L_{i+1}^{k,r,s}).
        \end{align}
        Finally, we can compute $L_r$ as
        \begin{equation}
            \begin{aligned}
                L&=\max_{1\leq i \leq N-1} \,\sum_{k=0}^{K-1}\sum_{(r,s)\in V} L_i^{k,r,s}\\
                &= \max_{1\leq i \leq N-1}\sum_{k=0}^{K-1}\sum_{r,s\in V} 2 d_{r,s}\cdot m_\Delta(x,k,\gP_{r,s}^{(i)})\cdot \overline{m}(x,\,k,\,\gP_{r,s}^{(i)}))
            \end{aligned}
        \end{equation}
        what concludes the proof.
    \end{proof}

\subsection{Scaling}
    \begin{customlem}{\ref{lem:scaling_lipschitz}}[restated]
        Let $x\in\R^{K\times W \times H}$ be a $K$-channel image and
        let $\phi_S$ be the scaling transformation.
        Then, a global Lipschitz constant $L$ for the functions $\{g_i\}_{i=1}^N$ is given by
        \begin{equation}
            L_{\mathrm{s}} = \max_{1\leq i \leq N-1}\sum_{k=0}^{K-1}\sum_{r,s\in \Omega\cap\N^2} \frac{\sqrt{2} d_{r,s}}{a^2}\cdot m_\Delta(x,k,\gP_{r,s}^{(i)})\cdot \overline{m}(x,\,k,\,\gP_{r,s}^{(i)})
        \end{equation}
        where $\Omega=[0,\,W-1]\times[0,\,H-1]$ and $a$ is the lower boundary value in $\gS=[a,\,b]$. The set $\gP_{r,s}^{(i)}$ is given by all integer grid pixels that are covered by the trajectory of source pixels of $(r,s)$ when scaling with factors from $\alpha_{i+1}$ to $\alpha_{i}$.
    \end{customlem}
    \begin{proof}
        Recall the Definition of the Scaling transformation $\phi_S$ given by $\phi_S\colon\R^{K\times W \times H}\times\R \to \R^{K\times W \times H}$, where
        \begin{equation}
            \phi_S(x,\,\alpha)_{k,r,s}:=Q_x\left(k,\,c_W + \frac{r - c_W}{s},\,c_H + \frac{s - c_H}{s}\right).
        \end{equation}
        Recall that the set $\Omega$ is given by $\Omega=[0,\,W-1]\times[0,\,H-1]=\{1,\,\ldots,\,K\}$ and let
        \begin{equation}
            \Omega_\N:=\Omega\cap\N^2
        \end{equation}
        be the set of integers in $\Omega$. Let $f_1^{r}\colon[a,\,b]\to\R$ and $f_2^{r,s}\colon[a,\,b]\to\R$ be functions defined as
        \begin{equation}
            \begin{aligned}
                f_1^{r}(\alpha) &:= c_W + \dfrac{r - c_W}{\alpha},\\
                f_2^{s}(\alpha) &:= c_H + \dfrac{s - c_H}{\alpha}.
            \end{aligned}
        \end{equation}
        Then, the value of the scaled image $\phi_S(x,\,\alpha)$ is given by
        \begin{equation}
            \label{eq:scaling_short_def}
            \phi_S(x,\,\alpha)_{k,r,s} = Q_x(k,\,f_1^{r}(\alpha),\,f_2^{s}(\alpha))
        \end{equation}
        where $Q_x$ denotes bilinear interpolation. Let
        \begin{align}
            \psi_k\colon[a,\,b]\to\R,\ \alpha\mapsto Q_x(k,\,f_1^{r}(\alpha),\,f_2^{s}(\alpha)).
        \end{align}
        We notice that, in contrast to rotations, $\psi_k$ is \emph{not} continuous at every $\alpha\in\R_{>0}$. Namely, when considering scaling factors in $(0,\,1)$, bilinear interpolation applies black padding to some $(r,\,s)\in\Omega$ resulting in discontinuities of $\psi_k$. To see this, consider the following. The interval $[\alpha_{i+1},\,\alpha_i]$ contains a discontinuity of $\psi_k$, if
        \begin{align}
            \begin{cases}
                \alpha_{i+1} < \dfrac{r - c_W}{c_W} < \alpha_{i},\hspace{1em}& r > c_W,\\
                \alpha_{i+1} < \dfrac{c_W - r}{c_W} < \alpha_{i},\hspace{1em}& r < c_W,
            \end{cases}
        \end{align}
        because then $\exists\,\alpha_0\in[\alpha_{i+1},\,\alpha_i]$ such that $f_1^{r}(\alpha_0)\in\{0,\,W-1\}\subseteq\Omega$ and hence
        \begin{equation}
            \phi_S(x,\,\alpha_0)_{k,r,s} \neq 0
        \end{equation}
        but, for $r > c_W$,
        \begin{equation}
            \phi_S(x,\,\alpha_0 + \varepsilon)_{k,r,s} = 0 \hspace{1em}\forall\varepsilon > 0
        \end{equation}
        or, when $r< c_W$,
        \begin{equation}
            \phi_S(x,\,\alpha_0 - \varepsilon)_{k,r,s} = 0 \hspace{1em}\forall\varepsilon > 0.
        \end{equation}
        A similar reasoning leads to a discontinuity in the $s$-coordinates. We can thus define the set of discontinuities of $\psi_k$ as
        \begin{equation}
            \gD:=\left(\bigcup_{r=0}^{W-1}\gD^{r}_1\right)\cup\left(\bigcup_{s=0}^{H-1}\gD^{s}_2\right)
        \end{equation}
        where
        \begin{equation}
            \begin{aligned}
                \gD^{r}_1 &:= \left\{\alpha_0\in [a,\,b]\mid f_1^r(\alpha_0) \in \{0,\,W-1\}\right\}\\
                \gD^{s}_2 &:= \left\{\alpha_0\in [a,\,b]\mid f_2^s(\alpha_0) \in \{0,\,H-1\} \right\}.
        \end{aligned}
        \end{equation}
        We notice that $|\gD| \leq H + W$ and hence for large enough $N$, each interval $[\alpha_i,\,\alpha_{i+1}]$ contains at most 1 discontinuity.

        Due to these continuities, we need to modify the general upper bound $M$ of the interpolation error $M_\gS$
        Recall that for $a<b$ and $\{\alpha_i\}_{i=1}^N$, the maximum $L_2$-sampling error $M_{a,b}$ is given by
        \begin{equation}
            \label{eq:maximum_l2_sampling_error_scale}
            M_{\gS} := \max_{a\leq \alpha \leq b}\min_{1\leq i \leq N}\left|\left|\phi_{S}(x,\,\alpha) - \phi_{S}(x,\,\alpha_i)\right|\right|_2.
        \end{equation}
        In order to compute an upper bound on~(\ref{eq:maximum_l2_sampling_error_scale}) for scaling, we are interested in finding $M\geq0$ such that
        \begin{equation}
            M_{\gS}^2 \leq M
        \end{equation}
        For scaling, similar as in the case for rotations, we sample $\alpha_i$ uniformly from $[a,\,b]$:
        \begin{equation}
            \alpha_i = a + \dfrac{b-a}{N-1}(i-1) \text{ for } 1 \le i \le N.
        \end{equation}
        and note that $\alpha_1 = b$ and $\alpha_N=a$. For $1\leq i\leq N$ Let $g_i$ be the functions $g_i\colon[a,\,b]\to\R_{\geq0}$ defined by
        \begin{align}
            g_i(\alpha) := \left\|\phi_S(x,\,\alpha) - \phi_S(x,\,\alpha_i)\right\|_2^2.
        \end{align}
        Note that $\forall\,\alpha\in[a,\,b]$, $\exists\,i$ such that $\alpha\in[\alpha_{i+1},\,\alpha_i]$. Suppose that $N$ is large enough such that $\forall\,i\colon |\gD\cap[\alpha_{i+1},\,\alpha_i]|\leq 1$ and denote the discontinuity in interval $[\alpha_{i+1},\,\alpha_i]$ by $t_i$ if it exists. Let
        \begin{align}
            \small
            M_i :=
                \begin{cases}
                    \max\limits_{\alpha_{i}\leq\alpha\leq\alpha_{i+1}}\min\{g_i(\alpha),\,g_{i+1}(\alpha)\} \hspace{1em}&[\alpha_{i},\,\alpha_{i+1}]\cap \gD = \varnothing\\
                    \max\left\{
                        \max\limits_{\alpha_{i}\leq\alpha\leq t_i} g_{i+1}(\alpha),\,
                        \max\limits_{t_i\leq\alpha \leq \alpha_{i+1}} g_{i}(\alpha)\right
                    \}
                    \hspace{1em} &[\alpha_{i},\,\alpha_{i+1}]\cap \gD = \{t_i\}
                \end{cases}
        \end{align}
        Similarly as in the case for rotations, we find
        \begin{equation}
            M_{\gS}^2 \leq \max_{1\leq i\leq N-1}\,M_i.
        \end{equation}
        For simplicity, we assume for the sequel that $\gD=\varnothing$. The case where discontinuities exist can be treated analogously. We further divide each interval $[\alpha_{i},\,\alpha_{i+1}]$ by sampling $n\in\N$ points $\{\gamma_{i,j}\}_{j=1}^n$ according to
        \begin{equation}
            \gamma_{i,j} := \alpha_i + \dfrac{\alpha_{i+1} - \alpha_i}{n-1}(j-1) \text{ for } 1 \le j \le n
        \end{equation}
        and define
        \begin{equation}
            m_{i,j}:=\max_{\gamma_{i,j}\leq \gamma \leq \gamma_{i,j+1}}\min\left\{g_i(\gamma),\,g_{i+1}(\gamma)\right\}.
        \end{equation}
        We can thus upper bound each $M_i$ by
        \begin{equation}
            M_i \leq \max_{1 \leq j \leq n-1}\,m_{i,j}.
        \end{equation}
        In order to find an upper bound on $M_{\gS}^2$, we thus need to find an upper bound on $m_{i,j}$ and can proceed analogously to rotations. Namely, setting
        \begin{align}
            \begin{split}
                &M:=\max_{1\leq i\leq N-1}\bigg\{\max_{1\leq j\leq n-1}\bigg\{\frac{1}{2}\cdot\left(\min\left\{g_i(\gamma_{i,j}) + g_i(\gamma_{i,j+1}),\right.\right.\\
                &\hspace{4em}\left.\left.\,g_{i+1}(\gamma_{i,j}) + g_{i+1}(\gamma_{i,j+1})\right\}\right) + L\cdot\frac{\gamma_{i,\,j+1} - \gamma_{i,\,j}}{2}\bigg\}\bigg\}
            \end{split}
        \end{align}
        yields a computable upper bound of the maximum $\ell_2$ interpolation error. Computing a Lipschitz constant for $g_i$ and $g_{i+1}$ is also analogous to rotations. The difference lies only in computing a Lipschitz constant for $\phi_S$ what we will explain in greater detail.

        Recall that Lemma~\ref{lem:lipschitz_constant} provides a Lipschitz constant for the function $t\mapsto\psi_k(t):=Q_x(k,\,\rho_1(t),\,\rho_2(t))$ where $\rho$ is a differentiable curve with values in $\R^2$. Namely, a Lipschitz constant for $\psi_k$ is given by
        \begin{equation}
            L_k = \max_{t\in[t_1,t_2]}\left(\sqrt{2}\left\|\dot\rho(t)\right\|_2\cdot m_\Delta(x,\,k,\,\floor{\rho(t)})\right).
        \end{equation}
        Consider the curve
        \begin{equation}
            \rho(t) := (f_1^r(t),\,f_2^s(t)),\hspace{1em}t>0
        \end{equation}
        and note that it is differentiable with derivatives
        \begin{align}
            \dfrac{df_1^r(t)}{dt} &= \dfrac{d}{dt}\left(c_W + \dfrac{r - c_W}{t}\right) = \dfrac{c_W - r}{t^2}\\
            \dfrac{df_2^s(t)}{dt} &= \dfrac{d}{dt}\left(c_H + \dfrac{s - c_H}{t}\right) = \dfrac{c_H - s}{t^2}
        \end{align}
        and
        \begin{equation}
            \left\|\dot\rho(t)\right\|_2 = \dfrac{1}{t^2}\sqrt{\left(c_W - r\right)^2 + \left(c_H - s\right)^2}.
        \end{equation}
        A Lipschitz constant for $\phi_S(x,\,\cdot)_{k,r,s}$ is thus given by
        \begin{align}
            \begin{split}
                L_k^{r,s} &= \max_{t\in[t_1,\,t_2]}\bigg(\frac{\sqrt{\left(c_W - r\right)^2 + \left(c_H - s\right)^2}}{t^2}
                \cdot\sqrt{2} m_\Delta(x,\,k,\,\floor{\rho(t)})\bigg)
            \end{split}\\
            &\leq\dfrac{\sqrt{\left(c_W - r\right)^2 + \left(c_H - s\right)^2}}{t_1^2}\cdot\sqrt{2}\cdot m_\Delta(x,\,k,\,\gP_{r,s})\\
            &\le\dfrac{\sqrt{\left(c_W - r\right)^2 + \left(c_H - s\right)^2}}{a^2}\cdot\sqrt{2}\cdot m_\Delta(x,\,k,\,\gP_{r,s})
        \end{align}
        where
        \begin{equation}
            \gP_{r,s} = \bigcup_{\alpha\in[t_1,\,t_2]}\{(\floor{f_1^r(t)},\,\floor{f_2^s(t)})\}.
        \end{equation}
        Finally, setting
        \begin{equation}
            L_i^{k,r,s}:= L_k^{r,s}\cdot \Bar{m}(x,\,k,\,\gP_{r,s}^{(i)})
        \end{equation}
        and
        \begin{equation}
            \begin{aligned}
                L_s &=\max_{1\leq i \leq N-1}\sum_{k=0}^{K-1}\sum_{(r,s)\in\Omega_\N}\,L_i^{k,r,s}\\
                &=\max_{1\leq i \leq N-1}\sum_{k=0}^{K-1}\sum_{r,s\in \Omega\cap\N^2} \frac{\sqrt{2} d_{r,s}}{a^2}\\
                &\hspace{6em}\cdot m_\Delta(x,k,\gP_{r,s}^{(i)})\cdot \overline{m}(x,\,k,\,\gP_{r,s}^{(i)})
            \end{aligned}
        \end{equation}
        yields the desired Lipschitz constant.
    \end{proof}
    \fi

\section{Algorithm Description for Differentially Resolvable Transformations}
\label{apx:algorithms}

    \begin{algorithm}[ht]
        \small
        \caption{Interpolation Error $M$ Computation for Rotation Transformation.}
        \begin{algorithmic}
            \label{algo:interpolation-error-M-compute}
            \STATE {\bf Input:}
            clean input image $x$; \\
            interval of rotation angle to certify $[a,\,b]$; \\
            number of first-level samples $N$; \\
            number of second-level samples $n$

            \STATE {\bf Output:}
            rotation angle samples $\{\alpha_i\}_{i=1}^N$; \\
            upper bound $M$ of squared $\ell_2$-interpolation error
            $$
                 M_{\gS}^2 = \argmax_{\alpha\in [a,b]} \min_{1\le i\le N} \|\Tilde{\phi}_R(x,\alpha) - \Tilde{\phi}_R(x,\alpha_i)\|_2^2.
            $$

            \STATE \textit{/* Compute Lipschitz constant $L_r$~\eqref{eq:lem_rotation_lipschitz} */}
            \STATE $\alpha_1 \gets a$
            \FOR{ $i=1,\dots,N-1$ }
                \STATE $\alpha_{i+1} \gets a + (b-a)\cdot\frac{i}{N-1}$~\eqref{eq:unif-in-a-b}
                \FORALL{ $(r,s) \in V$ }
                    \STATE \textit{/* $V$ and $\gP_{r,s}^{(i)}$ are defined in \Cref{lem:rotation_lipschitz} */}
                    \STATE Compute trajectory covered grid pixels $\gP_{r,s}^{(i)}$
                    \FOR{ $k=0,\dots,K-1$ }
                        \STATE Compute $2d_{r,s} \cdot m_{\Delta}(x,k,\cP_{r,s}^{(i)}) \cdot \bar m(x,k,\gP_{r,s}^{(i)})$~\eqref{eq:lem_rotation_lipschitz}
                    \ENDFOR
                \ENDFOR
                \STATE $L_{r,i}\gets \sum_{k=0}^{K-1} \sum_{(r,s) \in V}
                2 d_{r,s} \cdot m_{\Delta} (x,k,\cP_{r,s}^{(i)}) \cdot \bar m (x,k,\gP_{r,s}^{(i)})$.
            \ENDFOR
            \STATE $L_r \gets \max_{1\le i\le N - 1} L_{r,i}$~\eqref{eq:lem_rotation_lipschitz}

            \STATE \textit{/* Compute interpolation error bound $M$~\eqref{eq:sqrt-M} from stratified sampling */}

            \FOR{ $i=1,\dots,N-1$ }
                \FOR{ $j=1,\dots,n$ }
                \STATE \textit{/* Second-level sampling */}
                \STATE $\gamma_{i,j} \gets \alpha_i + (\alpha_{i+1}-\alpha_i)\cdot\frac{j-1}{n-1}$~\eqref{eq:gamma-i-j}
                \ENDFOR
                \STATE $M_i \gets 0$
                \FOR{ $j=1,\dots,n-1$ }
                    \STATE Compute $g_i(\gamma_{i,j})$, $g_i(\gamma_{i,j+1})$, $g_{i+1}(\gamma_{i,j})$, and $g_{i+1}(\gamma_{i,j+1})$~\eqref{eq:g_i_functions_def}
                    \STATE $M_i \gets \max\left\{ M_i, \min\left\{ g_i(\gamma_{i,j}) + g_i(\gamma_{i,j+1}), \right.\right.$
                    \STATE \hspace{2em} $\left.\left. g_{i+1}(\gamma_{i,j}) + g_{i+1}(\gamma_{i,j+1}) \right\} \right\}$
                \ENDFOR
                \STATE $M_i \gets \frac{1}{2} M_i + L\cdot \frac{b-a}{(N-1)(n-1)}$~\eqref{eq:M_i_expression}
            \ENDFOR
            \STATE {\bf Return:} $M \gets \max_{1\le i\le N-1} M_i$~\eqref{eq:sqrt-M}
        \end{algorithmic}
        \label{alg:interpolation-error-compute}
    \end{algorithm}

    \Cref{alg:interpolation-error-compute} presents a pseudo-code for interpolation error $M$ computation, taking rotation transformation as the example.
    It corresponds to the description in \Cref{subsec:interpolation-bound-computation}.
    \Cref{algo:progressive-sampling-cert} presents a pseudo-code for progressive sampling.
    It corresponds to the description in \Cref{subsec:progressive-sampling-with-interpolation-bound}.
    We remark that in practice, we sample in mini-batches with batch size $B$.
    The error tolerance $T$ is set to $M_{\gS}$~\eqref{eq:general_Ms} if certifying rotation or scaling, is set to $\sqrt{M_{\gS}^2 + \sigma^2 / \sigma_b^2 \cdot b_0^2}$~\eqref{eq:general_Ms_brightness_stmt} if certifying the composition of rotation or scaling with brightness change within $[-b_0,\,b_0]$;
    and is set to $\sqrt{(M_{\gS}+r)^2 + \sigma^2 / \sigma_b^2 \cdot b_0^2}$~\eqref{eq:general_Ms_brightness_l2_stmt} if certifying the composition of rotation or scaling, brightness change $[-b_0,\,b_0]$, and $\ell_2$ bounded perturbations within $r$.
    The two algorithms jointly constitute our pipeline \textbf{\framework-DR} for certifying against differentially resolvable transformations as shown in \Cref{fig:bounding-illustration}.

    \begin{algorithm}[!t]
        \small
        \caption{Progressive Sampling for Certification.}
        \begin{algorithmic}
            \label{algo:progressive-sampling-cert}
            \STATE {\bf Input:}
            clean input image $x$ with true class $k_A$;
            first-level parameter samples $\{\alpha_i\}_{i=1}^N$;
            perturbation random variable $\varepsilon$ with variance $\sigma^2$;
            $\ell_2$ error tolerance $T$;
            batch size $B$;
            sampling size limit $n_s$;
            confidence level $p$.
            \STATE {\bf Output:} with probability $1-p$, whether $g(\cdot;\varepsilon)$ is certifiably robust at $\phi(x,\,\alpha)$.

            \FOR { $i=1,\dots,N$ }
                \STATE $x^{(i)} \gets \phi(x,\alpha_i)$
                \STATE $j \gets 0$
                \WHILE {$j \le n_s$}
                    \STATE Sample $B$ instances of $\phi(x^{(i)}, \varepsilon)$, and use them to  update empirical mean $\hat{q}(y_A\lvert\,x^{(i)};\,\varepsilon)$.
                    \STATE $j \gets j + B$.

                    \STATE \textit{/* Lower confidence interval bound with these $j$ samples */}
                    \STATE $\underline{p_A}^{(i)} = \mathrm{LowerConfBound}(\hat{q}(y_A\lvert\,x^{(i)}; \varepsilon), j, 1 - p/N)$.

                    \IF{$\underline{R_i} = \sigma \Phi^{-1}\left( \underline{p_A}^{(i)} \right) > T$}
                        \STATE \textit{/* Already get the certification that $R_i > T$, break */}
                        \STATE \texttt{Break}
                    \ENDIF
                \ENDWHILE
                \IF{ $\underline{R_i} = \sigma \Phi^{-1}\left( \underline{p_A}^{(i)} \right) \le T$ }
                    \STATE \textit{/* Cannot ensure that $R_i > T$. So cannot ensure that $R = \min R_i > T$. Early halt */}
                    \STATE {\bf Return:} \texttt{false}
                \ENDIF
            \ENDFOR
            \STATE {\bf Return:} \texttt{true}
        \end{algorithmic}
    \end{algorithm}

\section{Omitted Experiment Details}
    \label{adxsec:experiment-details}

    Here we provide all omitted details about experiment setup, implementation, discussion about baselines, evaluation protocols, results, findings, and analyses.
    \ifnum\ArXiv=0
    Due to the space limit, most of the results are omitted to the corresponding appendix of the arXiv version~\cite{arxiv}.
    \fi

    \subsection{Model Preparation}
        \label{adxsec:model-preparation}

        \ifnum\ArXiv=0
            Due to the space limit, the full description of the model preparation including all hyperparameters are omitted to the corresponding section of the arXiv version~\cite{arxiv}.
        \fi

        \ifnum\ArXiv=1
        As previous work shows, an undefended model is very vulnerable even under simple random semantic attacks.
        Therefore, to obtain nontrivial certified robustness, we require the model itself to be trained to be robust against semantic transformations.
        We apply data augmentation training~\cite{cohen2019certified} combined with Consistency regularization~\cite{jeong2020consistency} to train the base classifiers.
        The data augmentation training randomly transforms the input by the specified transformation using parameters drawn from the specified smoothing distribution/strategy.
        The Consistency regularization further enhances the consistency of the base classifiers' prediction among the drawn parameters.
        Then,
        the base classifiers are used to construct smoothed classifiers by the specified smoothing distribution/strategy, and we compute its robustness certification with our approach.

        On relatively small datasets MNIST and CIFAR-10, the models are trained from scratch.
        On MNIST, we use a convolutional neural network~(CNN) composed of four convolutional layers and three fully connected layers.
        On CIFAR-10, we use the neural network ResNet-110, a $110$-layer ResNet model~\cite{he2016deep}.
        These model structures are the same as in the literature~\cite{salman2019provably,cohen2019certified,yang2020randomized} for direct comparison.
        On MNIST, we train $100$ epochs; on CIFAR-10, we train $150$ epochs.
        The batch sizes~($B$) are $400$ and $256$ on MNIST and CIFAR-10,  respectively.
        The learning rate on both datasets is initialized to $0.01$, and after every $50$ epochs, the learning rate is multiplied by $0.1$.
        For resolvable transformations, the data augmentation usually uses the same smoothing distribution/strategy as we will use to construct the smoothed classifier.
        In particular, for brightness and contrast transformation, we empirically observe that a larger variance during inference time helps to improve the certified accuracy under large attack radius, and for the composition of Gaussian blur, brightness, contrast, and translation, we additionally add small additive Gaussian noise to improve its ability to defend against other unforeseen attacks as we will discuss in \Cref{appendix:exp-corruption}.
        For differentially resolvable transformations,
        since Gaussian noise is required in constructing the smoothed classifier,
        the data augmentation jointly adds Gaussian noise and the transformation to certify against.
        The detailed hyperparameters such as distribution type and variance are listed in \Cref{tab:main-setting}.
        The weight of Consistency regularization is set to $10$ throughout the training.

        On the large ImageNet dataset, we finetune the existing trained models.
        For resolvable transformations, we finetune from ResNet-50 model in \texttt{torchvision} library~\cite{torchvisionzoo}.
        For differentially resolvable transformations, since the base classifier should also be robust under Gaussian noise, we finetune from Resnet-50 model in \cite{salman2019provably} that achieves state-of-the-art robustness under Gaussian noise.
        In either case, we follow the same data augmentation scheme as on MNIST and CIFAR-10, and we finetune for two epochs with batch size~($B$) $128$, learning rate $0.001$, and Consistency regularization weight $10$.
        During certification~(e.g., \Cref{algo:progressive-sampling-cert}), we use the same batch sizes as during training on these datasets.

        The channel-wise normalization is used for all models on these three datasets as in \cite{cohen2019certified,salman2019provably}.
        On all three datasets, in each training epoch, we feed in the whole training dataset without random shuffle.

        We remark that since our approach focuses on robustness certification and the smoothing strategy to improve certified robustness, we did not fully explore the potential of improving certified robustness from the training side, nor did we file-tune the training hyperparameters.
        Therefore, though we already achieved the state of the art by our effective robustness certification and smoothing strategies, we believe the results could be further improved by more effective training approaches.
        \fi

    \subsection{Implementation Details}
        \label{adxsec:implementation_details}
        We implement the whole approach along with the training scripts in a tool based on \texttt{PyTorch}.
        For resolvable transformations, we extend the smoothing module from Cohen et al \cite{cohen2019certified} to accommodate various smoothing strategies and smoothing distributions.
        The predict and certify modules are kept the same.
        For differentially resolvable transformations, since the stratified sampling requires $N\times n$ times of transformation to compute the interpolation error bound~(where $N$ is the number of first-level samples and $n$ the number of second-level samples), we implement a fast \texttt{C} module and integrate it to our \texttt{Python}-based tool.
        It empirically achieves $3-5$x speed gain compared with \texttt{OpenCV}\cite{opencvinterpolotation}-based transformation.
        For Lipschitz upper bound computation, since the loop in \texttt{Python} is slow, we reformulate the computation by loop-free tensor computations using \texttt{numpy}.
        It empirically achieves $20-40$x speed gain compared to the plain loop-based implementation.
        The full code implementation of our \framework tool along with all trained models are publicly available at \texttt{\small \url{https://github.com/AI-secure/semantic-randomized-smoothing}}.

    \subsection{Details on Attacks}
        \label{adxsec:attack_details}
        We use the following three attacks to evaluate the empirical accuracy of both \framework models and vanilla models: Random Attack, Random+ Attack, and PGD Attack.
        The Random Attack is used in previous work~\cite{balunovic2019certifying,fischer2019statistical} but does not consider the intrinsic characteristics of semantic transformations.
        Thus, we propose Random+ Attack and PGD Attack as the alternatives since they are adaptively designed for our smoothed \framework models and also consider the intrinsic characteristics of these transformations.
        \ifnum\ArXiv=0
        The detailed description of these three attacks can be found in the arXiv version~\cite{arxiv}.
        \fi

        \ifnum\ArXiv=1
        \subsubsection{Random Attack}
        The random attack is used to evaluate the empirical robust accuracy, which is an upper bound of the certified robust accuracy.
        The random attack reads in the clean input, and uniformly samples $N$ parameters from the pre-defined transformation parameter space to transform the input following uniform distribution.
        If the model gives a wrong prediction on any of these $N$ transformed inputs, we treat this sample as being successfully attacked; otherwise, the sample counts toward the empirical robust accuracy.
        We denote by $N$ the ``number of initial starts''.
        In the main experiments, we set $N = 100$, and in the following ablation study~(\Cref{appendix:exp-compare-of-adaptive-attacks}), we also compare the behaviors of the three attacks under $N=10/20/50$.

        For transformations with a hyper-rectangle parameter space, including brightness, contrast, scaling, rotation, Gaussian blur, and their compositions, we uniformly sample transformation parameters for each coordinate.
        For transformations with discrete parameter space, such as translation, we draw the parameter with equal probability.
        When the transformation is composed with $\ell_p$-bounded perturbations, we additionally generate the perturbation vector using FGSM attack~\cite{szegedy2013intriguing}, where the precise gradient is used for vanilla models, and the empirical mean gradient over $100$ samples is used for smoothed \framework models.

        \subsubsection{Adaptive Attack: Random+}
            The Random+ attack follows the same procedure as the Random attack.
            The only difference is that, instead of using uniform distribution for sampling transformation parameters, we use the Beta distribution $\text{Beta}(0.5, 0.5)$.

            Formally, suppose the transformation space is $[a,b]$.
            In Random attack, we generate the attack parameter $\varepsilon$ randomly as follows:
            \begin{equation}
                \varepsilon' \sim \text{Unif}(0,1), \quad \varepsilon \gets a + (b-a) \times \varepsilon'.
            \end{equation}
            In Random+ attack, we generate the attack parameter $\delta$ randomly as follows:
            \begin{equation}
                \delta' \sim \text{Beta}(0.5, 0.5), \quad \delta \gets a + (b-a) \times \delta'.
            \end{equation}
            We choose the Beta distribution because, intuitively, an adversarial example would be more likely to exist at the boundary, i.e., closer to $a$ or $b$.
            For example, suppose the rotation attacker has permitted angles in $[-r,\,r]$, then the adversarial samples may be more likely to have large rotation angle.
            As shown in \Cref{fig:compare_uniform_beta}, the Beta distribution helps to assign more mass when the parameter becomes closer to the boundary.
            Choosing other Beta distribution hyperparameters could control the trade-off on sampling weights over the boundary or over center, and we empirically find $\text{Beta}(0.5, 0.5)$ already works very well as shown in experiments~(\Cref{appendix:exp-compare-of-adaptive-attacks}).

            \begin{figure}[!t]
                \centering
                \includegraphics[width=0.7\linewidth]{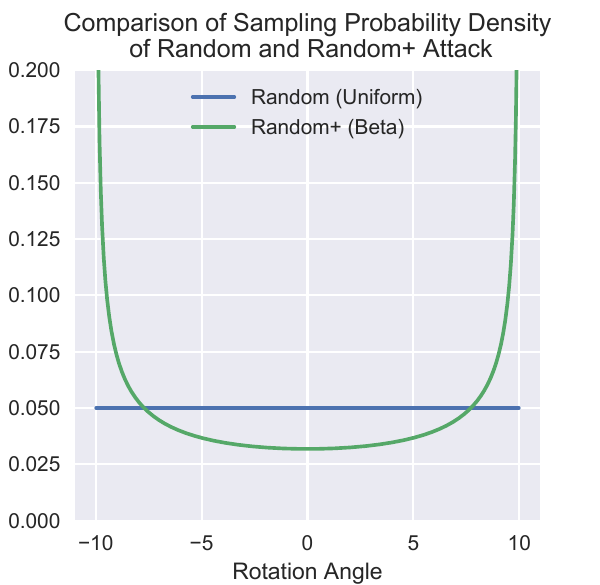}
                \caption{\small Comparison of probability density of Random and Random+ attack when attacking the rotation transformation with rotation angle between $-10^\circ$ and $+10^\circ$.}
                \label{fig:compare_uniform_beta}
            \end{figure}

        \subsubsection{Adaptive Attack: PGD}

            We propose the semantic transformation version of PGD attack as follows:
            (1)~Initialize the transformation parameter following the same process as in Random+ Attack;
            (2)~Suppose the current parameter is $(\alpha_1,\cdots,\alpha_z)$. The attack slightly perturbs each coordinate from $\alpha_i$ to $\alpha_i \pm \tau_i$ respectively and obtain $2z$ perturbed candidates.
            (3)~The attack clips each coordinate to be within the specified range, and then choose the candidate that yields the largest gain of cross-entropy loss (for vanilla models) or empirical mean cross-entropy loss (for \framework models) to update the current parameter.
            Then go to step (2).
            The loop repeats for $10$ iterations for each sample obtained in step (1).
            The perturbed step size $\tau_i = l_i / 10$ where $l_i$ is the length of the specified interval on the $i$-th coordinate.
            Finally, if the transformation is composed with $\ell_p$-bounded perturbations, we additionally generate the perturbation vector in the same way as in Random attack.

            For each Step (1) sample we would have an output and thus there are $N$ outputs.
            If any of these outputs fool the target model, we treat this sample as being successfully attacked; otherwise, the sample counts toward the empirical robust accuracy.
            Note that the translation transformation has a discrete parameter space, thus PGD attack is not applicable.

            We refer to the attack as the semantic transformation version of PGD attack because:
            (1)~It involves multiple initial starts;
            (2)~It leverages the local landscape information to maximize the loss function iteratively.
            (3)~It clips~(i.e., ``projects'') the parameters to be within the perturbation range.
            Compared to the classical PGD attack under $\ell_p$-norm constraint, we use coordinate-wise perturbations to probe the local landscape to circumvent the hardness of obtaining the gradient with respect to transformation parameters.
        \fi

    \subsection{Details on Baseline Approaches}
        \label{adxsec:baselines}
            \textbf{DeepG}~\cite{balunovic2019certifying} is based on linear relaxations.
            The code is open-sourced, and we utilize it to provide a direct comparison.
            The code provides trained models on MNIST and CIFAR-10, while on ImageNet the method is too slow and memory-consuming to run.
            On both MNIST and CIFAR-10, we use the provided trained models from the code.
            In terms of computation time, since our approach uses far less than $\SI{1000}{s}$ for certification per input on MNIST and CIFAR-10, we tune the hyperparameters to let the code spend roughly $\SI{1000}{s}$ for certification.

            \textbf{Interval}~\cite{singh2019abstract} is based on interval bound propagation.
            We also utilize the open-sourced code to provide a direct comparison.
            The settings are the same as in DeepG.

            \textbf{VeriVis}~\cite{pei2017towards} provides an enumeration-based solution when the number of possible transformation parameters or the number of possible transformed images is finite.
            In our evaluation, only translation satisfies this property.
            Therefore, as the baseline, we implement the enumeration-based robustness certification algorithm for our trained robust models.

            \textbf{Semantify-NN}~\cite{mohapatra2019towards} proposes to insert a preprocessing layer to reduce the verification against semantic transformations to the verification against classical $\ell_p$ noises.
            To our knowledge, the code has not been open-sourced yet.
            Therefore, we directly compare with the numbers reported in their paper.
            Since they report the average of certified robust magnitude, we apply Markov's inequality to obtain an upper bound of their certified robust accuracy.
            For example, they report $46.24$ degrees as the average certified robust rotation angles.
            It means that $\Prob[r \ge 50^\circ] \le \E[r] / 50 = 92.48\%$, i.e., the certified robust accuracy is no larger than $92.48\%$ when fixing the rotation angle to be $50^\circ$.

            For brightness and contrast changes, Semantify-NN considers first applying the change and then clipping to $[0, 1]$, while our TSS considers only brightness and contrast changes.
            This makes a one-to-one comparison with \cite{mohapatra2019towards} difficult, but since other baselines~(e.g., \cite{balunovic2019certifying}) consider the same setting as we do, and to align with most baselines, we slightly sacrifice comparability in this special case.
            For interested readers who would like to have an absolutely fair comparison with Semantify-NN on brightness and contrast changes, they can extend our \framework by modeling Semantify-NN's tranformation by $\phi_{BC}(x,(b,c)) \circ \phi_{clip}(x,t_l, t_h)$, where $\phi_{clip}$ clips the pixel intensities lower than $t_l$ and higher than $t_h$.
            Applying \framework-R on transformation parameters $(b,c,t_l,t_h)$ then derives the robustness certification under the same threat model as Semantify-NN.

            \textbf{DistSPT}~\cite{fischer2019statistical} combines randomized smoothing and interval bound propagation to provide certified robustness against semantic transformations.
            Concretely, the approach leverages interval bound propagation to compute the upper bound of interpolation error and then applies randomized smoothing.
            On small datasets such as MNIST and CIFAR-10, the approach is able to provide nontrivial robustness certification.
            Though the certified robust accuracy is inferior than \framework as reflected in \Cref{tab:main}.
            We use their reported numbers in \cite[Table 4]{fischer2019statistical} for DistSPT$^x$ for comparison, since the certification goal and evaluation protocol are the same as ours.
            On ImageNet, as described in \cite[Section 7.4]{fischer2019statistical}, the interval bound propagation is computationally expensive and loose.
            Therefore, they use sampling to estimate the interpolation error, which makes the robustness certification no longer hold against arbitrary attacks but just a certain random attack~(``worst-of-10'' attack).

            \textbf{IndivSPT}~\cite{fischer2019statistical} provides a different certification goal from the above approaches.
            At a high level, the approach uses a transformed image as the input where the transformation parameter is within predefined threshold.
            Then the approach certifies whether the prediction for the transformed image and the prediction for the original image are the same.
            In contrast, \framework and other baseline approaches take original image as the input and certifies whether there exists no transformed image that can mislead the model.
            Due to different certification goals, \framework is not comparable with IndivSPT.

    \subsection{Benign Accuracy}
        \label{adxsec:benign-acc}

        \Cref{tab:benign-acc} shows the benign accuracy of our models corresponding to \Cref{tab:main}.
        For comparison, the vanilla trained models have benign accuracy $98.6\%$ on MNIST, $88.6\%$ on CIFAR-10, and $74.4\%$ on ImageNet.

        We observe that though the trade-off between accuracy and (certified) robustness is widely reported both theoretically~\cite{yang2020randomized,mohapatra2020rethinking,zhang2020black} and empirically~(e.g., \cite{cohen2019certified,Zhai2020MACER,jeong2020consistency}) in classical $\ell_p$ threat model, it does not always exist in our semantic defense setting.
        Specifically, for resolvable transformations, we do not observe an apparent loss of benign accuracy in our certifiably robust models;
        while for differentially resolvable transformations~(those involving scaling and rotation), there is no loss on MNIST, slight losses on CIFAR-10, and apparent losses on ImageNet.
        When there does exist a trade-off between benign accuracy and certified robust accuracy, we show that smoothing variance levels control it in \Cref{appendix:exp-different-variance-levels}.

        \begin{table}[htbp]
            \centering
            \caption{\small Benign accuracy of our \framework models corresponding to those in \Cref{tab:main}. Certified robust accuracy shown as reference.}
            \resizebox{\linewidth}{!}{
            \begin{tabular}{c c c c c}
                \toprule
                \multirow{2}{*}{Transformation} & \multirow{2}{*}{Dataset} & \multirow{2}{*}{Attack Radius} & Certified & Benign \\
                & & & Robust Acc. & Acc. \\
                \midrule

                \multirow{3}{*}{Gaussian Blur} &
                MNIST & Squared Radius $\alpha \le 36$ & $90.6\%$ & $96.8\%$ \\
                & CIFAR-10 & Squared Radius $\alpha \le 16$ & $63.6\%$ & $76.2\%$ \\
                & ImageNet & Squared Radius $\alpha \le 36$ & $51.6\%$ & $59.2\%$ \\

                \midrule
                \multirow{3}{*}{\shortstack[c]{Translation \\ (Reflection Pad.)}} &
                MNIST & $\sqrt{\Delta x^2 + \Delta y^2} \le 8$ & $99.6\%$ & $99.6\%$ \\
                & CIFAR-10 & $\sqrt{\Delta x^2 + \Delta y^2} \le 20$ & $80.8\%$ & $87.0\%$ \\
                & ImageNet & $\sqrt{\Delta x^2 + \Delta y^2} \le 100$ & $50.0\%$ & $73.0\%$ \\

                \midrule
                \multirow{3}{*}{Brightness} &
                MNIST & $b \pm 50\%$ & $98.2\%$ & $98.2\%$ \\
                & CIFAR-10 & $b \pm 40\%$ & $87.0\%$ & $87.8\%$ \\
                & ImageNet & $b \pm 40\%$ & $70.0\%$ & $72.2\%$ \\

                \midrule
                \multirow{3}{*}{\shortstack[c]{Contrast\\ and\\ Brightness}} &
                MNIST & $c\pm 50\%, b\pm 50\%$ & $97.6\%$ & $98.0\%$ \\
                & CIFAR-10 & $c\pm 40\%, b\pm 40\%$ & $82.4\%$ & $86.8\%$ \\
                & ImageNet & $c\pm 40\%, b\pm 40\%$ & $61.4\%$ & $72.2\%$ \\

                \midrule
                \multirow{3}{*}{\shortstack[c]{Gaussian Blur,\\ Translation, Bright-\\ ness, and Contrast}} &  MNIST & $\alpha \le 1, c,b \pm 10\%, \sqrt{\Delta x^2 + \Delta y^2} \le 5$ &  $90.2\%$ &  $98.2\%$ \\
                &  CIFAR-10 &  $\alpha \le 1, c,b \pm 10\%, \sqrt{\Delta x^2 + \Delta y^2} \le 5$ &  $58.2\%$ &  $77.6\%$ \\
                &  ImageNet &  $\alpha \le 10, c,b\pm 20\%, \sqrt{\Delta x^2 + \Delta y^2} \le 10$ &  $32.8\%$ &  $61.6\%$ \\

                \midrule
                \multirow{4}{*}{Rotation} &
                MNIST & $r\pm 50^\circ$ & $97.4\%$ & $99.4\%$ \\
                & \multirow{2}{*}{CIFAR-10} & $r\pm 10^\circ$ & $70.6\%$ & $83.2\%$ \\
                & & $r\pm 30^\circ$ & $63.6\%$ & $82.6\%$ \\
                & ImageNet & $r\pm 30^\circ$ & $30.4\%$ & $46.2\%$ \\

                \midrule
                \multirow{3}{*}{Scaling} &
                MNIST & $s\pm 30\%$ & $99.0\%$ & $99.4\%$ \\
                & CIFAR-10 & $s\pm 30\%$ & $58.8\%$ & $79.8\%$ \\
                & ImageNet & $s\pm 30\%$ & $26.4\%$ & $50.8\%$ \\

                \midrule
                \multirow{4}{*}{\shortstack[c]{Rotation\\ and\\ Brightness}} &
                MNIST & $r\pm 50^\circ, b\pm 20\%$ & $97.0\%$ & $99.4\%$ \\
                & \multirow{2}{*}{CIFAR-10} & $r\pm 10^\circ, b\pm 10\%$ & $70.2\%$ & $83.0\%$ \\
                & & $r\pm 30^\circ, b\pm 20\%$ & $61.4\%$ & $82.6\%$ \\
                & ImageNet & $r\pm 30^\circ, b\pm 20\%$ & $26.8\%$ & $45.8\%$ \\

                \midrule
                \multirow{3}{*}{\shortstack[c]{Scaling\\ and\\ Brightness}} &
                MNIST & $s\pm 50\%, b\pm 50\%$ & $96.6\%$ & $99.4\%$ \\
                & CIFAR-10 & $s\pm 30\%, b\pm 30\%$ & $54.2\%$ & $79.6\%$ \\
                & ImageNet & $s\pm 30\%, b\pm 30\%$ & $23.4\%$ & $50.8\%$ \\

                \midrule
                \multirow{4}{*}{\shortstack[c]{Rotation,\\ Brightness,\\ and $\ell_2$}} &
                MNIST & $r\pm 50^\circ, b\pm 20\%, \|\delta\|_2 \le .05$ & $96.6\%$ & $99.4\%$ \\
                & \multirow{2}{*}{CIFAR-10} & $r\pm 10^\circ, b\pm 10\%, \|\delta\|_2\le .05$ & $64.2\%$ & $83.0\%$ \\
                & & $r\pm 30^\circ, b\pm 20\%, \|\delta\|_2\le .05$ & $55.2\%$ & $82.6\%$ \\
                & ImageNet & $r\pm 30^\circ, b\pm 20\%, \|\delta\|_2\le .05$ & $26.6\%$
         & $45.8\%$ \\

                \midrule
                \multirow{3}{*}{\shortstack[c]{Scaling,\\ Brightness,\\ and $\ell_2$}} &
                MNIST & $s\pm 50\%, b\pm 50\%, \|\delta\|_2\le .05$ & $96.4\%$ & $99.4\%$ \\
                & CIFAR-10 & $s\pm 30\%, b\pm 30\%, \|\delta\|_2\le .05$ & $51.2\%$ & $79.6\%$ \\
                & ImageNet & $s\pm 30\%, b\pm 30\%, \|\delta\|_2\le .05$ & $22.6\%$ & $50.8\%$ \\

                \bottomrule
            \end{tabular}
            }
            \label{tab:benign-acc}
        \end{table}

    \subsection{Smoothing Distributions and Running Time Statistics}
        \label{appendix:exp-dist-running-time}
        In \Cref{tab:main-setting}, we present the smoothing distributions with concrete parameters and average certification computing time per sample for results in main table~(\Cref{tab:main}).
        In the table,
        $\alpha$ is for squared kernel radius for Gaussian blur;
        $\Delta x$ and $\Delta y$ are for translation displacement on horizontal and vertical direction;
        $b$ and $c$ are for brightness shift and contrast change respectively as in $x \mapsto (1+c)x + b$;
        $r$ is for rotation angle;
        $s$ is for size scaling ratio;
        $\varepsilon$ is for additive noise vector;
        and $\|\delta\|_2$ for $\ell_2$ norm of permitted additional perturbations.
        Specifically, ``Training Distribution'' stands for
        the distributions for data augmentation during training the base classifiers;
        and ``Smoothing Distribution'' stands for the distributions for constructing the smoothed classifiers for certification.

        We select these distributions according to the principles in \Cref{adxsec:model-preparation}.

        \newcolumntype{a}{>{\gray}c}

\begin{table*}
    \caption{\small Detailed smoothing distributions and running time statistics for our \framework.
    $\gN(\mu,\Sigma)$ is the normal distribution, $\exp(\lambda)$ is the exponential distribution, $\gU([a,b])$ is the uniform distribution.
    Random variable $\epsilon$ is the elementwise noise as in \Cref{cor:rotations_scaling_certificate}.
    ``Cert.'' means certification.
    }

    \centering
    \resizebox{!}{0.45\textheight}{
    \begin{tabular}{c c c c c c}
        \toprule
        \multirow{3}{*}{Transformation} & \multirow{3}{*}{Dataset} & \multirow{3}{*}{Attack Radius} & \multirow{3}{*}{\shortstack[c]{Training\\ Distribution}} & \multirow{3}{*}{\shortstack[c]{Smoothing\\ Distribution}} & Avg. Cert. \\
        & & & & & Time per \\
        & & & & & Sample  \\

        \arrayrulecolor{black}
        \midrule
        \arrayrulecolor{rgray}
        \multirow{3}{*}{Gaussian Blur}
        & MNIST & Squared Radius $\alpha \le 36$ & \multicolumn{2}{c}{$\alpha\sim\Exp(1/10)$} & $\SI{7.9}{s}$  \\
        \cline{2-6}
        & CIFAR-10 & Squared Radius $\alpha \le 16$ & \multicolumn{2}{c}{$\alpha\sim\Exp(1/5)$} & $\SI{30.9}{s}$ \\
        \cline{2-6}
        & ImageNet & Squared Radius $\alpha \le 36$ & \multicolumn{2}{c}{$\alpha\sim\Exp(1/10)$} & $\SI{45.7}{s}$ \\

        \arrayrulecolor{black}
        \midrule
        \arrayrulecolor{rgray}
        \multirow{3}{*}{\shortstack[c]{Translation\\ (Reflection Pad.)}}
        & MNIST & $\sqrt{\Delta x^2 + \Delta y^2} \le 8$ & \multicolumn{2}{c}{$(\Delta x,\Delta y)\sim\gN(0, 10^2 I)$} & $\SI{10.2}{s}$  \\
        \cline{2-6}
        & CIFAR-10 & $\sqrt{\Delta x^2 + \Delta y^2} \le 20$ & \multicolumn{2}{c}{$(\Delta x,\Delta y)\sim\gN(0, 15^2 I)$} & $\SI{39.4}{s}$ \\
        \cline{2-6}
        & ImageNet & $\sqrt{\Delta x^2 + \Delta y^2} \le 100$ & \multicolumn{2}{c}{$(\Delta x,\Delta y)\sim\gN(0, 30^2 I)$} & $\SI{161.9}{s}$ \\

        \arrayrulecolor{black}
        \midrule
        \arrayrulecolor{rgray}
        \multirow{3}{*}{Brightness}
        & MNIST & $b\pm 50\%$ & \multicolumn{2}{c}{$b \sim \gN(0, 0.6^2)$} & $\SI{2.1}{s}$ \\
        \cline{2-6}
        & CIFAR-10 & $b\pm 40\%$ & \multicolumn{2}{c}{$b \sim \gN(0, 0.3^2)$} & $\SI{4.4}{s}$ \\
        \cline{2-6}
        & ImageNet & $b\pm 40\%$ & \multicolumn{2}{c}{$b \sim \gN(0, 0.4^2)$} & $\SI{45.1}{s}$ \\

        \arrayrulecolor{black}
        \midrule
        \arrayrulecolor{rgray}
        \multirow{3}{*}{\shortstack[c]{Contrast\\ and\\ Brightness}}
        & MNIST & $c\pm 50\%, b\pm 50\%$ & $(c,b)\sim\gN(0, 0.6^2 I)$ & $(c,b)\sim\gN(0, 1.0^2 I)$ & $\SI{9.8}{s}$ \\
        \cline{2-6}
        & CIFAR-10 & $c\pm 40\%, b\pm 40\%$ & $(c,b)\sim\gN(0, 0.4^2 I)$ & $(c,b)\sim\gN(0, 0.6^2 I)$ & $\SI{45.0}{s}$ \\
        \cline{2-6}
        & ImageNet & $c\pm 40\%, b\pm 40\%$ & \multicolumn{2}{c}{$(c,b)\sim\gN(0,0.4^2 I)$} & $\SI{325.6}{s}$ \\

        \arrayrulecolor{black}
        \midrule
        \arrayrulecolor{rgray}

        \multirow{12}{*}{\shortstack[c]{Gaussian Blur,\\ Translation, Bright-\\ ness, and Contrast}} &
        \multirow{4}{*}{MNIST} & \multirow{4}{*}{\shortstack[c]{$\alpha \le 5, c,b \pm 10\%$,\\ $\sqrt{\Delta x^2 + \Delta y^2} \le 5$}} &
        $\alpha \sim \Exp(1/10)$ & \multirow{4}{*}{\shortstack[c]{$\alpha \sim \Exp(1/10)$\\ $(\Delta x, \Delta y) \sim \gN(0, 10^2I)$\\ $(c,b) \sim \gN(0, 0.3^2I)$}} & \multirow{4}{*}{$\SI{12.9}{s}$} \\
        & & & $(\Delta x, \Delta y) \sim \gN(0, 10^2I)$ & \\
        & & & $(c,b) \sim \gN(0, 0.3^2I)$ & \\
        & & & $\epsilon\sim\gN(0, 0.05^2 I)$ & \\
        \cline{2-6}
        & \multirow{4}{*}{CIFAR-10} & \multirow{4}{*}{\shortstack[c]{$\alpha \le 1, c,b \pm 10\%$,\\ $\sqrt{\Delta x^2 + \Delta y^2} \le 5$}} &
        $\alpha \sim \Exp(1)$ & \multirow{4}{*}{\shortstack[c]{$\alpha \sim \Exp(1)$\\ $(\Delta x, \Delta y) \sim \gN(0, 10^2I)$\\ $(c,b) \sim \gN(0, 0.3^2I)$}} & \multirow{4}{*}{$\SI{43.1}{s}$} \\
        & & & $(\Delta x, \Delta y) \sim \gN(0, 10^2I)$ & \\
        & & & $(c,b) \sim \gN(0, 0.3^2I)$ \\
        & & & $\epsilon\sim\gN(0, 0.01^2 I)$ & \\
        \cline{2-6}
        & \multirow{4}{*}{ImageNet} & \multirow{4}{*}{\shortstack[c]{$\alpha \le 10, c,b \pm 20\%$,\\ $\sqrt{\Delta x^2 + \Delta y^2} \le 10$}} &
        $\alpha \sim \Exp(1/5)$ & \multirow{4}{*}{\shortstack[c]{$\alpha \sim \Exp(1/5)$\\ $(\Delta x, \Delta y) \sim \gN(0, 20^2I)$\\ $(c,b) \sim \gN(0, 0.4^2I)$}} & \multirow{4}{*}{$\SI{238.1}{s}$} \\
        & & & $(\Delta x, \Delta y) \sim \gN(0, 20^2I)$ & \\
        & & & $(c,b) \sim \gN(0, 0.4^2I)$ & \\
        & & & $\epsilon\sim\gN(0, 0.01^2 I)$ & \\

        \arrayrulecolor{black}
        \midrule
        \arrayrulecolor{rgray}
        \multirow{4}{*}{Rotation}
        & MNIST & $r\pm 50^\circ$ & \multirow{4}{*}{\shortstack[c]{Same as \\ Rotation and \\ Brightness }} & $\epsilon\sim\gN(0,0.12^2I)$ & $\SI{20.1}{s}$ \\
        \cline{2-3}\cline{5-6}
        & \multirow{2}{*}{CIFAR-10} & $r\pm 10^\circ$ & & $\epsilon\sim\gN(0,0.05^2I)$ & $\SI{52.8}{s}$ \\
        \cline{3-3}\cline{5-6}
        & & $r\pm 30^\circ$ & & $\epsilon\sim\gN(0,0.05^2I)$ & $\SI{141.0}{s}$ \\
        \cline{2-3}\cline{5-6}
        & ImageNet & $r\pm 30^\circ$ & & $\epsilon\sim\gN(0,0.5^2I)$ & $\SI{2358.1}{s}$ \\

        \arrayrulecolor{black}
        \midrule
        \arrayrulecolor{rgray}
        \multirow{3}{*}{Scaling}
        & MNIST & $s\pm 30\%$ & \multirow{3}{*}{\shortstack[c]{Same as \\ Scaling and \\ Brightness}} & $\epsilon\sim\gN(0,0.12^2I)$ & $\SI{17.7}{s}$ \\
        \cline{2-3}\cline{5-6}
        & CIFAR-10 & $s\pm 30\%$ & & $\epsilon\sim\gN(0,0.12^2I)$ & $\SI{42.2}{s}$ \\
        \cline{2-3}\cline{5-6}
        & ImageNet & $s\pm 30\%$ & & $\epsilon\sim\gN(0,0.5^2I)$ & $\SI{1201.2}{s}$ \\

        \arrayrulecolor{black}
        \midrule
        \arrayrulecolor{rgray}
        \multirow{12}{*}{\shortstack[c]{Rotation\\ and\\ Brightness}}
        & \multirow{3}{*}{MNIST} & \multirow{3}{*}{$r\pm 50^\circ, b\pm 20\%$} &
        $r \sim \gU([-55,55])$ & \multirow{3}{*}{\shortstack[c]{$\epsilon \sim \gN(0,0.12^2 I)$\\ $b\sim\gN(0,0.2^2)$}} & \multirow{3}{*}{$\SI{31.4}{s}$} \\
        & & & $\epsilon \sim \gN(0,0.12^2 I)$ \\
        & & & $b \sim \gN(0,0.2^2)$ \\
        \cline{2-6}
        & \multirow{6}{*}{CIFAR-10} & \multirow{3}{*}{$r\pm 10^\circ, b\pm 10\%$} &
        $r \sim \gU([-12.5,12.5])$ & \multirow{3}{*}{\shortstack[c]{$\epsilon \sim \gN(0,0.05^2 I)$\\ $b\sim\gN(0,0.2^2)$}} & \multirow{3}{*}{$\SI{62.3}{s}$} \\
        & & & $\epsilon \sim \gN(0,0.05^2 I)$ \\
        & & & $b \sim \gN(0,0.2^2)$ \\
        \cline{3-6}
        &  & \multirow{3}{*}{$r\pm 30^\circ, b\pm 20\%$} &
        $r \sim \gU([-35,35])$ & \multirow{3}{*}{\shortstack[c]{$\epsilon \sim \gN(0,0.05^2 I)$\\ $b\sim\gN(0,0.2^2)$}} & \multirow{3}{*}{$\SI{157.0}{s}$} \\
        & & & $\epsilon \sim \gN(0,0.05^2 I)$ \\
        & & & $b \sim \gN(0,0.2^2)$ \\
        \cline{2-6}
        & \multirow{3}{*}{ImageNet} & \multirow{3}{*}{$r\pm 30^\circ, b\pm 20\%$} &
        $r \sim \gU([-35,35])$ & \multirow{3}{*}{\shortstack[c]{$\epsilon \sim \gN(0,0.5^2 I)$\\ $b\sim\gN(0,0.2^2)$}} & \multirow{3}{*}{$\SI{2475.6}{s}$} \\
        & & & $\epsilon \sim \gN(0,0.5^2 I)$ \\
        & & & $b \sim \gN(0,0.2^2)$ \\

        \arrayrulecolor{black}
        \midrule
        \arrayrulecolor{rgray}
        \multirow{9}{*}{\shortstack[c]{Scaling\\ and\\ Brightness}}
        & \multirow{3}{*}{MNIST} & \multirow{3}{*}{$s\pm 50\%, b\pm 50\%$} &
        $s \sim \gU([0.45, 1.55])$ & \multirow{3}{*}{\shortstack[c]{$\epsilon\sim\gN(0,0.12^2I)$\\ $b\sim\gN(0,0.5^2)$}} & \multirow{3}{*}{$\SI{74.9}{s}$} \\
        & & & $\epsilon \sim \gN(0,0.12^2I)$ \\
        & & & $b \sim \gN(0, 0.5^2)$ \\
        \cline{2-6}
        & \multirow{3}{*}{CIFAR-10} & \multirow{3}{*}{$s\pm 30\%, b\pm 30\%$} &
        $s \sim \gU([0.65, 1.35])$ & \multirow{3}{*}{\shortstack[c]{$\epsilon\sim\gN(0,0.12^2I)$\\ $b\sim\gN(0,0.3^2)$}} & \multirow{3}{*}{$\SI{44.5}{s}$} \\
        & & & $\epsilon \sim \gN(0,0.12^2I)$ \\
        & & & $b \sim \gN(0, 0.3^2)$ \\
        \cline{2-6}
        & \multirow{3}{*}{ImageNet} & \multirow{3}{*}{$s\pm 30\%, b\pm 30\%$} & $s \sim \gU([0.65,1.35])$ & \multirow{3}{*}{\shortstack[c]{$\epsilon\sim\gN(0,0.5^2I)$\\ $b\sim\gN(0,0.3^2)$}} & \multirow{3}{*}{$\SI{1401.6}{s}$} \\
        & & & $\epsilon \sim \gN(0,0.5^2I)$ \\
        & & & $b \sim \gN(0, 0.3^2)$ \\

        \arrayrulecolor{black}
        \midrule
        \arrayrulecolor{rgray}
        \multirow{4}{*}{\shortstack[c]{Rotation,\\ Brightness,\\ and $\ell_2$}}
        & MNIST & $r\pm 50^\circ, b\pm 20\%, \|\delta\|_2\le .05$ & \multirow{4}{*}{\shortstack[c]{Same as\\ Rotation and\\ Brightness}} & \multirow{4}{*}{\shortstack[c]{Same as\\ Rotation and\\ Brightness}} & $\SI{35.1}{s}$ \\
        \cline{2-3}\cline{6-6}
        & \multirow{2}{*}{CIFAR-10} & $r\pm 10^\circ, b\pm 10\%, \|\delta\|_2\le .05$ & & & $\SI{132.5}{s}$ \\
        \cline{3-3}\cline{6-6}
        & & $r\pm 30^\circ, b\pm 20\%, \|\delta\|_2\le .05$ & & & $\SI{520.2}{s}$ \\
        \cline{2-3}\cline{6-6}
        & ImageNet & $r\pm 30^\circ, b\pm 20\%, \|\delta\|_2\le .05$ & & & $\SI{3463.8}{s}$ \\

        \arrayrulecolor{black}
        \midrule
        \arrayrulecolor{rgray}
        \multirow{3}{*}{\shortstack[c]{Scaling,\\ Brightness,\\ and $\ell_2$}}
        & MNIST & $s\pm 50\%, b\pm 50\%, \|\delta\|_2\le .05$ & \multirow{3}{*}{\shortstack[c]{Same as\\ Scaling and\\ Brightness}} & \multirow{3}{*}{\shortstack[c]{Same as\\ Scaling and\\ Brightness}} & $\SI{75.1}{s}$ \\
        \cline{2-3}\cline{6-6}
        & CIFAR-10 & $s\pm 30\%, b\pm 30\%, \|\delta\|_2\le .05$ & & & $\SI{50.0}{s}$ \\
        \cline{2-3}\cline{6-6}
        & ImageNet & $s\pm 30\%, b\pm 30\%, \|\delta\|_2\le .05$ & & & $\SI{1657.7}{s}$ \\

        \arrayrulecolor{black}
        \bottomrule
    \end{tabular}
    }
    \label{tab:main-setting}
\end{table*}

        \subsection{Comparison of Random Attack and Adaptive Attacks: Detail}
            \label{appendix:exp-compare-of-adaptive-attacks}

            \ifnum\ArXiv=0
                Detailed results are omitted to the corresponding section of the arXiv version~\cite{arxiv}.
            \fi

            \ifnum\ArXiv=1
            In \Cref{tab:main}, we compare the empirical robust accuracy of vanilla models and \framework models under random attacks and two adaptive attacks: Random+ and PGD.
            However, due to space limits, we omit the empirical accuracy of each adaptive attack and we just present the minimum empirical accuracy among them.
            In \Cref{tab:attack-detail-mnist}, \Cref{tab:attack-detail-cifar10}, and \Cref{tab:attack-detail-imagenet}, for all transformations on MNIST, CIFAR-10, and ImageNet respectively, we present the detailed empirical accuracy of each attack under different number of initial starts $N=10/20/50/100$.
            Note that the main table~(\Cref{tab:main}) shows the empirical accuracy with $N=100$ for all attacks.

            From these three tables, we cross-validate the findings shown in the main paper:
            the adaptive attack decreases the empirical accuracy of \framework models slightly, while it decreases that of vanilla models more.
            Moreover, when comparing these three attacks, we find that with a small number of initial starts~(e.g., $N=10$), the PGD attack is typically the most powerful.
            However, with a large number of initial starts~(e.g., $N=100$), Random+ attack sometimes becomes better.
            We conjecture that the optimization goal of PGD attack---maximization of cross-entropy loss---might be sub-optimal in terms of increasing the misclassification rate.
            Thus, with a small number of initial starts, PGD is better than Random/Random+ attack due to the iterative ascending.
            However, with a large number of initial starts, both PGD and Random+ attack can sufficiently explore the adversarial region, and PGD may be misled by the optimization goal to a benign region.
            It would be an interesting future work to study these intriguing properties of semantic attacks.

    \begin{table}[htbp]
        \centering
        \caption{\small Comparison between empirical robust accuracy against random and adaptive attacks and certified robust accuracy on \textbf{MNIST}. The attack radii are consistent with \Cref{tab:main}. The most powerful attack in each setting is highlighted in bold font. The adaptive attacks are shown in gray rows. Note that PGD attack cannot apply to translation transformation because the parameter space is discrete.
        }
        \resizebox{\linewidth}{!}{
        \begin{tabular}{ccccccccc}
            \toprule
            \multirow{3}{*}{Transformation} & \multirow{3}{*}{\shortstack[c]{Attack\\ Radius}} & \multirow{3}{*}{Model} & \multirow{3}{*}{Attack} & \multicolumn{4}{c}{Empirical Robust Accuracy} & Certified \\
            \cline{4-7}
            & & & \multicolumn{4}{c}{Initial Starts $N =$} & Robust \\
            & & & $10$ & $20$ & $50$ & $100$ & Accuracy \\
            \midrule
        
            \multirow{6}{*}{\shortstack[c]{Gaussian\\ Blur}}
            & \multirow{6}{*}{\shortstack[c]{Squared\\ Radius\\ $\alpha \le 36$}} & \multirow{3}{*}{\textbf{TSS}}
              & Random & $93.2\%$ & $92.2\%$ & $92.0\%$ & $91.4\%$ & \multirow{3}{*}{$90.6\%$} \\
            & & & \gray Random+ & \gray $92.4\%$ & \gray $92.2\%$ & \gray $\textbf{91.2\%}$ & \gray $\textbf{91.2\%}$ & \\
            & & & \gray PGD     & \gray $\textbf{91.6\%}$ & \gray $\textbf{91.6\%}$ & \gray $91.6\%$ & \gray $91.6\%$ & \\
            
            \cline{3-9}
            & & \multirow{3}{*}{Vanilla}
              & Random & $14.0\%$ & $12.4\%$ & $\textbf{12.2\%}$ & $\textbf{12.2\%}$ & \multirow{3}{*}{-} \\
            & & & \gray Random+ & \gray $12.4\%$ & \gray $12.4\%$ & \gray $\textbf{12.2\%}$ & \gray $\textbf{12.2\%}$ & \\
            & & & \gray PGD     & \gray $\textbf{12.2\%}$ & \gray $\textbf{12.2\%}$ & \gray $\textbf{12.2\%}$ & \gray $\textbf{12.2\%}$ & \\
            \midrule

            \multirow{6}{*}{\shortstack[c]{Translation\\ (Reflection Pad.)}}
            & \multirow{6}{*}{\shortstack[c]{$\sqrt{\Delta x^2 + \Delta y^2}$\\ $\le 8$}} & \multirow{3}{*}{\textbf{TSS}}
              & Random & $\textbf{99.6\%}$ & $\textbf{99.6\%}$ & $\textbf{99.6\%}$ & $\textbf{99.6\%}$ & \multirow{3}{*}{$99.6\%$} \\
            & & & \gray Random+ & \gray $\textbf{99.6\%}$ & \gray $\textbf{99.6\%}$ & \gray $\textbf{99.6\%}$ & \gray $\textbf{99.6\%}$ & \\
            & & & \gray PGD     & \gray - & \gray - & \gray - & \gray - & \\
            
            \cline{3-9}
            & & \multirow{3}{*}{Vanilla}
              & Random & $\textbf{0.0\%}$ & $\textbf{0.0\%}$ & $\textbf{0.0\%}$ & $\textbf{0.0\%}$ & \multirow{3}{*}{-} \\
            & & & \gray Random+ & \gray $\textbf{0.0\%}$ & \gray $\textbf{0.0\%}$ & \gray $\textbf{0.0\%}$ & \gray $\textbf{0.0\%}$ & \\
            & & & \gray PGD     & \gray - & \gray - & \gray - & \gray - & \\
            \midrule

            \multirow{6}{*}{\shortstack[c]{Brightness}}
            & \multirow{6}{*}{\shortstack[c]{$b\pm 50\%$}} & \multirow{3}{*}{\textbf{TSS}}
              & Random & $\textbf{98.2\%}$ & $\textbf{98.2\%}$ & $\textbf{98.2\%}$ & $\textbf{98.2\%}$ & \multirow{3}{*}{$98.2\%$} \\
            & & & \gray Random+ & \gray $\textbf{98.2\%}$ & \gray $\textbf{98.2\%}$ & \gray $\textbf{98.2\%}$ & \gray $\textbf{98.2\%}$ & \\
            & & & \gray PGD     & \gray $\textbf{98.2\%}$ & \gray $\textbf{98.2\%}$ & \gray $\textbf{98.2\%}$ & \gray $\textbf{98.2\%}$ & \\
            
            \cline{3-9}
            & & \multirow{3}{*}{Vanilla}
              & Random & $97.2\%$ & $\textbf{96.6\%}$ & $\textbf{96.6\%}$ & $\textbf{96.6\%}$ & \multirow{3}{*}{-} \\
            & & & \gray Random+ & \gray $96.8\%$ & \gray $\textbf{96.6\%}$ & \gray $\textbf{96.6\%}$ & \gray $\textbf{96.6\%}$ & \\
            & & & \gray PGD     & \gray $\textbf{96.6\%}$ & \gray $\textbf{96.6\%}$ & \gray $\textbf{96.6\%}$ & \gray $\textbf{96.6\%}$ & \\
            \midrule

            \multirow{6}{*}{\shortstack[c]{Contrast\\ and\\ Brightness}}
            & \multirow{6}{*}{\shortstack[c]{$c\pm 50\%$, \\ $b\pm 50\%$}} & \multirow{3}{*}{\textbf{TSS}}
              & Random & $\textbf{98.0\%}$ & $\textbf{98.0\%}$ & $\textbf{98.0\%}$ & $\textbf{98.0\%}$ & \multirow{3}{*}{$97.6\%$} \\
            & & & \gray Random+ & \gray $\textbf{98.0\%}$ & \gray $\textbf{98.0\%}$ & \gray $\textbf{98.0\%}$ & \gray $\textbf{98.0\%}$ & \\
            & & & \gray PGD     & \gray $\textbf{98.0\%}$ & \gray $\textbf{98.0\%}$ & \gray $\textbf{98.0\%}$ & \gray $\textbf{98.0\%}$ & \\
            
            \cline{3-9}
            & & \multirow{3}{*}{Vanilla}
              & Random & $96.8\%$ & $95.8\%$ & $95.0\%$ & $94.6\%$ & \multirow{3}{*}{-} \\
            & & & \gray Random+ & \gray $95.8\%$ & \gray $94.4\%$ & \gray $93.8\%$ & \gray $93.6\%$ & \\
            & & & \gray PGD     & \gray $\textbf{93.6\%}$ & \gray $\textbf{93.4\%}$ & \gray $\textbf{93.2\%}$ & \gray $\textbf{93.2\%}$ & \\
            \midrule

            \multirow{6}{*}{\shortstack[c]{Gaussian Blur,\\ Translation\\ Contrast,\\ and Brightness}}
            & \multirow{6}{*}{\shortstack[c]{$\alpha\le 5$, \\ $c\pm 10\%$,\\ $b\pm 10\%$,\\ $\sqrt{\Delta x^2 + \Delta y^2}$\\ $\le 5$}} & \multirow{3}{*}{\textbf{TSS}}
              & Random & $97.6\%$ & $97.6\%$ & $97.6\%$ & $97.2\%$ & \multirow{3}{*}{$90.2\%$} \\
            & & & \gray Random+ & \gray $97.6\%$ & \gray $\textbf{97.2\%}$ & \gray $\textbf{97.0\%}$ & \gray $\textbf{97.0\%}$ & \\
            & & & \gray PGD     & \gray $\textbf{97.4\%}$ & \gray $97.4\%$ & \gray $97.2\%$ & \gray $\textbf{97.0\%}$ & \\
            
            \cline{3-9}
            & & \multirow{3}{*}{Vanilla}
              & Random & $10.0\%$ & $4.4\%$ & $1.4\%$ & $\textbf{0.4\%}$ & \multirow{3}{*}{-} \\
            & & & \gray Random+ & \gray $\textbf{6.8\%}$ & \gray $2.4\%$ & \gray $1.2\%$ & \gray $\textbf{0.4\%}$ & \\
            & & & \gray PGD     & \gray $7.0\%$ & \gray $\textbf{1.4\%}$ & \gray $\textbf{0.8\%}$ & \gray $\textbf{0.4\%}$ & \\
            \midrule

            \multirow{6}{*}{\shortstack[c]{Rotation}}
            & \multirow{6}{*}{\shortstack[c]{$r\pm 50^\circ$}} & \multirow{3}{*}{\textbf{TSS}}
              & Random & $98.6\%$ & $\textbf{98.4\%}$ & $\textbf{98.2\%}$ & $98.4\%$ & \multirow{3}{*}{$97.4\%$} \\
            & & & \gray Random+ & \gray $98.6\%$ & \gray $98.6\%$ & \gray $98.4\%$ & \gray $\textbf{98.2\%}$ & \\
            & & & \gray PGD     & \gray $\textbf{98.2\%}$ & \gray $\textbf{98.4\%}$ & \gray $98.4\%$ & \gray $\textbf{98.2\%}$ & \\
            
            \cline{3-9}
            & & \multirow{3}{*}{Vanilla}
              & Random & $27.2\%$ & $17.4\%$ & $13.8\%$ & $12.2\%$ & \multirow{3}{*}{-} \\
            & & & \gray Random+ & \gray $\textbf{15.4\%}$ & \gray $\textbf{13.0\%}$ & \gray $\textbf{11.0\%}$ & \gray $\textbf{11.0\%}$ & \\
            & & & \gray PGD     & \gray $16.4\%$ & \gray $15.6\%$ & \gray $15.4\%$ & \gray $15.2\%$ & \\
            \midrule

            \multirow{6}{*}{\shortstack[c]{Scaling}}
            & \multirow{6}{*}{\shortstack[c]{$s\pm 30\%$}} & \multirow{3}{*}{\textbf{TSS}}
              & Random & $\textbf{99.2\%}$ & $\textbf{99.2\%}$ & $\textbf{99.2\%}$ & $\textbf{99.2}\%$ & \multirow{3}{*}{$97.2\%$} \\
            & & & \gray Random+ & \gray $\textbf{99.2\%}$ & \gray $\textbf{99.2\%}$ & \gray  $\textbf{99.2\%}$& \gray $\textbf{99.2\%}$ & \\
            & & & \gray PGD     & \gray $\textbf{99.2\%}$ & \gray $\textbf{99.2\%}$ & \gray $\textbf{99.2\%}$ & \gray $\textbf{99.2\%}$ & \\
            
            \cline{3-9}
            & & \multirow{3}{*}{Vanilla}
              & Random & $92.0\%$ & $91.4\%$ & $90.2\%$ & $90.2\%$ & \multirow{3}{*}{-} \\
            & & & \gray Random+ & \gray $\textbf{90.0\%}$ & \gray $\textbf{89.4\%}$ & \gray $\textbf{89.2\%}$ & \gray $\textbf{89.2\%}$ & \\
            & & & \gray PGD     & \gray $90.4\%$ & \gray $90.2\%$ & \gray $90.2\%$ & \gray $90.2\%$ & \\
            \midrule

            \multirow{6}{*}{\shortstack[c]{Rotation\\ and\\ Brightness}}
            & \multirow{6}{*}{\shortstack[c]{$r\pm 50\%$,\\ $b\pm 20\%$}} & \multirow{3}{*}{\textbf{TSS}}
              & Random & $98.8\%$ & $98.4\%$ & $98.2\%$ & $98.2\%$ & \multirow{3}{*}{$97.0\%$} \\
            & & & \gray Random+ & \gray $98.6\%$ & \gray $98.2\%$ & \gray $\textbf{98.0\%}$ & \gray $98.2\%$ & \\
            & & & \gray PGD     & \gray $\textbf{98.2\%}$ & \gray $\textbf{98.0\%}$ & \gray $\textbf{98.0\%}$ & \gray $\textbf{98.0\%}$ & \\
            
            \cline{3-9}
            & & \multirow{3}{*}{Vanilla}
              & Random & $28.8\%$ & $17.8\%$ & $12.6\%$ & $11.0\%$ & \multirow{3}{*}{-} \\
            & & & \gray Random+ & \gray $16.6\%$ & \gray $\textbf{11.6\%}$ & \gray $\textbf{10.4\%}$ & \gray $\textbf{10.4\%}$ & \\
            & & & \gray PGD     & \gray $\textbf{13.4\%}$ & \gray $13.6\%$ & \gray $13.0\%$ & \gray $12.6\%$ & \\
            \midrule

            \multirow{6}{*}{\shortstack[c]{Scaling\\ and\\ Brightness}}
            & \multirow{6}{*}{\shortstack[c]{$s\pm 50\%$,\\ $b\pm 50\%$}} & \multirow{3}{*}{\textbf{TSS}}
              & Random & $98.6\%$ & $98.6\%$ & $98.4\%$ & $\textbf{97.8\%}$ & \multirow{3}{*}{$96.6\%$} \\
            & & & \gray Random+ & \gray $98.4\%$ & \gray $98.0\%$ & \gray $\textbf{97.8\%}$ & \gray $\textbf{97.8\%}$ & \\
            & & & \gray PGD     & \gray $\textbf{98.2\%}$ & \gray $\textbf{97.8\%}$ & \gray $\textbf{97.8\%}$ & \gray $\textbf{97.8\%}$ & \\
            
            \cline{3-9}
            & & \multirow{3}{*}{Vanilla}
              & Random & $57.4\%$ & $46.0\%$ & $31.0\%$ & $24.8\%$ & \multirow{3}{*}{-} \\
            & & & \gray Random+ & \gray $40.4\%$ & \gray $28.0\%$ & \gray $\textbf{19.8\%}$ & \gray $\textbf{15.6\%}$ & \\
            & & & \gray PGD     & \gray $\textbf{29.0\%}$ & \gray $\textbf{25.2\%}$ & \gray $25.0\%$ & \gray $24.0\%$ & \\
            \midrule

            \multirow{6}{*}{\shortstack[c]{Rotation,\\ Brightness, \\ and $\ell_2$}}
            & \multirow{6}{*}{\shortstack[c]{$r\pm 50\%$,\\ $b\pm 20\%$,\\ $\|\delta\|_2\le .05$}} & \multirow{3}{*}{\textbf{TSS}}
              & Random & $98.2\%$ & $97.8\%$ & $\textbf{97.6\%}$ & $97.6\%$ & \multirow{3}{*}{$96.6\%$} \\
            & & & \gray Random+ & \gray $98.4\%$ & \gray $98.0\%$ & \gray $97.8\%$ & \gray $97.6\%$ & \\
            & & & \gray PGD     & \gray $\textbf{97.6\%}$ & \gray $\textbf{97.6\%}$ & \gray $\textbf{97.6\%}$ & \gray $\textbf{97.4\%}$ & \\
            
            \cline{3-9}
            & & \multirow{3}{*}{Vanilla}
              & Random & $27.6\%$ & $17.2\%$ & $11.4\%$ & $10.8\%$ & \multirow{3}{*}{-} \\
            & & & \gray Random+ & \gray $15.2\%$ & \gray $\textbf{11.2\%}$ & \gray $\textbf{9.4\%}$ & \gray $\textbf{9.0\%}$ & \\
            & & & \gray PGD     & \gray $\textbf{13.4\%}$ & \gray $11.8\%$  & \gray $12.0\%$  & \gray $11.8\%$  & \\
            \midrule

            \multirow{6}{*}{\shortstack[c]{Scaling,\\ Brightness,\\ and $\ell_2$}}
            & \multirow{6}{*}{\shortstack[c]{$s\pm 50\%$,\\ $b\pm 50\%$,\\ $\|\delta\|_2\le .05$}} & \multirow{3}{*}{\textbf{TSS}}
              & Random & $98.4\%$ & $98.4\%$ & $\textbf{97.6\%}$ & $\textbf{97.6\%}$ & \multirow{3}{*}{$96.4\%$} \\
            & & & \gray Random+ & \gray $\textbf{97.8\%}$ & \gray $97.8\%$ & \gray $\textbf{97.6\%}$ & \gray $\textbf{97.6\%}$ & \\
            & & & \gray PGD     & \gray $\textbf{97.8\%}$ & \gray $\textbf{97.6\%}$ & \gray $\textbf{97.6\%}$ & \gray $\textbf{97.6\%}$ & \\
            
            \cline{3-9}
            & & \multirow{3}{*}{Vanilla}
              & Random & $50.4\%$ & $38.2\%$ & $28.2\%$ & $22.2\%$ & \multirow{3}{*}{-} \\
            & & & \gray Random+ & \gray $34.4\%$ & \gray $23.2\%$ & \gray $\textbf{13.4\%}$ & \gray $\textbf{12.2\%}$ & \\
            & & & \gray PGD     & \gray $\textbf{23.4\%}$ & \gray $\textbf{22.0\%}$ & \gray $21.6\%$ & \gray $20.8\%$ & \\
            
            \bottomrule
        \end{tabular}
        }
        \label{tab:attack-detail-mnist}
    \end{table}

    \begin{table}[htbp]
        \centering
        \caption{\small Comparison between empirical robust accuracy against random and adaptive attacks and certified robust accuracy on \textbf{CIFAR-10}. The attack radii are consistent with \Cref{tab:main}. The most powerful attack in each setting is highlighted in bold font. The adaptive attacks are shown in gray rows. Note that the PGD attack cannot apply to translation transformation because the parameter space is discrete. 
        }
        \resizebox{!}{0.44\textheight}{
        \begin{tabular}{ccccccccc}
            \toprule
            \multirow{3}{*}{Transformation} & \multirow{3}{*}{\shortstack[c]{Attack\\ Radius}} & \multirow{3}{*}{Model} & \multirow{3}{*}{Attack} & \multicolumn{4}{c}{Empirical Robust Accuracy} & Certified \\
            \cline{5-8}
            & & & & \multicolumn{4}{c}{Initial Starts $N =$} & Robust \\
            & & & & $10$ & $20$ & $50$ & $100$ & Accuracy \\
            \midrule

            \multirow{6}{*}{\shortstack[c]{Gaussian\\ Blur}} & \multirow{6}{*}{\shortstack[c]{Squared\\ Radius\\ $\alpha \le 16$}} & \multirow{3}{*}{\textbf{TSS}} & Random & $66.4\%$ & $66.4\%$ & $\textbf{65.8\%}$ & $\textbf{65.8\%}$ & \multirow{3}{*}{$63.6\%$} \\
            & & & \gray Random+ & \gray $66.8\%$ & \gray $66.0\%$ & \gray $\textbf{65.8\%}$ & \gray  $\textbf{65.8\%}$ & \\
            & & & \gray PGD & \gray $\textbf{65.8\%}$ & \gray $\textbf{65.8\%}$ & \gray $\textbf{65.8\%}$ & \gray $\textbf{65.8\%}$ & \\
            \cline{3-9}
            
            & &  \multirow{3}{*}{Vanilla} & Random & $4.8\%$ & $4.2\%$ & $\textbf{3.4\%}$ & $\textbf{3.4\%}$ & \multirow{3}{*}{-} \\
            & & & \gray Random+ & \gray $4.6\%$ & \gray $4.0\%$ & \gray $3.6\%$ & \gray $\textbf{3.4\%}$ & \\
            & & & \gray PGD & \gray $\textbf{3.4\%}$ & \gray $\textbf{3.4\%}$ & \gray $\textbf{3.4\%}$ & \gray $\textbf{3.4\%}$ & \\
            \midrule

            \multirow{6}{*}{\shortstack[c]{Translation\\ (Reflection Pad.)}} & \multirow{6}{*}{\shortstack[c]{$\sqrt{\Delta x^2 + \Delta y^2}$\\ $\le 20$}} & \multirow{3}{*}{\textbf{TSS}} & Random & $\textbf{86.2\%}$ & $\textbf{86.0\%}$ & $86.2\%$ & $86.2\%$ & \multirow{3}{*}{$80.8\%$} \\
            & & & \gray Random+ & \gray $86.4\%$ & \gray $\textbf{86.0\%}$ & \gray $\textbf{86.0\%}$ & \gray $\textbf{86.0\%}$ & \\
            & & & \gray PGD & \gray - & \gray - & \gray - & \gray - & \\
            \cline{3-9}
            
            & & \multirow{3}{*}{Vanilla} & Random & $\textbf{8.0\%}$ & $\textbf{7.0\%}$ & $4.4\%$ & $\textbf{4.2\%}$ & \multirow{3}{*}{-} \\
            & & & \gray Random+ & \gray $8.2\%$ & \gray $7.2\%$ & \gray $\textbf{4.2\%}$ & \gray $\textbf{4.2\%}$ & \\
            & & & \gray PGD & \gray - & \gray - & \gray - & \gray - & \\
            \midrule

            \multirow{6}{*}{\shortstack[c]{Brightness}} & \multirow{6}{*}{\shortstack[c]{$b \pm 40\%$}} & \multirow{3}{*}{\textbf{TSS}} & Random & $87.2\%$ & $87.2\%$ & $87.4\%$ & $\textbf{87.2\%}$ & \multirow{3}{*}{$87.0\%$} \\
            & & & \gray Random+ & \gray $\textbf{87.0\%}$ & \gray $\textbf{87.0\%}$ & \gray $\textbf{87.0\%}$ & \gray $87.4\%$ & \\
            & & & \gray PGD & \gray $87.4\%$ & \gray $87.4\%$ & \gray $87.4\%$ & \gray $87.4\%$ & \\
            \cline{3-9}
            
            & & \multirow{3}{*}{Vanilla} & Random & $57.8\%$ & $51.2\%$ & $45.8\%$ & $44.4\%$ & \multirow{3}{*}{-} \\
            & & & \gray Random+ & \gray $\textbf{49.8\%}$ & \gray $\textbf{44.2\%}$ & \gray $\textbf{42.8\%}$ & \gray $\textbf{42.6\%}$ & \\
            & & & \gray PGD & \gray $52.4\%$ & \gray $51.0\%$ & \gray $50.8\%$ & \gray $50.8\%$ & \\
            \midrule

            \multirow{6}{*}{\shortstack[c]{Contrast\\ and\\ Brightness}} & \multirow{6}{*}{\shortstack[c]{$c \pm 40\%$,\\ $b\pm 40\%$}} & \multirow{3}{*}{\textbf{TSS}} & Random & $86.2\%$ & $86.2\%$ & $86.2\%$ & $86.0\%$ & \multirow{3}{*}{$82.4\%$} \\
            & & & \gray Random+ & \gray $\textbf{85.8\%}$ & \gray $86.2\%$ & \gray $86.0\%$ & \gray $\textbf{85.8\%}$ & \\
            & & & \gray PGD & \gray $\textbf{85.8\%}$ & \gray $\textbf{85.8\%}$ & \gray $\textbf{85.8\%}$ & \gray $\textbf{85.8\%}$ & \\
            \cline{3-9}
            
            & & \multirow{3}{*}{Vanilla} & Random & $48.0\%$ & $40.0\%$ & $27.2\%$ & $21.0\%$ & \multirow{3}{*}{-} \\
            & & & \gray Random+ & \gray $32.0\%$ & \gray $23.2\%$ & \gray $14.8\%$ & \gray $\textbf{9.6\%}$ & \\
            & & & \gray PGD & \gray $\textbf{17.0\%}$ & \gray $\textbf{13.0\%}$ & \gray $\textbf{12.2\%}$ & \gray $11.8\%$ & \\
            \midrule

            \multirow{6}{*}{\shortstack[c]{Gaussian Blur,\\ Translation, \\ Contrast, \\ and Brightness}} & \multirow{6}{*}{\shortstack[c]{$\alpha\le 1$,\\ $c\pm 10\%$, \\ $b \pm 10\%$,\\ $\sqrt{\Delta x^2 + \Delta y^2}$\\ $\le 5$}} & \multirow{3}{*}{\textbf{TSS}} & Random & $71.0\%$ & $\textbf{69.2\%}$ & $\textbf{68.0\%}$ & $\textbf{67.6\%}$ & \multirow{3}{*}{$58.2\%$} \\
            & & & \gray Random+ & \gray $70.6\%$ & \gray $69.8\%$ & \gray $68.4\%$ & \gray $67.8\%$ & \\
            & & & \gray PGD & \gray $\textbf{69.8\%}$ & \gray $69.8\%$ & \gray $69.0\%$ & \gray $68.0\%$ & \\
            \cline{3-9}
            & & \multirow{3}{*}{Vanilla} & Random & $21.2\%$ & $16.6\%$ & $12.0\%$ & $9.6\%$ & \multirow{3}{*}{-} \\
            & & & \gray Random+ & \gray $18.6\%$ & \gray $14.2\%$ & \gray $9.0\%$ & \gray $7.2\%$ & \\
            & & & \gray PGD & \gray $\textbf{12.8\%}$ & \gray $\textbf{9.8\%}$ & \gray $\textbf{6.8\%}$ & \gray $\textbf{5.6\%}$ & \\
            \midrule

            \multirow{12}{*}{Rotation} & \multirow{6}{*}{\shortstack[c]{$r \pm 10^\circ$}} & \multirow{3}{*}{\textbf{TSS}} & Random & $78.0\%$ & $77.0\%$ & $76.8\%$ & $76.6\%$ & \multirow{3}{*}{$70.6\%$} \\
            & & & \gray Random+ & \gray $77.4\%$ & \gray $\textbf{76.8\%}$ & \gray $\textbf{76.4\%}$ & \gray $\textbf{76.4\%}$  \\
            & & & \gray PGD & \gray $\textbf{76.8\%}$ & \gray $\textbf{76.8\%}$ & \gray $76.8\%$ & \gray $76.6\%$  & \\
            \cline{3-9}
            
            & & \multirow{3}{*}{Vanilla} & Random & $69.2\%$ & $68.0\%$ & $\textbf{65.6\%}$ & $65.6\%$ & \multirow{3}{*}{-} \\
            & & & \gray Random+ & \gray $68.4\%$ & \gray $67.2\%$ & \gray $66.0\%$ & \gray $65.6\%$ & \\
            & & & \gray PGD & \gray $\textbf{66.4\%}$ & \gray $\textbf{66.0\%}$ & \gray $\textbf{65.6\%}$ & \gray $\textbf{65.4\%}$ \\
            \cmidrule{2-9}
            
            & \multirow{6}{*}{\shortstack[c]{$r \pm 30^\circ$}} & \multirow{3}{*}{\textbf{TSS}} & Random & $71.8\%$ & $70.2\%$ & $69.8\%$ & $\textbf{69.2}\%$ & \multirow{3}{*}{$63.6\%$} \\
            & & & \gray Random+ & \gray $71.0\%$ & \gray $\textbf{69.4\%}$ & \gray $\textbf{69.2}\%$ & \gray $69.4\%$ \\
            & & & \gray PGD & \gray $\textbf{70.4\%}$ & \gray $70.0\%$ & \gray $70.0\%$ & \gray $69.8\%$ & \\
            \cline{3-9}
            
            & & \multirow{3}{*}{Vanilla} & Random & $31.6\%$ & $27.4\%$ & $22.6\%$ & $21.6\%$ & \multirow{3}{*}{-} \\
            & & & \gray Random+ & \gray $32.2\%$ & \gray $27.2\%$ & \gray $23.8\%$ & \gray $\textbf{21.4}\%$ & \\
            & & & \gray PGD & \gray $\textbf{25.2}\%$ & \gray $\textbf{23.8}\%$ & \gray $\textbf{23.2}\%$ & \gray $23.2\%$ \\
            \midrule

            \multirow{6}{*}{Scaling} & \multirow{6}{*}{\shortstack[c]{$s\pm 30\%$}} & \multirow{3}{*}{\textbf{TSS}} & Random & $69.6\%$ & $67.8\%$ & $67.8\%$ & $67.2\%$ & \multirow{3}{*}{$58.8\%$} \\
            & & & \gray Random+ & \gray $69.2\%$ & \gray $68.4\%$ & \gray $67.4\%$ & \gray $\textbf{67.0\%}$ & \\
            & & & \gray PGD & \gray $\textbf{67.8\%}$ & \gray $\textbf{67.6\%}$ & \gray $\textbf{67.2\%}$ & \gray $\textbf{67.0\%}$ & \\
            \cline{3-9}
            
            & & \multirow{3}{*}{Vanilla} & Random & $60.0\%$ & $54.6\%$ & $52.8\%$ & $51.6\%$ & \multirow{3}{*}{-} \\
            & & & \gray Random+ & \gray $56.6\%$ & \gray $53.8\%$ & \gray $52.2\%$ & \gray $\textbf{51.2\%}$ & \\
            & & & \gray PGD & \gray $\textbf{53.2\%}$ & \gray $\textbf{52.4\%}$ & \gray $\textbf{52.0\%}$ & \gray $52.0\%$ & \\
            \midrule

            \multirow{12}{*}{\shortstack[c]{Rotation\\ and \\ Brightness}} & \multirow{6}{*}{\shortstack[c]{$r\pm 10^\circ$,\\ $b\pm 10\%$}} & \multirow{3}{*}{\textbf{TSS}} & Random & $77.2\%$ & $76.8\%$ & $77.0\%$ & $76.6\%$ & \multirow{3}{*}{$70.6\%$} \\
            & & & \gray Random+ & \gray $77.2\%$ & \gray $\textbf{76.6\%}$ & \gray $\textbf{76.4\%}$ & \gray $\textbf{76.0\%}$ \\
            & & & \gray PGD & \gray $\textbf{76.6\%}$ & \gray $\textbf{76.6\%}$ & \gray $\textbf{76.4\%}$ & \gray $76.4\%$ & \\
            \cline{3-9}
            
            & & \multirow{3}{*}{Vanilla} & Random & $67.2\%$ & $64.8\%$ & $60.6\%$ & $59.4\%$ & \multirow{3}{*}{-} \\
            & & & \gray Random+ & \gray $66.0\%$ & \gray $63.0\%$ & \gray $59.4\%$ & \gray $57.8\%$ & \\
            & & & \gray PGD & \gray $\textbf{57.8\%}$ & \gray $\textbf{57.6\%}$ & \gray $\textbf{57.0\%}$ & \gray $\textbf{56.8\%}$ \\
            \cmidrule{2-9}

            & \multirow{6}{*}{\shortstack[c]{$r\pm 30^\circ$,\\ $b\pm 20\%$}} & \multirow{3}{*}{\textbf{TSS}} & Random & $72.0\%$ & $70.2\%$ & $68.8\%$ & $68.4\%$ & \multirow{3}{*}{$61.4\%$} \\
            & & & \gray Random+ & \gray $70.6\%$ & \gray $68.8\%$ & \gray $\textbf{68.0}\%$ & \gray $\textbf{68.2\%}$ \\
            & & & \gray PGD & \gray $\textbf{69.2}\%$ & \gray $\textbf{68.6}\%$ & \gray $68.6\%$ & \gray $68.6\%$ & \\
            \cline{3-9}
            
            & & \multirow{3}{*}{Vanilla} & Random & $26.6\%$ & $20.2\%$ & $15.8\%$ & $13.0\%$ & \multirow{3}{*}{-} \\
            & & & \gray Random+ & \gray $18.8\%$ & \gray $16.0\%$ & \gray $11.6\%$ & \gray $9.4\%$ \\
            & & & \gray PGD &  \gray $\textbf{12.2}\%$ & \gray $\textbf{10.4}\%$ & \gray $\textbf{9.2}\%$ & \gray $\textbf{9.0}\%$ \\
            \midrule

            \multirow{6}{*}{\shortstack[c]{Scaling\\ and\\ Brightness}} & \multirow{6}{*}{\shortstack[c]{$s\pm 30\%$,\\ $b\pm 30\%$}} & \multirow{3}{*}{\textbf{TSS}} & Random & $68.6\%$ & $68.6\%$ & $67.4\%$ & $67.2\%$ & \multirow{3}{*}{$54.2\%$} \\
            & & & \gray Random+ & \gray $68.4\%$ & \gray $68.0\%$ & \gray $67.0\%$ & \gray $\textbf{66.8\%}$ & \\
            & & & \gray PGD & \gray $\textbf{67.4\%}$ & \gray $\textbf{67.4\%}$ & \gray $\textbf{66.8\%}$ & \gray $\textbf{66.8\%}$ & \\
            \cline{3-9}
            
            & & \multirow{3}{*}{Vanilla} & Random & $39.2\%$ & $30.6\%$ & $20.0\%$ & $17.4\%$ & \multirow{3}{*}{-} \\
            & & & \gray Random+ & \gray $30.4\%$ & \gray $19.4\%$ & \gray $15.4\%$ & \gray $\textbf{11.6\%}$ & \\
            & & & \gray PGD & \gray $\textbf{16.0\%}$ & \gray $\textbf{14.4\%}$ & \gray $\textbf{13.0\%}$ & \gray $13.0\%$ & \\
            \midrule

            \multirow{12}{*}{\shortstack[c]{Rotation,\\ Brightness,\\ and $\ell_2$}} & \multirow{6}{*}{\shortstack[c]{$r\pm 10^\circ$, \\ $b\pm 10\%$, \\ $\|\delta\|_2\le .05$}} &
            \multirow{3}{*}{\textbf{TSS}} 
            & Random & $74.2\%$ & $72.8\%$ & $71.8\%$ & $71.6\%$ & \multirow{3}{*}{$64.2\%$} \\
            & & & \gray Random+ & \gray $72.8\%$ & \gray $72.2\%$ & \gray $71.8\%$ & \gray $\textbf{71.2\%}$ & \\
            & & & \gray PGD & \gray $\textbf{71.6\%}$ & \gray $\textbf{71.6\%}$ & \gray $\textbf{71.6\%}$ & \gray $71.6\%$ & \\
            \cline{3-9}
            & & \multirow{3}{*}{Vanilla}
            & Random & $40.4\%$ & $35.8\%$ & $34.4\%$ & $31.8\%$ & \multirow{3}{*}{-} \\
            & & & \gray Random+ & \gray $36.4\%$ & \gray $\textbf{34.6\%}$ & \gray $\textbf{30.8\%}$ & \gray $\textbf{29.6\%}$ & \\
            & & & \gray PGD & \gray $\textbf{36.0\%}$ & \gray $35.0\%$ & \gray $34.6\%$ & \gray $34.6\%$ & \\
            \cmidrule{2-9}

            & \multirow{6}{*}{\shortstack[c]{$r\pm 30^\circ$, \\ $b\pm 20\%$,\\ $\|\delta\|_2\le .05$}} & \multirow{3}{*}{\textbf{TSS}} & Random & $67.6\%$ & $66.2\%$ & $64.8\%$ & $65.2\%$ & \multirow{3}{*}{$55.2\%$} \\
            & & & \gray Random+ & \gray $65.6\%$ & \gray $65.6\%$ & \gray $65.2\%$ & \gray $64.4\%$ & \\
            & & & \gray PGD & \gray $\textbf{65.2\%}$ & \gray $\textbf{64.6\%}$ & \gray $\textbf{64.0\%}$ & \gray $\textbf{64.0\%}$ & \\
            \cline{3-9}
            & & \multirow{3}{*}{Vanilla} & Random & $7.6\%$ & $5.4\%$ & $2.6\%$ & $0.8\%$ & \multirow{3}{*}{-} \\
            & & & \gray Random+ & \gray $3.8\%$ & \gray $2.4\%$ & \gray $1.2\%$ & \gray $\textbf{0.4\%}$ & \\
            & & & \gray PGD & \gray $\textbf{1.2\%}$ & \gray $\textbf{0.6\%}$ & \gray $\textbf{0.6\%}$ & \gray $0.6\%$ & \\
            \midrule
            
            \multirow{6}{*}{\shortstack[c]{Scaling,\\ Brightness,\\ and $\ell_2$}} & \multirow{6}{*}{\shortstack[c]{$s\pm 30\%$, \\ $b\pm 30\%$, \\ $\|\delta\|_2\le .05$}} & \multirow{3}{*}{\textbf{TSS}} & Random & $67.6\%$ & $66.8\%$ & $65.2\%$ & $65.0\%$ & \multirow{3}{*}{$51.2\%$} \\
            & & & \gray Random+ & \gray $66.0\%$ & \gray $66.2\%$ & \gray $64.6\%$ & \gray $64.4\%$ & \\
            & & & \gray PGD & \gray $\textbf{64.2\%}$ & \gray $\textbf{62.2\%}$ & \gray $\textbf{61.8\%}$ & \gray $\textbf{61.8\%}$ & \\
            \cline{3-9}
            & & \multirow{3}{*}{Vanilla} & Random & $15.6\%$ & $11.4\%$ & $5.8\%$ & $4.4\%$ & \multirow{3}{*}{-} \\
            & & & \gray Random+ & \gray $8.2\%$ & \gray $5.0\%$ & \gray $\textbf{2.0\%}$ & \gray $\textbf{2.0\%}$ & \\
            & & & \gray PGD & \gray $\textbf{3.8\%}$ & \gray $\textbf{2.8\%}$ & \gray $2.8\%$ & \gray $2.6\%$ & \\
            \bottomrule
        \end{tabular}
        }
        \label{tab:attack-detail-cifar10}
        \vspace{1em}
    \end{table}

    \begin{table}[tb]
        \centering
        \caption{\small Comparison between empirical robust accuracy against random and adaptive attacks and certified robust accuracy on \textbf{ImageNet}. The attack radii are consistent with \Cref{tab:main}. The most powerful attack in each setting is highlighted in bold font. The adaptive attacks are shown in gray rows. Note that PGD attack cannot apply to translation transformation because the parameter space is discrete.
        }
        \resizebox{0.95\linewidth}{!}{
        \begin{tabular}{ccccccccc}
            \toprule
            \multirow{3}{*}{Transformation} & \multirow{3}{*}{\shortstack[c]{Attack\\ Radius}} & \multirow{3}{*}{Model} & \multirow{3}{*}{Attack} & \multicolumn{4}{c}{Empirical Robust Accuracy} & Certified \\
            \cline{5-8}
            & & & & \multicolumn{4}{c}{Initial Starts $N =$} & Robust \\
            & & & & $10$ & $20$ & $50$ & $100$ & Accuracy \\
            \midrule
        
            \multirow{6}{*}{\shortstack[c]{Gaussian\\ Blur}}
            & \multirow{6}{*}{\shortstack[c]{Squared\\ Radius\\ $\alpha\le 36$}} & \multirow{3}{*}{\textbf{TSS}}
              & Random & $53.2\%$ & $\textbf{52.8\%}$ & $\textbf{52.8\%}$ & $52.8\%$ & \multirow{3}{*}{$51.6\%$} \\
            & & & \gray Random+ & \gray $53.2\%$ & \gray $\textbf{52.8\%}$ & \gray $\textbf{52.8\%}$ & \gray $52.8\%$ & \\
            & & & \gray PGD     & \gray $\textbf{52.8\%}$ & \gray $\textbf{52.8\%}$ & \gray $\textbf{52.8\%}$ & \gray $\textbf{52.6\%}$ & \\
            
            \cline{3-9}
            & & \multirow{3}{*}{Vanilla}
              & Random & $9.6\%$ & $8.6\%$ & $8.4\%$ & $8.4\%$ & \multirow{3}{*}{-} \\
            & & & \gray Random+ & \gray $8.8\%$ & \gray $\textbf{8.2\%}$ & \gray $\textbf{8.2\%}$ & \gray $\textbf{8.2\%}$ & \\
            & & & \gray PGD     & \gray $\textbf{8.4\%}$ & \gray $\textbf{8.2\%}$ & \gray $\textbf{8.2\%}$ & \gray $\textbf{8.2\%}$ & \\
            \midrule

            \multirow{6}{*}{\shortstack[c]{Translation\\ (Reflection Pad.)}}
            & \multirow{6}{*}{\shortstack[c]{$\sqrt{\Delta x^2 + \Delta y^2}$\\ $\le 100$}} & \multirow{3}{*}{\textbf{TSS}}
              & Random & $70.0\%$ & $69.6\%$ & $\textbf{69.2\%}$ & $\textbf{69.2\%}$ & \multirow{3}{*}{$50.0\%$} \\
            & & & \gray Random+ & \gray $\textbf{69.4\%}$ & \gray $\textbf{69.2\%}$ & \gray $\textbf{69.2\%}$ & \gray $\textbf{69.2\%}$ & \\
            & & & \gray PGD     & \gray - & \gray - & \gray - & \gray - & \\
            
            \cline{3-9}
            & & \multirow{3}{*}{Vanilla}
              & Random & $\textbf{55.8\%}$ & $\textbf{53.4\%}$ & $\textbf{48.8\%}$ & $46.6\%$ & \multirow{3}{*}{-} \\
            & & & \gray Random+ & \gray $57.2\%$ & \gray $54.6\%$ & \gray $50.6\%$ & \gray $\textbf{46.2\%}$ & \\
            & & & \gray PGD     & \gray - & \gray - & \gray - & \gray - & \\
            \midrule

            \multirow{6}{*}{\shortstack[c]{Brightness}}
            & \multirow{6}{*}{\shortstack[c]{$b\pm 40\%$}} & \multirow{3}{*}{\textbf{TSS}}
              & Random & $70.8\%$ & $\textbf{70.4\%}$ & $\textbf{70.4\%}$ & $\textbf{70.4\%}$ & \multirow{3}{*}{$70.0\%$} \\
            & & & \gray Random+ & \gray $\textbf{70.4\%}$ & \gray $\textbf{70.4\%}$ & \gray $\textbf{70.4\%}$ & \gray $\textbf{70.4\%}$ & \\
            & & & \gray PGD     & \gray $\textbf{70.4\%}$ & \gray $\textbf{70.4\%}$ & \gray $\textbf{70.4\%}$ & \gray $\textbf{70.4\%}$ & \\
            
            \cline{3-9}
            & & \multirow{3}{*}{Vanilla}
              & Random & $31.6\%$ & $26.6\%$ & $21.6\%$ & $19.6\%$ & \multirow{3}{*}{-} \\
            & & & \gray Random+ & \gray $22.8\%$ & \gray $\textbf{19.8\%}$ & \gray $\textbf{18.4\%}$ & \gray $\textbf{18.4\%}$ & \\
            & & & \gray PGD     & \gray $\textbf{22.0\%}$ & \gray $22.4\%$ & \gray $21.8\%$ & \gray $21.8\%$ & \\
            \midrule

            \multirow{6}{*}{\shortstack[c]{Contrast\\ and\\ Brightness}}
            & \multirow{6}{*}{\shortstack[c]{$c\pm 40\%$,\\ $b\pm 40\%$}} & \multirow{3}{*}{\textbf{TSS}}
              & Random & $70.4\%$ & $69.2\%$ & $\textbf{68.4\%}$ & $\textbf{68.4\%}$ & \multirow{3}{*}{$61.4\%$} \\
            & & & \gray Random+ & \gray $69.2\%$ & \gray $68.8\%$ & \gray $\textbf{68.4\%}$ & \gray $\textbf{68.4\%}$ & \\
            & & & \gray PGD     & \gray $\textbf{68.4\%}$ & \gray $\textbf{68.4\%}$ & \gray $\textbf{68.4\%}$ & \gray $\textbf{68.4\%}$ & \\
            
            \cline{3-9}
            & & \multirow{3}{*}{Vanilla}
              & Random & $20.8\%$ & $10.4\%$ & $3.6\%$ & $1.2\%$ & \multirow{3}{*}{-} \\
            & & & \gray Random+ & \gray $8.0\%$ & \gray $2.0\%$ & \gray $0.4\%$ & \gray $\textbf{0.0\%}$ & \\
            & & & \gray PGD     & \gray $\textbf{1.8\%}$ & \gray $\textbf{0.2\%}$ & \gray $\textbf{0.0\%}$ & \gray $\textbf{0.0\%}$ & \\
            \midrule

            \multirow{6}{*}{\shortstack[c]{Gaussian Blur,\\ Translation\\ Contrast,\\ and Brightness}}
            & \multirow{6}{*}{\shortstack[c]{$\alpha\le 10$,\\ $c\pm 20\%$,\\ $b\pm 20\%$,\\ $\sqrt{\Delta x^2 + \Delta y^2}$\\ $\le 10$}} & \multirow{3}{*}{\textbf{TSS}}
              & Random & $51.8\%$ & $50.2\%$ & $49.2\%$ & $48.8\%$ & \multirow{3}{*}{$32.8\%$} \\
            & & & \gray Random+ & \gray $51.4\%$ & \gray $\textbf{49.6\%}$ & \gray $\textbf{48.0\%}$ & \gray $48.2\%$ & \\
            & & & \gray PGD     & \gray $\textbf{49.6\%}$ & \gray $\textbf{49.6\%}$ & \gray $48.2\%$ & \gray $\textbf{47.4\%}$ & \\
            
            \cline{3-9}
            & & \multirow{3}{*}{Vanilla}
              & Random & $20.6\%$ & $17.4\%$ & $12.0\%$ & $9.4\%$ & \multirow{3}{*}{-} \\
            & & & \gray Random+ & \gray $15.2\%$ & \gray $12.8\%$ & \gray $7.8\%$ & \gray $6.6\%$ & \\
            & & & \gray PGD     & \gray $\textbf{11.2\%}$ & \gray $\textbf{8.0\%}$ & \gray $\textbf{6.0\%}$ & \gray $\textbf{4.0\%}$ & \\
            \midrule

            \multirow{6}{*}{\shortstack[c]{Rotation}}
            & \multirow{6}{*}{\shortstack[c]{$r\pm 30\%$}} & \multirow{3}{*}{\textbf{TSS}}
              & Random & $40.2\%$ & $\textbf{38.4\%}$ & $38.4\%$ & $\textbf{37.8\%}$ & \multirow{3}{*}{$30.4\%$} \\
            & & & \gray Random+ & \gray $\textbf{39.0\%}$ & \gray $38.6\%$ & \gray $\textbf{38.0\%}$ & \gray $\textbf{37.8\%}$ & \\
            & & & \gray PGD     & \gray $40.4\%$ & \gray $39.8\%$ & \gray $39.8\%$ & \gray $39.4\%$ & \\
            
            \cline{3-9}
            & & \multirow{3}{*}{Vanilla}
              & Random & $47.8\%$ & $44.4\%$ & $41.4\%$ & $40.0\%$ & \multirow{3}{*}{-} \\
            & & & \gray Random+ & \gray $45.0\%$ & \gray $43.6\%$ & \gray $40.6\%$ & \gray $38.8\%$ & \\
            & & & \gray PGD     & \gray $\textbf{39.6\%}$ & \gray $\textbf{38.4\%}$ & \gray $\textbf{37.8\%}$ & \gray $\textbf{37.0\%}$ & \\
            \midrule

            \multirow{6}{*}{\shortstack[c]{Scaling}}
            & \multirow{6}{*}{\shortstack[c]{$s\pm 30\%$}} & \multirow{3}{*}{\textbf{TSS}}
              & Random & $40.2\%$ & $38.0\%$ & $37.4\%$ & $37.4\%$ & \multirow{3}{*}{$26.4\%$} \\
            & & & \gray Random+ & \gray $\textbf{38.8\%}$ & \gray $\textbf{37.2\%}$ & \gray $\textbf{36.8\%}$ & \gray $\textbf{36.4\%}$ & \\
            & & & \gray PGD     & \gray $39.0\%$ & \gray $37.8\%$ & \gray $37.6\%$ & \gray $37.8\%$ & \\
            
            \cline{3-9}
            & & \multirow{3}{*}{Vanilla}
              & Random & $55.2\%$ & $53.0\%$ & $51.2\%$ & $50.0\%$ & \multirow{3}{*}{-} \\
            & & & \gray Random+ & \gray $55.6\%$ & \gray $52.8\%$ & \gray $50.6\%$ & \gray $50.0\%$ & \\
            & & & \gray PGD     & \gray $\textbf{50.6\%}$ & \gray $\textbf{49.8\%}$ & \gray $\textbf{49.4\%}$ & \gray $\textbf{49.8\%}$ & \\
            \midrule

            \multirow{6}{*}{\shortstack[c]{Rotation\\ and\\ Brightness}}
            & \multirow{6}{*}{\shortstack[c]{$r\pm 30^\circ$\\ $b\pm 20\%$}} & \multirow{3}{*}{\textbf{TSS}}
              & Random & $\textbf{38.8\%}$ & $\textbf{38.0\%}$ & $37.2\%$ & $37.4\%$ & \multirow{3}{*}{$26.8\%$} \\
            & & & \gray Random+ & \gray $39.0\%$ & \gray $38.2\%$ & \gray $\textbf{37.0\%}$ & \gray $\textbf{36.8\%}$ & \\
            & & & \gray PGD     & \gray $39.6\%$ & \gray $39.4\%$ & \gray $38.6\%$ & \gray $38.8\%$ & \\
            
            \cline{3-9}
            & & \multirow{3}{*}{Vanilla}
              & Random & $40.4\%$ & $35.4\%$ & $29.2\%$ & $22.4\%$ & \multirow{3}{*}{-} \\
            & & & \gray Random+ & \gray $35.2\%$ & \gray $31.2\%$ & \gray $25.2\%$ & \gray $\textbf{21.2\%}$ & \\
            & & & \gray PGD     & \gray $\textbf{25.0\%}$ & \gray $\textbf{23.2\%}$ & \gray $\textbf{22.2\%}$ & \gray $21.4\%$ & \\
            \midrule

            \multirow{6}{*}{\shortstack[c]{Scaling\\ and\\ Brightness}}
            & \multirow{6}{*}{\shortstack[c]{$s\pm 30\%$,\\ $b\pm 30\%$}} & \multirow{3}{*}{\textbf{TSS}}
              & Random & $40.2\%$ & $38.0\%$ & $\textbf{36.4\%}$ & $36.4\%$ & \multirow{3}{*}{$23.4\%$} \\
            & & & \gray Random+ & \gray $38.0\%$ & \gray $\textbf{37.0\%}$ & \gray $36.6\%$ & \gray $\textbf{36.0\%}$ & \\
            & & & \gray PGD     & \gray $\textbf{37.0\%}$ & \gray $\textbf{37.0\%}$ & \gray $36.6\%$ & \gray $36.6\%$ & \\
            
            \cline{3-9}
            & & \multirow{3}{*}{Vanilla}
              & Random & $34.4\%$ & $26.2\%$ & $19.4\%$ & $16.0\%$ & \multirow{3}{*}{-} \\
            & & & \gray Random+ & \gray $21.0\%$ & \gray $\textbf{15.0\%}$ & \gray $\textbf{12.4\%}$ & \gray $\textbf{8.8\%}$ & \\
            & & & \gray PGD     & \gray $\textbf{17.6\%}$ & \gray $15.2\%$ & \gray $13.8\%$ & \gray $13.4\%$ & \\
            \midrule

            \multirow{6}{*}{\shortstack[c]{Rotation,\\ Brightness, \\ and $\ell_2$}}
            & \multirow{6}{*}{\shortstack[c]{$r\pm 30^\circ$,\\ $b\pm 20\%$,\\ $\|\delta\|_2\le .05$}} & \multirow{3}{*}{\textbf{TSS}}
              & Random & $39.4\%$ & $38.2\%$ & $37.8\%$ & $37.0\%$ & \multirow{3}{*}{$26.6\%$} \\
            & & & \gray Random+ & \gray $\textbf{38.2\%}$ & \gray $\textbf{37.8\%}$ & \gray $\textbf{36.6\%}$ & \gray $\textbf{36.4\%}$ & \\
            & & & \gray PGD     & \gray $38.8\%$ & \gray $38.8\%$ & \gray $38.4\%$ & \gray $38.0\%$ & \\
            
            \cline{3-9}
            & & \multirow{3}{*}{Vanilla}
              & Random & $26.0\%$ & $23.2\%$ & $19.8\%$ & $17.6\%$ & \multirow{3}{*}{-} \\
            & & & \gray Random+ & \gray $21.4\%$ & \gray $18.4\%$ & \gray $16.0\%$ & \gray $14.4\%$ & \\
            & & & \gray PGD     & \gray $\textbf{16.6\%}$ & \gray $\textbf{14.6\%}$ & \gray $\textbf{14.2\%}$ & \gray $\textbf{14.0\%}$ & \\
            \midrule

            \multirow{6}{*}{\shortstack[c]{Scaling,\\ Brightness,\\ and $\ell_2$}}
            & \multirow{6}{*}{\shortstack[c]{$s\pm 30\%$,\\ $b\pm 30\%$,\\ $\|\delta\|_2\le .05$}} & \multirow{3}{*}{\textbf{TSS}}
              & Random & $40.2\%$ & $38.2\%$ & $37.2\%$ & $36.0\%$ & \multirow{3}{*}{$22.6\%$} \\
            & & & \gray Random+ & \gray $38.0\%$ & \gray $\textbf{36.4\%}$ & \gray $\textbf{35.8\%}$ & \gray $\textbf{35.6\%}$ & \\
            & & & \gray PGD     & \gray $\textbf{36.8\%}$ & \gray $\textbf{36.4\%}$ & \gray $36.4\%$ & \gray $36.0\%$ & \\
            
            \cline{3-9}
            & & \multirow{3}{*}{Vanilla}
              & Random & $24.4\%$ & $17.2\%$ & $11.4\%$ & $7.4\%$ & \multirow{3}{*}{-} \\
            & & & \gray Random+ & \gray $13.8\%$ & \gray $\textbf{8.4\%}$ & \gray $\textbf{5.8\%}$ & \gray $\textbf{4.8\%}$ & \\
            & & & \gray PGD     & \gray $\textbf{9.8\%}$ & \gray $8.8\%$ & \gray $7.4\%$ & \gray $7.4\%$ & \\
            
            \bottomrule
        \end{tabular}
        }
        \label{tab:attack-detail-imagenet}
    \end{table}

                \begin{table}[H]
                    \centering
                    \caption{\small Comparison of Empirical Accuracy for each corruption evaluated from the highest severity level~(5) of CIFAR-10-C and ImageNet-C.}
                    \resizebox{\linewidth}{!}{
                    \begin{tabular}{c|c|ccg|ccg}
                        \toprule
                         \multicolumn{2}{c|}{Corruption} & \multicolumn{3}{c|}{CIFAR-10} & \multicolumn{3}{c}{ImageNet} \\
                        \hline
                        Category & Type & Vanilla & AugMix~\cite{hendrycks2019augmix} & \framework
                        & Vanilla & AugMix~\cite{hendrycks2019augmix} & \framework \\
                        \hline
                        Weather
                        & Snow &
                        $68.2\%$ & $\textbf{75.6\%}$ & $69.4\%$ &
                        $16.0\%$ & $\textbf{22.6\%}$ & $13.8\%$ \\
                        & Fog &
                        $63.4\%$ & $\textbf{65.4\%}$ & $62.0\%$ &
                        $\textbf{24.0\%}$ & $22.2\%$ & $18.0\%$ \\
                        & Frost &
                        $59.2\%$ & $67.8\%$ & $\textbf{73.8\%}$ &
                        $21.6\%$ & $\textbf{24.8\%}$ & $22.6\%$ \\
                        & Brightness &
                        $\textbf{82.4\%}$ & $\textbf{82.4\%}$ & $71.8\%$ &
                        $\textbf{56.8\%}$ & $56.6\%$ & $35.8\%$ \\
                        \hline
                        Blur
                        & Zoom Blur &
                        $52.6\%$ & $70.8\%$ & $\textbf{75.2\%}$ &
                        $21.4\%$ & $\textbf{31.0\%}$ & $20.4\%$ \\
                        & Glass Blur &
                        $46.6\%$ & $50.2\%$ & $\textbf{72.2\%}$ &
                        $8.0\%$ & $\textbf{14.0\%}$ & $13.8\%$ \\
                        & Motion Blur &
                        $54.8\%$ & $68.6\%$ & $\textbf{70.2\%}$ &
                        $14.2\%$ & $\textbf{25.2\%}$ & $11.4\%$ \\
                        & Defocus Blur &
                        $49.0\%$ & $72.2\%$ & $\textbf{75.6\%}$ &
                        $14.0\%$ & $22.6\%$ & $\textbf{25.6\%}$ \\
                        \hline
                        Noise
                        & Impulse Noise &
                        $29.8\%$ & $\textbf{51.0\%}$ & $46.2\%$ &
                        $4.0\%$ & $9.8\%$ & $\textbf{12.0\%}$ \\
                        & Gaussian Noise &
                        $34.8\%$ & $56.4\%$ & $\textbf{62.8\%}$ &
                        $4.4\%$ & $9.6\%$ & $\textbf{12.8\%}$ \\
                        & Shot Noise &
                        $43.0\%$ & $\textbf{63.4\%}$ & $62.6\%$ &
                        $4.0\%$ & $13.0\%$ & $\textbf{14.0\%}$ \\
                        \hline
                        Digital
                        & Pixelate &
                        $42.0\%$ & $59.0\%$ & $\textbf{76.0\%}$ &
                        $19.6\%$ & $39.2\%$ & $\textbf{55.6\%}$ \\
                        & Elastic Transform &
                        $71.4\%$ & $65.2\%$ & $\textbf{74.4\%}$ &
                        $14.8\%$ & $\textbf{23.8\%}$ & $23.6\%$ \\
                        & Contrast &
                        $23.8\%$ & $26.0\%$ & $\textbf{49.8\%}$ &
                        $4.2\%$ & $\textbf{11.6\%}$ & $5.0\%$ \\
                        & JPEG Compression &
                        $70.8\%$ & $\textbf{73.0\%}$ & $71.8\%$ &
                        $33.6\%$ & $\textbf{45.4\%}$ & $31.6\%$ \\
                        \hline
                        Extra
                        & Saturate &
                        $79.6\%$ & $\textbf{83.4\%}$ & $63.6\%$ &
                        $41.6\%$ & $\textbf{43.4\%}$ & $25.8\%$ \\
                        & Spatter &
                        $72.8\%$ & $\textbf{82.0\%}$ & $69.0\%$ &
                        $22.4\%$ & $\textbf{30.6\%}$ & $17.6\%$ \\
                        & Speckle Noise &
                        $45.2\%$ & $\textbf{64.0\%}$ & $58.8\%$
                        & $11.4\%$ & $\textbf{27.4\%}$ & $23.6\%$ \\
                        & Gaussian Blur &
                        $34.6\%$ & $67.4\%$ & $\textbf{75.8\%}$ &
                        $11.2\%$ & $15.2\%$ & $\textbf{33.0\%}$ \\
                        \hline
                        \multicolumn{2}{c|}{Average} &
                        $53.89\%$ & $65.46\%$ & $\textbf{67.42\%}$ &
                        $18.27\%$ & $\textbf{25.68\%}$ & $21.89\%$ \\
                        \bottomrule
                    \end{tabular}
                    }
                    \label{tab:cifar10-imagenet-c-breakdown}
                \end{table}
            \fi

        \subsection{Empirical Robustness against Unforeseen Attacks: Evaluation Protocol and Result Breakdown}
            \label{appendix:exp-corruption}

            \ifnum\ArXiv=0
                Detailed results are omitted to the corresponding section of the arXiv version~\cite{arxiv}.
            \fi

            \ifnum\ArXiv=1
            In the main text~(\Cref{subsec:generalization-to-unforeseen-physical-attacks}), we briefly show that \framework is generalizable to defend against unforeseen physical attacks by evaluating on CIFAR-10-C and ImageNet-C.
            Here, we first introduce the detailed evaluation protocol, then a breakdown of the result table~(\Cref{tab:cifar10-imagenet-c}) in the main text and show empirical accuracy on each type of corruption.

            \subsubsection{Evaluated Models}
                On either CIFAR-10-C and ImageNet-C, we choose three models for evaluation: the vanilla model, the AugMix~\cite{hendrycks2019augmix} trained model, and our \framework model for defending the composition of Gaussian blur, translation, brightness, and contrast.
                The vanilla models and \framework models are the same models as used in main experiments.
                The AugMix is the state-of-the-art empirical defense on the CIFAR-10-C and ImageNet-C dataset according to \cite{croce2020robustbench}.
                For AugMix, on CIFAR-10-C, since the model weights are not released, we use the official implementation of AugMix~(\url{https://github.com/google-research/augmix}) and extend the code to support our used model architecture~(ResNet-110) for a fair comparison.
                We run the code with the suggested hyperparameters and achieve similar performance as reported in their paper.
                On ImageNet-C, we directly use the officially released model weights.
                The model has the same architecture~(ResNet-50) as ours so the comparison is naturally fair.
                Note that all these models are trained only on the clean CIFAR-10 or ImageNet training set.
                We do not include the evaluation of MNIST models since there is no corrupted MNIST dataset available within our best knowledge.

            \subsubsection{Empirical Accuracy Computation}
                We compute the \textbf{empirical accuracy}~(on CIFAR-10-C/ImageNet-C) as the ratio of correctly predicted samples among the test samples, where the test samples are all corrupted at the highest severity level~(level 5) to model the strongest unforeseen semantic attacker.
                For each corruption type, there is a full test set generated by 1-to-1 mapping from the original clean test set samples processed with the corruption.
                Being consistent with the main experiment's protocol, for each corruption type, we uniformly pick $500$ samples from the corresponding test set.
                Then, we compute the average empirical accuracy among all $19$ corruptions and report it in \Cref{tab:cifar10-imagenet-c} in main text.

                As a reference, we also include their certified accuracy against the composition of Gaussian blur, brightness, contrast, and translation.
                Since our \framework  provides robustness certification only for smoothed models, we apply the same smoothing strategy as our \framework models, hoping for providing robustness certificates for baseline models.
                As shown in \Cref{tab:cifar10-imagenet-c}, only \framework models can be certified with nontrivial certified robust accuracy.

            \subsubsection{Breakdown}

                In \Cref{tab:cifar10-imagenet-c-breakdown}, we show the breakdown of empirical accuracy for all models evaluated in \Cref{tab:cifar10-imagenet-c}.
                Note that our \framework models are trained using only four of these 19 corruptions~(brightness, contrast, Gaussian blur, and additive Gaussian noise).
                Almost on all the corruptions, \framework has higher accuracy than vanilla models and sometimes higher than the state-of-the-art defense---AugMix.
                Interestingly,  we find \framework models have different generalization abilities on these corruptions.
                The additive Gaussian noise has the best generalization ability, because \framework model also achieves much higher accuracy against impulse noise and shot noise than all the baselines.
                The Gaussian blur also generalizes well, because we can see significantly higher accuracy of \framework models against zoom blur, glass blur, motion blur, and defocus blur especially on CIFAR-10-C.
                Finally, brightness and contrast, even though they seem to be among the simplest transformations, have the poorest generalization ability.
                For example, under severe corruptions, the empirical accuracy of brightness is even below than that of vanilla models.
                Manual inspection of the corrupted images showed that corrupted brightness or contrast images are severely altered so that they are hard to be distinguished even by humans, giving a hint for possible reasons for the poor performance on these images.
                We thus conjecture that overly severe corruption could be the reason, and we think that it would be an interesting future direction to study these different generalization abilities in depth.
            \fi

        \subsection{Certified Accuracy Beyond Certified Radii}
            \label{appendix:exp-beyond-radii}

            \ifnum\ArXiv=0
                Detailed results are omitted to the corresponding section of the arXiv version~\cite{arxiv}.
            \fi

            \ifnum\ArXiv=1
            \begin{figure}[t]
                \centering
                \includegraphics[width=\linewidth]{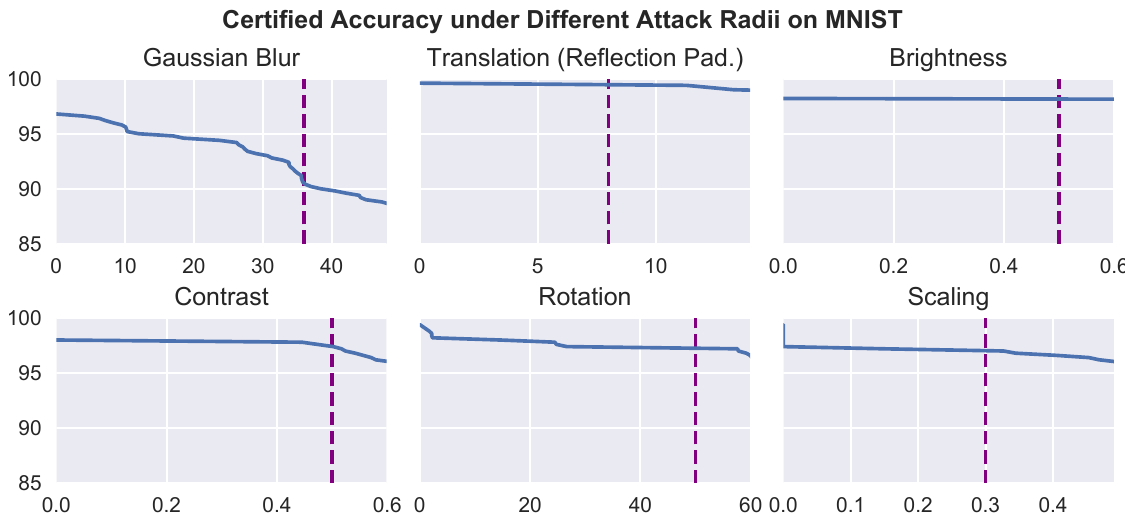}
                \caption{\small
                The blue curves show the certified robust accuracy on \textbf{MNIST}.
                The predefined certified radii are shown as purple vertical dotted lines.
                \textbf{No significant degradation after exceeding the predefined radii}.
                For Contrast subfigure, we allow additional $50\%$ brightness change.
                }
                \label{fig:mnist-beyond-radii-curve}
            \end{figure}

            \begin{figure}[htbp]
                \centering
                \includegraphics[width=\linewidth]{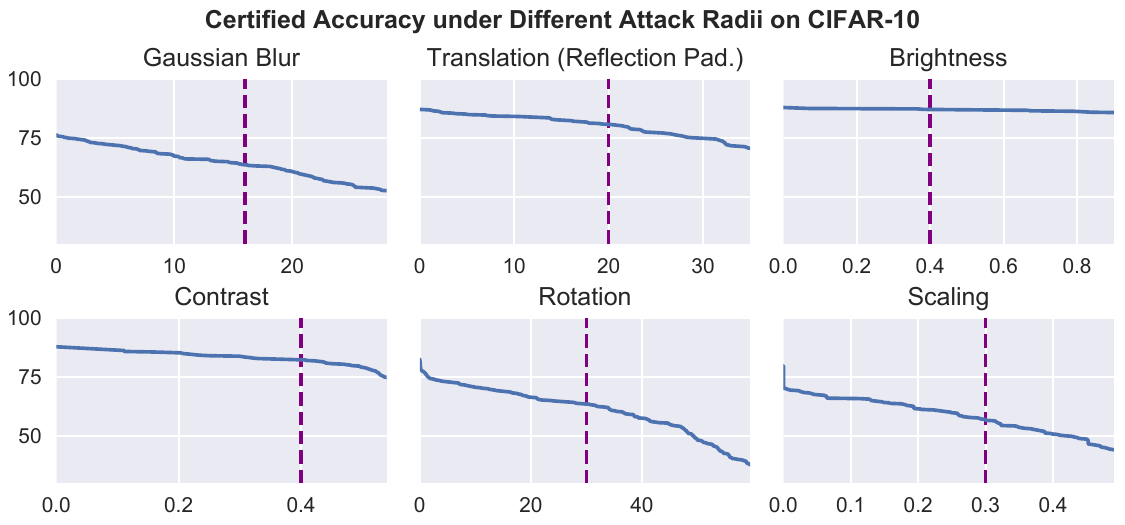}
                \caption{\small
                The blue curves show the certified robust accuracy on \textbf{CIFAR-10}.
                The predefined certified radii are shown as purple vertical dotted lines.
                \textbf{No significant degradation after exceeding the predefined certified radii}.
                For Contrast subfigure, we allow additional $40\%$ brightness change.
                }
                \label{fig:cifar-10-beyond-radii-curve}
            \end{figure}

    \begin{table}[htbp]
        \centering
        \caption{\small Empirical and certified robust accuracy on \textbf{MNIST} when attack radii go beyond the predefined one. The predefined certified radii are consistent with \Cref{tab:main}.\\
        *~For Contrast we allow additional $\pm 50\%$ brightness change since a single Contrast attack is not powerful enough. 
        }
        \resizebox{0.95\linewidth}{!}{
        \begin{tabular}{ccccggg}
            & & & & \multicolumn{3}{g}{Beyond Predefined Radii} \\
            \toprule
            
            \multirow{4}{*}{\shortstack[c]{Gaussian\\ Blur}} & & Radii: & $36$ & $40$ & $44$ & $48$ \\
            \cline{2-7}
            & \multirow{2}{*}{\textbf{TSS}} & Empirical & $91.2\%$ & $90.8\%$ & $90.4\%$ & $89.2\%$ \\
            & & Certified & $90.6\%$ & $90.0\%$ & $89.6\%$ & $88.8\%$ \\
            \cline{2-7}
            & Vanilla & Empirical & $12.2\%$ & $11.8\%$ & $11.2\%$ & $11.2\%$ \\
            \midrule
            
            \multirow{4}{*}{\shortstack[c]{Translation\\ (Reflection Pad.)}} & & Radii: & $8$ & $10$ & $12$ & $14$ \\
            \cline{2-7}
            & \multirow{2}{*}{\textbf{TSS}} & Empirical & $99.6\%$ & $99.6\%$ & $99.6\%$ & $99.6\%$ \\
            & & Certified & $99.6\%$ & $99.6\%$ & $99.4\%$ & $99.0\%$ \\
            \cline{2-7}
            & Vanilla & Empirical & $0.0\%$ & $0.0\%$ & $0.0\%$ & $0.0\%$ \\
            \midrule
            
            \multirow{4}{*}{Brightness} & & Radii: & $50\%$ & $52\%$ & $55\%$ & $60\%$ \\
            \cline{2-7}
            & \multirow{2}{*}{\textbf{TSS}} & Empirical & $98.2\%$ & $98.2\%$ & $98.2\%$ & $98.2\%$ \\
            & & Certified & $98.2\%$ & $98.2\%$ & $98.2\%$ & $98.2\%$ \\
            \cline{2-7}
            & Vanilla & Empirical & $96.6\%$ & $96.2\%$ & $95.6\%$ & $94.4\%$ \\
            \midrule
            
            \multirow{4}{*}{Contrast*} & & Radii: & $50\%$ & $52\%$ & $55\%$ & $60\%$ \\
            \cline{2-7}
            & \multirow{2}{*}{\textbf{TSS}} & Empirical & $98.0\%$ & $98.0\%$ & $98.0\%$ & $98.0\%$ \\
            & & Certified & $97.6\%$ & $97.2\%$ & $96.8\%$ & $96.2\%$ \\
            \cline{2-7}
            & Vanilla & Empirical & $93.2\%$ & $93.2\%$ & $93.2\%$ & $93.0\%$ \\
            \midrule
            
            \multirow{4}{*}{\shortstack{Rotation}} & & Radii: & $50^\circ$ & $52^\circ$ & $55^\circ$ & $60^\circ$ \\
            \cline{2-7}
             & \multirow{2}{*}{\textbf{TSS}} & Empirical & $98.2\%$ & $98.2\%$ & $98.2\%$ & $97.8\%$ \\
             & & Certified & $97.4\%$ & $97.4\%$ & $97.4\%$ & $96.6\%$ \\
            \cline{2-7}
             & Vanilla & Empirical & $11.0\%$ & $9.8\%$ & $8.4\%$ & $7.2\%$  \\
            \midrule
            
            \multirow{4}{*}{\shortstack{Scaling}} & & Radii: & $30\%$ & $35\%$ & $40\%$ & $50\%$ \\
            \cline{2-7}
             & \multirow{2}{*}{\textbf{TSS}} & Empirical & $99.2\%$ & $98.8\%$ & $98.8\%$ & $98.6\%$ \\
             & & Certified & $97.2\%$ & $96.8\%$ & $96.8\%$ & $96.0\%$ \\
            \cline{2-7}
             & Vanilla & Empirical & $89.2\%$ & $82.6\%$ & $72.8\%$ & $45.4\%$ \\
            \bottomrule
        \end{tabular}
        }
        \label{tab:beyond-radii-mnist}
    \end{table}

    \begin{table}[H]
        \centering
        \caption{\small Empirical and certified robust accuracy on \textbf{CIFAR-10} when attack radii go beyond the predefined one. The predefined certified radii are consistent with \Cref{tab:main}.\\
        * For Contrast we allow additional $40\%$ brightness change since a single Contrast attack is not powerful enough.
        }
        \resizebox{0.85\linewidth}{!}{
        \begin{tabular}{ccccggg}
            
            & & & & \multicolumn{3}{g}{Beyond Predefined Radii} \\
            \toprule
            
            \multirow{4}{*}{\shortstack[c]{Gaussian\\ Blur}} & & Radii: & $16$ & $20$ & $24$ & $28$ \\
            \cline{2-7}
            & \multirow{2}{*}{\textbf{TSS}} & Empirical & $65.8\%$ & $63.4\%$ & $61.2\%$ & $56.4\%$ \\
            & & Certified & $63.6\%$ & $60.8\%$ & $56.0\%$ & $52.6\%$ \\
            \cline{2-7}
            & Vanilla & Empirical & $3.4\%$ & $3.2\%$ & $3.0\%$ & $2.8\%$ \\ 
            \midrule
            
            \multirow{4}{*}{\shortstack[c]{Tranlation\\ (Reflection Pad.)}} & & Radii: & $20$ & $25$ & $30$ & $35$ \\
            \cline{2-7}
            & \multirow{2}{*}{\textbf{TSS}} & Empirical & $86.0\%$ & $86.0\%$ & $85.8\%$ & $85.8\%$ \\
            & & Certified & $80.8\%$ & $77.4\%$ & $74.8\%$ & $70.6\%$ \\
            \cline{2-7}
            & Vanilla & Empirical & $4.2\%$ & $3.6\%$ & $3.6\%$ & $2.8\%$ \\ 
            \midrule
            
            \multirow{4}{*}{Brightness} & & Radii: & $40\%$ & $45\%$ & $50\%$ & $55\%$ \\
            \cline{2-7}
            & \multirow{2}{*}{\textbf{TSS}} & Empirical & $87.2\%$ & $87.0\%$ & $87.0\%$ & $87.0\%$ \\
            & & Certified & $87.0\%$ & $87.0\%$ & $87.0\%$ & $87.0\%$ \\
            \cline{2-7}
            & Vanilla & Empirical & $42.6\%$ & $32.8\%$ & $21.2\%$ & $14.0\%$ \\ 
            \midrule
            
            \multirow{4}{*}{Contrast*} & & Radii & $40\%$ & $45\%$ & $50\%$ & $55\%$ \\
            \cline{2-7}
            & \multirow{2}{*}{\textbf{TSS}} & Empirical & $85.8\%$ & $85.8\%$ & $85.4\%$ & $85.2\%$ \\
            & & Certified & $82.4\%$ & $80.8\%$ & $79.2\%$ & $71.8\%$ \\
            \cline{2-7}
            & Vanilla & Empirical & $9.6\%$ & $7.8\%$ & $5.6\%$ & $4.8\%$ \\
            \midrule
            
            \multirow{4}{*}{Rotation} & & Radii: & $30^\circ$ & $40^\circ$ & $50^\circ$ & $60^\circ$ \\
            \cline{2-7}
             & \multirow{2}{*}{\textbf{TSS}} & Empirical & $69.2\%$ & $64.0\%$ & $57.8\%$ & $46.8\%$  \\
             & & Certified & $63.6\%$ & $57.6\%$ & $48.2\%$ & $37.4\%$ \\
            \cline{2-7}
             & Vanilla & Empirical & $21.4\%$ & $8.8\%$ & $5.0\%$ & $3.2\%$ \\
            \midrule

            \multirow{4}{*}{Scaling} & & Radii: & $30\%$ & $35\%$ & $40\%$ & $50\%$ \\
            \cline{2-7}
             & \multirow{2}{*}{\textbf{TSS}} & Empirical & $67.0\%$ & $65.0\%$ & $60.6\%$ & $54.8\%$  \\
             & & Certified & $58.8\%$ & $53.6\%$ & $51.0\%$ & $43.4\%$ \\
            \cline{2-7}
             & Vanilla & Empirical & $51.2\%$ & $43.0\%$ & $34.8\%$ & $21.2\%$ \\
            \bottomrule
        \end{tabular}
        }
        \label{tab:beyond-radii-cifar10}
    \end{table}

    \begin{table}[htbp]
        \centering
        \caption{\small Empirical and certified robust accuracy on \textbf{ImageNet} when attack radii go beyond the predefined one. The predefined certified radii are consistent with \Cref{tab:main}.\\
        *~For Contrast/Rotation/Scaling we allow additional $\pm 40\%$/$\pm 20\%$/$\pm 30\%$ brightness change since a single Contrast/Rotation/Scaling attack is not powerful enough.}
        \resizebox{0.85\linewidth}{!}{
        \begin{tabular}{ccccgg}
            & & & & \multicolumn{2}{g}{Beyond Predefined Radii} \\
            \toprule
            
            \multirow{4}{*}{\shortstack[c]{Gaussian\\ Blur}} & & Radii: & $36$ & $40$ & $44$ \\
            \cline{2-6}
            & \multirow{2}{*}{\textbf{TSS}} & Empirical & $52.6\%$ & $51.2\%$ & $49.8\%$ \\
            & & Certified & $51.6\%$ & $50.0\%$ & $48.8\%$ \\
            \cline{2-6}
            & Vanilla & Empirical & $8.2\%$ & $7.2\%$ & $6.2\%$ \\
            \midrule
            
            \multirow{4}{*}{\shortstack[c]{Translation\\ (Reflection Pad.)}} & & Radii: & $100$ & $105$ & $110$ \\
            \cline{2-6}
            & \multirow{2}{*}{\textbf{TSS}} & Empirical & $69.2\%$ & $69.2\%$ & $69.0\%$ \\
            & & Certified & $50.0\%$ & $49.4\%$ & $46.8\%$ \\
            \cline{2-6}
            & Vanilla & Empirical & $46.2\%$ & $37.6\%$ & $36.6\%$ \\
            \midrule
            
            \multirow{4}{*}{Brightness} & & Radii: & $40\%$ & $45\%$ & $50\%$ \\
            \cline{2-6}
            & \multirow{2}{*}{\textbf{TSS}} & Empirical & $70.4\%$ & $70.2\%$ & $70.0\%$ \\
            & & Certified & $70.0\%$ & $69.8\%$ & $69.6\%$ \\
            \cline{2-6}
            & Vanilla & Empirical & $18.4\%$ & $10.0\%$ & $5.2\%$ \\
            \midrule
            
            \multirow{4}{*}{Contrast*} & & Radii: & $40\%$ & $45\%$ & $50\%$ \\
            \cline{2-6}
            & \multirow{2}{*}{\textbf{TSS}} & Empirical & $68.4\%$ & $68.2\%$ & $67.6\%$ \\
            & & Certified & $61.4\%$ & $55.8\%$ & $45.0\%$ \\
            \cline{2-6}
            & Vanilla & Empirical & $0.0\%$ & $0.0\%$ & $0.0\%$ \\
            \midrule
            
            \multirow{4}{*}{\shortstack{Rotation*}} & & Radii: & $30^\circ$ & $35^\circ$ & $45^\circ$ \\
            \cline{2-6}
             & \multirow{2}{*}{\textbf{TSS}} & Empirical & $36.8\%$ & $36.4\%$ & $33.4\%$ \\
             & & Certified & $26.8\%$ & $26.2\%$ & $21.8\%$ \\
            \cline{2-6}
             & Vanilla & Empirical & $21.2\%$ & $19.4\%$ & $16.2\%$ \\
            \midrule
            
            \multirow{4}{*}{\shortstack{Scaling*}} & & Radii: & $30\%$ & $40\%$ & $50\%$ \\
            \cline{2-6}
             & \multirow{2}{*}{\textbf{TSS}} & Empirical & $36.0\%$ & $32.4\%$ & $26.6\%$  \\
             & & Certified & $23.4\%$ & $18.4\%$ & $11.6\%$  \\
            \cline{2-6}
             & Vanilla & Empirical & $8.8\%$ & $8.8\%$ & $7.0\%$ \\
            \bottomrule
        \end{tabular}
        }
        \label{tab:beyond-radii-imagenet}
    \end{table}

            In \Cref{fig:mnist-beyond-radii-curve} and \Cref{fig:cifar-10-beyond-radii-curve}, on MNIST and CIFAR-10,
            the purple vertical dotted lines stand for the predefined certified radii that the models aim to defend, and the blue curves show the certified robust accuracy~($y$ axis) with respect to attack radii~($x$ axis).
            The figures imply that the \framework models that aim to defend against transformations within certain thresholds still maintain \textbf{high certified accuracy} when the transformation parameters go even far beyond the thresholds.

            In \Cref{tab:beyond-radii-mnist}, \Cref{tab:beyond-radii-cifar10}, and \Cref{tab:beyond-radii-imagenet}, we further list the empirical robust accuracy of \framework and vanilla models when the attacker goes beyond the predefined certified radius.
            The empirical robust accuracy is computed as the minimum among all three attacks: Random, Random+, and PGD.
            We observe that the empirical robust accuracy follows the same tendency.
            For example, on CIFAR-10 dataset, the \framework model is trained to defend against the rotation transformation within $30^\circ$ where it achieves $69.2\%$/$63.6\%$ empirical/certified accuracy.
            When the rotation angle goes up to $60^\circ$ the model still preserves $46.8\%$/$37.4\%$ empirical/certified accuracy.
            On the contrary, the vanilla model's empirical accuracy is reduced from $21.4\%$~($30^\circ$ rotation) to $3.2\%$~($60^\circ$ rotation).

    \begin{table}[htbp]
        \centering
        \caption{\small Study of the impact of different smoothing variance levels on certified robust accuracy and benign accuracy on \textbf{MNIST} for \framework. The attack radii are consistent with \Cref{tab:main}. The ``Dist.'' refers to both training and smoothing distribution. The variance used in \Cref{tab:main} is labeled in gray.}
        \resizebox{\linewidth}{!}{

            \begin{tabular}{cccccc}
                \toprule
                \multirow{2}{*}{Transformation} & \multirow{2}{*}{\shortstack[c]{Attack\\ Radius}} & \multicolumn{4}{c}{Certified Accuracy and Benign Accuracy} \\
                & & \multicolumn{4}{c}{under Different Variance Levels} \\
                \midrule

                \multirow{3}{*}{\shortstack[c]{Gaussian Blur}} & \multirow{3}{*}{$\alpha\le 36$} & Dist. of $\alpha$ & $\Exp(1/5)$ & \gray $\Exp(1/10)$ & $\Exp(1/20)$ \\
                \cline{3-6}
                & & Cert. Rob. Acc. & $90.4\%$ & $\textbf{90.6\%}$ & $89.2\%$ \\
                & & Benign Acc. & $\textbf{97.0\%}$ & $96.8\%$ & $93.4\%$ \\
                \hline

                \multirow{3}{*}{\shortstack[c]{Translation\\ (Reflection Pad.)}} & \multirow{3}{*}{\shortstack[c]{$\sqrt{\Delta x^2 + \Delta y^2}$\\ $\le 8$}} & Dist. of $(\Delta x, \Delta y)$ & $\gN(0,5^2 I)$ & \gray $\gN(0, 10^2 I)$ & $\gN(0, 15^2 I)$ \\
                \cline{3-6}
                & & Cert. Rob. Acc. & $99.0\%$ & $\textbf{99.6\%}$ & $99.4\%$ \\
                & & Benign Acc. & $\textbf{99.6\%}$ & $\textbf{99.6\%}$ & $\textbf{99.6\%}$ \\
                \hline

                \multirow{3}{*}{\shortstack[c]{Brightness}} & \multirow{3}{*}{$b\pm 50\%$} & Dist. of $(c,b)$ & $\gN(0, 0.5^2 I)$ & \gray $\gN(0, 0.6^2 I)$ & $\gN(0, 0.7^2 I)$ \\
                \cline{3-6}
                & & Cert. Rob. Acc. & $\textbf{98.4\%}$ & $98.2\%$ & $\textbf{98.4\%}$ \\
                & & Benign Acc. & $\textbf{98.4\%}$ & $\textbf{98.4\%}$ & $\textbf{98.4\%}$ \\
                \hline

                \multirow{3}{*}{\shortstack[c]{Contrast}} & \multirow{3}{*}{$c\pm 50\%$} & Dist. of $(c,b)$ & $\gN(0, 0.5^2 I)$ & \gray $\gN(0, 0.6^2 I)$ & $\gN(0, 0.7^2 I)$ \\
                \cline{3-6}
                & & Cert. Rob. Acc. & $0.0\%$ & $98.0\%$ & $\textbf{98.4\%}$  \\
                & & Benign Acc. & $\textbf{98.4\%}$ & $\textbf{98.4\%}$ & $\textbf{98.4\%}$ \\
                \hline

                \multirow{3}{*}{Rotation} & \multirow{3}{*}{$r\pm 50^\circ$} & Dist. of $\epsilon$ & $\gN(0, 0.05^2 I)$ & \gray $\gN(0, 0.12^2 I)$ & $\gN(0, 0.20^2 I)$ \\
                \cline{3-6}
                & & Cert. Rob. Acc. & $\textbf{97.6\%}$ & $97.4\%$ & $\textbf{97.6\%}$ \\
                & & Benign Acc. & $99.2\%$ & $\textbf{99.4\%}$ & $99.2\%$ \\
                \hline

                \multirow{3}{*}{Scaling} & \multirow{3}{*}{$s\pm 30\%$} & Dist. of $\epsilon$ & $\gN(0, 0.05^2 I)$ & \gray $\gN(0, 0.12^2 I)$ & $\gN(0, 0.20^2 I)$ \\
                \cline{3-6}
                & & Cert. Rob. Acc. & $96.6\%$ & $\textbf{97.2\%}$ & $96.0\%$ \\
                & & Benign Acc. & $\textbf{99.4\%}$ & $\textbf{99.4\%}$ & $99.0\%$ \\

                \bottomrule
            \end{tabular}
        }
        \label{tab:var-level-mnist}
    \end{table}

    \begin{table}[htbp]
        \centering
        \caption{\small Study of the impact of different smoothing variance levels on certified robust accuracy and benign accuracy on \textbf{CIFAR-10} for \framework. The attack radii are consistent with \Cref{tab:main}. The ``Dist.'' refers to both training and smoothing distribution. The variance used in \Cref{tab:main} is labeled in gray.}
        \resizebox{\linewidth}{!}{

            \begin{tabular}{cccccc}
                \toprule
                \multirow{2}{*}{Transformation} & \multirow{2}{*}{\shortstack[c]{Attack\\ Radius}} & \multicolumn{4}{c}{Certified Accuracy and Benign Accuracy} \\
                & & \multicolumn{4}{c}{under Different Variance Levels} \\
                \midrule

                \multirow{3}{*}{\shortstack[c]{Gaussian Blur}} & \multirow{3}{*}{$\alpha \le 16$} & Dist. of $\alpha$ & \gray $\Exp(1/5)$ & $\Exp(1/10)$ & $\Exp(1/20)$ \\
                \cline{3-6}
                & & Cert. Rob. Acc. & $\textbf{63.6\%}$ & $60.6\%$ & $53.0\%$ \\
                & & Benign Acc. & $\textbf{76.2\%}$ & $68.0\%$ & $57.4\%$ \\
                \hline

                \multirow{3}{*}{\shortstack[c]{Translation\\ (Reflection Pad.)}} & \multirow{3}{*}{\shortstack[c]{$\sqrt{\Delta x^2 + \Delta y^2}$\\ $\le 20$}} & Dist. of $(\Delta x, \Delta y)$ & $\gN(0,10^2 I)$ & \gray $\gN(0, 15^2 I)$ & $\gN(0, 20^2 I)$ \\
                \cline{3-6}
                & & Cert. Rob. Acc. & $76.2\%$ & $\textbf{80.8\%}$ & $74.4\%$ \\
                & & Benign Acc. & $\textbf{89.0\%}$ & $87.0\%$ & $84.6\%$ \\
                \hline

                \multirow{3}{*}{\shortstack[c]{Brightness}} & \multirow{3}{*}{$b \pm 40\%$} & Dist. of $(c,b)$ &  $\gN(0, 0.2^2 I)$ & \gray $\gN(0, 0.3^2 I)$ & $\gN(0, 0.4^2 I)$ \\
                \cline{3-6}
                & & Cert. Rob. Acc. & $\textbf{87.4\%}$ & $87.0\%$ & $86.2\%$  \\
                & & Benign Acc. & $\textbf{87.8\%}$ & $\textbf{87.8\%}$ & $86.4\%$ \\
                \hline

                \multirow{3}{*}{\shortstack[c]{Contrast}} & \multirow{3}{*}{$c \pm 40\%$} &  Dist. of $(c,b)$ &  $\gN(0, 0.2^2 I)$ & $\gN(0, 0.3^2 I)$ & \gray $\gN(0, 0.4^2 I)$ \\
                \cline{3-6}
                & & Cert. Rob. Acc. & $0.0\%$ & $\textbf{82.4\%}$ & $\textbf{82.4\%}$ \\
                & & Benign Acc. & $\textbf{87.8\%}$ & $\textbf{87.8\%}$ & $86.4\%$ \\
                \hline

                \multirow{3}{*}{Rotation} & \multirow{3}{*}{$r \pm 30^\circ$} &  Dist. of $\epsilon$ & $\gN(0, 0.05^2 I)$ \gray & $\gN(0, 0.09^2 I)$ & $\gN(0, 0.12^2 I)$ \\
                \cline{3-6}
                & & Cert. Rob. Acc. & $\textbf{63.6\%}$ & $62.0\%$  & $59.0\%$  \\
                & & Benign Acc. & $\textbf{82.0\%}$ & $78.6\%$ & $72.2\%$  \\
                \hline

                \multirow{3}{*}{Scaling} & \multirow{3}{*}{$s \pm 30\%$} &  Dist. of $\epsilon$ & $\gN(0, 0.05^2 I)$ & $\gN(0, 0.09^2 I)$ & \gray $\gN(0, 0.12^2 I)$ \\
                \cline{3-6}
                & & Cert. Rob. Acc. & $59.0\%$ & $\textbf{59.4\%}$ & $58.8\%$ \\
                & & Benign Acc. & $\textbf{85.4\%}$ & $81.6\%$ & $79.2\%$ \\

                \bottomrule
            \end{tabular}
        }
        \label{tab:var-level-cifar}
    \end{table}
            \fi

        \subsection{Different Smoothing Variance Levels: More Results}
            \label{appendix:exp-different-variance-levels}

            \ifnum\ArXiv=0
                Detailed results are omitted to the corresponding section of the arXiv version~\cite{arxiv}.
            \fi

            \ifnum\ArXiv=1
            In \Cref{subsec:different-smoothing-variance} we have shown the study on smoothing variance levels on ImageNet~(\Cref{tab:var-level-imagenet}).
            Here, we further present our study of smoothing variance levels on MNIST and CIFAR-10.
            They are shown in \Cref{tab:var-level-mnist} and \Cref{tab:var-level-cifar} respectively.
            The smoothing variances shown in the two tables are for both training and inference-time smoothing.
            Except for smoothing variance, all other hyperparameters for training and certification are kept the same and consistent with the main experiments~(\Cref{tab:main}).
            As we can observe, the same conclusion still holds: usually, when the smoothing variance increases, the benign accuracy drops and the certified robust accuracy first rises and then drops.
            The reason is that larger smoothing variance makes the input more severely transformed so that the benign accuracy becomes smaller.
            On the other hand, larger smoothing variance makes the robustness easier to be certified as we can observe in various robustness conditions in \Cref{adx-sec:distribution_proofs}, where the required lower bound of $p_A$ becomes smaller.
            This is the reason for the ``first rise'' on certified accuracy.
            However, when the smoothing variance becomes too large, the benign accuracy becomes too low, and according to our definition, the certified accuracy is upper bounded by the benign accuracy~(precondition of robustness is correctness).
            This is the reason for the ``then drop'' on certified accuracy.

            We again observe that the range of acceptable variance is usually wide.
            For example, on CIFAR-10, for rotation transformation, the certified robust accuracy is $63.6\%$/$62.0\%$/$59.0\%$ across a wide range of smoothing variance: $0.05$, $0.09$, $0.12$.
            Thus, even in the presence of such trade-off, without fine-tuning the smoothing variances, we can still obtain high certified robust accuracy and high benign accuracy as reported in \Cref{tab:main} and \Cref{tab:benign-acc}~respectively.

            We remark that the gray cells in \Cref{tab:var-level-mnist} and \Cref{tab:var-level-cifar} indicate the smoothing variances used in our main experiments.
            We did not tune the smoothing variances so these cells might be sub-optimal~(though usually close to optimal) and they are just placed here for indication purpose.
        \fi

    \subsection{Tightness-Efficiency Trade-Off}
        \label{appendix:exp-tightness-efficiency-trade-off}

            \ifnum\ArXiv=0
                Detailed results are omitted to the corresponding section of the arXiv version~\cite{arxiv}.
            \fi

            \ifnum\ArXiv=1
        We notice that as we increase the number of samples when estimating the interpolation error in~\eqref{eq:sqrt-M} and~\eqref{eq:M_i_expression}, the interpolation error $M_{\gS}$ and the upper bound $\sqrt M \ge M_{\gS}$ become smaller and the certification becomes tighter, leading to higher certified robust accuracy. However, the computation time is also increased, resulting in a trade-off between speed and accuracy.
        In \Cref{tab:tightness-efficiency-trade-off-cifar} and \Cref{tab:tightness-efficiency-trade-off-mnist}, we illustrate this trade-off on two differentially resolvable transformations: composition of rotation and brightness on CIFAR-10, and composition of scaling and brightnes on MNIST.
        From the tables, we find that, for these compositions, as the sample numbers $N$ and $n$ increase, the interpolation error decreases and computing time increases (linearly with $N$ and $n$). As a consequence, if using a large number of samples, we can decrease the smoothing noise level $\sigma$ and achieve both higher certified accuracy and higher benign accuracy at the cost of larger computation time.

        \begin{table}[htbp]
            \centering
            \caption{\small
            Average interpolation upper bound $\sqrt{M}$~\eqref{eq:sqrt-M}, average computation time, and ``Certified accuracy~(average certification time)'' for varying number of samples and smoothing noise levels. Results on \textbf{CIFAR-10} against the composition of \textbf{rotation} $\pm 10^\circ$ and \textbf{brightness} change $\pm 10\%$.
            }
            \resizebox{\linewidth}{!}{
                \begin{tabular}{c c | c c | c c c}
                    \toprule
                    \multicolumn{2}{c|}{Number of Samples} & \multicolumn{2}{c|}{Interpolation} & \multicolumn{3}{c}{Smoothing Noise Level $\sigma$} \\
                    First-Level & Second-Level & Avg. $\sqrt M$ & Avg. Comp. Time & $0.05$ & $0.09$ & $0.12$  \\
                    \hline
                    \multirow{2}{*}{$N=556$} & \multirow{2}{*}{$n=2,000$} & \multirow{2}{*}{$0.050$} & \multirow{2}{*}{$\SI{22.50}{s}$} & $70.2\%$ & $65.2\%$ & $61.2\%$  \\
                    & & & & ($\SI{62.32}{s}$) & ($\SI{86.60}{s}$) & ($\SI{53.73}{s}$) \\
                    \multirow{2}{*}{$N=556$} & \multirow{2}{*}{$n=200$} & \multirow{2}{*}{$0.131$} & \multirow{2}{*}{$\SI{1.97}{s}$} & $42.0\%$ & $59.2\%$ & $60.4\%$ \\
                    & & & & ($\SI{490.21}{s}$) & ($\SI{93.19}{s}$) & ($\SI{86.60}{s}$) \\
                    \multirow{2}{*}{$N=56$} & \multirow{2}{*}{$n=2,000$} & \multirow{2}{*}{$0.322$} & \multirow{2}{*}{$\SI{1.90}{s}$} & $1.2\%$ & $12.6\%$ & $29.2\%$  \\
                    & & & & ($\SI{6.18}{s}$) & ($\SI{16.64}{s}$) & ($\SI{25.77}{s}$) \\
                    \multirow{2}{*}{$N=56$} & \multirow{2}{*}{$n=200$} & \multirow{2}{*}{$0.499$} & \multirow{2}{*}{$\SI{0.27}{s}$} & $0.0\%$ & $1.2\%$ & $3.4\%$ \\
                    & & & & ($\SI{5.22}{s}$) & ($\SI{5.68}{s}$) & ($\SI{8.49}{s}$) \\
                    \hline
                    \multicolumn{4}{r|}{Benign Accuracy:} & $83.0\%$ & $79.2\%$ & $79.6\%$ \\
                    \bottomrule
                \end{tabular}
            }
            \label{tab:tightness-efficiency-trade-off-cifar}
        \end{table}

        \begin{table}[htbp]
            \centering
            \caption{\small Average interpolation upper bound $\sqrt{M}$~\eqref{eq:sqrt-M}, average bound computation time, and ``Certified robust accuracy~(average certification time)'' when using different number of samples and various smoothing noise levels. Data is collected on \textbf{MNIST} dataset against the composition of \textbf{scaling} $\pm 50\%$ and \textbf{brightness} change $\pm 50\%$.}
            \resizebox{\linewidth}{!}{
                \begin{tabular}{c c | c c | c c c}
                    \toprule
                    \multicolumn{2}{c|}{Number of Samples} & \multicolumn{2}{c|}{Interpolation} & \multicolumn{3}{c}{Smoothing Noise Level $\sigma$} \\
                    First-Level & Second-Level & Avg. $\sqrt M$ & Avg. Comp. Time & $0.05$ & $0.09$ & $0.12$  \\
                    \midrule
                    \multirow{2}{*}{$N=2,500$} & \multirow{2}{*}{$n=500$} & \multirow{2}{*}{$0.064$} & \multirow{2}{*}{$\SI{10.52}{s}$} & $97.2\%$ & $97.4\%$ & $96.6\%$ \\
                    & & & & ($\SI{92.36}{s}$) & ($\SI{76.25}{s}$) & ($\SI{67.44}{s}$) \\
                    \multirow{2}{*}{$N=2,500$} & \multirow{2}{*}{$n=50$} & \multirow{2}{*}{$0.163$} & \multirow{2}{*}{$\SI{0.90}{s}$} & $18.8\%$ & $97.0\%$ & $95.0\%$ \\
                    & & & & ($\SI{157.48}{s}$) & ($\SI{217.97}{s}$) & ($\SI{97.91}{s}$) \\
                    \multirow{2}{*}{$N=250$} & \multirow{2}{*}{$n=500$} & \multirow{2}{*}{$0.441$} & \multirow{2}{*}{$\SI{0.74}{s}$} & $0.0\%$ & $6.0\%$ & $16.2\%$ \\
                    & & & & ($\SI{0.80}{s}$) & ($\SI{4.91}{s}$) & ($\SI{12.48}{s}$) \\
                    \multirow{2}{*}{$N=250$} & \multirow{2}{*}{$n=50$} & \multirow{2}{*}{$0.641$} & \multirow{2}{*}{$\SI{0.13}{s}$} & $0.0\%$ & $0.0\%$ & $0.6\%$ \\
                    & & & & ($\SI{0.79}{s}$) & ($\SI{0.71}{s}$) & ($\SI{1.60}{s}$) \\
                    \hline
                    \multicolumn{4}{r|}{Benign Accuracy:} & $99.4\%$ & $99.6\%$ & $99.4\%$ \\
                    \bottomrule
                \end{tabular}
            }
            \label{tab:tightness-efficiency-trade-off-mnist}
        \end{table}
            \fi

\end{document}